\documentclass[12pt]{article}

\usepackage[utf8]{inputenc}
\usepackage[export]{adjustbox}

\usepackage[sectionbib]{natbib}
\usepackage{array,epsfig,rotating}
\usepackage{hyperref}
\hypersetup{
  colorlinks,
  linkcolor={red!50!black},
  citecolor={blue},
  urlcolor={blue!80!black}
}
\usepackage{color}
\usepackage[dvipsnames]{xcolor}
\usepackage{tabularx}
\usepackage{amsmath}
\usepackage{amssymb}
\usepackage{amsfonts}
\usepackage{amsthm}
\usepackage{float}
\usepackage{verbatim}
\usepackage{subcaption}
\usepackage[ruled,vlined]{algorithm2e}
\usepackage[nodisplayskipstretch]{setspace}
\usepackage{booktabs}
\usepackage{multirow}
\usepackage[noabbrev,capitalize]{cleveref}
\usepackage[shortlabels]{enumitem}
\usepackage{bbm}

\usepackage{forest}

\newtheorem{theorem}{Theorem}
\newtheorem{assumption}{Assumption}
\newtheorem{lemma}{Lemma}

\newtheorem{proposition}{Proposition}
\newtheorem{definition}{Definition}
\newtheorem{example}{Example}
\newtheorem{remark}{Remark}

\newtheorem*{theorem*}{Theorem}

\DeclareMathOperator{\aaone}{\boldsymbol \alpha^{(1)}}
\DeclareMathOperator{\aa2}{\boldsymbol \alpha^{(2)}}
\DeclareMathOperator{\gl}{\emph{g}_{\text{logistic}}}

\DeclareMathOperator{\K}{\K^{(1)}}

\newcommand{\aaat}{\tilde{\boldsymbol \alpha}}
\newcommand{\aaa}{\boldsymbol \alpha}

\newcommand{\bbb}{\boldsymbol \beta}
\newcommand{\ggamma}{\boldsymbol \gamma}
\newcommand{\bo}{\boldsymbol }

\newcommand{\yy}{\mathbf y}

\newcommand{\diag}{\text{\normalfont{diag}}}

\newcommand{\rank}{\text{\normalfont{rank}}}

\newcommand{\pa}{\text{\normalfont{pa}}}
\newcommand{\ch}{\text{\normalfont{ch}}}

\newcommand{\bbeta}{{\boldsymbol \beta}}

\newcommand{\bu}{\mathbf u}
\newcommand{\bs}{\mathbf s}
\newcommand{\bv}{\mathbf v}

\newcommand{\pp}{{\boldsymbol p}}

\newcommand{\ma}{\mathbf A}
\newcommand{\mck}{\mathcal K}
\newcommand{\mcb}{\mathcal B}
\newcommand{\mci}{\mathcal I}
\newcommand{\mcg}{\mathcal G}
\newcommand{\B}{\mathbf B}
\newcommand{\YY}{\mathbf Y}
\newcommand{\ZZ}{\mathbf Z}
\newcommand{\M}{\mathbf M}
\newcommand{\NN}{\mathbf N}

\newcommand{\GG}{\mathbf G}
\newcommand{\I}{\mathbf I}

\newcommand{\PP}{\mathbb P}

\newcommand{\BB}{\mathbf B}
\newcommand{\UU}{\mathbf U}
\newcommand{\VV}{\mathbf V}

\newcommand{\TT}{\boldsymbol \Theta}

\newcommand{\mcy}{\mathcal Y}
\newcommand{\gt}{\tilde{g}}

\newcommand{\E}{\mathbb E}

\def\hat{\widehat}
\def\tilde{\widetilde}

\DeclareMathOperator*{\argmin}{argmin}
\DeclareMathOperator*{\argmax}{argmax}

\usepackage[margin=1in]{geometry}

\newcolumntype{C}{>{\centering\arraybackslash}m{6em}}

\usepackage[noindentafter]{titlesec}
\titlespacing{\section}{0pt}{6pt plus 6pt minus 3pt}{6pt plus 6pt minus 3pt}
\titlespacing{\subsection}{0pt}{6pt plus 6pt minus 3pt}{6pt plus 6pt minus 3pt}

\usepackage{tikz-qtree}
\usepackage{tikz}

\usetikzlibrary{shapes, calc, shapes.multipart, arrows, positioning}

\tikzstyle{qedge}=[->,thick,black]
\tikzstyle{pre}=[->,thick,dotted]
\tikzstyle{pres}=[-,dotted]

\definecolor{myblue}{rgb}{0.0265,    0.6137,    0.8135}
\definecolor{myyellow}{rgb}{0.9290,    0.6940,    0.1250}

\newcommand{\darkblue}[1]{\textcolor{black}{#1}}

\usepackage{bbm}
\tikzstyle{neuron}=[draw, circle,minimum size=25pt,inner sep=0pt, fill=black!20]

\tikzstyle{hidden}=[draw, circle,minimum size=25pt,inner sep=0pt, fill=white]

\tikzstyle{hiddens}=[draw,circle,minimum size=17pt,inner sep=0pt, fill=white]

\usetikzlibrary{arrows.meta}
\tikzset{>={Latex[width=2mm,length=2mm]}}

\tikzstyle{arr}=[->, thick, black]

\usetikzlibrary{decorations.text}
\usetikzlibrary{intersections}
\usetikzlibrary{shadows,shadings,shapes.symbols}
\tikzset{
    double color fill/.code 2 args={
        \pgfdeclareverticalshading[%
            tikz@axis@top,tikz@axis@middle,tikz@axis@bottom%
        ]{diagonalfill}{100bp}{%
            color(0bp)=(tikz@axis@bottom);
            color(50bp)=(tikz@axis@bottom);
            color(50bp)=(tikz@axis@middle);
            color(50bp)=(tikz@axis@top);
            color(100bp)=(tikz@axis@top)
        }
        \tikzset{shade, left color=#1, right color=#2, shading=diagonalfill}
    }
}

\def\spacingset#1{\renewcommand{\baselinestretch}%
{#1}\small\normalsize} 
\spacingset{1}
\allowdisplaybreaks

\newcommand\blfootnote[1]{%
  \begingroup
  \renewcommand\thefootnote{}\footnote{#1}%
  \addtocounter{footnote}{-1}%
  \endgroup
}

\title{Deep Discrete Encoders: Identifiable Deep Generative Models for Rich Data with Discrete Latent Layers}
\author{\large Seunghyun Lee and Yuqi Gu}
\date{\large Department of Statistics, Columbia University}

\begin{document}
\spacingset{1}
\maketitle
\blfootnote{Corresponding Author: Yuqi Gu. Email:  \texttt{yuqi.gu@columbia.edu}.}

\vspace{-10mm}
\begin{abstract}
In the era of generative AI, deep generative models (DGMs) with latent representations have gained tremendous popularity. Despite their impressive empirical performance, the statistical properties of these models remain underexplored. DGMs are often overparametrized, non-identifiable, and uninterpretable black boxes, raising serious concerns when deploying them in high-stakes applications. Motivated by this, we propose  interpretable deep generative models for rich data types with discrete latent layers, called \emph{Deep Discrete Encoders} (DDEs). A DDE is a directed graphical model with multiple binary latent layers. Theoretically, we propose transparent identifiability conditions for DDEs, which imply progressively smaller sizes of the latent layers as they go deeper. Identifiability ensures consistent parameter estimation and inspires an interpretable design of the deep architecture. Computationally, we propose a scalable estimation pipeline of a layerwise nonlinear spectral initialization followed by a penalized stochastic approximation EM algorithm.  This procedure can efficiently estimate models with exponentially many latent components. Extensive simulation studies for high-dimensional data and deep architectures validate our theoretical results and demonstrate the excellent performance of our algorithms. We apply DDEs to three diverse real datasets with different data types to perform hierarchical topic modeling, image representation learning, and response time modeling in educational testing.
\end{abstract}


\spacingset{1.7}

\section{Introduction}
In the era of generative AI, deep generative models (DGMs) with latent representations have gained tremendous popularity across various domains. DGMs achieve impressive empirical success due to their rich modeling power and predictive capacity, and are useful tools to generate images, text, and audio \citep{hinton2006fast, lee2009convolutional, kingma2014stochastic, salakhutdinov2015learning}. However, these models are often subject to statistical issues regarding the model identifiability, interpretability, and parameter estimation reliability. Indeed, most deep learning models are heavily overparametrized black boxes, with more parameters than the number of samples, and are fundamentally non-identifiable. The lack of identifiability means that there may be very different parameter values that give the same marginal distribution of the observed data, leading to inconsistent parameter estimation. In such cases, it is impossible to guarantee the reproducibility of the learned latent representations across different training instances and the validity of the downstream inference. Additionally, the deep layers often merely serve as tools for inducing flexible data distributions but without a meaningful substantive interpretation. These issues raise serious concerns when deploying DGMs in high-stakes applications such as education, medicine, and health care.

We address the above problems from statisticians' perspectives by proposing a broad family of interpretable and identifiable deep generative models for rich types of data, called \emph{Deep Discrete Encoders} (DDEs). DDEs are directed graphical models with potentially deep discrete latent layers to generate the bottom-layer multivariate observed data. A key feature of DDEs is that the latent layers are discrete and organized in progressively smaller sizes as they go deeper; see \cref{fig:model}. This architecture induces very expressive models via the exponentially many mixture components (since each configuration of the discrete latent vector gives a mixture component), and also has the nice interpretation of increasingly more general latent features in deeper layers \citep{bengio2013representation}. DDEs are motivated by both popular generative models in deep learning and latent variable models in educational and psychological measurement. While these two areas rarely intersect in the past, we leverage the insights from both fields to inspire the theory and methodology of DDEs.

\cref{fig:model} displays the graphical model representations of a typical DDE alongside a DDE estimated from real data. 
The right panel of \cref{fig:model} shows that fitting DDEs to a dataset of text documents uncovers interesting hierarchical latent topics as well as an interpretable word generating mechanism; see more details in Section \ref{sec:real data}. 
We emphasize that there is no restriction on the types of observed data; for example, the bottom data layer can be modeled by any exponential family distributions.
In the text data example, Poisson distribution is used to model the word counts in documents. 
Our real data examples in Section \ref{sec:real data} range from word \emph{counts} in text documents to \emph{binary} pixel values in images, to \emph{continuous} response times of students in digital educational assessments.
Such flexibility makes DDEs attractive for many practical applications ranging from machine learning to domain sciences. 

\begin{figure}[h!]
\centering
\begin{minipage}{0.54\textwidth}
\centering
\resizebox{\textwidth}{!}{
    \begin{tikzpicture}[scale=1.5]
\tikzset{
    node/.style={circle, draw, minimum size=1.2cm, inner sep = 0pt, align=center,}
    shaded/.style={circle, draw, fill=gray!30, minimum size=1.2cm, inner sep = 0pt, align=center},
}

    \node (h21)[hidden] at (1,0.2) {$A^{(3)}_{1}$};
    \node (h22)[hidden] at (2,0.2) {$A^{(3)}_{K^{(3)}}$};

    \node (hb1)[hidden] at (0,-1) {$A^{(2)}_{1}$};
    \node (hb2)[hidden] at (1,-1) {$A^{(2)}_{2}$};
    \node (hb3)[hidden] at (2,-1) {$\cdots$};
    \node (hb4)[hidden] at (3,-1) {$A^{(2)}_{K^{(2)}}$};
    
    \node at (1,-1.6) {$\cdots$};
    \node at (2,-1.6) {$\cdots$};
    \node at (1.5,-0.4) {$\cdots$};
       
    \node (v2)[hidden] at (-0.8,-2.2) {$A_1^{(1)}$};
    \node (v3)[hidden] at (0,-2.2) {$A_2^{(1)}$};
    \node (v4)[hidden] at (0.8,-2.2) {$\cdots$};
    \node (v5)[hidden] at (1.6,-2.2) {$\cdots$};
    \node (v6)[hidden] at (2.4,-2.2) {$\cdots$};
    \node (v7)[hidden] at (3.2,-2.2) {$\cdots$};
    \node (v9)[hidden] at (4,-2.2) {$A_{K^{(1)}}^{(1)}$};

    \node (vv1)[neuron] at (-1.6,-3.4) {$Y_1$};
    \node (vv2)[neuron] at (-0.8,-3.4) {$Y_2$};
    \node (vv3)[neuron] at (0,-3.4) {$\cdots$};
    \node (vv4)[neuron] at (0.8,-3.4) {$\cdots$};
    \node (vv5)[neuron] at (1.6,-3.4) {$\cdots$};
    \node (vv6)[neuron] at (2.4,-3.4) {$\cdots$};
    \node (vv7)[neuron] at (3.2,-3.4) {$\cdots$};
    \node (vv8)[neuron] at (4,-3.4) {$\cdots$};
    \node (vv10)[neuron] at (4.8,-3.4) {$Y_J$};

    \draw[arr] (v2) -- (vv1);
    \draw[arr] (v3) -- (vv3);
    \draw[arr] (v2) -- (vv2);
    \draw[arr] (v9) -- (vv10);
    \draw[arr] (v7) -- (vv10);
    
    \node at (0,-2.8) {$\cdots$};
    \node at (1,-2.8) {$\cdots$};
    \node at (2,-2.8) {$\cdots$};
    \node at (3,-2.8) {$\cdots$};
    \node at (4,-2.8) {$\cdots$};
    
    \draw[arr, blue!70!black] (hb1) -- (v2);
    \draw[arr, blue!70!black] (hb1) -- (v3);
    \draw[arr, blue!70!black] (hb2) -- (v3);
    
    \draw[arr, blue!70!black] (hb2) -- (v3);
    \draw[arr, blue!70!black] (hb2) -- (v4);
    
    \draw[arr, blue!70!black] (hb4) -- (v7);
    \draw[arr, blue!70!black] (hb4) -- (v9);
    
    \draw[arr, green!70!black] (h21) -- (hb1);
    \draw[arr, green!70!black] (h21) -- (hb2);
    \draw[arr, green!70!black] (h21) -- (hb3);
    
    \draw[arr, green!70!black] (h22) -- (hb3);
    \draw[arr, green!70!black] (h22) -- (hb4);

    \node[anchor=west] (g0) at (4.6, -2.6) {$\mathbf G^{(1)}: J\times K^{(1)}$};
    \node[anchor=west] (g1) at (4.6, -1.4) {\darkblue{$\mathbf G^{(2)}: K^{(1)} \times K^{(2)}$}};
    \node[anchor=west] (g2) at (4.6, -0.2) {\textcolor{green!70!black}{$\mathbf G^{(3)}: K^{(2)} \times K^{(3)}$}};

\end{tikzpicture}
}
\end{minipage}
\begin{minipage}{0.42\textwidth}
\centering
\resizebox{\textwidth}{!}{
\begin{tikzpicture}

\tikzset{
    node/.style={circle, draw, minimum size=1.5cm, inner sep = 0pt, align=center, font=\sffamily\bfseries},
    shaded/.style={circle, draw, fill=gray!30, minimum size=1.5cm, inner sep = 0pt, align=center, font=\sffamily\bfseries},
}

\node[node] (A1) at (-2, 6) {rec};
\node[node] (A2) at (2, 6) {tech};

\node[node] (B1) at (-4, 3) {sports};
\node[node] (B2) at (-1.5, 3) {car};
\node[node] (B3) at (1.5, 3) {graphics};
\node[node] (B4) at (4, 3) {space};

\node[shaded] (C1) at (-6, 0) {league};
\node[shaded] (C2) at (-4, 0) {playoff};
\node[shaded] (C3) at (-2, 0) {parts};
\node[shaded] (C4) at (0, 0) {opinion};
\node[shaded] (C5) at (2, 0) {color};
\node[shaded] (C6) at (4, 0) {nasa};
\node[shaded] (C7) at (6, 0) {orbit};

\draw[->, line width=0.94mm] (A1)--(B1);
\draw[->, line width=0.3mm] (A1)--(B2);
\draw[->, line width=0.57mm, red] (A2)--(B1);
\draw[->, line width=0.8mm] (A2)--(B2);
\draw[->, line width=0.95mm] (A2)--(B3);
\draw[->, line width=0.3mm] (A2)--(B4);

\draw[->, line width = 1.25mm] (B1)--(C1);
\draw[->, line width = 1.25mm] (B1)--(C2);
\draw[->, line width = 0.6mm, red] (B1)--(C6);
\draw[->, line width = 0.36mm] (B2)--(C1);
\draw[->, line width = 0.6mm] (B2)--(C3);
\draw[->, line width = 0.4mm, red] (B3)--(C2);
\draw[->, line width = 0.65mm] (B3)--(C5);
\draw[->, line width = 0.22mm, red] (B3)--(C6);
\draw[->, line width = 0.5mm, red] (B3)--(C7);
\draw[->, line width = 1.14mm] (B4)--(C6);
\draw[->, line width = 0.75mm] (B4)--(C7);
\end{tikzpicture}

}
\end{minipage}
    \caption{\textbf{Left}: Example graphical model representation of DDEs. \textbf{Right}: Simplified estimated DDE structure of the latent topics for the 20 newsgroups dataset. Only the shaded nodes are observed. In the right panel, edge widths are proportional to coefficients' absolute values. The black/red edge colors imply positive/negative coefficients, respectively.}
    \label{fig:model}
    \vspace{-5mm}
\end{figure}

The main contributions of this paper include rigorous identifiability theory and scalable computational pipelines for DDEs.
\emph{For identifiability}, we propose general identifiability conditions in terms of the probabilistic graph structures between layers in the graphical model {(corresponding to the directed arrows in \cref{fig:model})}.
We work under the minimal possible assumptions in order to flexibly cover the various examples and data types mentioned above. Next, we provide an informal statement of the identifiability conditions. Under an identifiable DDE, we also prove that a penalized-likelihood based estimator is consistent for estimating both the continuous parameters as well as the discrete graph structure.

\begin{theorem*}[Informal version of Theorems \ref{thm:deep id} and \ref{thm:deep gid}]
    The DDE  is identifiable up to latent variable permutation in each layer, as long as each latent variable has at least three pure children (which only has one parent variable). 
    Under a weaker notion of ``generic identifiability'', this condition can be relaxed to that each latent variable has at least three children that are not necessarily pure.
\end{theorem*}

\emph{For computation}, we propose a scalable estimation pipeline for DDEs.
The multiple layers of nonlinear latent variables in DDEs lead to a {highly nonconvex optimization landscape with potentially exponentially many local optima.
To address this challenging issue, our computational pipeline features a nuanced layerwise nonlinear spectral initialization followed by a penalized stochastic approximation EM algorithm.
This procedure can efficiently handle models with a large number of latent variables. We achieve favorable simulation results for as many as $K^{(1)}=30$ binary latent variables in the shallowest latent layer, which amount to $2^{K^{(1)}}$ mixture components and define a very expressive model. 
Extensive simulation studies not only validate the identifiability results, but also demonstrate the excellent performance of the proposed algorithms.
We apply DDEs to real data in three diverse tasks, including hierarchical topic modeling, image representation learning, and multimodal modeling in digital educational testing.
Across these applications, DDEs extract highly interpretable latent structures and learn useful representations for downstream analyses.

We make brief remarks to place DDEs in the rapidly emerging field of generative AI. For the considered unsupervised learning setting, we use ``generative model'' to refer to probabilistic models where observed data are generated conditional on hidden variables. 
Recently, powerful multilayer models have been proposed for more complex tasks, which either consist of multiple latent layers or a single latent layer transformed by deep neural networks. Popular models in machine learning include deep belief networks \citep[DBNs,][]{hinton2006fast}, deep Boltzmann machines \citep[DBMs,][]{salakhutdinov2009deep}, variational autoencoders \citep{ranzato2008semi},
generative adversarial networks \citep{goodfellow2014generative}, diffusion models \citep{sohl2015deep}, and transformer-based models \citep{vaswani2017attention} such as large language models. Our proposed models are most closely related to DBNs and DBMs.
See \cref{subsec:connections} for more discussions.

\vspace{1mm}
\noindent\textbf{Organization.}
\cref{sec:model} formally defines DDEs and elaborates on its connection with existing models. \cref{sec:theory} provides identifiability results of DDEs and proves the consistency of a penalized-likelihood estimator. \cref{sec:algorithm} presents scalable computational algorithms for estimating DDEs.
\cref{sec:simulations} and \cref{sec:real data} present simulation studies and real-data applications under various data types. \cref{sec:discussion} concludes the paper. All technical proofs, and additional details about computation and data analysis are in the Supplementary Material.

\section{Deep Discrete Encoders}\label{sec:model}
\noindent\textbf{Notation.}
For a positive integer $K$, denote $[K] = \{1, \ldots, K\}$.
For a matrix $\GG$ with $J$ rows, let $\mathbf{g}_1, \ldots, \mathbf{g}_J$ denote its row vectors.
Let $\mathbf{0}_K, \mathbf{1}_K$ be the all-zero vector and all-one vector of length $K$, respectively. 
Let $\mathbf{e}_k$ be the $K$-dimensional canonical basis vector.
Let $\gl(x) := {e^x}/({1+e^x})$ denote the logistic/sigmoid function. For a finite set $I$, let $S_I$ denote the collection of all permutation maps of $I$, and let $\text{id}_I$ be the identity permutation on $I$.

\subsection{Model Definition}\label{subsec:model def}

The $D$-latent-layer DDE is a generative model with $D$ discrete latent layers.
For each $d \in [D]$, assume that the $(d-1)$th layer is generated conditional on the $d$th. Here, only the bottom layer (indexed by $d=0$) is observed, and all other layers are latent. The bottom layer data can take arbitrary values, but all latent variables are binary, {similar to the celebrated deep belief networks \citep{hinton2006fast}.}
Let $\YY = (Y_1, \ldots, Y_J) \in \times_{j=1}^J \mcy_j$ denote the $J$-dimensional observed responses, where $\mcy_j$ is the sample space for the $j$th response; see the end of this subsection for concrete examples. 
We work under the minimal assumption that each $\mcy_j$ is a separable metric space.
Denote the $d$th latent layer as $\ma^{(d)} = \big(A_1^{(d)}, \ldots, A_{K^{(d)}}^{(d)}\big) \in \{0, 1\}^{K^{(d)}}$, which is a $K^{(d)}$-dimensional binary vector. 
A DDE has a shrinking-ladder-shaped deep architecture,  with the dimension of each layer decreasing as $d$ increases: $K^{(D)} < \ldots < K^{(1)} < J$. 
See \cref{fig:model} for a graphical model representation.

We further specify the distribution of the directed graphical model in a top-down manner. The directed edges in \cref{fig:model} are pointing downward, meaning that the deepest latent variables in the top layer $d = D$ are generated first.
The top layer latent variables are assumed to be independent Bernoullis with parameter $\pp = (p_1, \ldots, p_{K^{(D)}})$:
\begin{align}\label{eq:top layer}
    \PP( \ma^{(D)} = \aaa^{(D)}) = \prod_{k=1}^{K^{(D)}} \PP(A_k^{(D)} = \alpha_k^{(D)}) = \prod_{k=1}^{K_D} p_k^{\alpha_k^{(D)}} (1-p_k)^{1-\alpha_k^{(D)}}, \quad \forall \aaa^{(D)} \in \{0, 1\}^{K^{(D)}}.
\end{align}
Next, define the \emph{middle latent layers} inductively as follows. For each $d > 1$, suppose that $\PP(\ma^{(d)})$, the distribution of the $d$th layer, is given. Then, we define the $(d-1)$th layer distribution by assuming the conditional independence of $A_1^{(d-1)}, \ldots, A_{K^{(d-1)}}^{(d-1)}$ given $\ma^{(d)}$:
\begin{align}\label{eq:middle layer}
    \PP(\ma^{(d-1)} = \aaa^{(d-1)} \mid \ma^{(d)}) = \prod_{k=1}^{K^{(d-1)}} \PP(A^{(d-1)}_k = \alpha^{(d-1)}_k \mid \ma^{(d)}), \quad \forall \aaa^{(d-1)} \in \{0, 1\}^{K^{(d-1)}}. 
\end{align}
Here, we additionally model each conditional distribution in \eqref{eq:middle layer} as
\begin{align}\label{eq:latent conditional distribution}
    A_{k}^{(d-1)}\mid \ma^{(d)} \sim \text{Bernoulli} \Big( \gl \big( \beta_{k,0}^{(d)} + \sum_{l=1}^{K^{(d)}} \beta_{k,l}^{(d)} A_l^{(d)} \big) \Big),
\end{align}
where $\gl$ is the logistic function that maps the real-valued linear combinations to the $[0, 1]$-valued Bernoulli parameters (one may use alternative link functions $g:\mathbb{R}\to[0,1]$, such as the probit link). Collect the $\beta_{k,l}$-parameters in a $K^{(d-1)} \times (K^{(d)}+1)$ matrix $\BB^{(d)}$, whose first column is the intercepts $\big(\beta_{k,0}^{(d)}\big)_{k \in [K^{(d-1)}]}$ and remaining parts are $\big(\beta_{k,l}^{(d)}\big)_{k \in [K^{(d-1)}], l \in [K^{(d)}]}$.

Finally, we model the bottom layer for the observed data. The observed $\YY=(Y_1,\ldots,Y_J)$ are modeled by assuming the conditional independence of $Y_1, \ldots, Y_J$ given $\ma^{(1)}$.
As the observations $\YY$ are not necessarily binary, we replace the Bernoulli conditional distributions in \eqref{eq:latent conditional distribution} by a general parametric family of the form
\begin{equation}\label{eq:observed exp fam}
    Y_j \mid \ma^{(1)} \sim \text{ParFam}_j \Big(g_j \Big(\beta_{j, 0}^{(1)} + \sum_{k=1}^{K^{(1)}} \beta_{j, k}^{(1)} A_k^{(1)}, \gamma_j \Big) \Big).
\end{equation}
Here, $\text{ParFam}_j$ denotes a specific identifiable parametric family, and let $H_j$ denote its parameter space. For convenience, let $p_j$ be the probability mass/density function of $\text{ParFam}_j$. The $g_j: \mathbb{R} \times [0, \infty) \rightarrow H_j$ is a known link function that maps the linear combinations to the parameters of the given parametric family. Here, $\gamma_j > 0$ denotes the dispersion parameter, when there exists one. Throughout the paper, we will state all results under the more general assumption that there exists a dispersion parameter in the parametric family in \eqref{eq:observed exp fam}. If not, one may ignore the notation $\gamma$.
We further elaborate on the specific parameterizations in \eqref{eq:observed exp fam} for various response types $\mcy_j$ at the end of this section.

Following the definition of directed graphical models \citep{koller2009pgm}, we can write the joint distribution of all observed and latent variables based on \eqref{eq:top layer}--\eqref{eq:observed exp fam}:
\begin{align*}
    \PP \big(\YY, \{\ma^{(d)}\}_{d \in [D]} \big) = \PP (\YY \mid \ma^{(1)}) \prod_{d=2}^{D} \PP (\ma^{(d-1)} \mid \ma^{(d)}) \PP(\ma^{(D)}).
\end{align*}
The marginal distribution of $\YY$ is obtained by marginalizing out all of the $D$ latent layers:
{\small
\begin{align}
    ~\PP(\YY) 
    \label{eq:marginal probability}
    =&~ \sum_{\substack{\aaa^{(d)} \in \{0,1\}^{K^{(d)}} \\ \forall d \in [D]}} \PP (\YY |\ma^{(1)} = \aaa^{(1)}) \prod_{d=2}^{D} \PP (\ma^{(d-1)} = \aaa^{(d-1)} \mid \ma^{(d)} =\aaa^{(d)}) \PP(\ma^{(D)} = \aaa^{(D)}).
\end{align}}
The $D$-latent-layer DDE is parametrized by $(\pp, \mcb, \ggamma)$, where $\mcb := \{\BB^{(d)}\}_{d \in [D]}$ and $\ggamma := (\gamma_1, \ldots, \gamma_J)$.
Upon the above marginalization, the induced DDE is a highly expressive model with exponentially many latent mixture components. However, this expressivity also introduces identifiability and computation challenges, which we address in Sections \ref{sec:theory} and \ref{sec:algorithm}.

In many scenarios, it is desirable for the coefficients $\mcb$ to be sparse, as this leads to a more interpretable and parsimonious data-generating mechanism. 
{The interpretability stems from that if a latent variable is connected to only a few, rather than all, variables in the layer below, then these children variables can help pinpoint the interpretation of the latent parent.}
{Similar sparse deep generative architectures have been considered in deep exponential families \citep{ranganath2015deep}, Bayesian pyramids \citep{gu2023bayesian}, and deep cognitive diagnostic models \citep{gu2024deepCDM}.}
Additionally, as will be shown in \cref{sec:identifiability}, the sparsity of the coefficients play a key role in facilitating identifiability. To encode the sparsity of $\BB^{(d)}$, for each $d \in [D]$, define a $K^{(d-1)} \times K^{(d)}$ binary matrix $\GG^{(d)} = (g_{k,l})$, where $g_{k,l}=1$ if the corresponding coefficient $\beta_{k,l}^{(d)}$ is nonzero, and 0 otherwise. Define $K^{(0)} := J$ for notational convenience.
By \eqref{eq:middle layer} and \eqref{eq:latent conditional distribution}, $\GG^{(d)}$ can also be viewed as the {adjacency matrix} or ``graphical matrix'' between the $(d-1)$th layer and the $d$th layer in the graphical model representation (see \cref{fig:model}). 
Collect all $\GG^{(d)}$s by defining $\mcg := \{\GG^{(d)}\}_{d \in [D]}$, \darkblue{and collect the number of latent variables in each layer by defining $\mck := \{K^{(d)}\}_{d \in [D]}.$}
Now, we formally define $D$-latent-layer DDEs, which also incorporate $\mcg$ as unknown parameters.

\begin{definition}[DDE]\label{def:dde}
    A $D$-latent-layer DDE with parameters $(\pp, \mcb, \mcg,\ggamma)$ 
    is a statistical model with marginal distribution of the observed data as given in \eqref{eq:marginal probability}. When $\mck$ is known and fixed, DDEs can be viewed as parametric families with parameters $(\TT, \mcg)$ and probability mass/density functions $\PP_{\TT, {\mcg}}$, where $\TT:= (\pp, \mcb, \ggamma)$ denotes all continuous parameters.
\end{definition}

Next, we give some examples of the various response types $\mcy_j$ allowed in the DDE framework,
along with the corresponding link functions $g$ and parametrizations for \eqref{eq:observed exp fam}. 
As mentioned in the Introduction, the later numerical experiments consider three types of responses: (i) binary, (ii) count, and (iii) continuous.
We model each of these responses using (i) Bernoulli with $g = g_{\text{logistic}}$, (ii) Poisson with an exponential link $g(x) = e^x$, and (iii) Normal with an identity link $g(x,y) = (x,y)$, respectively. Modeling other data types is also straightforward.
While not required, we typically consider that the data types are the same across the $p$ features.
In such cases, we omit the subscript in $\mcy_j, g_j$, and write $\mcy$ and $g$.

\subsection{Connections with Existing Models}
\label{subsec:connections}

Related deep latent variable models in the machine learning literature include the deep Boltzmann machine \citep[DBM,][]{salakhutdinov2009deep}, deep belief networks \citep[DBNs,][]{hinton2006fast}, and deep exponential families \citep[DEFs,][]{ranganath2015deep}. The DBM and DBN both contain multiple binary latent layers and differ in the directions of the edges; see \cref{fig:dbn dbm}. DBM and DBN have been originally proposed to model binary data and have been later extended to handle {continuous or count}
responses \citep{cho2013gaussian, gan2015scalable, li2019novel}. 
We recommend the review \cite{salakhutdinov2015learning} for more details and references on DBM and DBNs. 
DEFs are an unsupervised modeling framework that uses exponential families to model conditional distributions for each layer.

\begin{figure}[h!]
\centering
\begin{minipage}[c]{0.22\linewidth}
\resizebox{0.8\textwidth}{!}{
\begin{tikzpicture}[scale=1.2]
	\node (h1)[hiddens] at (0, 3) {$A_1^{(2)}$};
    \node (h2)[hiddens] at (1.5, 3) {$A_2^{(2)}$};
    \node (h3)[hiddens] at (3, 3) {$A_3^{(2)}$};
    \node (h4)[hiddens] at (0, 1.5) {$A_1^{(1)}$};
    \node (h5)[hiddens] at (1.5, 1.5) {$A_2^{(1)}$};
    \node (h6)[hiddens] at (3, 1.5) {$A_3^{(1)}$};
    \node (y1)[neuron] at (0, 0) {$Y_1$};
    \node (y2)[neuron] at (1.5, 0) {$Y_2$};
    \node (y3)[neuron] at (3, 0) {$Y_3$};
    
    \draw[-] (h1) -- (h4) node [midway,above=-0.12cm,sloped] {}; 
    \draw[-] (h1) -- (h5) node [midway,above=-0.12cm,sloped] {}; 
    \draw[-] (h1) -- (h6) node [midway,above=-0.12cm,sloped] {}; 
    \draw[-] (h2) -- (h4) node [midway,above=-0.12cm,sloped] {}; 
    \draw[-] (h2) -- (h5) node [midway,above=-0.12cm,sloped] {}; 
    \draw[-] (h2) -- (h6) node [midway,above=-0.12cm,sloped] {}; 
    \draw[-] (h3) -- (h4) node [midway,above=-0.12cm,sloped] {}; 
    \draw[-] (h3) -- (h5) node [midway,above=-0.12cm,sloped] {}; 
    \draw[-] (h3) -- (h6) node [midway,above=-0.12cm,sloped] {}; 
    \draw[-] (h4) -- (y1) node [midway,above=-0.12cm,sloped] {}; 
    \draw[-] (h4) -- (y2) node [midway,above=-0.12cm,sloped] {}; 
    \draw[-] (h4) -- (y3) node [midway,above=-0.12cm,sloped] {}; 
    \draw[-] (h5) -- (y1) node [midway,above=-0.12cm,sloped] {}; 
    \draw[-] (h5) -- (y2) node [midway,above=-0.12cm,sloped] {}; 
    \draw[-] (h5) -- (y3) node [midway,above=-0.12cm,sloped] {}; 
    \draw[-] (h6) -- (y1) node [midway,above=-0.12cm,sloped] {}; 
    \draw[-] (h6) -- (y2) node [midway,above=-0.12cm,sloped] {}; 
    \draw[-] (h6) -- (y3) node [midway,above=-0.12cm,sloped] {}; 
\end{tikzpicture}
} 
\caption*{\hspace{-7mm}DBM}
\end{minipage}
\begin{minipage}[c]{0.22\linewidth}
\resizebox{0.8\textwidth}{!}{
\begin{tikzpicture}[scale=1.2]
	\node (h1)[hiddens] at (0, 3) {$A_1^{(2)}$};
    \node (h2)[hiddens] at (1.5, 3) {$A_2^{(2)}$};
    \node (h3)[hiddens] at (3, 3) {$A_3^{(2)}$};
    \node (h4)[hiddens] at (0, 1.5) {$A_1^{(1)}$};
    \node (h5)[hiddens] at (1.5, 1.5) {$A_2^{(1)}$};
    \node (h6)[hiddens] at (3, 1.5) {$A_3^{(1)}$};
    \node (y1)[neuron] at (0, 0) {$Y_1$};
    \node (y2)[neuron] at (1.5, 0) {$Y_2$};
    \node (y3)[neuron] at (3, 0) {$Y_3$};
    
    \draw[-] (h1) -- (h4) node [midway,above=-0.12cm,sloped] {}; 
    \draw[-] (h1) -- (h5) node [midway,above=-0.12cm,sloped] {}; 
    \draw[-] (h1) -- (h6) node [midway,above=-0.12cm,sloped] {}; 
    \draw[-] (h2) -- (h4) node [midway,above=-0.12cm,sloped] {}; 
    \draw[-] (h2) -- (h5) node [midway,above=-0.12cm,sloped] {}; 
    \draw[-] (h2) -- (h6) node [midway,above=-0.12cm,sloped] {}; 
    \draw[-] (h3) -- (h4) node [midway,above=-0.12cm,sloped] {}; 
    \draw[-] (h3) -- (h5) node [midway,above=-0.12cm,sloped] {}; 
    \draw[-] (h3) -- (h6) node [midway,above=-0.12cm,sloped] {}; 
    \draw[qedge] (h4) -- (y1) node [midway,above=-0.12cm,sloped] {}; 
    \draw[qedge] (h4) -- (y2) node [midway,above=-0.12cm,sloped] {}; 
    \draw[qedge] (h4) -- (y3) node [midway,above=-0.12cm,sloped] {}; 
    \draw[qedge] (h5) -- (y1) node [midway,above=-0.12cm,sloped] {}; 
    \draw[qedge] (h5) -- (y2) node [midway,above=-0.12cm,sloped] {}; 
    \draw[qedge] (h5) -- (y3) node [midway,above=-0.12cm,sloped] {}; 
    \draw[qedge] (h6) -- (y1) node [midway,above=-0.12cm,sloped] {}; 
    \draw[qedge] (h6) -- (y2) node [midway,above=-0.12cm,sloped] {}; 
    \draw[qedge] (h6) -- (y3) node [midway,above=-0.12cm,sloped] {}; 
\end{tikzpicture}
} 
\caption*{\hspace{-7mm}DBN}
\end{minipage}
\begin{minipage}[c]{0.22\linewidth}
\resizebox{0.8\textwidth}{!}{
\begin{tikzpicture}[scale=1.2]
	\node (h1)[hiddens] at (0, 3) {$A_1^{(2)}$};
    \node (h2)[hiddens] at (1.5, 3) {$A_2^{(2)}$};
    \node (h3)[hiddens] at (3, 3) {$A_3^{(2)}$};
    \node (h4)[hiddens] at (0, 1.5) {$A_1^{(1)}$};
    \node (h5)[hiddens] at (1.5, 1.5) {$A_2^{(1)}$};
    \node (h6)[hiddens] at (3, 1.5) {$A_3^{(1)}$};
    \node (y1)[neuron] at (0, 0) {$Y_1$};
    \node (y2)[neuron] at (1.5, 0) {$Y_2$};
    \node (y3)[neuron] at (3, 0) {$Y_3$};
    
    \draw[qedge] (h1) -- (h4) node [midway,above=-0.12cm,sloped] {}; 
    \draw[qedge] (h1) -- (h5) node [midway,above=-0.12cm,sloped] {}; 
    \draw[qedge] (h1) -- (h6) node [midway,above=-0.12cm,sloped] {}; 
    \draw[qedge] (h2) -- (h4) node [midway,above=-0.12cm,sloped] {}; 
    \draw[qedge] (h2) -- (h5) node [midway,above=-0.12cm,sloped] {}; 
    \draw[qedge] (h2) -- (h6) node [midway,above=-0.12cm,sloped] {}; 
    \draw[qedge] (h3) -- (h4) node [midway,above=-0.12cm,sloped] {}; 
    \draw[qedge] (h3) -- (h5) node [midway,above=-0.12cm,sloped] {}; 
    \draw[qedge] (h3) -- (h6) node [midway,above=-0.12cm,sloped] {}; 
    \draw[qedge] (h4) -- (y1) node [midway,above=-0.12cm,sloped] {}; 
    \draw[qedge] (h4) -- (y2) node [midway,above=-0.12cm,sloped] {}; 
    \draw[qedge] (h4) -- (y3) node [midway,above=-0.12cm,sloped] {}; 
    \draw[qedge] (h5) -- (y1) node [midway,above=-0.12cm,sloped] {}; 
    \draw[qedge] (h5) -- (y2) node [midway,above=-0.12cm,sloped] {}; 
    \draw[qedge] (h5) -- (y3) node [midway,above=-0.12cm,sloped] {}; 
    \draw[qedge] (h6) -- (y1) node [midway,above=-0.12cm,sloped] {}; 
    \draw[qedge] (h6) -- (y2) node [midway,above=-0.12cm,sloped] {}; 
    \draw[qedge] (h6) -- (y3) node [midway,above=-0.12cm,sloped] {}; 
\end{tikzpicture}
} 
\hspace{-3cm}
\caption*{\hspace{-7mm}DEF}
\end{minipage}
\begin{minipage}[c]{0.22\linewidth}
\resizebox{0.8\textwidth}{!}{
\begin{tikzpicture}[scale=1.2]
    \node (h1)[hiddens] at (1.5, 3) {$A_1^{(2)}$};
    \node (h2)[hiddens] at (0.75, 1.5) {$A_1^{(1)}$};
    \node (h3)[hiddens] at (2.25, 1.5) {$A_2^{(1)}$};
    \node (y1)[neuron] at (0, 0) {$Y_1$};
    \node (y2)[neuron] at (1.5, 0) {$Y_2$};
    \node (y3)[neuron] at (3, 0) {$Y_3$};
    
    \draw[qedge] (h1) -- (h2) node [midway,above=-0.12cm,sloped] {}; 
    \draw[qedge] (h1) -- (h3) node [midway,above=-0.12cm,sloped] {}; 
    \draw[qedge] (h2) -- (y1) node [midway,above=-0.12cm,sloped] {};
    \draw[qedge] (h2) -- (y2) node [midway,above=-0.12cm,sloped] {};
    \draw[qedge] (h3) -- (y2) node [midway,above=-0.12cm,sloped] {}; 
    \draw[qedge] (h3) -- (y3) node [midway,above=-0.12cm,sloped] {}; 
\end{tikzpicture}
} 
\caption*{\hspace{-7mm}DDE}
\end{minipage}

\caption{
Comparison of the graphical structure of DDEs to relevant deep generative models. DBM: binary data and binary latent. DBN: binary data and binary latent. DEF: exponential family conditional distributions. DDE: general-response data and binary latent.}
\label{fig:dbn dbm}
\end{figure}
\vspace{-5mm}

Despite their numerous empirical successes, the aforementioned machine learning models are usually fundamentally non-identifiable due to an enormous number of latent variables and parameters organized in a complex nonlinear architecture.
These models are typically heavily overparametrized, making it challenging to understand and interpret the latent representations. 
Moreover, popular existing estimation procedures for these models are developed to maximize a tractable but less theoretically understood alternative to the likelihood, such as contrastive divergence or layer-wise variational approximation.

As to be shown later, DDEs resolve all these issues by assuming a shrinking-ladder architecture of entirely discrete latent variables and potentially sparse layerwise connections (see \cref{fig:dbn dbm}).
This allows us to establish identifiability and consistency as well as effectively reduce the model dimension and interpret the latent structure. 

On a related note, several identifiable deep models have recently been proposed.
\cite{gu2023bayesian} proposed \emph{Bayesian pyramids}, which are identifiable multilayer discrete latent variable models with a pyramid structure.
But the methodology and identifiability theory therein are restricted to multivariate categorical data. DDEs significantly broaden the applicability of Bayesian pyramids by modeling arbitrarily flexible data types while still remaining identifiable. 
\darkblue{\cite{kong2024learning} considered a modeling framework allowing both discrete and continuous latent variables, where the discrete part of the model can have more flexible graphical structures than the multi-layer structure. However, as a tradeoff, they consider a weaker notion of identifiability (up to atomic cover structures) and propose less explicit conditions for identifiability.
Finally,  \cite{anandkumar2013learning} and \cite{xie2024generalized} consider {continuous} latent variable models and prove identifiability under potentially more general graph structures, but assume that all conditional distributions are \emph{linear}. In contrast, DDEs allow arbitrary nonlinear link functions, which improves representation power. We remark that DDEs are identifiable even when higher-order interaction effects among latent variables are in the model (as detailed in Supplement S.1.6). We provide additional literature review regarding identifiable VAEs and psychometric models in Supplement S.6.}

\section{Theoretical Guarantees for DDEs}\label{sec:theory}
\subsection{Model Identifiability}
\label{sec:identifiability}
In this section, we establish the identifiability of DDEs. Assuming known numbers of latent variables, we prove that both the continuous model parameters and the discrete graph structures between layers can be uniquely identified 
under certain conditions on the true graphical matrices $\mathbf G^{(d)}$'s.
We first make an assumption to address some trivial ambiguities. For notational convenience, denote $K^{(0)} := J$ and $A_j^{(0)} := Y_j$.

\begin{assumption}\label{assmp:proportion}
    Assume that the graphical matrices $\mcg = \{\GG^{(d)}\}_{d \in [D]}$ and proportion parameters $\bo p$ satisfy the following conditions.
    \begin{enumerate}[(a),itemsep=0pt]
        \item For all $k \in [K^{(D)}]$, $p_k \in (0, 1)$. 
        \item For all $d \in [D]$, $\GG^{(d)}$ does not have all-zero columns and is faithful in the sense that for any $k \in [K^{(d-1)}], l \in [K^{(d)}]$, $g_{k,l}^{(d)} = 0$ if and only if $\beta_{k,l+1}^{(d)} = 0$. 
        \item For any $d \in [D]$, all column sums of $\BB^{(d)}$ except for the first column are strictly positive.
    \end{enumerate}
\end{assumption}

Condition (a) and the first part of (b) is required for the latent dimension $\mck$ to be well-defined, in the sense that removing or adding a latent variable must change the marginal distribution \eqref{eq:marginal probability}. 
Condition (b) is a standard faithfulness assumption in graphical models \citep[e.g. see Definition 3.8 in][]{koller2009pgm} that follows from our definition of $\GG^{(d)}$.
Condition (c) is introduced to avoid trivial sign-flipping of latent variables. This condition ensures that for each latent variable $A_k^{(d)}$, the value $A_k^{(d)} =1$ implies a larger coefficient of the $(k+1)$th row of $\BB^{(d)}$, and may be replaced with other monotonicity assumptions. For example, one can alternatively assume that the first nonzero coefficient in each column of $\BB^{(d)}$ is positive. We emphasize that condition (c) is much weaker compared to the nonnegative coefficient assumption $\beta_{j,k} \ge 0$, which is a popular assumption for various identifiable latent variable models \citep{donoho2003does, chen2020sparse, lee2024new}.

Now, we formally define the parameter space and the notion of identifiability. For multilayer latent variable models, there are inevitable latent variable permutation issues within each latent layer. Hence, we introduce the notion of identifiability up to equivalence.

\begin{definition}[Parameter space]\label{def:parameter space}
    Consider a $D$-latent-layer DDE with $\mathcal{K}=\{K^{(d)}\}_{d \in [D]}$ latent variables. We define the parameter space of the continuous parameters $\TT$ given $\mcg$ as
    $\Omega_{\mathcal{K}}(\TT; \mcg) := \{\TT: \beta_{l,k}^{(d)} \neq 0 \text{ if and only if } g_{l,k}^{(d)} = 1, \gamma_j > 0 \}.$
    Define the joint parameter space for all parameters to be $\Omega_{\mathcal{K}}(\TT, \mcg) :=\{(\TT, \mcg): \TT \in \Omega_{\mathcal{K}}(\TT ; \mcg) \}.$
\end{definition}

\begin{definition}[Identifiability up to equivalence]\label{def:identifiability up to permutation}

    For a D-layer DDE with $\mck$ latent variables, we define an equivalence relationship ``$\sim_{\mck}$'' by setting $(\TT, \mcg) \sim_{\mck} (\tilde{\TT}, \tilde{\mcg})$ if and only if $ \bo\gamma = \tilde{ \bo\gamma}$ and there exists permutations $\sigma^{(d)} \in S_{[K^{(d)}]}$
    for all $d \in [D]$ such that the following hold:
    \begin{itemize}[itemsep=0pt]
        \item $p_k = \tilde{p}_{\sigma^{(D)}(k)}$ 
        \item ${g}_{l,k}^{(d)} = \tilde{{g}}_{\sigma^{(d-1)}(l) ,\sigma^{(d)}(k)}^{(d)}$ and $\beta_{l,k}^{(d)} = \tilde{\beta}_{\sigma^{(d-1)}(l) ,\sigma^{(d)}(k)}^{(d)}$ for all $d \in [D]$, $k \in [K^{(d)}], l \in [K^{(d-1)}]$
    \end{itemize}
    Here, we set $\sigma^{(0)} = \textnormal{id}_{[J]}$.
    We say that the DDE with true parameters $(\TT^\star, \mcg^\star)$ is identifiable up to $\sim_{\mck}$ when for any alternate parameter value $(\TT, \mcg) \in \Omega_{\mck}(\TT, \mcg)$ with $\PP_{\TT, \mcg} = \PP_{\TT^\star, \mcg^\star}$, it holds that $(\TT, \mcg) \sim_{\mck} (\TT^\star, \mcg^\star)$. Here, $\PP_{\TT, \mcg}$ is the marginal distribution of $\YY$ defined in \eqref{eq:marginal probability}.
\end{definition}
Despite the seemingly heavy notation, the equivalence relationship is quite natural. \darkblue{For example, in the right panel of Figure 1, the location of the latent variables in each layer can be arbitrarily permuted without changing the likelihood. In other words, the latent variable ``sports'' and ``car'' can be equivalently indexed by $(1, 2)$ or $(2, 1)$ as long as the associated arrows are permuted accordingly.}
For the $d$th latent layer, the permutation $\sigma^{(d)}$ corresponds to this fundamental yet relatively trivial label-switching map.
The following theorem is our first main result on the identifiability of DDEs. Here, \darkblue{we say that a variable $X$ is a ``pure child'' of $Y$ if $Y$ is the only parent of $X$.}

\begin{theorem}[Identifiability of DDEs]\label{thm:deep id}
Consider a $D$-latent-layer DDE with true parameters $(\TT^\star, \mcg^\star)$, \darkblue{where the number of latent variables, $\mck$, is given}. Suppose that for all $d \le D-1$, the true graphical structures $\GG^{(d) \star}$ and parameters $\BB^{(d) \star}$ satisfy the following conditions:
\begin{enumerate}
    \item[A.] Each latent variable $A_k^{(d)}$ has at least \darkblue{two pure children $A_{j_{k,1}}^{(d-1)}, A_{j_{k,2}}^{(d-1)}$}  according to $\GG^{(d)\star}$.
    \item[B.] For any $\aaa^{(d)} \neq \aaa'^{(d)} \in\{0,1\}^{K^{(d)}}$, there exists $j \in [K^{(d-1)}] \setminus \cup_{k=1}^{K^{(d)}} \{j_{k,1}, j_{k,2}\}$ such that $\sum_{k=1}^{K^{(d)}} \beta_{j,k}^{(d)\star} (\alpha_k^{(d)} - \alpha_k^{'(d)}) \neq 0.$
\end{enumerate}
Then, the model components $(\TT, \mcg)$ are identifiable.

\end{theorem}

Our key observation behind proving the identifiability of the complex deep generative structures in DDEs is that, with discrete latent layers, it suffices to establish identifiability for a one-latent-layer model and proceed in a layer-wise manner inductively.
To elaborate, consider a ``collapsed'' DDE where the latent layers indexed by $d = 2, \ldots, D$ are marginalized out to give a probability mass function (pmf) for the first latent layer: $\{\PP(\ma^{(1)} = \aaa^{(1)}): \aaa^{(1)} \in \{0,1\}^{K^{(1)}}\}$.
{If the collapsed model is proven to be identifiable, that means the conditional distributions $\mathbb P(Y_j\mid \ma^{(1)})$ and the marginal distribution $\PP(\ma^{(1)})$ can be uniquely identified from the data distribution $\mathbb P(\YY)$}. In this case, we can determine the \darkblue{structured pmf up to the inevitable permutation of the latent variables (indexed by $\sigma^{(1)}\in S_{[K^{(1)}]}$ in \cref{def:identifiability up to permutation})}. 
Then, we can theoretically treat the shallowest latent layer $\mathbf A^{(1)}$ as if it was observed, because its probability mass function is now identified and known. Then, viewing this pmf as the marginal distribution of $K^{(1)}$-dimensional observed variables of a $(D-1)$-latent-layer Bernoulli-DDE, we can inductively identify all parameters via a layerwise argument. 
Please see Supplementary Material S.1 for the detailed proof.

We next give an interpretation of the conditions in \cref{thm:deep id}. \darkblue{Condition A requires each latent variable $A_k^{(d)}$ to have at least two \emph{pure children} in the layer below. For example, the latent variable ``tech'' in the right panel of Figure \ref{fig:model} has two pure children ``graphics'' and ``space''.}
Condition B is more technical and is introduced to distinguish the different binary latent configurations $\aaa \neq \aaa'$. 
For example, condition B holds when each latent variable $A_k^{(d)}$ has a third pure child that is distinct from those in condition A. \darkblue{We also provide identifiability guarantees for the number of latent variables in Supplement S.1.6.}

The pure children requirements in condition A can be further relaxed under a weaker notion of \emph{generic identifiability}. Generic identifiability allows a measure-zero subset of the parameter space to be non-identifiable, and often holds under weaker conditions than that in Definition \ref{def:identifiability up to permutation}.
As the concept was originally proposed under a continuous parameter space \citep{allman2009identifiability}, we consider the following modified definition that considers a smaller parameter space for the coefficients $\mcb$ given the true graphical matrices.

\begin{definition}[Generic identifiability]\label{def:gid}
    Consider a $D$-latent-layer DDE with $\mck$ latent variables, graphical matrices $\mcg^\star$, and true parameters belonging to 
    $\Omega_{\mck} (\TT; \mcg^\star)$. 
    The model is generically identifiable up to $\sim_{\mck}$ when
        $\{\TT \in \Omega_{\mck} (\TT; \mcg^\star)$:  there exists  $(\tilde{\TT}, \tilde{\mcg}) \not \sim_\mck (\TT,\mcg^\star)$ such that $\PP_{\tilde{\TT}, \tilde{\mcg}} = \PP_{\TT,\mcg^\star} \}$
    is a measure-zero set with respect to $\Omega_{\mck}(\TT; \mcg^\star)$.
\end{definition}
\noindent

For generic identifiability, we introduce an additional assumption on the parametric families used to model the conditional distribution $Y_j \mid \ma^{(1)}$ in \eqref{eq:observed exp fam}. This is a technical assumption that arises from our proof technique for dealing with measure-zero sets. This assumption holds for all example parametric families described in \cref{subsec:model def}, and more generally for exponential families with an analytic {log}-partition function.

\begin{assumption}[Analytic family]\label{assmp:monotone family}
    Let $p(\cdot ; \eta, \gamma)$ be the pmf/pdf of a parametric family, indexed by $\eta, \gamma$ and equipped with a sample space $\mathcal{Y}$.
    We say that $p$ is \emph{analytic} when the pmf/pdf $p(Y; \eta, \gamma)$ is analytic in both $\eta, \gamma$, for all $Y \in \mcy$.
\end{assumption}

We next state the generic identifiability result for DDEs. \darkblue{In graph theory, a bipartite graph is said to have a ``perfect matching'' if it contains a set of edges without common vertices that covers every vertex of the graph \citep{hall2011combinatorial}; see \cref{ex:id condition} for an illustration.}} 

\begin{theorem}[Generic identifiability of DDEs]\label{thm:deep gid}
Consider the $D$-latent-layer DDE where the number of latent variables $\mck$ is given, and all parametric families and link function $g_j$s are analytic. Let $\mcg^\star$ denote the true graphical matrices and suppose the true parameter lives in $\Omega_{\mck}(\TT; \mcg^\star)$. Suppose that for all $d \le D-1$, the true graphical structure $\GG^{(d)}$ satisfy the following condition C:
\begin{enumerate}
    \item[C.] \darkblue{There exists a partition of the $(d-1)$th layer variable indices $[K^{(d-1)}] = \mathcal{I}_1 \cup \mathcal{I}_2 \cup \mathcal{I}_3$ that satisfies the following properties: (i) for $a=1, 2$, there is a perfect matching in the bipartite graph between $\ma^{(d-1)}_{\mathcal{I}_a}$ and $\ma^{(d)}$, (ii) each latent variable $A_k^{(d)}$ has at least one child among $\ma^{(d-1)}_{\mathcal{I}_3}$.}
\end{enumerate}
Then, the model components $(\TT, \mcg)$ are generically identifiable.
\end{theorem}

Condition C relaxes conditions A and B in Theorem \ref{thm:deep id}. While condition A requires two pure children for each latent variable, condition C does not require any pure child and allows more complex dependence structures. Additionally, condition B for the remaining $J-2K^{(d)}$ variables is relaxed to a simple non-zero column condition on $\GG^{(d)}_3$.
Thus, condition C does not concern any continuous parameter values and just depends on the graph structure $\GG^{(d)}$, and can be used as a practical criterion to assess identifiability.

\begin{example}[Illustration of identifiability conditions]\label{ex:id condition}
    \darkblue{We provide a toy example to illustrate and compare conditions in \cref{thm:deep id} and \cref{thm:deep gid}. In \cref{fig:pure children}, condition A holds when only the solid edges exist; in this case, the latent variable $A_1^{(d)}$ has two pure children $A_1^{(d-1)}, A_3^{(d-1)}$, and latent variable $A_2^{(d)}$ also has two pure children. In contrast, condition C allows arbitrary additional dashed arrows in the graph structure; in this case, by taking $\mathcal{I}_1 = \{A_1^{(d-1)}, A_2^{(d-1)}\}$ and $\mathcal{I}_2 = \{A_3^{(d-1)}, A_4^{(d-1)}\}$, the red and blue solid edges each form a perfect matching.
    Thus, condition C significantly relaxes the pure child condition in condition A, by allowing many more additional edges in the bipartite graph.}
\end{example}

\begin{figure}[h!]
\centering
\resizebox{0.35\textwidth}{!}{
    \begin{tikzpicture}[scale=2, baseline=(current bounding box.center)]

    \node (v1)[neuron] at (0, 0) {$A_1^{(d-1)}$};
    \node (v2)[neuron] at (0.8, 0) {$A_2^{(d-1)}$};
    \node (v3)[neuron] at (1.6, 0) {$A_3^{(d-1)}$};
    \node (v4)[neuron] at (2.4, 0) {$A_4^{(d-1)}$};
    \node (v5)[neuron] at (3.2, 0) {$A_5^{(d-1)}$};
       
    \node (h1)[hidden] at (1.07, 1) {$A_1^{(d)}$};
    \node (h2)[hidden] at (2.14, 1) {$A_2^{(d)}$};

    \draw[ultra thick, qedge, blue] (h1) -- (v1) node [midway,above=-0.12cm,sloped] {}; 
    \draw[ultra thick, qedge, blue] (h2) -- (v2) node [midway,above=-0.12cm,sloped] {};  
    \draw[ultra thick, qedge, red] (h1) -- (v3) node [midway,above=-0.12cm,sloped] {}; 
    \draw[ultra thick, qedge, red] (h2) -- (v4) node [midway,above=-0.12cm,sloped] {}; 
    \draw[qedge, blue, dashed] (h1) -- (v2) node [midway,above=-0.12cm,sloped] {}; 
    \draw[qedge, red, dashed] (h1) -- (v4) node [midway,above=-0.12cm,sloped] {}; 
     \draw[qedge, red, dashed] (h2) -- (v3) node [midway,above=-0.12cm,sloped] {}; 
     \draw[qedge, blue, dashed] (h2) -- (v1) node [midway,above=-0.12cm,sloped] {}; 
    \draw[qedge] (h1) -- (v5) node [midway,above=-0.12cm,sloped] {}; 
    \draw[qedge] (h2) -- (v5) node [midway,above=-0.12cm,sloped] {}; [midway,above=-0.12cm,sloped] {};

\end{tikzpicture}
}
\quad \quad $\GG^{(d)} = \begin{pmatrix}
    \multicolumn{2}{c}{\color{blue} 1 \quad *} & \multicolumn{2}{c}{\color{red} 1 \quad *} & 1 \\
    \multicolumn{2}{c}{\underbrace{\color{blue} * \quad 1}_{\mathcal{I}_1}} &  \multicolumn{2}{c}{\underbrace{\color{red} * \quad 1}_{\mathcal{I}_2}} & \underbrace{1}_{\mathcal{I}_3} \\
\end{pmatrix}^\top$
\caption{\darkblue{Graphical illustration of conditions A and C. While condition A holds when all dashed arrows are ignored (or equivalently, $* = 0$ in the matrix representation), condition C can allow arbitrarily many dashed arrows (where each $*$ can be either zero or one).}}
\label{fig:pure children}
\end{figure}
\color{black}

\vspace{-2mm}
We place our identifiability results in the literature on the identifiability of generative models and latent variable models. While the identifiability of generative models has attracted increasing attention in machine learning \citep{hyvarinen2019nonlinear, khemakhem2020variational, moran2021identifiable, kivva2022identifiability}, many of these results require additional information such as auxiliary covariates.
More importantly, almost all of these results build on nonlinear independent component analysis \citep{oja2000independent} or variational autoencoders \citep{ranzato2008semi}, both of which essentially have only one latent layer of random variables transformed by deterministic deep neural networks.
Consequently, these results cannot be applied to DDEs with multiple latent layers organized in a probabilistic graphical model.
Since uncertainty occurs and accumulates in each layer of a DDE, addressing identifiability in such cases requires fundamentally different techniques due to a complicated marginal likelihood.
{At the high level, our proof techniques are based on transforming the marginal distribution of data into a tensor and invoking the uniqueness of tensor decompositions to establish identifiability.}

In statistics, the study of identifiability has a long history but also mainly concerns relatively simple latent structures with only one latent layer \citep{anderson1956statistical, koopmans1950identification, allman2009identifiability,xu2018identifying}. \darkblue{In particular, \cite{gu2023bayesian} establish identifiability of a multi-layer latent structure model, but their result is \emph{only applicable to categorical response data}, in contrast to the general and rich types of responses considered in this work;
moreover, \cite{gu2023bayesian}'s strict identifiability result requires each latent variable to have at least three pure children.
On the other hand, \cite{lee2024new} considered models for general responses which are similar to ours, but their model is restricted to only one latent layer, and they prove identifiability under a strong assumption that the true graph between the latent and observed layers is known.  Compared to these existing results for related models, our results (i) apply to arbitrary response types as opposed to only categorical responses, (ii) require weaker identifiability conditions compared to \cite{gu2023bayesian} (in terms of both strict and generic identifiability), and (iii) do not require a known graph structure as in \cite{lee2024new}.
}

\subsection{Estimation Consistency}
In this section, we propose a penalized maximum likelihood estimation method for DDEs, and show that the estimator is consistent for both the continuous parameters and the discrete graph structures.
Suppose that the numbers of latent variables in all layers are known.
We maximize the following objective function to estimate parameters ${\TT}=(\pp, \mcb, \ggamma)$:
\begin{equation}\label{eq:penalized optimization}
    \hat{\TT} \in \argmax_{\TT} \left[ \ell (\TT \mid \YY) - \sum_{d=1}^D p_{\lambda_N,\tau_N} (\B^{(d)})\right],
\end{equation}
where $\ell(\TT \mid \YY) = \sum_{i=1}^N \log \mathbb P(\YY_i\mid \TT)$ denotes the marginal log-likelihood function given a sample $\YY_1,\ldots,\YY_N$ of size $N$, defined based on the marginal distribution of $\mathbf Y$ in \eqref{eq:marginal probability}.
From now on, we slightly abuse notation and let $\mathbf Y$ denote the $N\times J$ data matrix including $\YY_1,\ldots,\YY_N$ as rows.
Using the estimated coefficients in $\hat{\mcb}=(\hat{\beta}_{l,k}^{(d)})$, the layer-wise graphical matrices in $\mcg$ can be estimated by reading off the sparsity pattern of $\hat{\mcb}$:
\begin{equation}\label{eq:graphical matrix estimator}
    \hat{g}_{l,k}^{(d)} := \mathbbm{1}(\hat{\beta}_{l,k}^{(d)} \neq 0) \quad \text{for all} \quad d \in [D], ~k \in [K^{(d)}], ~l \in [K^{(d-1)}].
\end{equation}
For some tuning parameters $\lambda_N, \tau_N > 0$, $p_{\lambda_N,\tau_N}:\mathbb{R}\to [0,\infty)$ is a sparsity-inducing symmetric penalty that satisfies several technical conditions postponed to Supplement S.1.4.
Our assumption for the penalty $p_{\lambda_N, \tau_N}$ includes common truncated sparsity-inducing penalties such as the TLP \citep[Truncated Lasso Penalty;][]{shen2012likelihood} and SCAD \citep[Smoothly Clipped Absolute Deviation;][]{fan2001variable}. 
With a slight abuse of notation, in \eqref{eq:penalized optimization}, we view the penalty $p_{\lambda_N,\tau_N}$ as a function of matrices by letting it be the sum of the entrywise penalties: $p_{\lambda_N,\tau_N} (\B^{(d)}) = \sum_{k \in [K^{(d-1)}], \ell \in [K^{(d)}]} p_{\lambda_N,\tau_N} (\beta_{k,\ell}^{(d)}).$

Assuming a compact parameter space with bounded coefficients $\B^{(d)}$, we prove that the estimator defined in \eqref{eq:penalized optimization} and \eqref{eq:graphical matrix estimator} results in consistent estimation. 

\begin{theorem}\label{thm:estimation consistency}
Consider a $D$-latent-layer DDE with true parameters $\TT^\star, \mcg^\star$ and known $\mck$.
Assume that the model at $\TT^\star$ is identifiable, has a non-singular Fisher information, and all entries of $\{\BB^{(d)}\}_{d=1}^D$ are bounded.
Then, the estimator $\hat{\TT}$ in \eqref{eq:penalized optimization} is $\sqrt{N}-$consistent in the sense that there exists some $\tilde{\TT} \sim_{\mck} \hat{\TT}$ such that $\lVert \tilde{\TT} - \TT^\star \rVert = O_p({1}/{\sqrt{N}})$. Here, $\|\cdot\|$ denotes the vectorized $L^2$ norm.
Additionally, the graphical matrices are consistently estimated: for $\tilde{\mcg}$ resulting from $\tilde{\TT}$ according to \eqref{eq:graphical matrix estimator}, we have $\PP (\tilde{\mcg} \neq \mcg^\star) \rightarrow 0$.
\end{theorem}

\section{Scalable Computation Pipeline}
\label{sec:algorithm}
We present a two-stage scalable computational pipeline to compute the penalized maximum likelihood estimator in \eqref{eq:penalized optimization}. Our proposed method builds upon the standard penalized EM algorithm (see Supplement S.3.1) by including a spectral initialization and stochastic approximation. We separate out each stage into Sections \ref{subsec-spectral-ini} (stage one) and \ref{subsec:saem} (stage two). For notational simplicity, we describe the proposed method assuming $D = 2$ latent layers, which straightforwardly extends to an arbitrary number of latent layers.

\subsection{Stage One: Layerwise Double-SVD Initialization}
\label{subsec-spectral-ini}
The multiple layers of nonlinearity in DDEs lead to a {highly nonconvex optimization landscape with potentially exponentially many local optima} \citep{sutskever2013importance}.
For example, if the EM algorithm starts with an initialization close to a local optima, it can get stuck and fail to converge to the global maximizer of the penalized log-likelihood function.
Hence, for highly complex latent variable models such as DDEs, it is critical to initialize the optimization algorithm wisely.
We propose a novel layerwise nonlinear spectral initialization strategy, which enjoys low computational complexity and reasonably high accuracy.
This spectral initialization serves as the first stage in the proposed computational pipeline.

Spectral methods have mostly been used for linear low-rank models, and existing approaches are not directly applicable for DDEs with a deep nonlinear structure.
To address this, we propose a nuanced layerwise procedure utilizing the double SVD procedure for denoising low-rank generalized linear factor models \citep{zhang2020note} and the SVD-based Varimax to find sparse rotations of the factor loadings \citep{rohe2023vintage}. 

\begin{algorithm}[h!]
\caption{Outline of the Layerwise Double-SVD Initialization}
\KwData{Data matrix $\YY_{N\times J}$, latent dimensions $\mck$.}
    \begin{enumerate}[leftmargin=*,itemsep=0pt]
        \item {De-noise the data matrix $\YY$ using a first SVD, and linearize this matrix by applying the inverse-link function $(\mu \circ g)^{-1}$ elementwisely. Let $\hat{\mathbf{Z}}$ denote the inverted matrix.}
        \item Let $\hat{\mathbf{Z}}_0$ be the column-centered version of $\hat{\mathbf{Z}}$, and compute its rank-$K^{(1)}$ approximation by a second SVD: $\hat{\mathbf{Z}}_0 \approx \mathbf{U}_{N \times K^{(1)}} \mathbf{\Sigma}_{K^{(1)} \times K^{(1)}} \mathbf{V}_{J \times K^{(1)}}^\top$.  
        \item Rotate $\mathbf{V}$ according to the Varimax criteria, and denote it as $\hat{\BB}^{(1)}$. Modify the sign ($\pm$) of each column so that \cref{assmp:monotone family}(c) is satisfied.
        \item Use the relationship $\hat{\mathbf{Z}} \approx [\mathbf{1}_N, {\ma}^{(1)}] {\BB}^{(1) \top}$ to estimate $\ma^{(1)}$.
        \item Now, suppose that the estimated $\hat{\ma}^{(1)}$ is the ``observed data'' of a one-latent-layer DDE. Repeat steps 1-4 to estimate ${\BB}^{(2)}$ and $\ma^{(2)}$.
    \end{enumerate}\label{algo-init}
\end{algorithm}
\vspace{-5mm}

Our main idea is to view the responses $\YY$ as a perturbation of a population expectation
$\E (\YY \mid \ma^{(1)}, \BB^{(1)}) = (\mu \circ g)\left( [\mathbf{1}_N, \ma ^{(1)}] \BB^{(1) \top} \right)$, which is an elementwise (nonlinear) transformation of a low-rank matrix.
Here, $\mu:H \to \mcy$ is the known mean function of the observed-layer parametric family in \eqref{eq:observed exp fam} and $g:\mathbb{R}\to H$ is the link function. 
The function $(\mu \circ g)$ is equal to $\gl$ for Bernoulli responses, the exponential function for Poisson, and the identity function for Normal.
When $(\mu \circ g)$ is nonlinear, we use the aforementioned double SVD procedure \citep{zhang2020note}. This procedure applies a first SVD to de-noise the data, and then linearizes the data through inverting the link function $(\mu \circ g)$, and finally performs a second SVD to find the low-rank matrix $\hat{\mathbf{Z}} \approx [\mathbf{1}_N, \ma ^{(1)}] \BB^{(1) \top}$. 
Next, noting that the true coefficient matrix $\BB^{(1)}$ is sparse, we estimate it by seeking a sparse rotation of the right singular subspace of $\hat{\mathbf{Z}}$, using the popular Varimax criterion \citep{kaiser1958varimax, rohe2023vintage}.
This procedure also provides an estimate of the latent variables $\ma^{(1)}$, which can be treated as the ``observed data'' to initialize the deeper layer's $\BB^{(2)}$ and $\ma^{(2)}$ in a similar fashion as described above. This layer-by-layer algorithm readily generalizes to deeper models and is reminiscent of the greedy learning strategy for DBNs \citep{hinton2006fast,salakhutdinov2015learning}.

We summarize the overall procedure in \cref{algo-init}, and postpone further details of each step to Supplementary Material S.2. In the Supplement, we also illustrate the effectiveness of the spectral initialization by comparing the estimation accuracy of our two-stage computational pipeline to that of the EM algorithm with a random initialization.

\subsection{Stage Two: Penalized SAEM Algorithm}\label{subsec:saem}

For DDEs with a large number of latent variables, the E-step in standard EM algorithms (see Supplement S.3.1) computes  all conditional probabilities $\PP (\ma^{(1)}_i = \aaa^{(1)}, \ma^{(2)}_i = \aaa^{(2)} \mid \YY; \TT^{[t]})$ for all $\aaa^{(1)}\in\{0,1\}^{K^{(1)}}$ and $\aaa^{(2)}\in\{0,1\}^{K^{(2)}}$. This requires computing and storing $O(N \times 2^{\sum_{d=1}^D K^{(d)}})$ terms.
The exponential dependency in $K^{(d)}$ is concerning even for moderately large latent dimensions, say $K^{(d)}=10$, and quickly becomes prohibitive for larger $K^{(d)}$. 
Therefore, we propose a penalized Stochastic Approximate EM  \citep[SAEM; see][]{delyon1999convergence, kuhn2004coupling} by modifying both the E-step and M-step to more scalable versions using approximate sampling. As we illustrate below, this is a method with linear dependence of $\sum_{d=1}^D K^{(d)}$ on the computing time as well as memory.

We elaborate on the details on deriving the SAEM. \textit{First}, we replace the E-step to a simulation step, which consists of simulating only a small number (denoted as $C$) of posterior samples of the latent variables. As exact sampling from the joint distribution $\PP(\ma_i^{(1)}, \ma_i^{(2)} \mid \YY; \TT^{[t]})$ is expensive, we sample each latent variable from their complete conditionals. That is, we sample each $A_{i,k}^{(1),[t+1]}$ from $\PP(A_{i,k}^{(1)} \mid (-), \TT^{[t]})$, where $(-)$ denotes the estimates of all other latent variables from the $t$th iteration. Since the latent variables are binary, the conditional distribution is Bernoulli and easy to evaluate.
Consequently, the computationally expensive E-step is replaced by the simulation step, which computes and stores only $O(N \times \sum_{d=1}^D K^{(d)})$ terms. 
In terms of choosing the number of samples $C$ in each iteration, we empirically find that taking $C=1$ is computationally efficient without sacrificing much accuracy (see Supplement S.4.4). 
This choice of $C=1$ was also suggested in the original paper that proposed the SAEM \citep{delyon1999convergence}.

\textit{Second}, we modify the standard M-step objective function (expected complete data log-likelihood) by
(i) replacing the exact conditional probability values 
to sample-based quantities, (ii) stochastically averaging the objective functions, and (iii) including the sparsity-inducing penalties in \eqref{eq:penalized optimization}. The M-step objective for updating $\BB^{(2)}$ is
\begin{align}\notag
    & Q^{(2),[t+1]}(\BB^{(2)}) \\  
    \label{eq:SAEM stochastic approx} 
    & :=
    (1 - \theta_{t+1}) Q^{(2),[t]}(\BB^{(2)}) + \theta_{t+1} \sum_{i=1}^N \log \PP(\ma_{i}^{(1)} = \ma_{i}^{(1), [t+1]} \mid \ma_i^{(2)} = \ma_{i}^{(2), [t+1]}; \BB^{(2)}),\\
    & \BB^{(2),[t+1]} := \argmax_{\BB^{(2)}} \left[Q^{(2),[t+1]}(\BB^{(2)}) - p_{\lambda_N,\tau_N}(\BB^{(2)}) \right], \label{eq:SAEM maximization B2}
\end{align}
where $Q^{(2),[0]} = 0$ and $\theta_{t+1}$ is a pre-determined step size that decreases in $t$. Here, in the log probability term in \eqref{eq:SAEM stochastic approx}, $\ma_i^{(d)}$ denotes the \emph{latent random variable} and $\ma_{i}^{(d), [t+1]}$ denotes the \emph{realized sample} from the simulation step in the current $(t+1)$th iteration.
In the $(t+1)$th iteration, we update the objective function $Q^{(2),[t+1]}$ by taking a weighted average of the previous objective function $Q^{(2),[t]}$, and the conditional probabilities computed using the current iteration's simulated samples $\ma^{(1),[t+1]}$. Then, in \eqref{eq:SAEM maximization B2} we compute the parameters that maximize the penalized objective function.

\begin{algorithm}[h!]
\caption{Penalized SAEM algorithm for the two-latent-layer DDE}
\label{algo-saem}
\SetKwInOut{Input}{Input}
\SetKwInOut{Output}{Output}

\KwData{$\YY, \mck$, tuning parameters $\lambda_{N}, \tau_N$.}
Initialize $\ma^{(1),[0]}, \ma^{(2),[0]}$ and $\TT^{[0]}$ based on \cref{algo-init}.  \\
\While{$\|{\TT}^{[t]} - {\TT}^{[t-1]} \|$ is larger than a threshold}{
 In the $t$th iteration,
 
\texttt{// Simulation-step}

Sample each $A_{i,k}^{(1), [t+1]}, A_{i,l}^{(2), [t+1]}$ from the complete conditionals using the previous parameter estimates $\TT^{[t]}, \ma^{[t]}$

\vspace{2mm}
\texttt{// Stochastic approximation M-step}

update the parameters $\TT^{[t+1]}$ by maximizing the stochastic averaged objectives (e.g. see \eqref{eq:SAEM maximization B2})

}
Estimate $\GG$ based on the sparsity structure of $\hat{\BB}$ according to \eqref{eq:graphical matrix estimator}.\\
\Output{Estimated continuous parameters $\hat{\TT}$, estimated graphical matrices $\mcg$.}
\end{algorithm}
\begin{algorithm}[h!]
    \caption{\darkblue{Practical methods for selecting the latent dimensions $\mck$}}
    \label{algo:select_K}
    \SetKwInOut{Input}{Input}
    \SetKwInOut{Output}{Output}
    
    \KwData{$\hat{\ma}^{(0)} := \YY_{N\times J}, \hat{K}^{(0)}:=J$, the number of latent layers $D$}
    For each $1 \le d \le D$, repeat:
    \begin{enumerate}[topsep=0pt,itemsep=-1ex]
        \item Let $\mathfrak{K}^{(d)} := \{\lceil \hat{K}^{(d-1)}/4 \rceil, \ldots, \lfloor \hat{K}^{(d-1)}/2 \rfloor \}$ be the candidate grid for $\hat{K}^{(d)}$. 
        \item Define $\hat{K}^{(d)}$ based on the largest spectral ratio of the denoised/linearized $\hat{\ma}^{(d-1)}$ (see step 1 in Algorithm \ref{algo-init}):
        $$\hat{K}^{(d)} := \argmax_{k \in \mathfrak{K}^{(d)}} \sigma_{k}/\sigma_{k+1}-1.$$
        \item Given $\hat{K}^{(d)}$, estimate the $d$-th layer latent variables $\hat{\ma}^{(d)}_{N\times \hat{K}^{(d)}}$ using \cref{algo-init}.
    \end{enumerate}
    \Output{Estimated latent dimensions for all layers $\hat K^{(1)},\ldots,\hat K^{(D)}$.}
\end{algorithm}

\cref{algo-saem} summarizes our proposed SAEM algorithm with $C=1$, where detailed formulas are deferred to Supplement S.3.2. Here, the detailed M-step updates can be written in terms of low-dimensional maximizations for each row of the coefficient matrices.

\paragraph{Selecting the latent dimensions.}\label{sec:select k}
To apply the above computational pipeline to real data, one also needs to specify the number of latent variables, $\mck$. \darkblue{We propose a layer-wise estimation strategy in \cref{algo:select_K}, which can be incorporated into our initialization procedure in \cref{algo-init}.} Recall the denoised data matrix $\hat{\mathbf{Z}}$ from Step 1 of \cref{algo-init} and let $\sigma_{1},\sigma_2, \ldots$ be its singular values in the descending order. Given a candidate grid $\mathfrak{K}^{(1)}$, we estimate the size of the first latent layer based on the largest spectral ratio:
    $\hat{K}^{(1)} := \argmax_{k \in \mathfrak{K}} ({\sigma_{k}}/{\sigma_{k+1}}) - 1$.
Now, given $\hat{K}^{(1)}$, we proceed with the remaining steps of \cref{algo-init} to estimate the first-layer latent variables $\ma^{(1)}$. Inductively treating the estimated $\hat{\ma}^{(d)}$ as the observed variables of a one-latent-layer DDE, we can repeat the above procedure to estimate $\hat{K}^{(d+1)}$. See Supplement S.3.3 for alternative selection criteria for $\mck$.

\section{Simulation Studies}\label{sec:simulations}
We conduct extensive simulation studies in various settings to assess the performance of the proposed computation pipeline (Algorithms \ref{algo-init} and \ref{algo-saem}) and validate our identifiability conditions (in \cref{sec:identifiability}).

\vspace{-4mm}
\paragraph{Two-latent-layer DDEs with general response types.}
First, we generate data from two-latent-layer DDEs, exploring a total of 90 true settings by varying the following:
\begin{enumerate}[(a),topsep=0pt,itemsep=-1ex]
    \item three \emph{parametric families}: Bernoulli, Poission, Normal, 
    \item three \emph{paramter dimensions}: $(J, K^{(1)}, K^{(2)}) = (18, 6, 2), (54, 18, 6), (90, 30, 10)$,
    \item two \emph{parameter values}: see Supplement S.4.1,
    \item five varying \emph{sample sizes}: $N = 500, 1000, 2000, 4000, 8000$.
\end{enumerate}
Here, we consider two sets of parameter values that each satisfy the strict and generic identifiability conditions in Theorems \ref{thm:deep id} and \ref{thm:deep gid}. Regarding the parameter dimensions, given a value of $K^{(2)}$, we set $K^{(1)} = 3K^{(2)}$ and $J = 9K^{(2)}$. This allows a large latent dimension with as many as $K^{(1)} = 30$ binary latent variables in the shallowest latent layer.

\begin{figure}[h!]
    \centering
    \includegraphics[width=0.9\linewidth]{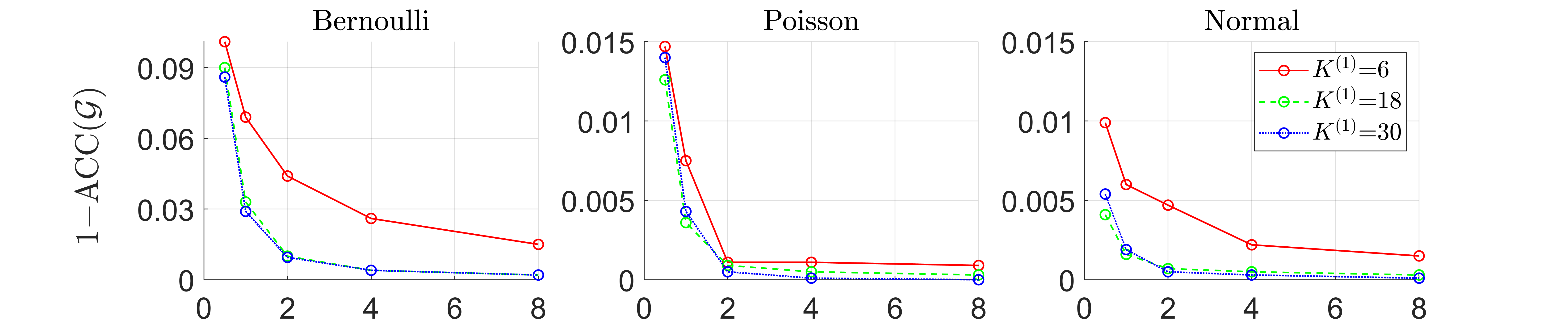}
    \includegraphics[width=0.9\linewidth]{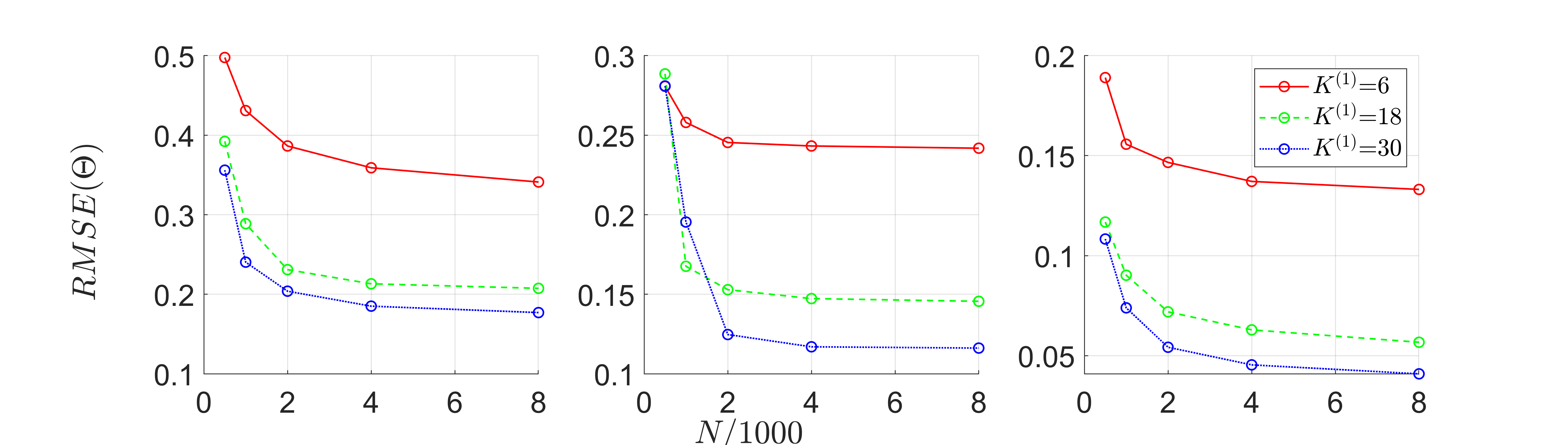}
    \caption{Estimation error for $\mcg$ and $\TT$ under the two-latent-layer DDE with strictly identifiable true parameters and various observed-layer parametric families. }
    \label{fig:acc g sid}
    \vspace{-2mm}
\end{figure}

For each scenario, we run $100$ independent simulations, and display the estimation results in Figure \ref{fig:acc g sid}.
The results under the generic identifiable parameters {and computation times} are included in Supplement S.4.3.
We measure the estimation accuracy of the graphical matrices $\mcg$ by computing the average entrywise accuracy. For the continuous parameters $\TT$, we report the root mean squared error (RMSE). Under all response types and parameter values, the estimation errors for both $\mcg$ and $\TT$ decrease as the sample size $N$ increases. This empirically justifies the identifiability and consistency results. Additionally, by comparing the estimation accuracy across different parametric families, we observe that the Bernoulli is the most challenging to estimate, and the Normal is the easiest. 

\vspace{-2mm}
\paragraph{\darkblue{Simulation studies under deeper models.}}
We assess the scalability of our proposed method for \emph{deeper models} (with $D=3,4,5$ latent layers) with potentially \emph{high-dimensional responses} (with the observed data dimension set to $J = 54, 162, 486$, respectively). For all settings, we consider Normal observations with $K^{(D)} = 2$ latent variables for the top (deepest) layer and $K^{(d-1)} = 3K^{(d)}$ variables for the $d$th layer for each  $d=D, \ldots, 2$.

We summarize the estimation accuracy of the graph structure in each layer and the average runtime for DDEs with $D=5$ latent layers in Table \ref{tab:D=5}. The results for $D=3,4$ are displayed in Supplement S.4.6. The computation time illustrates that our proposed method is scalable to deeper models. 
More specifically, the $D = 3, 4$ case only took a few minutes, and the most challenging case with $D = 5$ and $N = 16,000$ also took less than an hour on a personal laptop. 
In terms of estimation accuracy of graph structures $\mathcal{G}$, the accuracy is higher for shallower layers than for deeper layers. This results from the accumulation of uncertainty for deeper layers, which is inevitable as each layer consists of stochastic latent variables. 
Note that even for such deeper models with $D = 4, 5$ latent layers, the shallower structures $\GG^{(1)}, \GG^{(2)}$ are recovered with high accuracy. This indicates that one can include additional latent layers to increase the models' representational power, without sacrificing the accuracy of learning graph structures closer to the data layer.

\begin{table}[h!]
\centering
\begin{tabular}{cccccccccc}
\toprule
Layer \textbackslash $N$ & 500 & 1000 & 2000 & 4000 & 8000 & 16000 \\
\midrule
$\GG^{(1)}$ & 1.00 & 1.00 & 1.00 & 1.00 & 1.00 & 1.00 \\
$\GG^{(2)}$ & 0.81 & 0.90 & 0.95 & 0.97 & 0.98 & 0.98 \\
$\GG^{(3)}$ & 0.77 & 0.83 & 0.86 & 0.88 & 0.90 & 0.92 \\
$\GG^{(4)}$ & 0.61 & 0.62 & 0.65 & 0.67 & 0.66 & 0.74 \\
$\GG^{(5)}$ & 0.60 & 0.59 & 0.62 & 0.62 & 0.65 & 0.64 \\
\midrule
runtime (s) & 220 & 259 & 326 & 563 & 893 & 2925 \\
\bottomrule
\end{tabular}
\caption{\darkblue{Average entrywise-accuracy of estimating the graphical matrices and runtime in seconds for DDEs with $D=5$ latent layers.}}
\label{tab:D=5}
\vspace{-3mm}
\end{table}

\paragraph{\darkblue{Ablation studies.}}
As our model closely resemble DBNs, we mainly compare DDEs with DBNs. For fair comparison, we considered two settings: (a) data generated from a Berounlli-DDE, and (b) data generated from a DBN with identical coefficients. The results in  \cref{fig:ablation true DDE} demonstrate that our proposed algorithm has better estimation accuracy for both the graphical structure and the continuous parameters under both well-specified and mis-specified settings, even when data are generated from a DBN with undirected edges in the top layer. This is because the DBN algorithm does not learn sparse coefficients, and also fundamentally suffers from local optima due to a random initialization. 
Note that our identifiability results in Propositions S.1-2 immediately guarantee the identifiability of DBNs, and theoretically justify the above positive results of our algorithm.

The ablation studies also indicate that the current implementations of DDEs are slower than DBNs. We believe that our algorithm can be further scaled up by replacing the M-step in the SAEM algorithm to first-order optimization methods (e.g. gradient ascent or stochastic gradient ascent), which we leave for future work.

\begin{figure}[h!]
    \centering
    \begin{minipage}{0.24\textwidth}
     \centering
     \includegraphics[width=\linewidth]{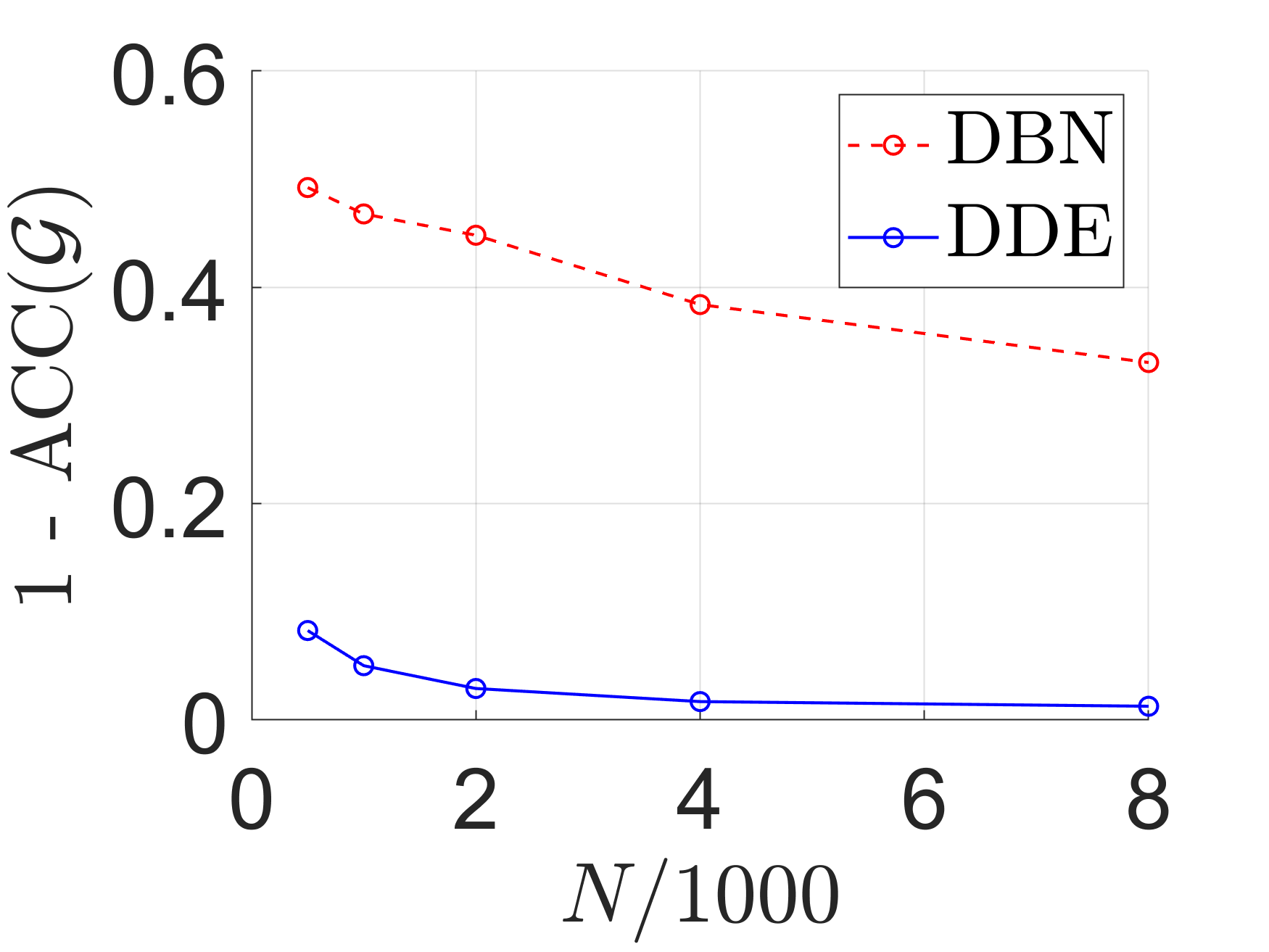}
     \subcaption{}
    \end{minipage}
    \begin{minipage}{0.24\textwidth}
     \centering
    \includegraphics[width=\linewidth]{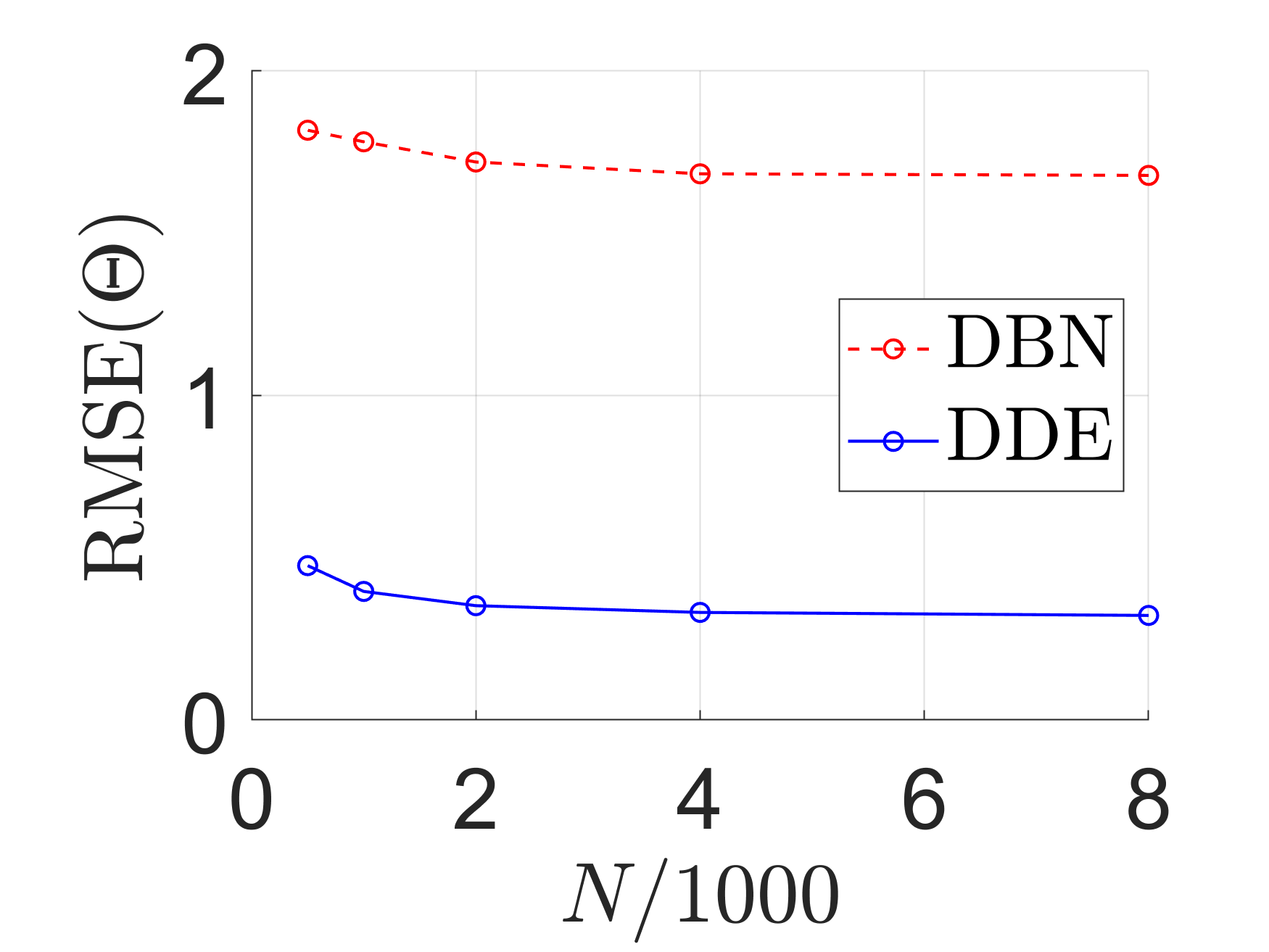}
    \subcaption{}
    \end{minipage}
    \begin{minipage}{0.24\textwidth}
     \centering
    \includegraphics[width=\linewidth]{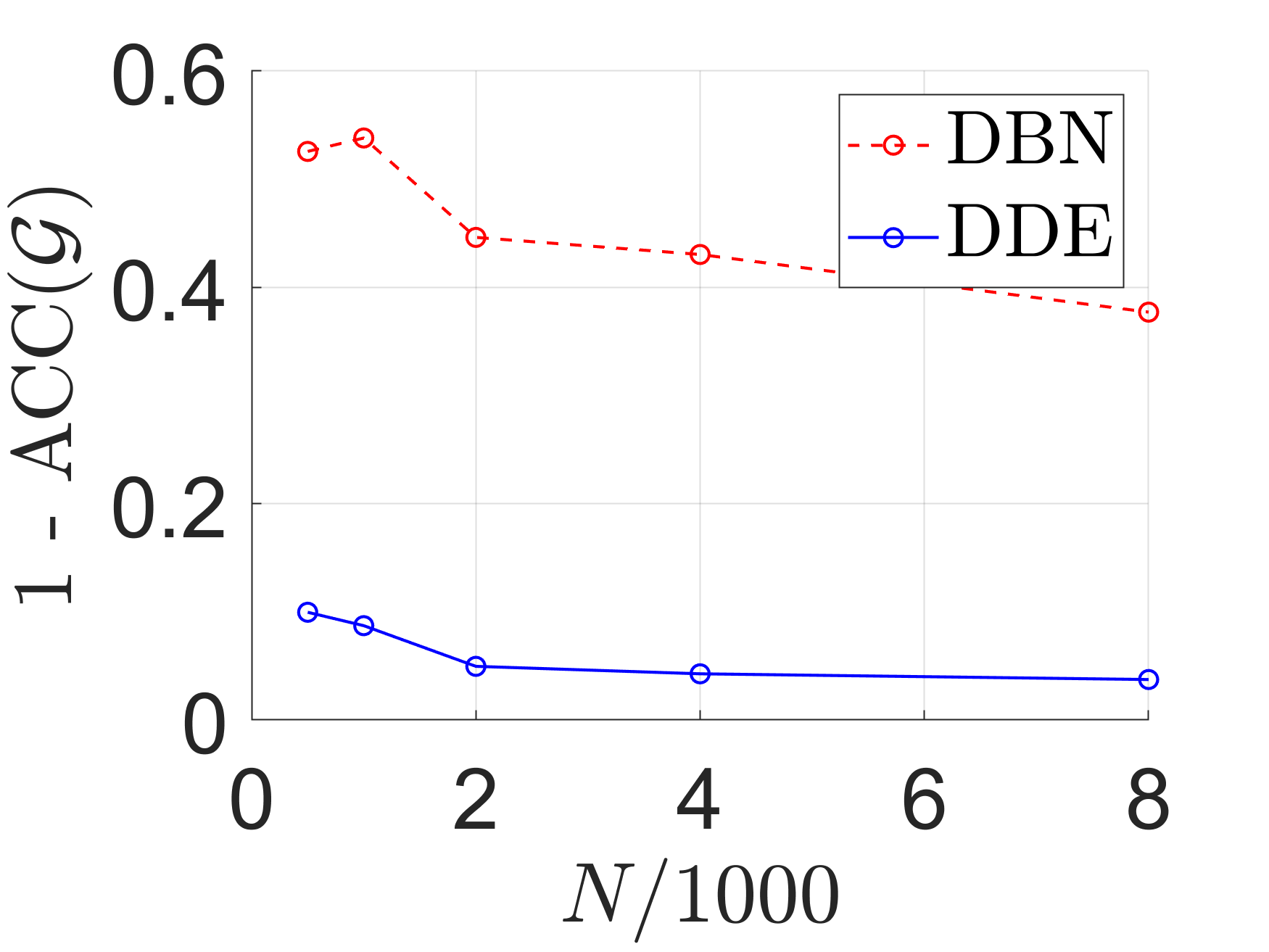}\subcaption{}
    \end{minipage}
    \begin{minipage}{0.24\textwidth}
     \centering
    \includegraphics[width=\linewidth]{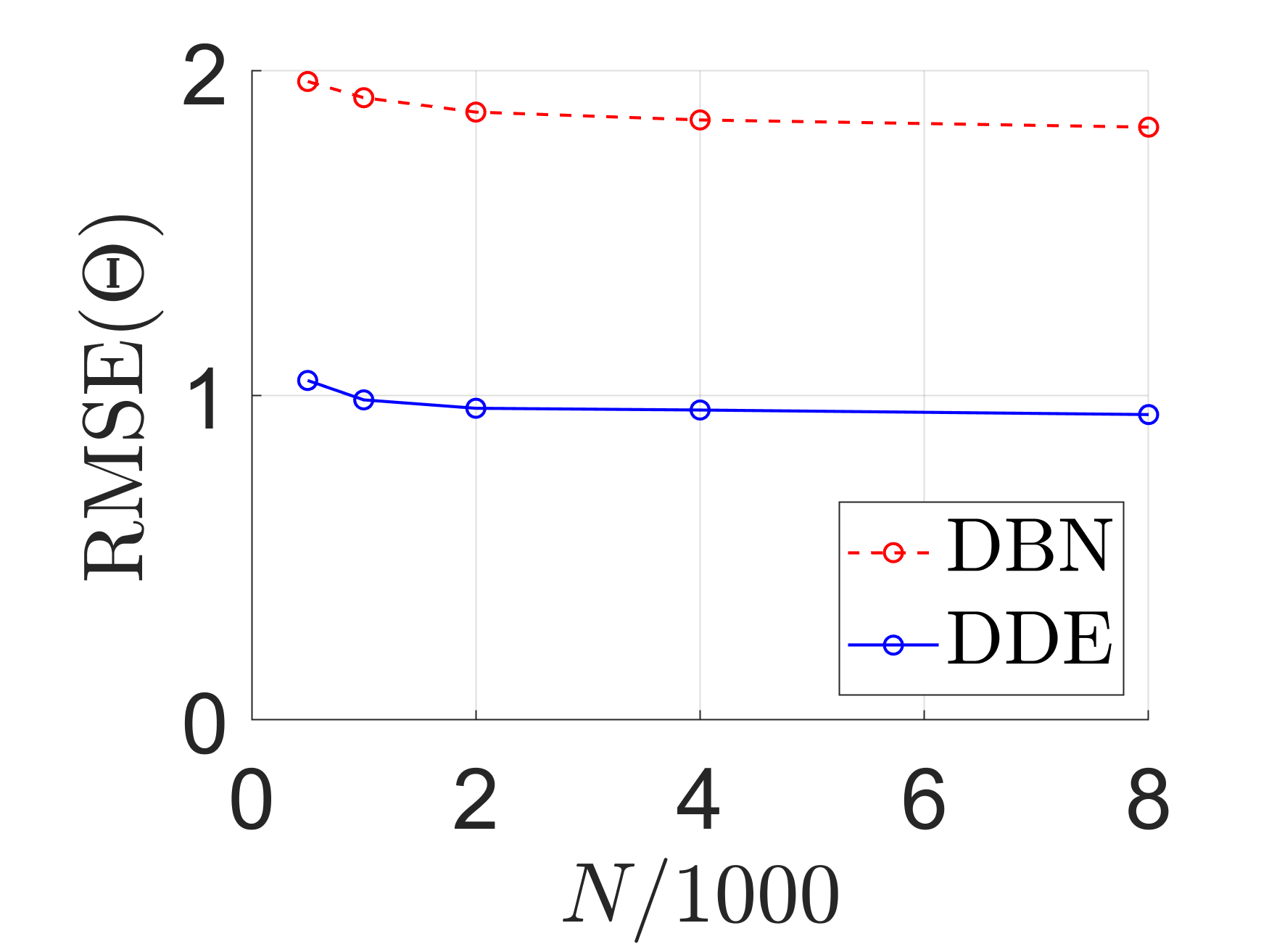}
    \subcaption{}
    \end{minipage}
    \caption{\darkblue{Comparison of Bernoulli-DDE vs DBN.  True data are generated from a \textbf{DDE} in (a)--(b), and from a \textbf{DBN} in (c)--(d). Red and blue lines show results obtained by the DBN algorithm and our DDE algorithm, respectively.
    (a) and (c): estimation error for the graphs $\mathcal{G}$. (b) and (d): estimation error for continuous parameters $\Theta$. {Smaller} values are better.}}
    \label{fig:ablation true DDE}
\end{figure}

\paragraph{Additional simulation results.}
In Supplement S.4.4 and S.4.6, we present additional simulation results where the latent dimension $\mck$ is unknown. We illustrate that the spectral-gap estimator has near-perfect selection accuracy for a large $N$ (e.g., larger than $4000$), and is superior in terms of both accuracy and computation time compared to alternatives. 

\section{Real Data Applications}\label{sec:real data}
We illustrate DDEs' interpretability, representation power, and downstream prediction accuracy on three diverse real-world datasets. Supplementary Material S.5.1 gives the preprocessing details of all datasets.

\subsection{Binary Data: Bernoulli-DDE for MNIST Handwritten Digits} 
The MNIST dataset for handwritten digits is very popular for classification as well as unsupervised learning \citep{MNIST}. We fit the two-latent-layer DDE with binary responses (Bernoulli-DDE), where the observed layer distributions in \eqref{eq:observed exp fam} are
$Y_j \mid (\ma^{(1)} = \aaa^{(1)}) \sim \text{Ber}\big(\gl(\beta_{j,0}^{(1)} + \sum_{k \in [K^{(1)}]} \beta_{j,k}^{(1)} \alpha_k^{(1)})\big)$.
This resembles existing generative models for images such as DBN and DBM, but we instead consider a much low-dimensional shrinking-ladder shaped latent structure that is identifiable and interpretable.
For easier presentation, we consider the subset of images whose true digit labels are 0, 1, 2, and 3.
After preprocessing, our training set consists of $N = 20,679$ images each with $J = 264$ binary pixels. Compared to many existing works that analyzed MNIST, we are considering a more challenging fully-unsupervised setting by holding out all other information about the images, such as the true labels, number of classes, and the spatial location among the pixels.

\begin{table}
\centering
\resizebox{\textwidth}{!}{
\begin{tabular}{{l}*6{C}@{}}
\toprule
Basis image & \includegraphics[width = \linewidth]{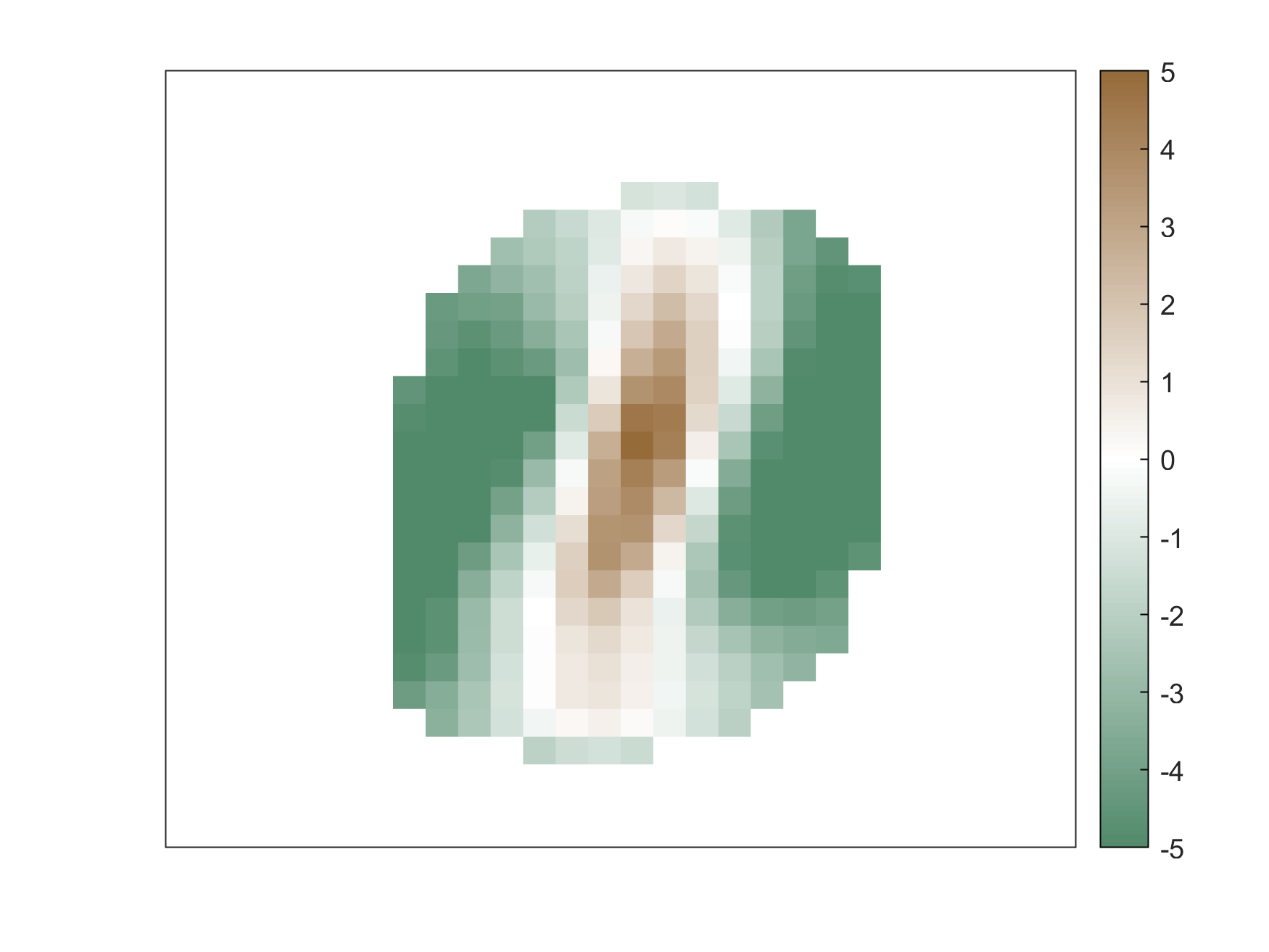} &
\includegraphics[width = \linewidth]{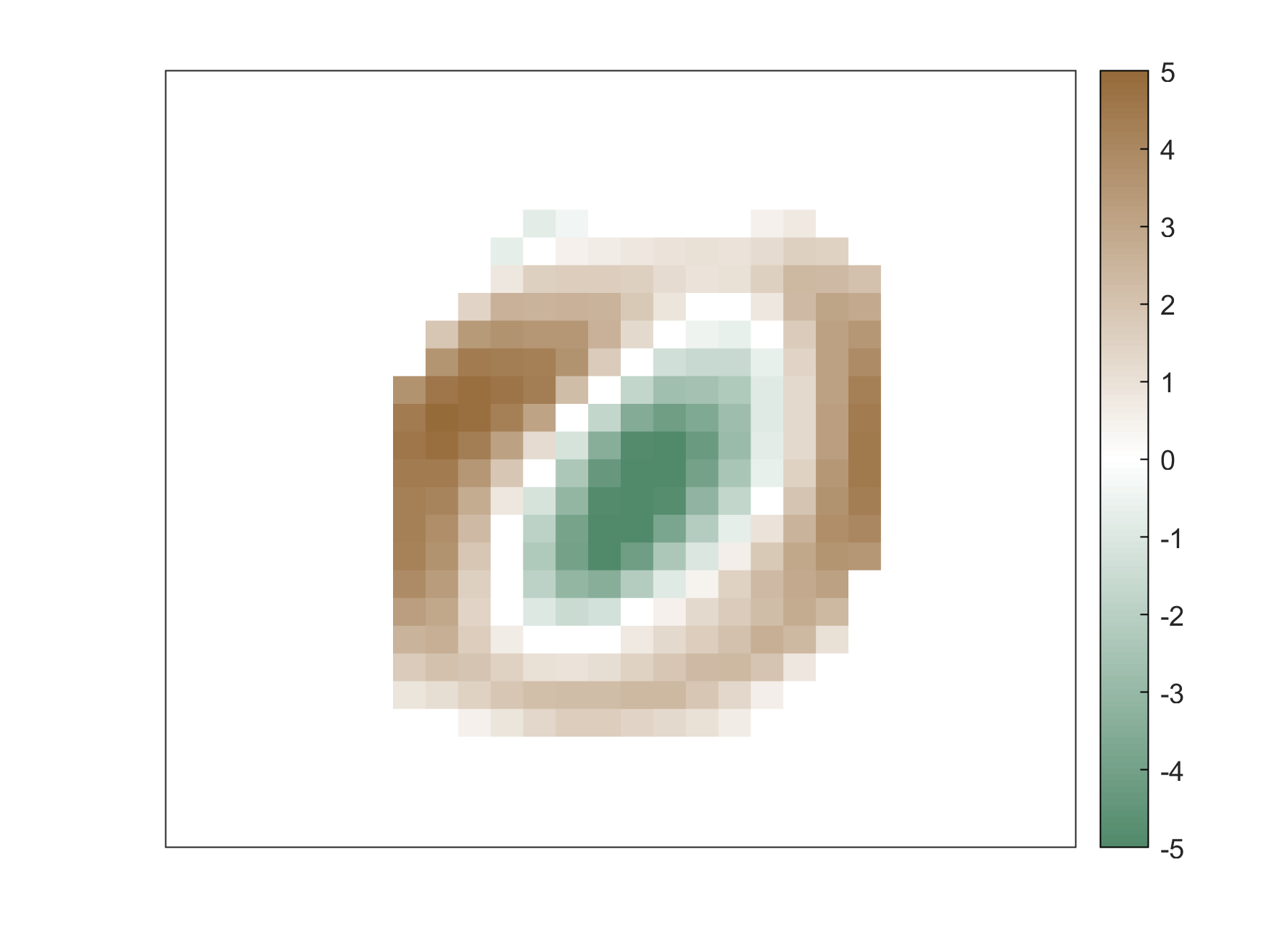} & \includegraphics[width = \linewidth]{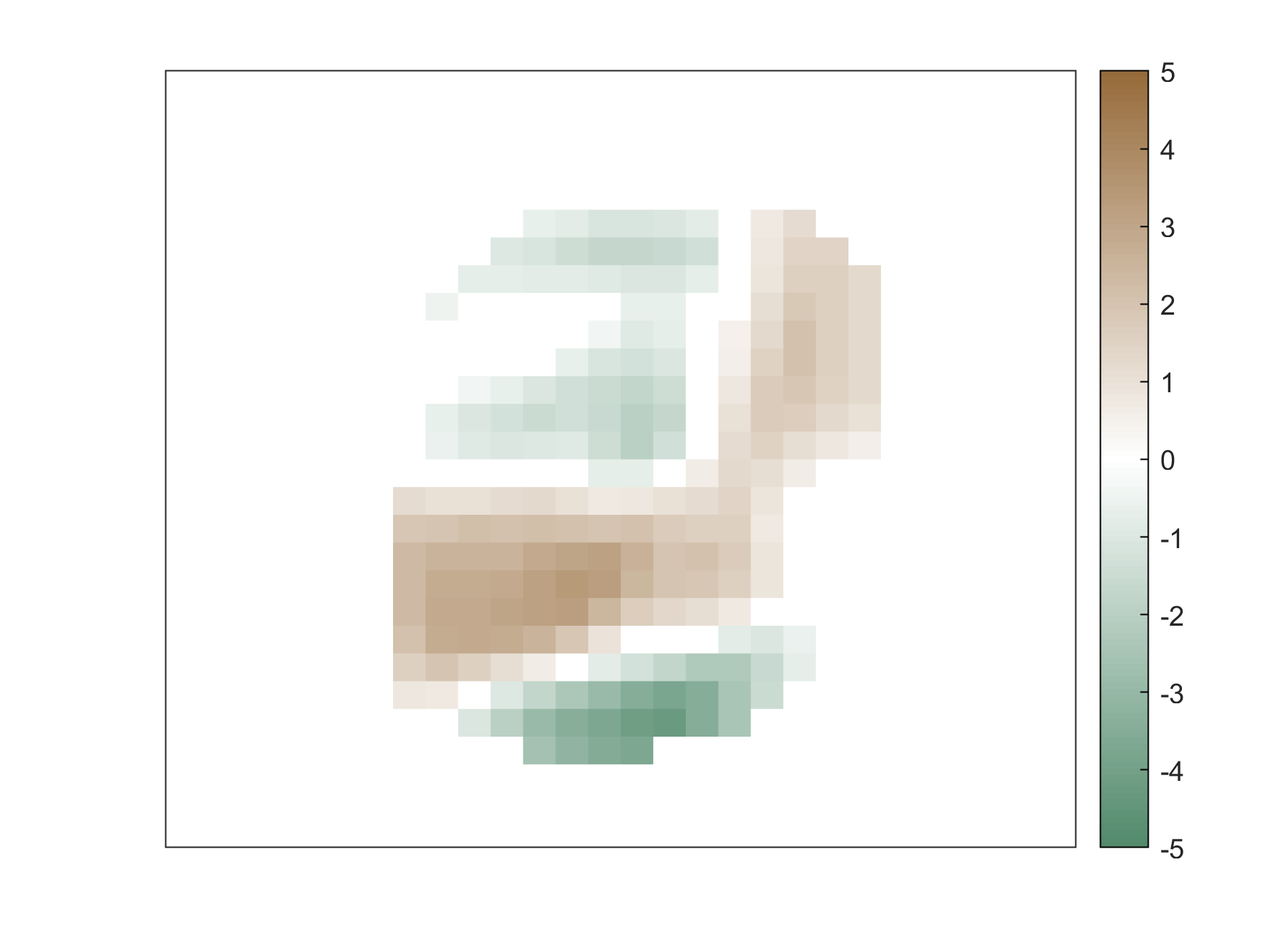} & \includegraphics[width = \linewidth]{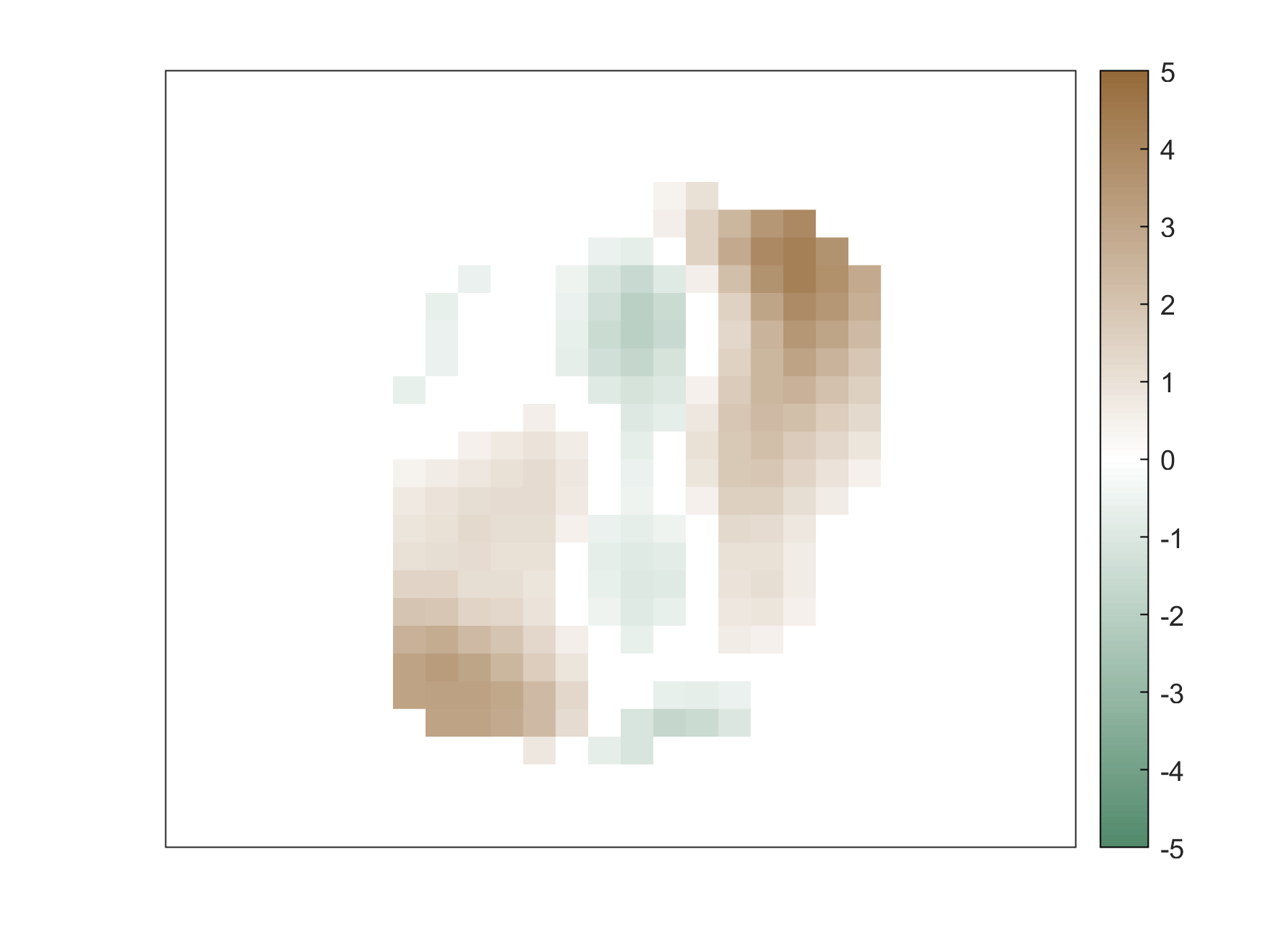} & \includegraphics[width = \linewidth]{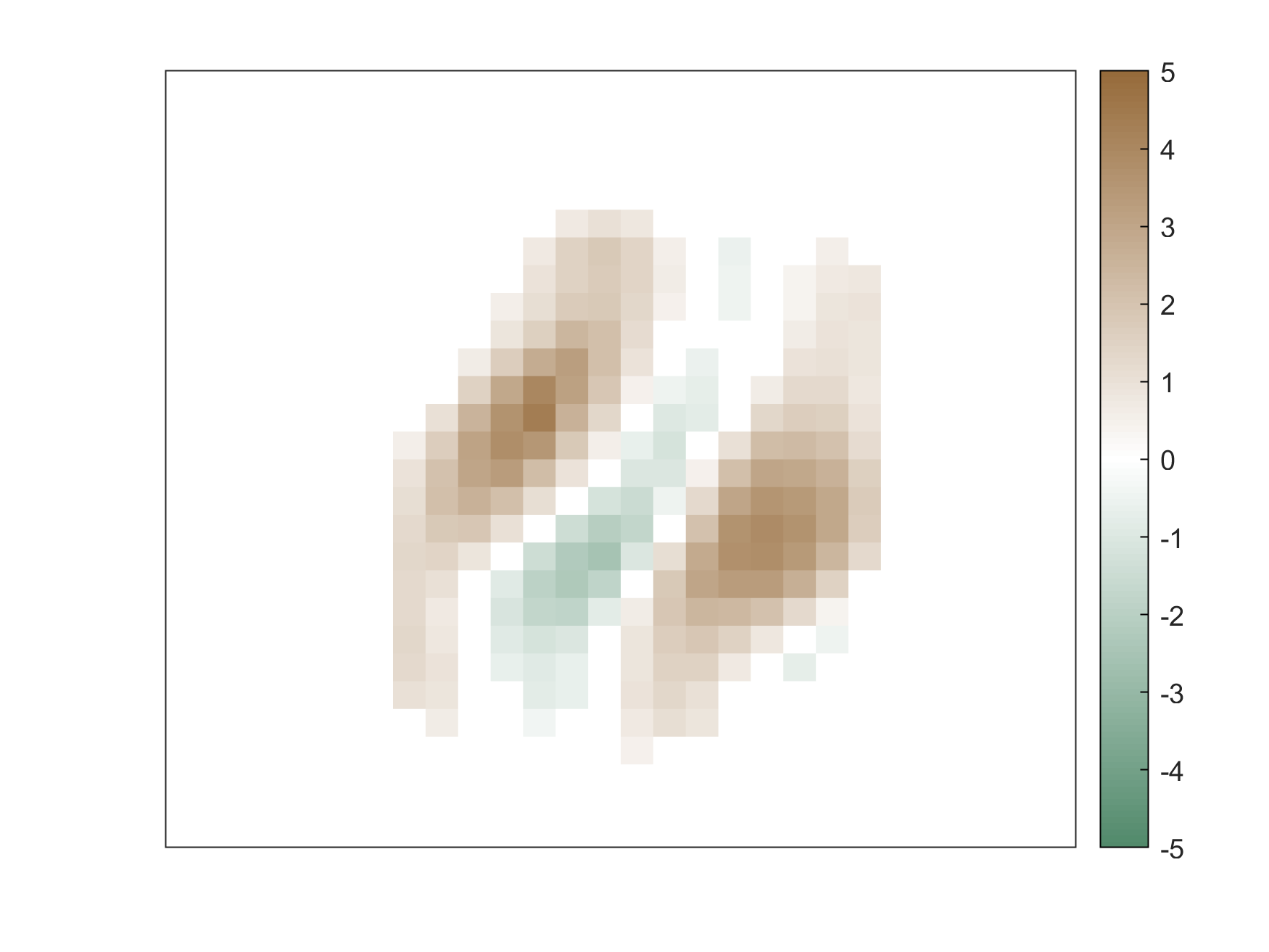} & \includegraphics[width = \linewidth]{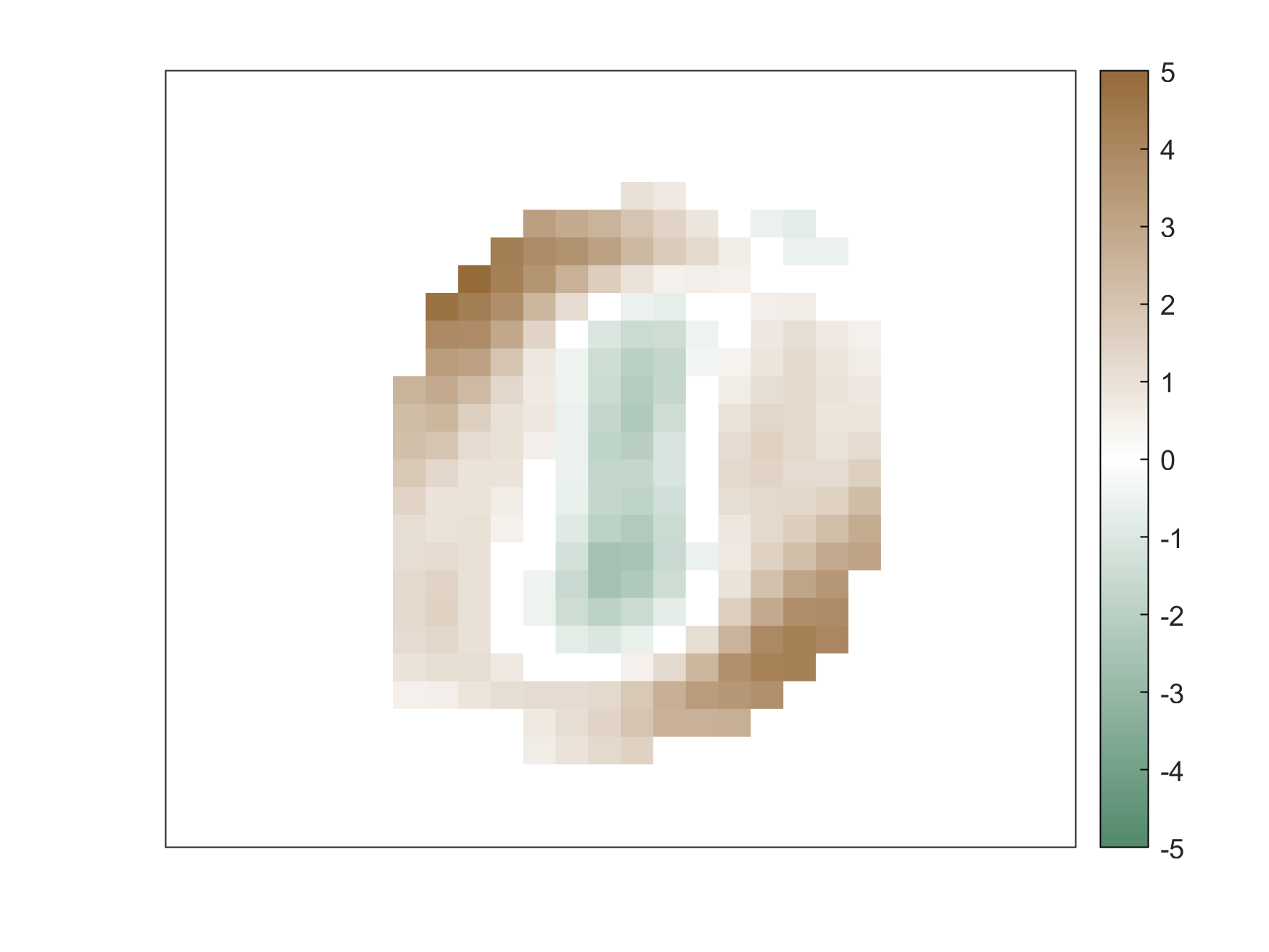} \\
\midrule
Positive part & \includegraphics[width = \linewidth]{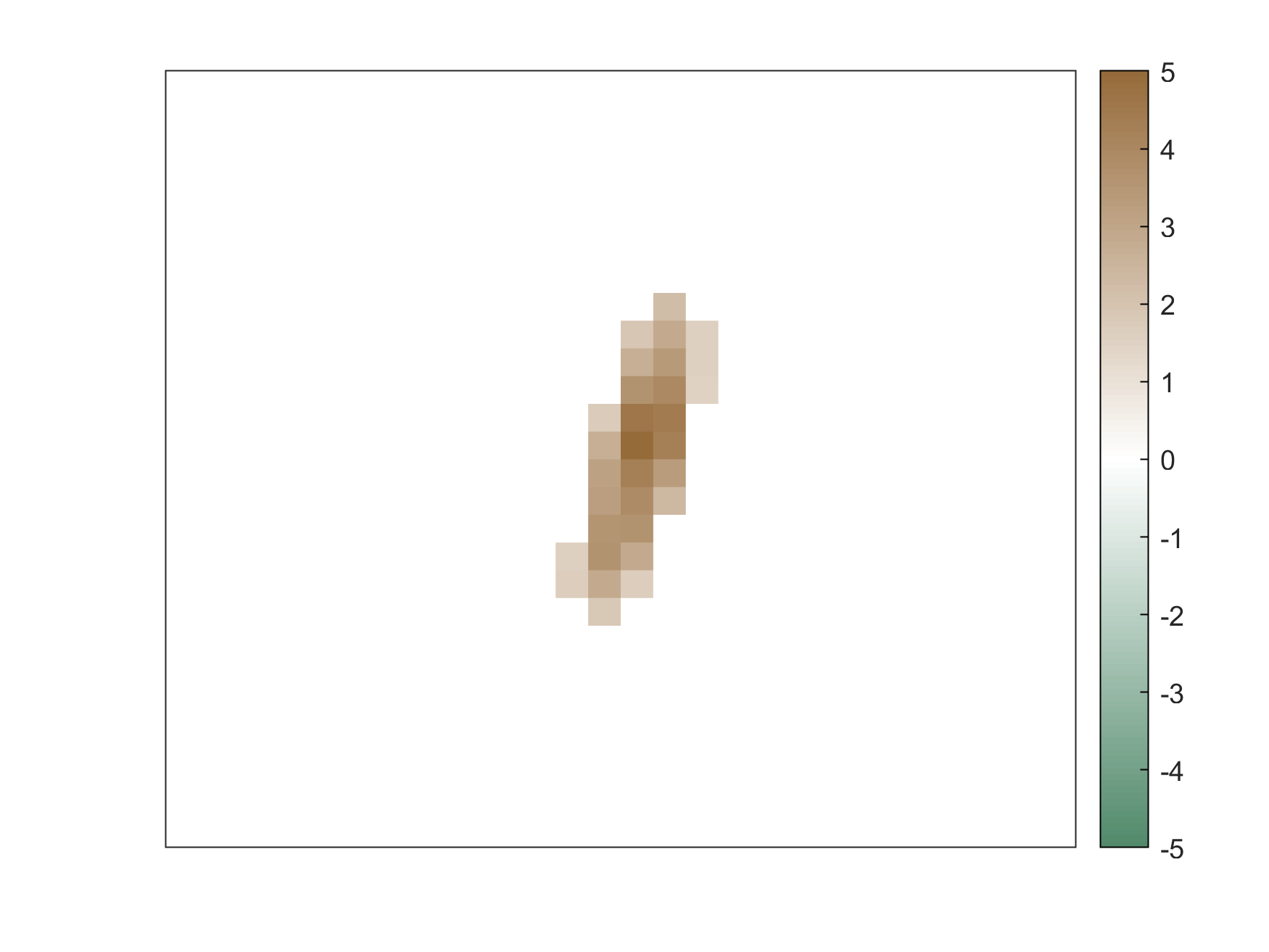} &
\includegraphics[width = \linewidth]{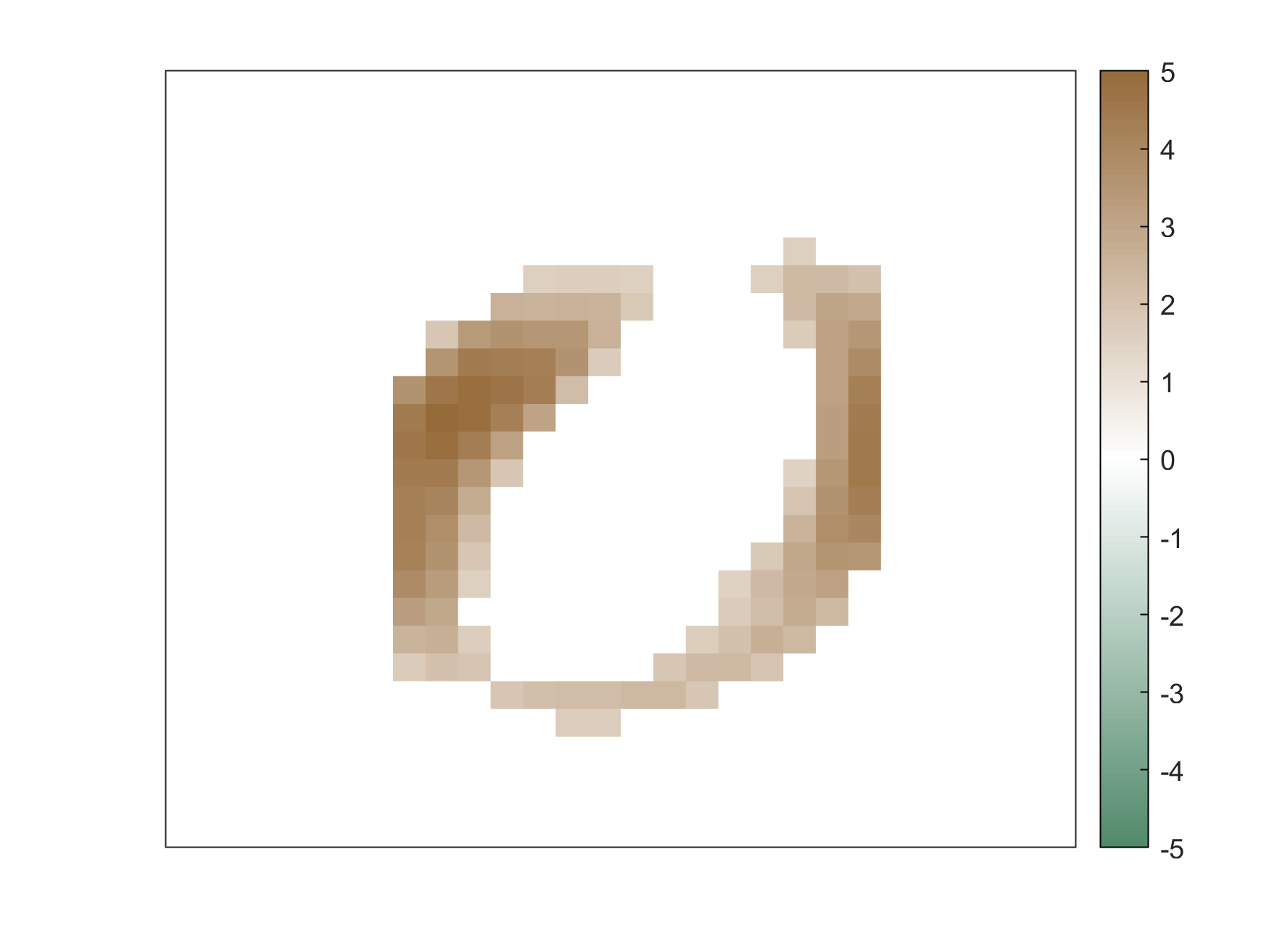} & \includegraphics[width = \linewidth]{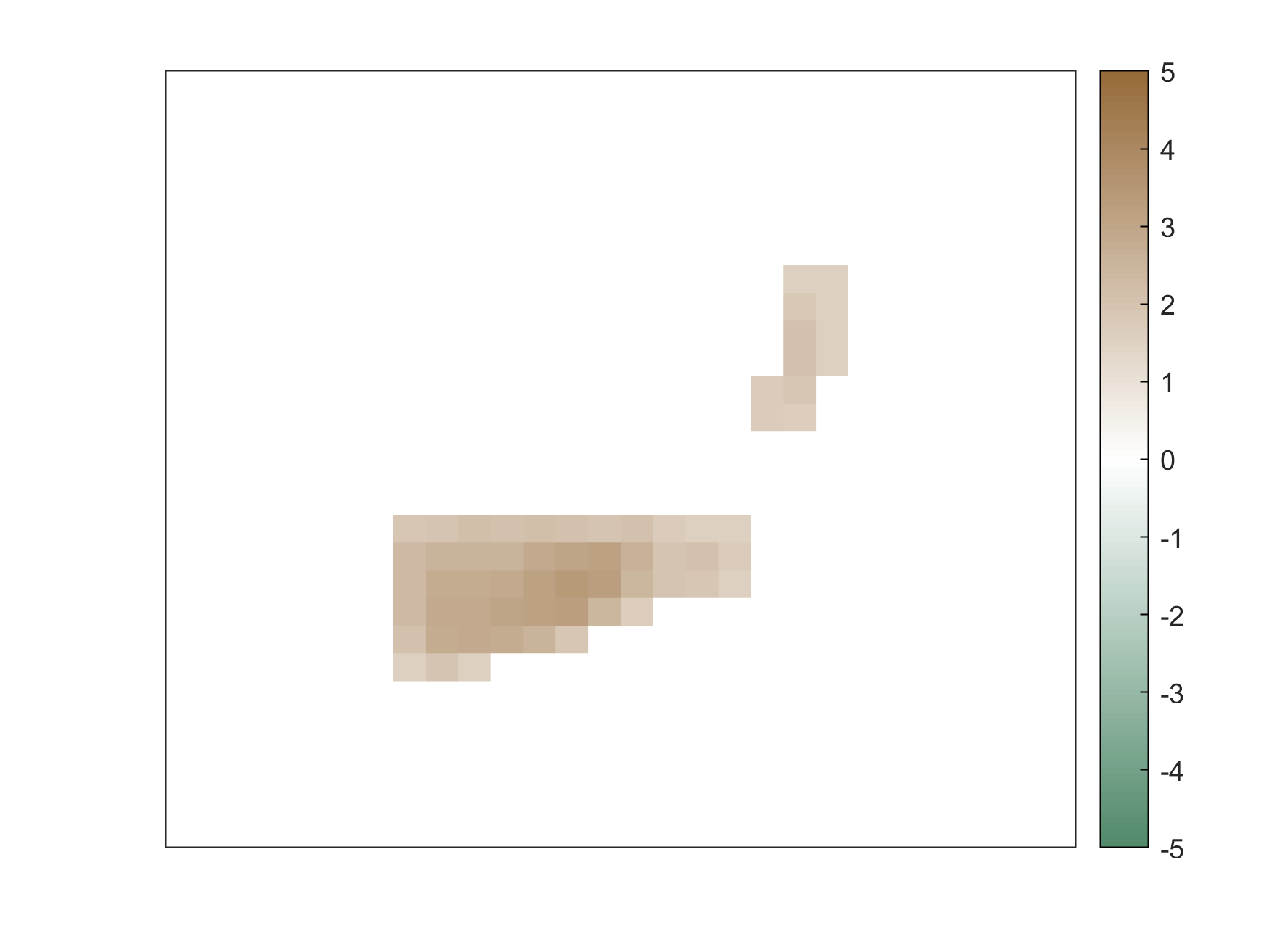} & \includegraphics[width = \linewidth]{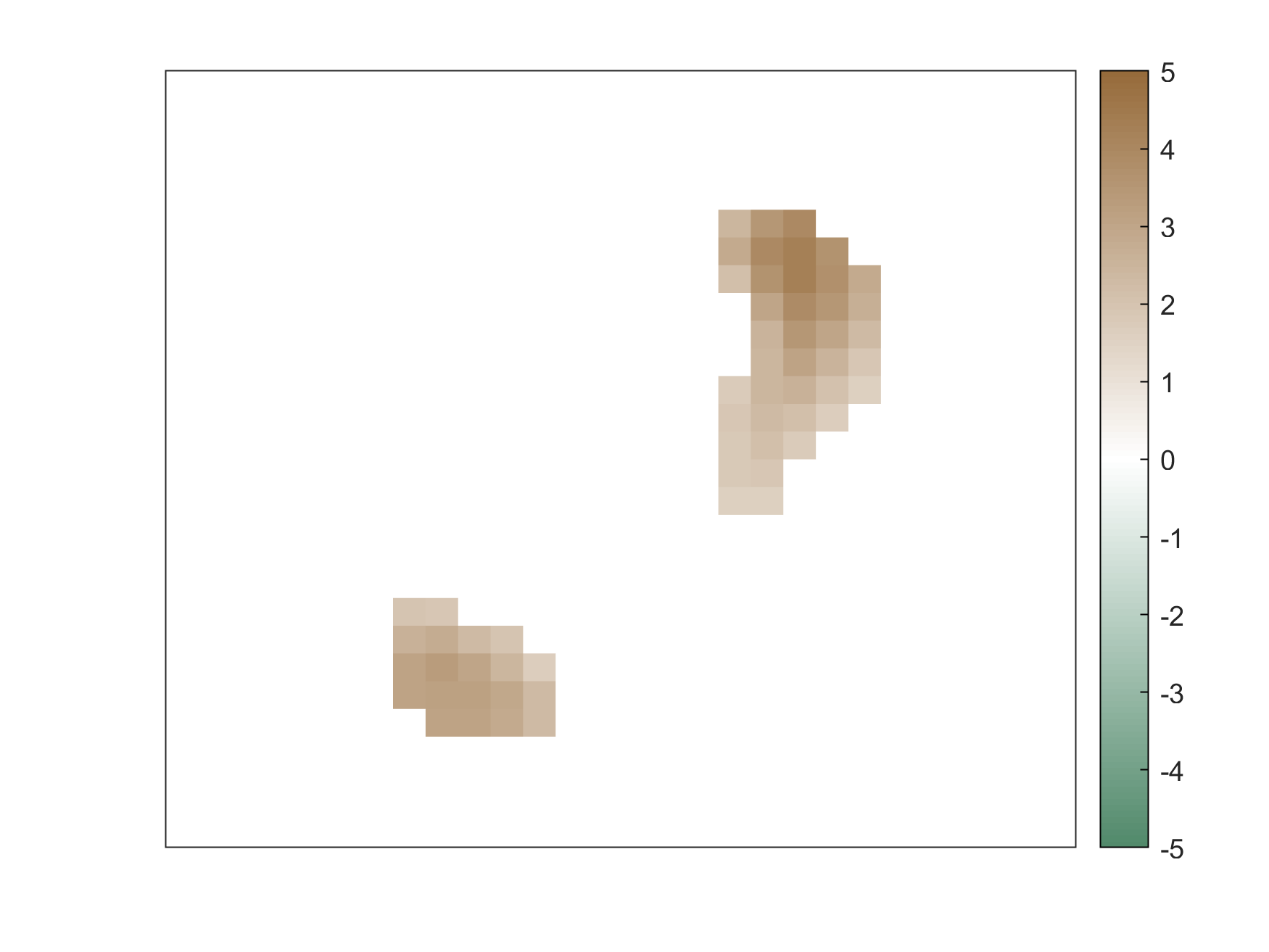} & \includegraphics[width = \linewidth]{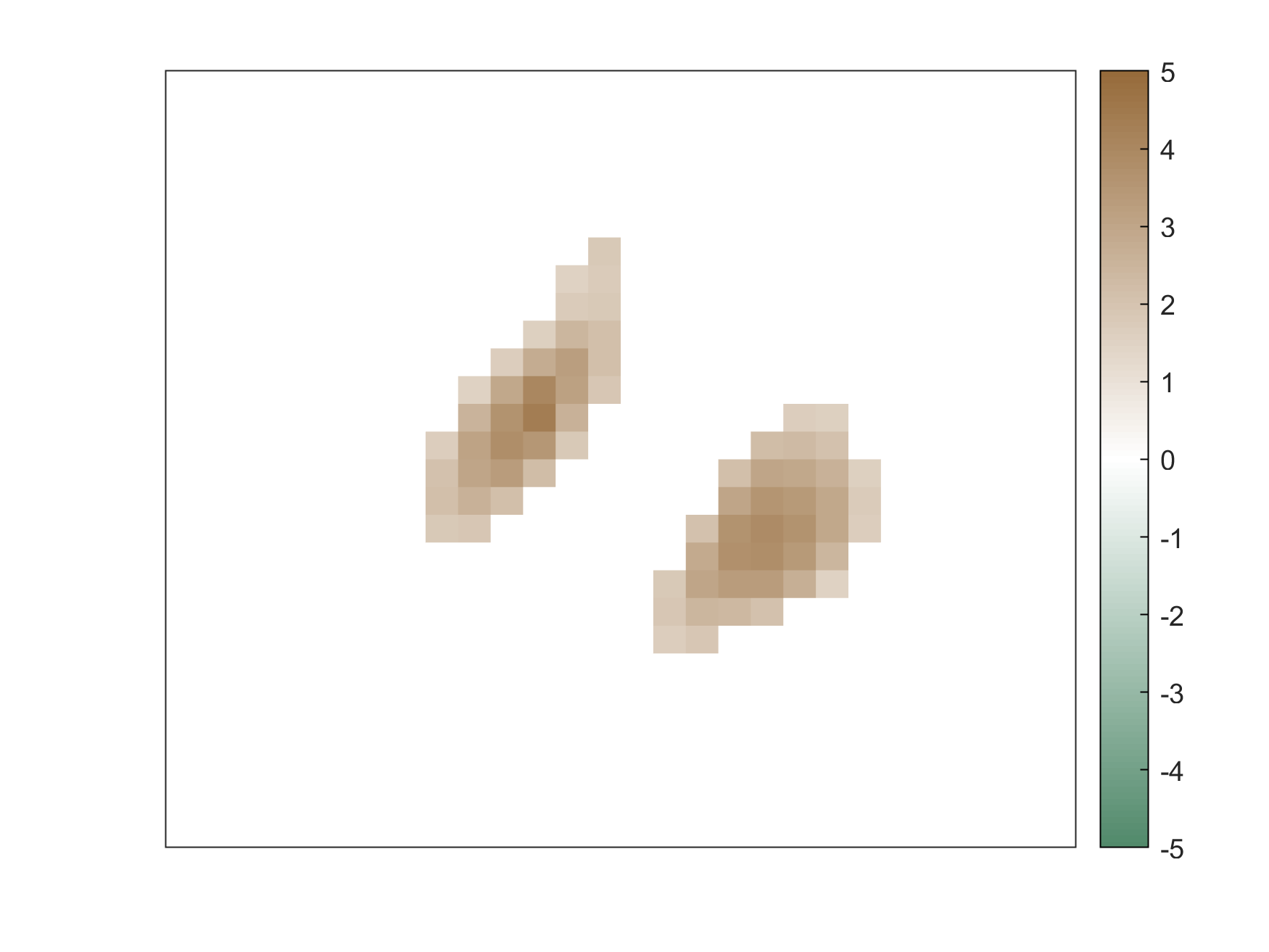} & \includegraphics[width = \linewidth]{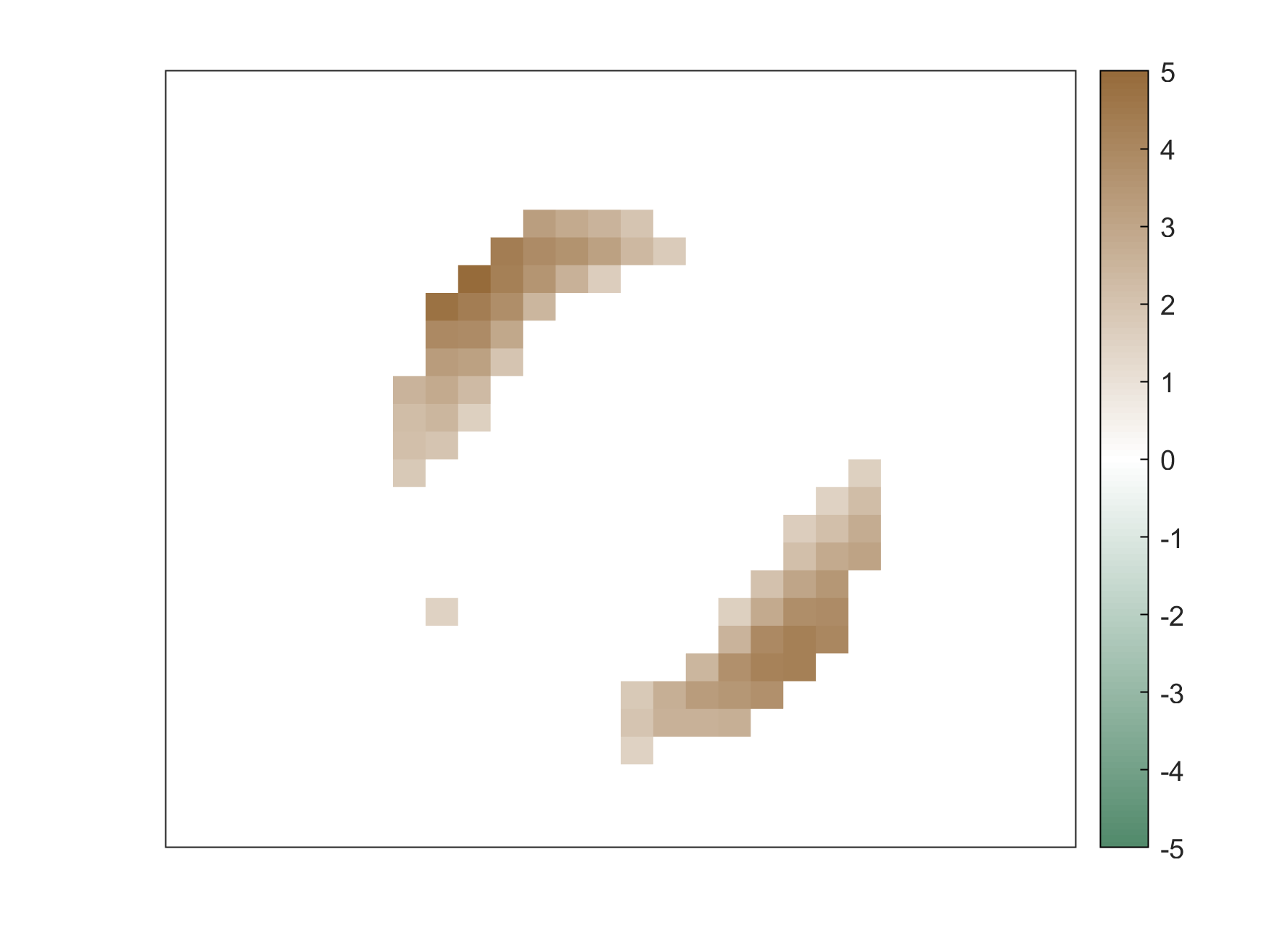} \\
Negative part & \includegraphics[width = \linewidth]{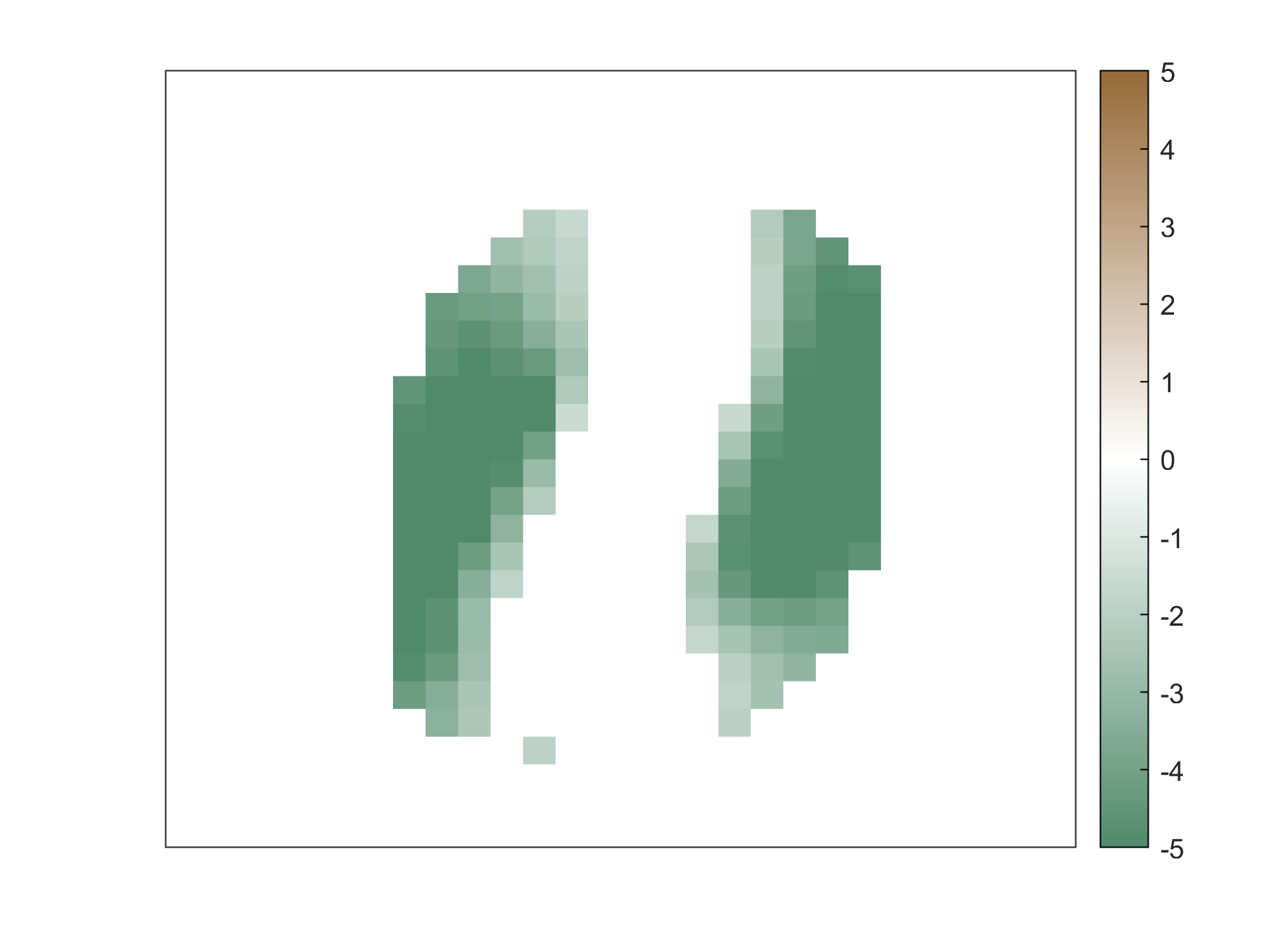} &
\includegraphics[width = \linewidth]{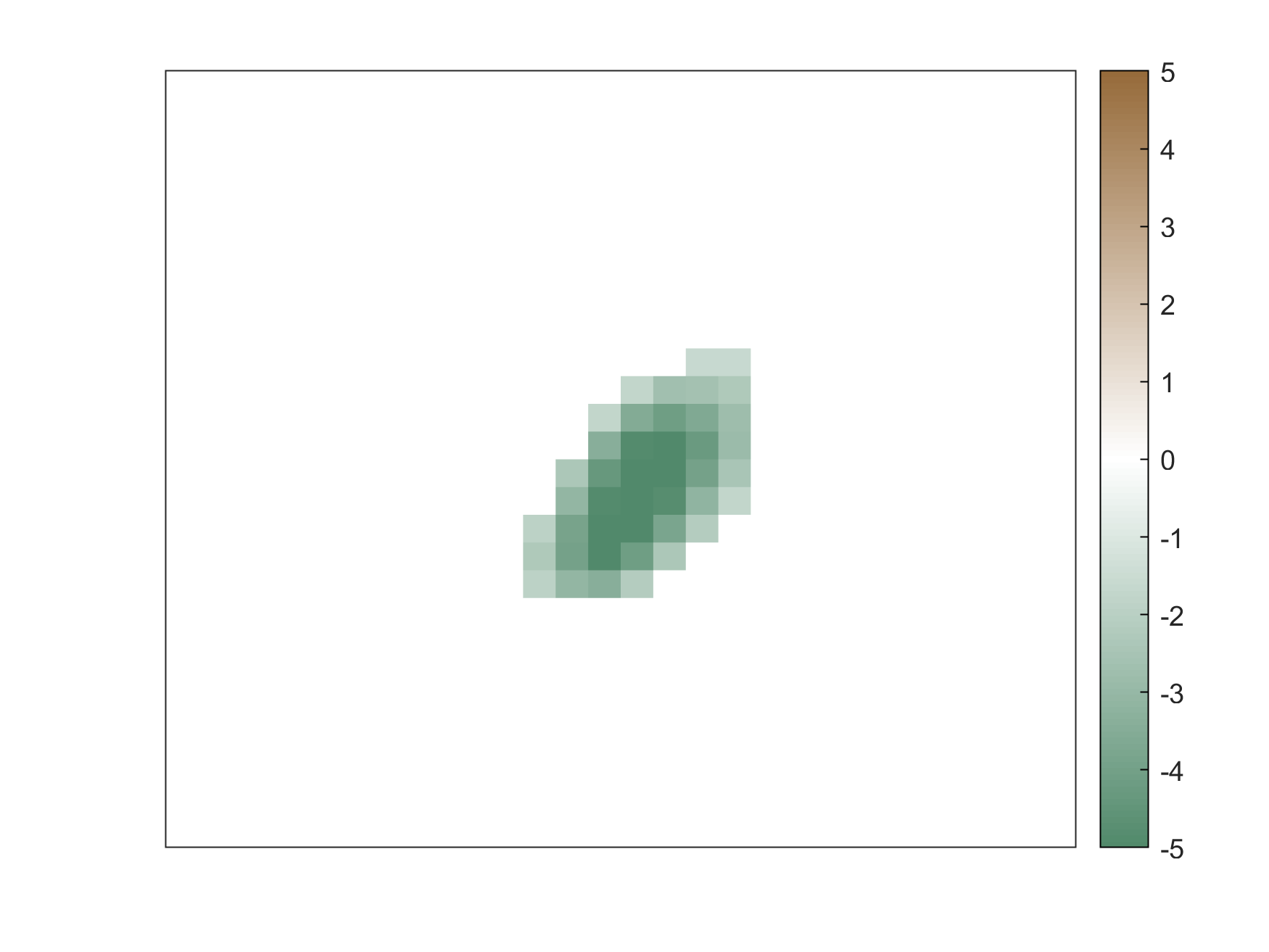} & \includegraphics[width = \linewidth]{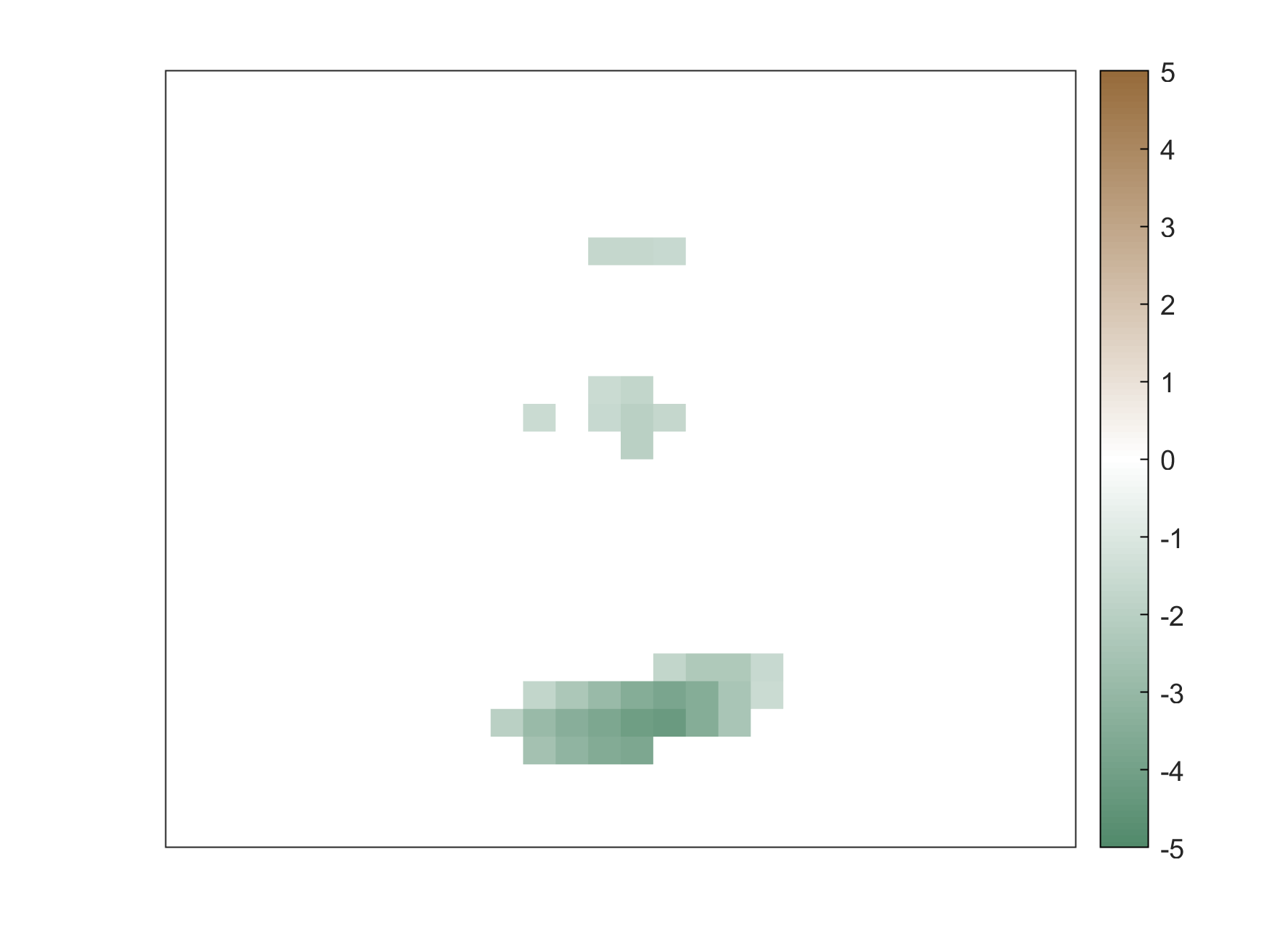} & \includegraphics[width = \linewidth]{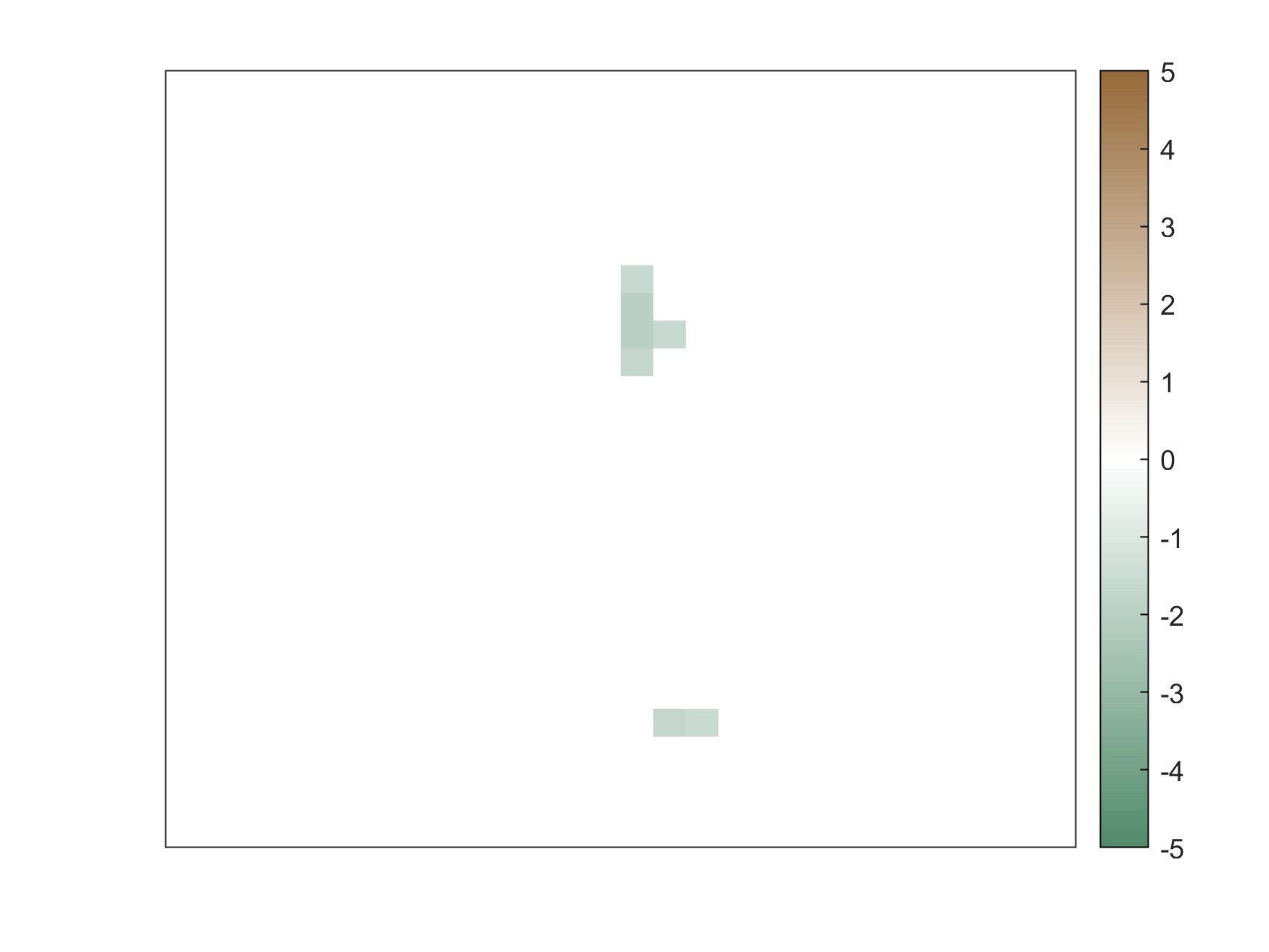} & \includegraphics[width = \linewidth]{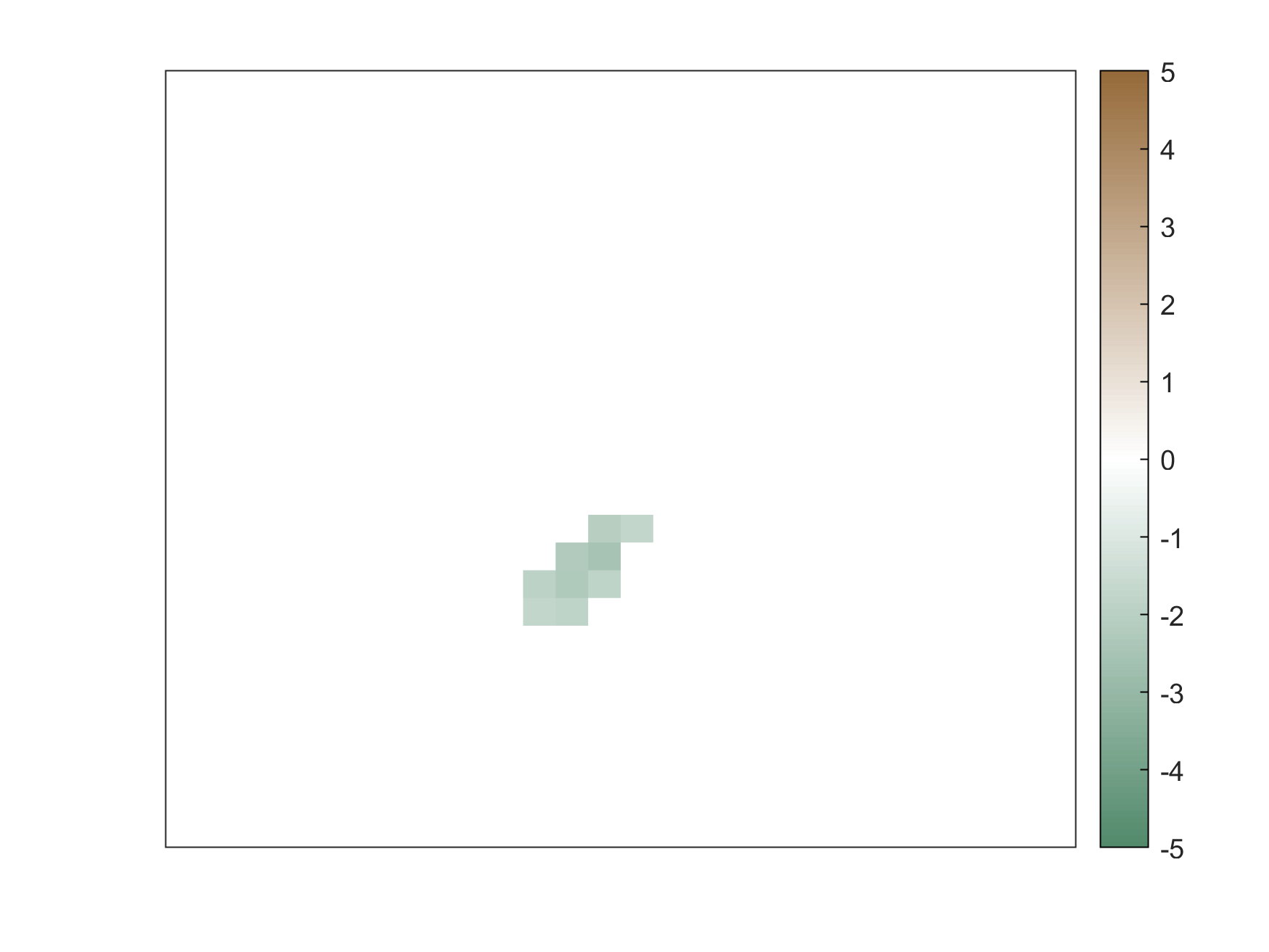} & \includegraphics[width = \linewidth]{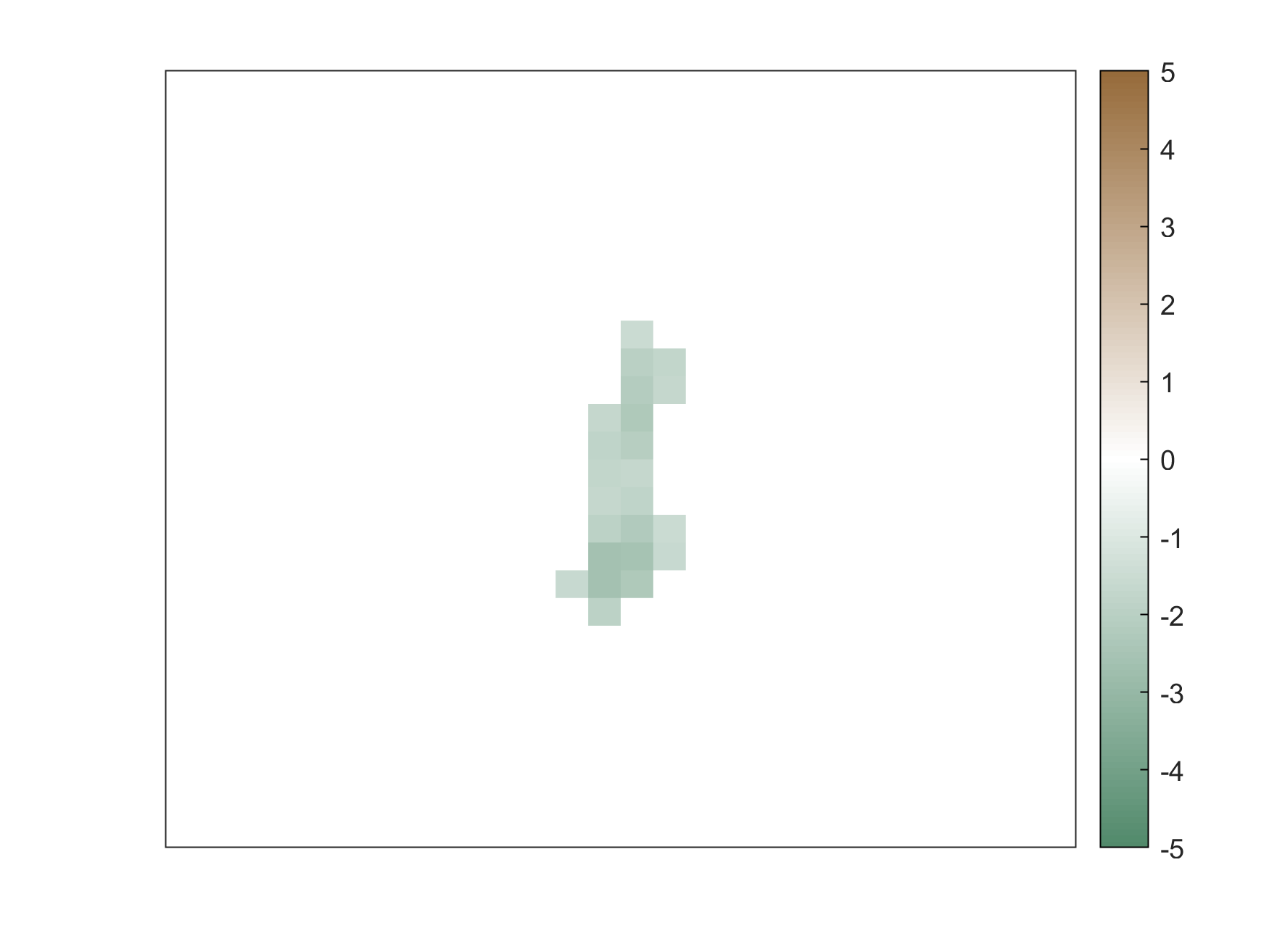} \\
\midrule
& $\bbeta_{0,0}^{(1)}$ & $\bbeta_{0,1}^{(1)}$ & $\bbeta_{0,2}^{(1)}$ & $\bbeta_{0,3}^{(1)}$ & $\bbeta_{0,4}^{(1)}$ & $\bbeta_{0,5}^{(1)}$ \\
\end{tabular}
}
\caption{From left to right: Estimated basis images (reshaped from the estimated coefficients $\bbeta_{0,k}^{(1)}$) from the MNIST data, which define sparse subregions in the $28 \times 28$ image. The second and third row shows the most significant positive and negative parts of the basis images by thresholding at the value $\pm 1.5$.
}
\label{tab:basis images MNIST}
\vspace{-2mm}
\end{table}

We fit the two-latent-layer DDE with the latent dimensions set to $K^{(1)} = 5 $ and $K^{(2)} = 2$ (see the Supplementary Material S.5.2 for the rationale for this choice).
\cref{tab:basis images MNIST} displays the first-layer coefficients $\B^{(1)}$ by rearranging the each column into the original $28\times 28$ grid. Note that the value of $\B^{(1)}$ is in the logit-scale, so the negative coefficients make the corresponding pixel more likely to be zero. From the reshaped columns $\bbeta_{0,0}^{(1)}$, $\bbeta_{0,1}^{(1)}$, \ldots, $\bbeta_{0,5}^{(1)}$ of $\B^{(1)}$, we can interpret the meaning of each latent variable: $A_1^{(1)} = 1$ indicates a zero-like shape, $A_5^{(1)} = 1$ indicates symmetric curves on the upper-left and bottom-right corners, and the other latent variables represent different rotations.
Additionally, using the deeper graphical matrix $\mathbf{G}^{(2)}$ displayed in the Supplementary Material S.5.4, we can also interpret the top layer latent variables as broader information about the images. For example, $A_1^{(2)}$ indicates large pixel density and $A_2^{(2)}$ indicates symmetry with respect to the x-axis.

\begin{figure}[h!]
\centering
\resizebox{0.4\textwidth}{!}{
\vspace{-2mm}
    \includegraphics[width = 0.5\textwidth]{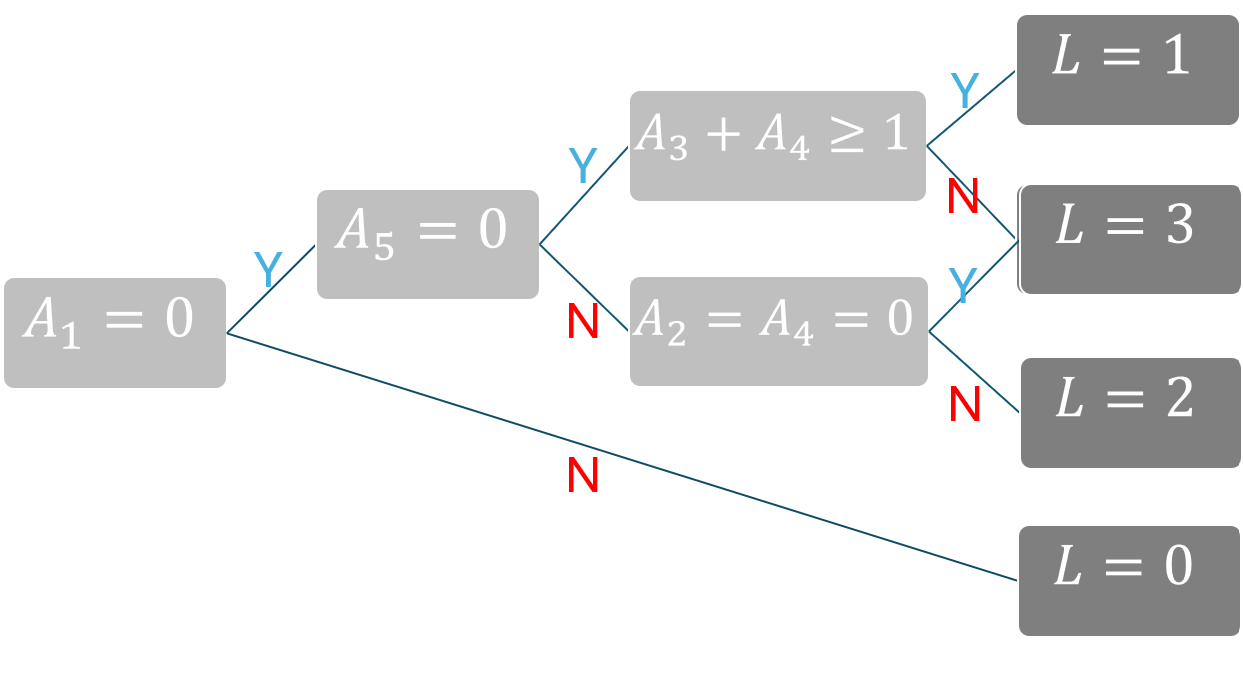}
}\quad
\resizebox{0.55\textwidth}{!}{
\includegraphics[height = 5cm]{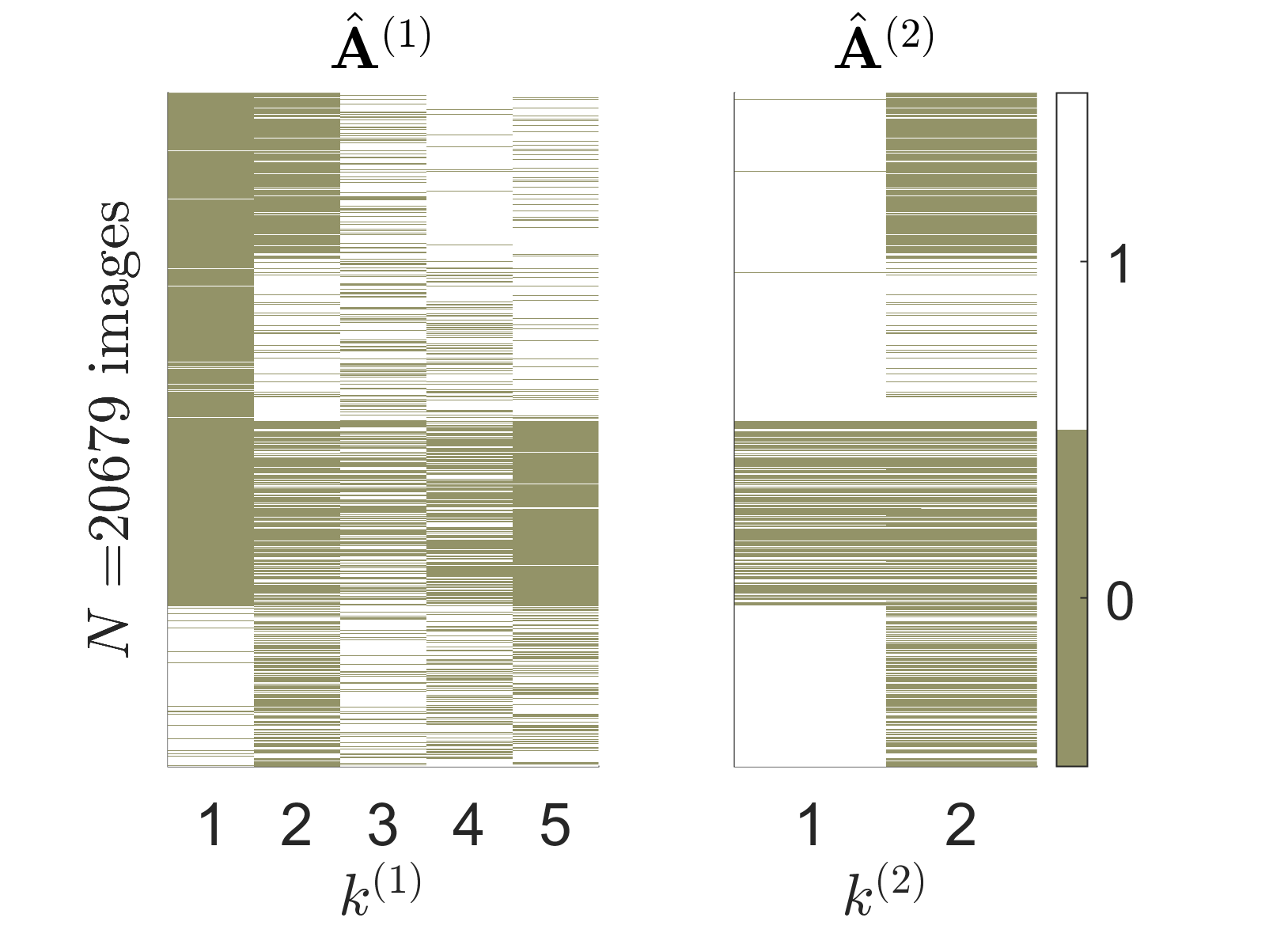}
\includegraphics[height = 5cm]{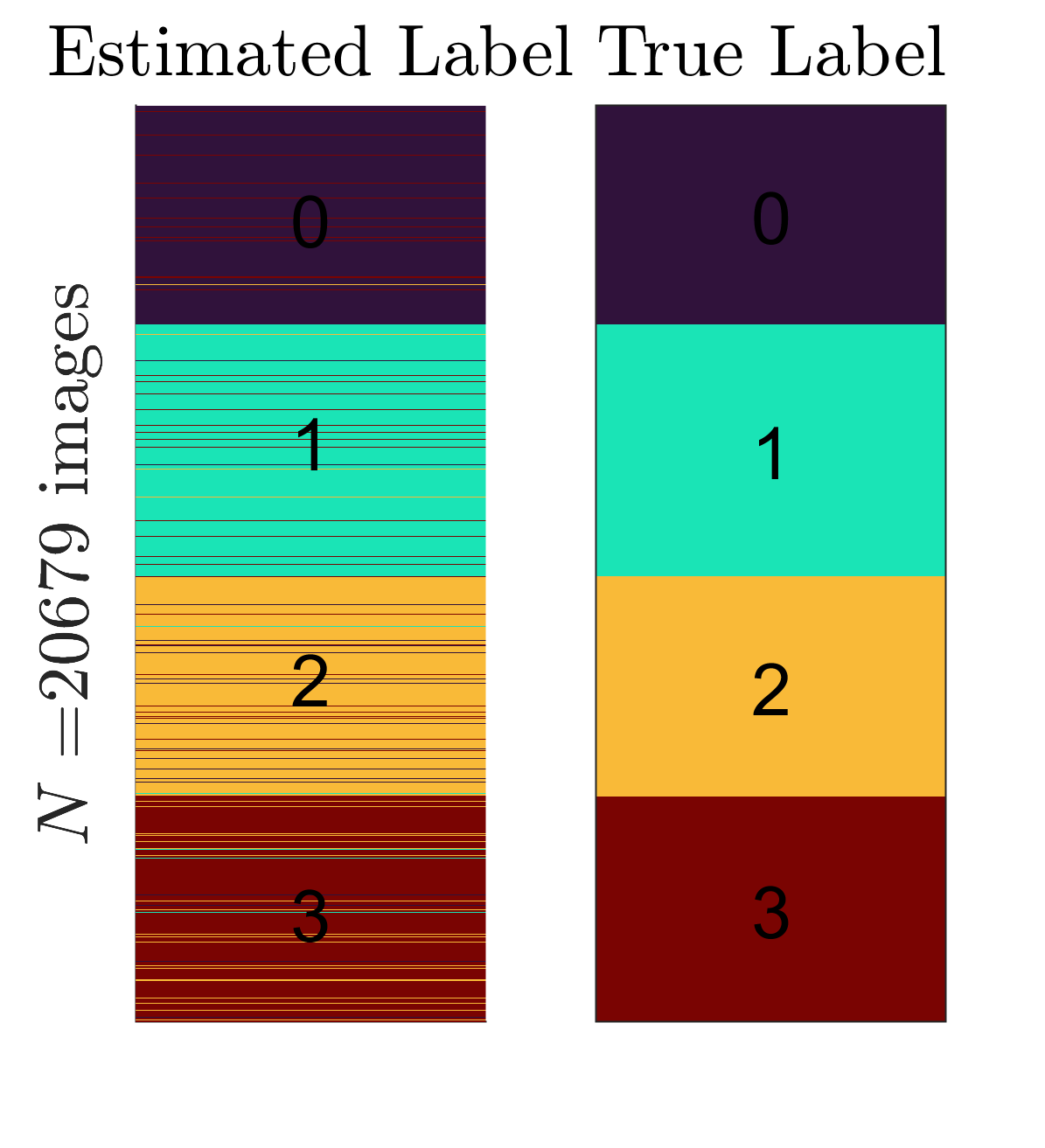}
}
\caption{\textbf{Left}: Decision tree to estimate the digit $L = 0, 1, 2, 3$. \textbf{Center}: Estimated latent representations $\hat{\mathbf A}^{(1)}, \hat{\mathbf A}^{(2)}$. \textbf{Right}: Estimated and true digits in the train set. 
}
\label{fig:mnist label estimation}
\vspace{-4mm}
\end{figure}

The learned shallower and deeper latent representations $(\hat\ma^{(1)}, \hat\ma^{(2)})$ are evaluated under two measures of performance: \emph{classification accuracy} and \emph{reconstruction accuracy}.
We first estimate the latent variables using the $\varphi$ matrix in the EM algorithm:
\begin{equation}\label{eq:latent variable estimation}
    (\hat{\ma}^{(1)}_i, \hat{\ma}^{(2)}_i) = \argmax_{\aaa^{(1)}\in\{0,1\}^{K_1},~ \aaa^{(2)}\in\{0,1\}^{K_2}} \varphi_{i, \aaa^{(1)}, \aaa^{(2)}};
    \quad \aaa^{(d)}\in\{0,1\}^{K^{(d)}},~d=1,2.
\end{equation}
Then, a decision tree that classifies the binary latent representations to the categorical label {is built by using the misclassification error as the splitting criteria}; see the left panel of \cref{fig:mnist label estimation}.
The center and right panels display the estimated latent traits and estimated labels. Our classifier leads to a high train/test accuracy of $92.0\%/92.6\%$, even though our model is not fine-tuned for image classification.
Although the state-of-art machine learning methods can achieve an accuracy as high as $97\%$ \citep{monnier2020deep}, we point out that the main goal here is not classification, but on interpretability and parsimony; indeed, the DDE uses more limited information and is a less complex but more interpretable model that provides the generative process for the image. 
In \cref{tab:reconstruction MNIST}, we compare the classification accuracy and pixel-wise reconstruction accuracy of the two-latent-layer DDE with alternative interpretable models (latent class model (LCM) and the one-latent-layer DDE), as well as \darkblue{popular unsupervised machine learning models (two-latent-layer DBN and VAE)}. 
The results show that the two-latent-layer DDE performs well for both measures. We also display example images generated from DDE alongside their latent representations in \cref{fig:generated digits}, which illustrate various handwriting styles for each digit. We provide implementation details, \darkblue{detailed comparison with iVAEs}, and additional visualization in Supplementary Material S.5.3.

\begin{figure}[h!]
\adjustbox{valign=c}{
  \begin{minipage}[b]{.44\linewidth}
    \centering
\resizebox{\textwidth}{!}{
    \begin{tabular}[h!]{cccc}
    \toprule
    $[1, 0, 1, 0, 1]$ & $[0, 0, 0, 0, 0]$  & $[0, 1, 0, 1, 1]$ & $[0, 0, 1, 1, 1]$ \\
    \includegraphics[height = 2cm, valign=m]{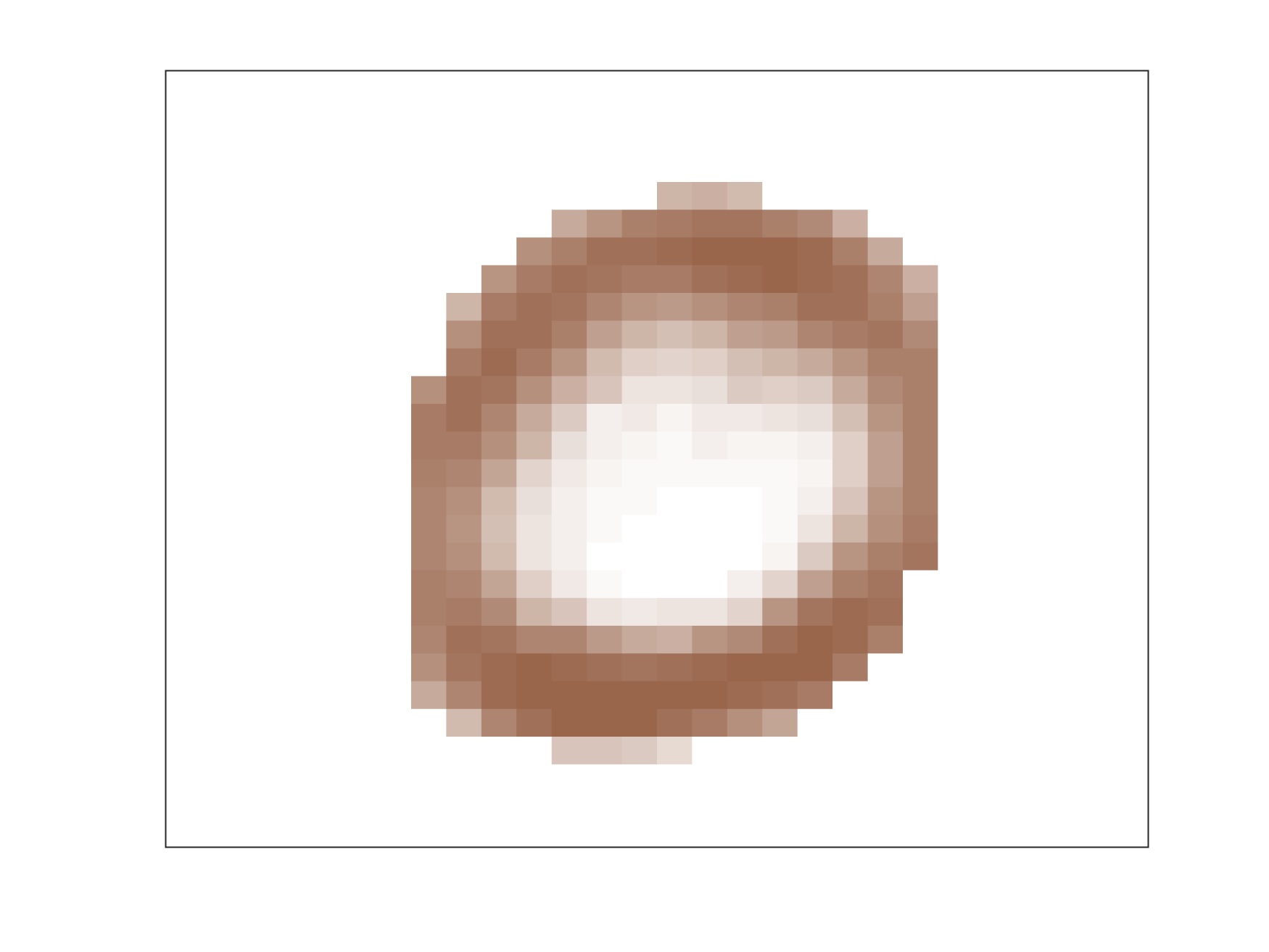} & \includegraphics[height = 2cm, valign=m]{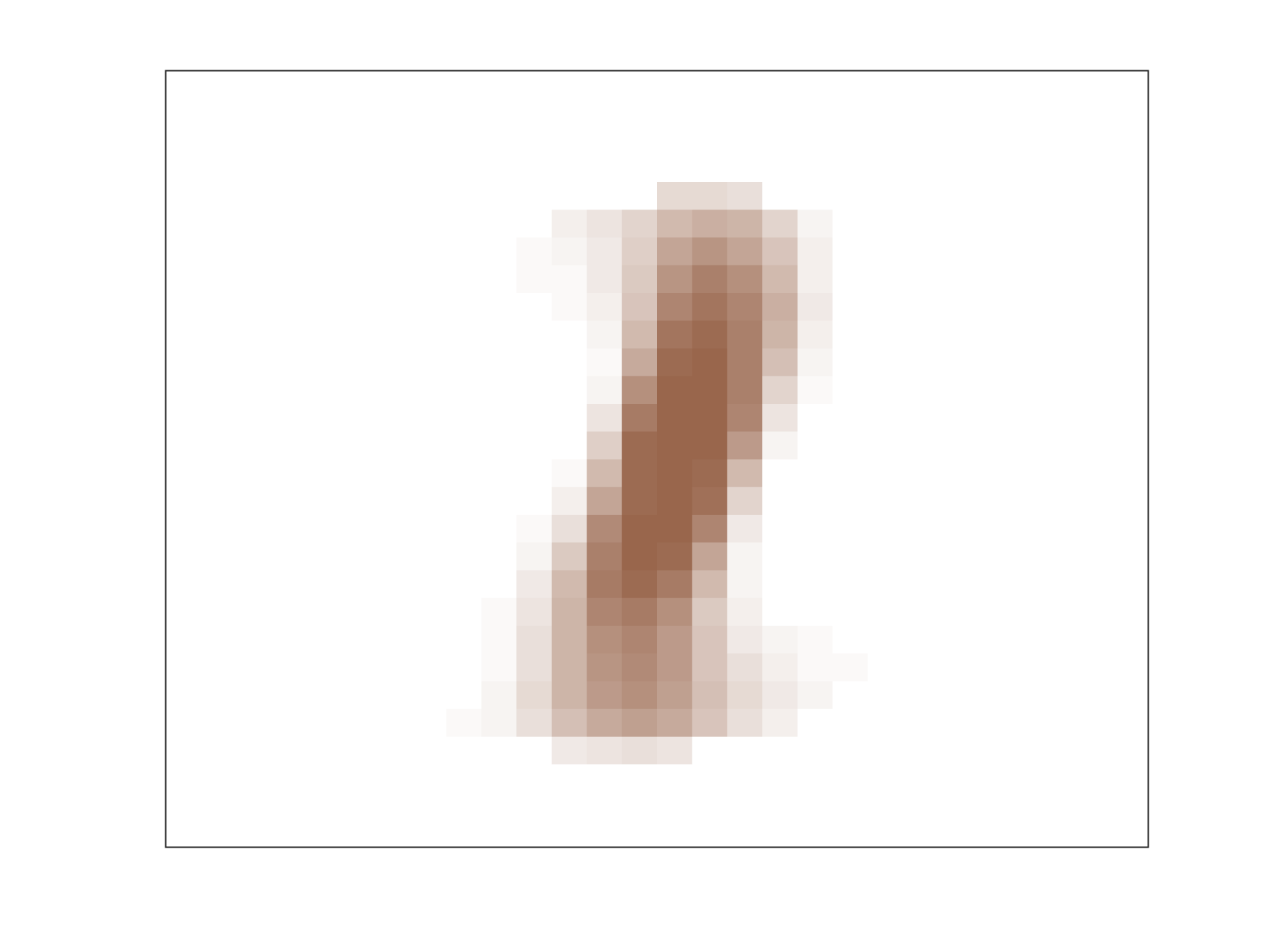} & \includegraphics[height = 2cm, valign=m]{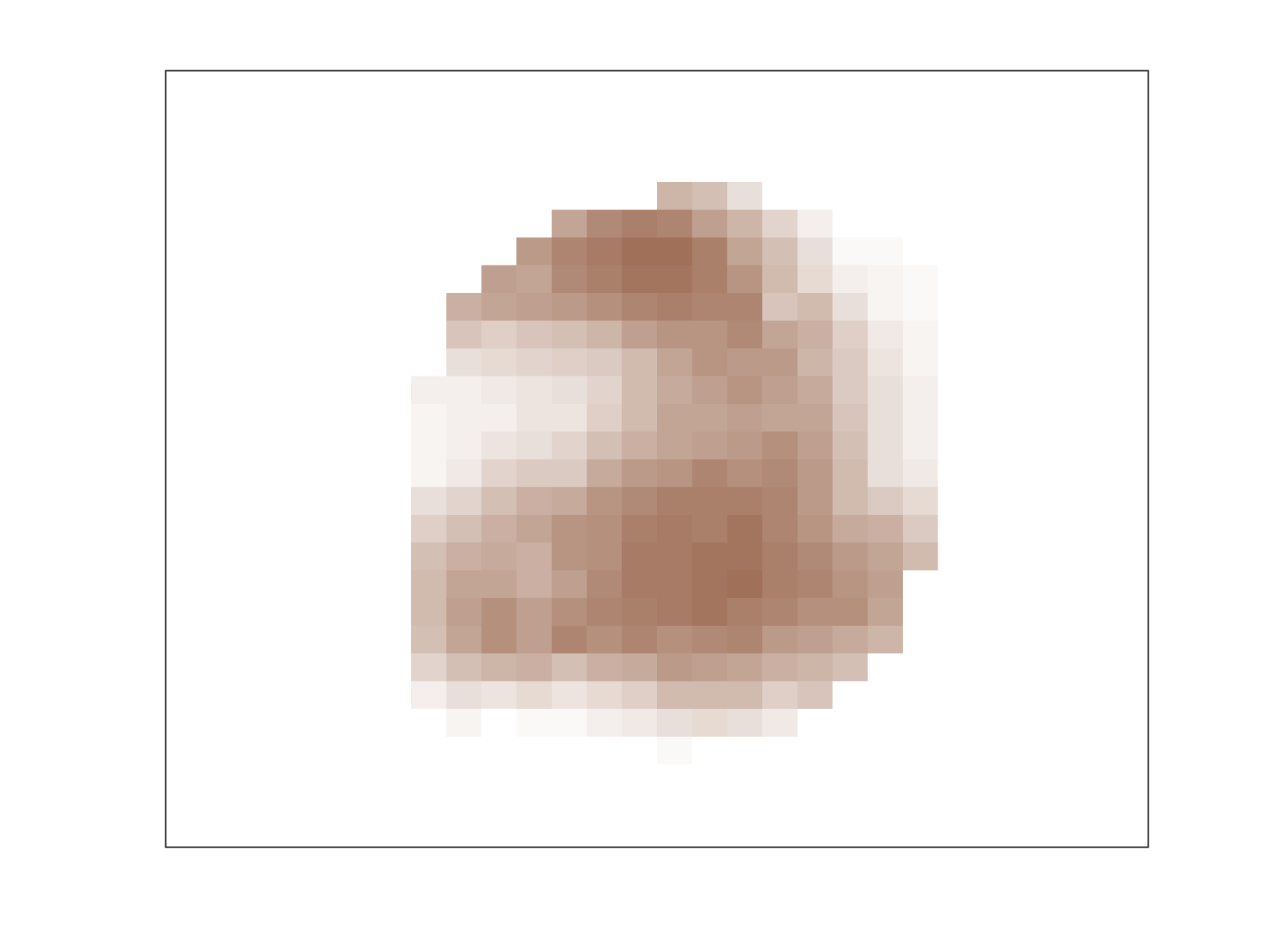} & \includegraphics[height = 2cm, valign=m]{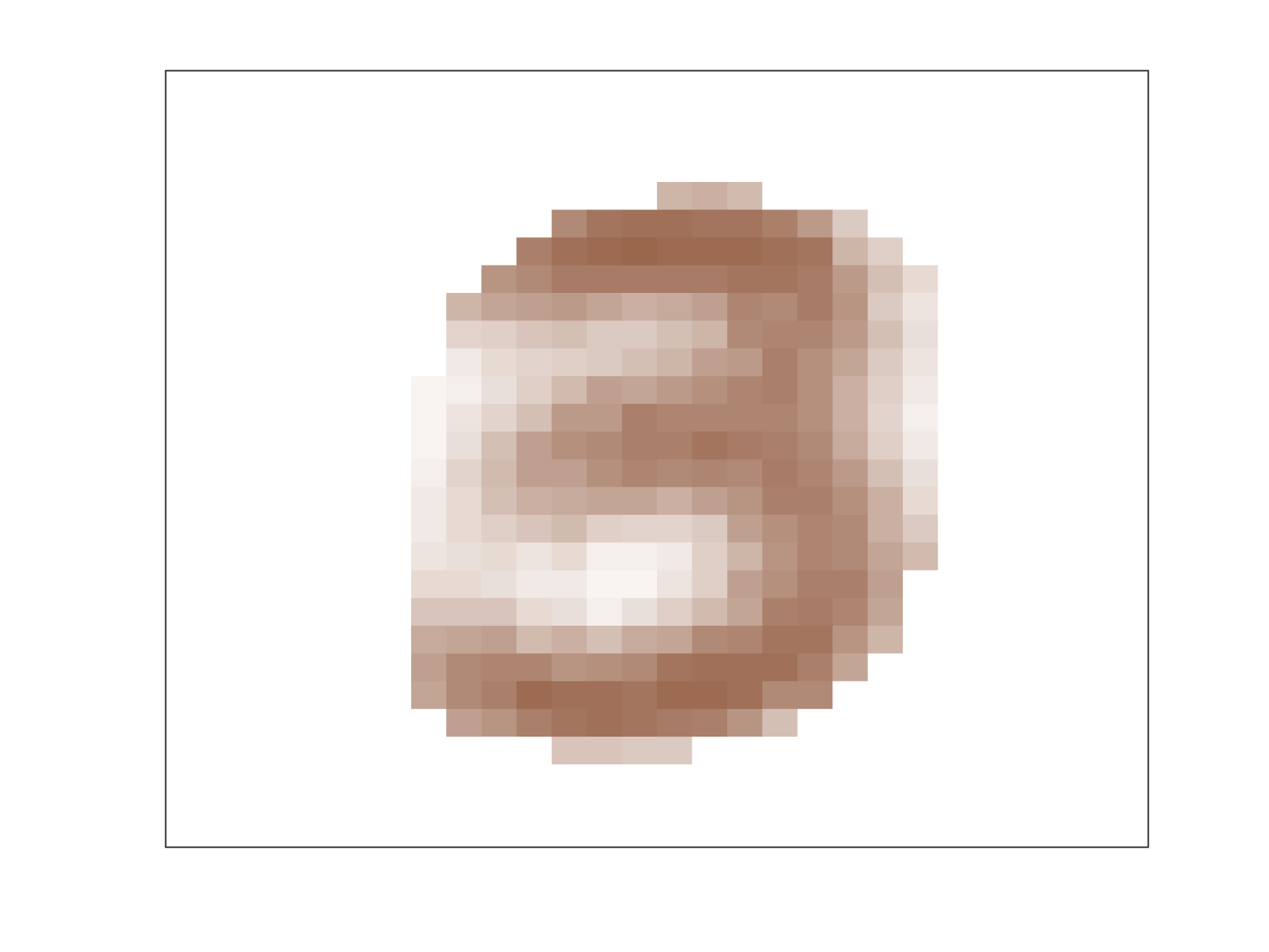} \\
    \midrule
    $[1, 0, 1, 1, 1]$ & $[0, 0, 0, 1, 0]$ & $[0, 1, 1, 0, 1]$ & $[0, 0, 1, 1, 0]$ \\
    \includegraphics[height = 2cm, valign=m]{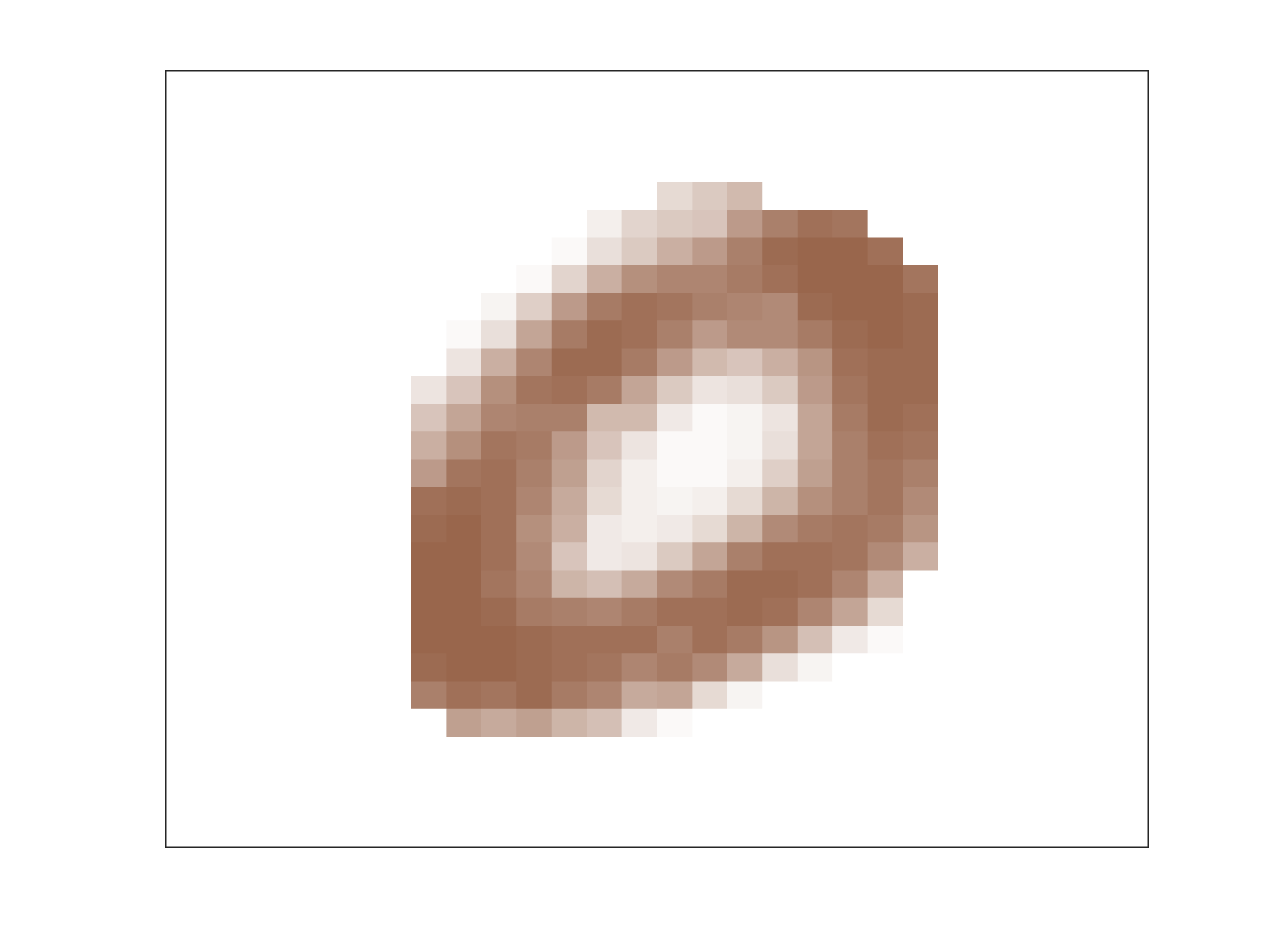}
     & \includegraphics[height = 2cm, valign=m]{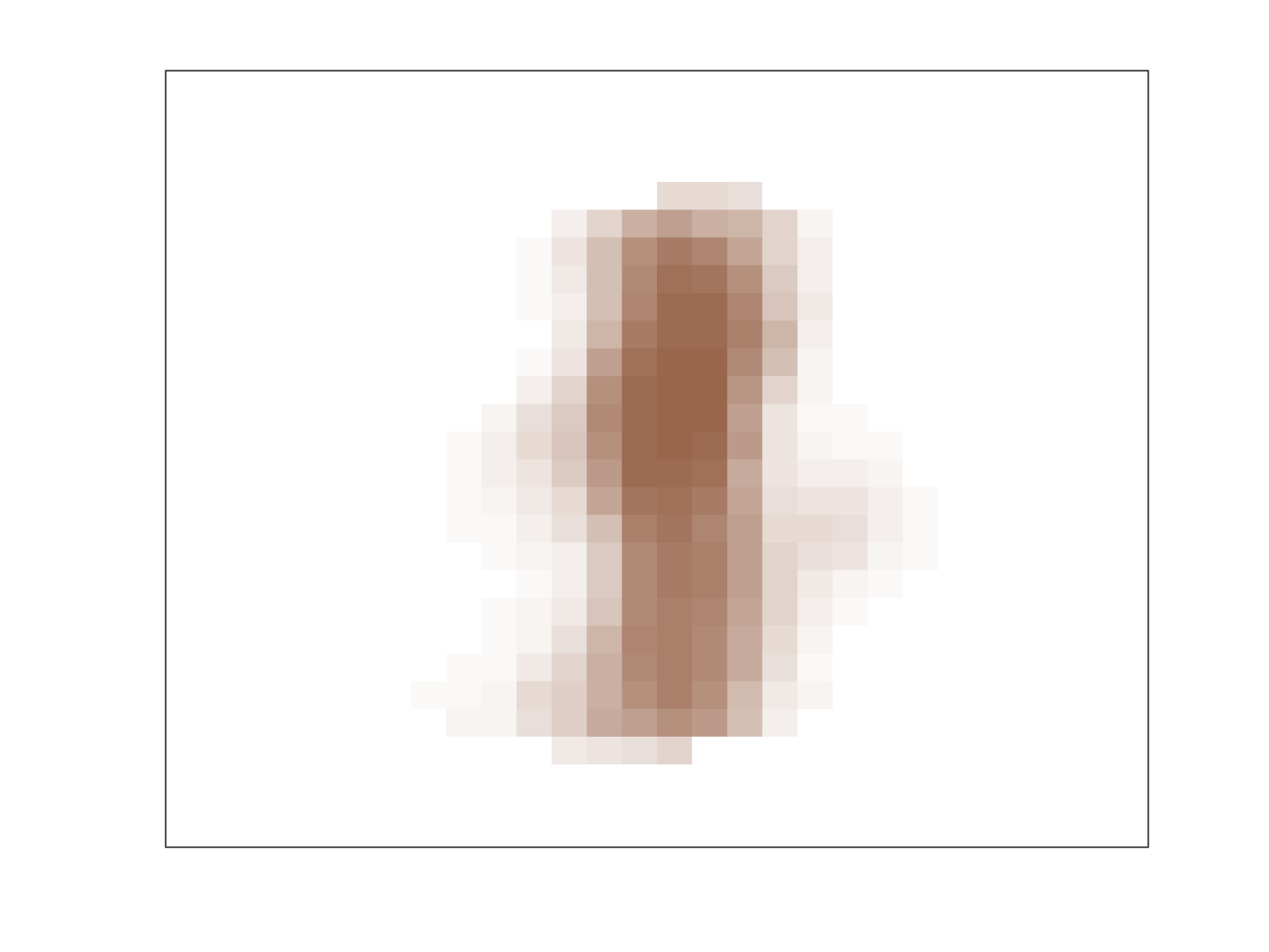} & \includegraphics[height = 2cm, valign=m]{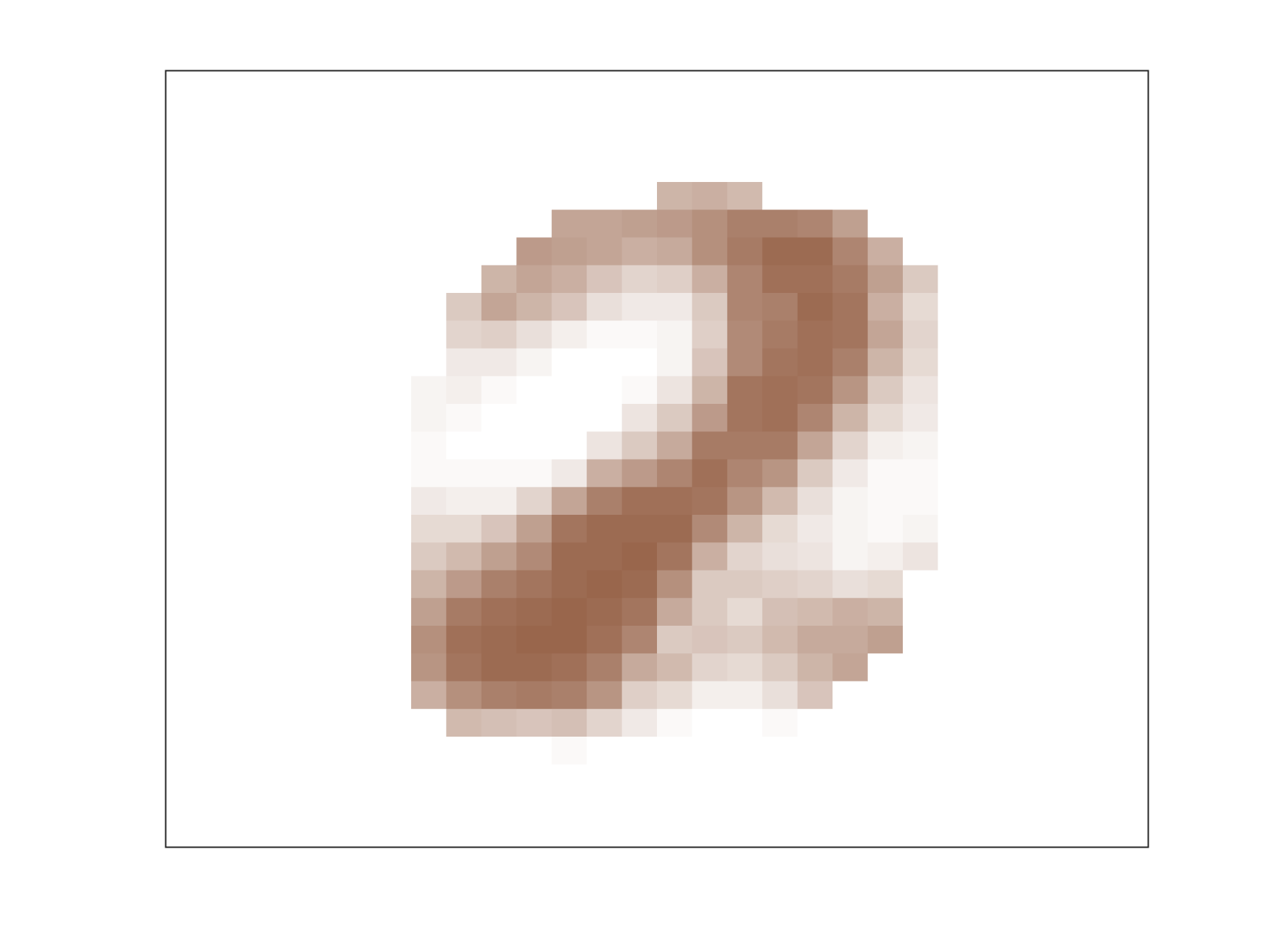} & \includegraphics[height = 2cm, valign=m]{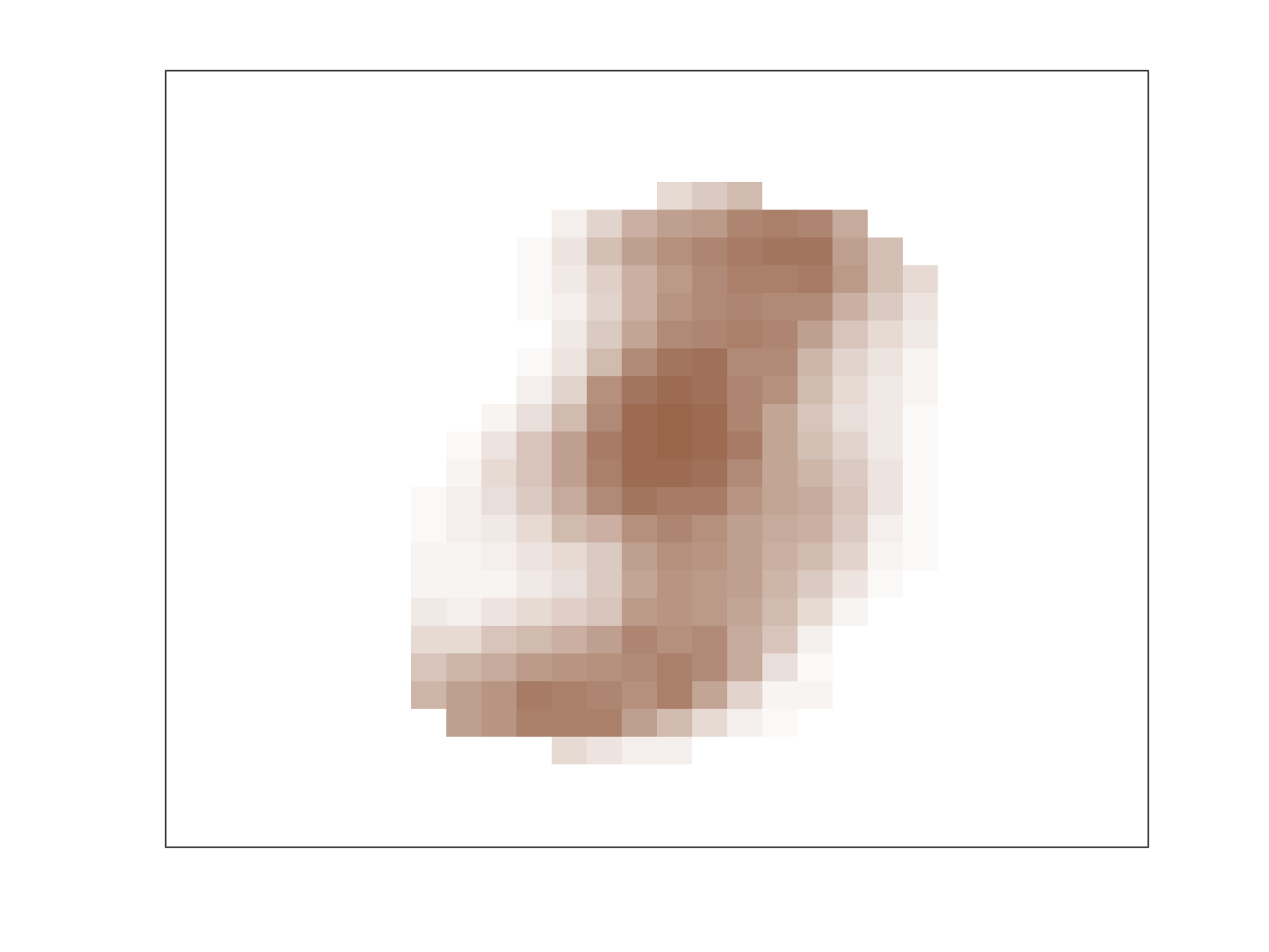} \\
    \bottomrule
    \end{tabular}
}
    \captionof{table}{{\small Example images generated from DDE and their latent representations.}}
    \label{fig:generated digits}
  \end{minipage}
  }
  \hfill
  \adjustbox{valign=c}{
  \begin{minipage}[b]{.51\linewidth}
    \centering
    \resizebox{\textwidth}{!}{
        \begin{tabular}[h!]{llllll}
        \toprule
        Accuracy & LCM & 1-DDE & \textbf{2-DDE} & \darkblue{2-DBN} & \darkblue{VAE} \\ 
        \midrule
        Train classif. (\%) & 89.5 & 84.8 & \textbf{92.0} & 89.3 & 97.1 \\
        Test classif. (\%) & 90.9 & 86.2 & \textbf{92.6} & 90.6 & 92.0\\
        Train recon.  (\%) & 48.2 & 79.5 & \textbf{79.6} & 76.4 & 82.5 \\
        Test recon. (\%) & 48.5 & 79.1 & \textbf{79.9} & 77.0 & 82.7 \\
        \bottomrule 
        \end{tabular}
    }
    \captionof{table}{{\small Classification accuracy and pixel-wise reconstruction accuracy for the MNIST dataset.}}
    \label{tab:reconstruction MNIST}
  \end{minipage}
  }
  \vspace{-5mm}
\end{figure}

\subsection{Count Data: Poisson-DDE Hierarchical Topic Modeling} 
Next, we apply DDEs to learn hierarchical latent topics from text documents. Within the field of topic modeling, it is natural for the topics to be correlated with each other \citep{blei2006correlated}, and hierarchical topic modeling is often adopted \citep{griffiths2003hierarchical, paisley2014nested, chakraborty2024learning}. While many of the existing works assume a tree-structured hierarchy, DDEs flexibly allow multiple parents for each variable.

\begin{figure}
\centering
\resizebox{0.8\textwidth}{!}{
\begin{tikzpicture}[
  auto,
  node distance=0.5cm,
  every node/.style={
    draw,
    rectangle,
    minimum size=1cm,
    font=\tiny\sffamily\bfseries
  },
  arrow/.style={
    ->,
    line width=#1
  },
  every text node part/.style={align=center}
  ]

  \node (T1) {recreation};

  \node (M3) [below=of T1] {cars};
  \node (M2) [left=of M3] {sports};
  \node (M1) [left=of M2] {motorcycle};
  \node (M4) [right=of M3] {software};
  \node (M5) [right=of M4] {location/names};
  \node (M6) [right=of M5] {hardware};
  \node (M7) [right=of M6] {graphics};
  \node (M8) [right=of M7] {space};

  \node (T2) [above=of M6] {technology};

  \node (B1) [below=of M1] {geb \\	bike \\	banks \\	riding \\	pitt \\	dod \\	bmw \\ gordon};
  \node (B2) [below=of M2] {hockey\\	nhl\\	playoffs\\	teams\\	players\\	league\\	season\\	baseball};
  \node (B3) [below=of M3] {engine\\	cheap\\	parts\\	turned\\	models\\	dealer\\	market\\ car};
  \node (B4) [below=of M4] {files\\	directory\\	file\\	ms\\	copy\\	deleted\\	size\\	charles};
  \node (B5) [below=of M5] {peter \\	ny	\\chicago\\	devices\\	period\\	york\\	thomas\\ usa};
  \node (B6) [below=of M6] {motherboard \\ dx \\ CPU \\	mhz \\	controller \\	board \\ ram \\ cards};
  \node (B7) [below=of M7] {mouse \\	motif \\	lines \\	color \\	display \\	screen \\	window \\ image};
  \node (B8) [below=of M8] {orbit \\	earth \\	research \\	science	 \\nasa \\	development	 \\project \\ institute};

  \draw[arrow=0.98mm] (T1) -- (M1);
  \draw[arrow=0.94mm] (T1) -- (M2);
  \draw[arrow=0.3mm] (T1) -- (M3);
  \draw[arrow=0.69mm] (T1) -- (M4);
  \draw[arrow=1.07mm] (T1) -- (M5);
  \draw[arrow=0.8mm] (T2) -- (M3);
  \draw[arrow=0.71mm] (T2) -- (M4);
  \draw[arrow=0.515mm] (T2) -- (M5);
  \draw[arrow=0.4mm] (T2) -- (M6);
  \draw[arrow=0.95mm] (T2) -- (M7);
  \draw[arrow=0.3mm] (T2) -- (M8);
  \draw[red, arrow=0.57mm] (T2) -- (M2);
  
  \draw[dotted] (M1) -- (B1);
  \draw[dotted] (M2) -- (B2);
  \draw[dotted] (M3) -- (B3);
  \draw[dotted] (M4) -- (B4);
  \draw[dotted] (M5) -- (B5);
  \draw[dotted] (M6) -- (B6);
  \draw[dotted] (M7) -- (B7);
  \draw[dotted] (M8) -- (B8);
\end{tikzpicture}
}
\caption{DDE estimated from the 20 newsgroups dataset. For each shallow layer latent variable, we display the top eight representative words. The width of the upper layer arrows is proportional to the corresponding coefficients and the red arrow indicates negative values.}
    \label{fig:topic model}
    \vspace{-3mm}
\end{figure}
We analyze the text corpus from the 20 newsgroups dataset \citep{lang95}, which was previously analyzed by other topic models with binary latent variables \citep{srivastava2013modeling, gan2015scalable}.
After preprocessing and focusing on 12 newsgroups,
the dataset consists of $N=5,883$ documents and $J=653$ words. 
We fit the two-latent-layer DDE with a Poisson-distributed data layer (Poisson-DDE) $Y_j \mid (\ma^{(1)} = \aaa^{(1)}) \sim \text{Poi}\big(\exp \big[\beta_{j,0}^{(1)} + \sum_{k \in [K^{(1)}]} \beta_{j,k}^{(1)} \alpha_k^{(1)}\big] \big)$ and latent dimensions $K^{(1)} = 8, K^{(2)}=2$, and display the estimated latent structure in \cref{fig:topic model}.
Additional details behind this choice are given in Supplementary Material S.5.2.
To better interpret individual latent variables, we define representative words for each topic $k$ based on the discrepancy between $\beta_{j,k}$ and all other coefficients, $(\beta_{j,l})_{l \neq k}$.
That is, for each index $k$, we choose the words $j$ with the largest values of 
$\max\{ \min\{\beta_{j,k} - \beta_{j,l}: l \neq k \}, 0\}$. We display the representative words for each latent variable in the bottom row of \cref{fig:topic model}. Here, each latent variable is named based on the representative words and the held-out newsgroup categories.

Compared to the held-out tree structure of the 12 newsgroup labels (see Figure S.7 in the Supplement), the DDE discovered a lower-dimensional structure in \cref{fig:topic model}. The latent structure in DDEs allow multiple parents for each topic and effectively model the complex label dependence. For example, `cars', `software', `location/names' have both `recreation' and `technology' as parents, and many bottom-layer words are assigned to multiple topics.
We also observe that similar true labels are combined into a single latent variable, for example `computer' and `science' are combined into `technology' in the second latent layer, and `baseball' and `hockey' are combined as `sports'
in the bottom latent layer.

We also compare our model fit to existing directed graphical models with matching latent dimensions:
LDA \citep{blei2003latent} and DPFA-SBN \citep{gan2015scalable}. LDA has a single latent layer with mixed membership scores as continuous latent variables, and DPFA-SBN is a multilayer model with binary latent variables similar to DDEs.
We consider the following three metrics widely used in topic modeling to measure different aspects of fit \citep{chen2023learning}. The first is the \emph{perplexity}, measuring the predictive likelihood of the words in the held-out set. 
The second is the average negative \emph{coherence}, measuring the quality within each topic by computing 
$-\frac{1}{K^{(1)}}\sum_{k=1}^{K^{(1)}} \sum_{v_1, v_2 \in V_k} \log \left({(\text{freq}(v_1, v_2) + 1)}/{\text{freq}(v_2)} \right),$ where $V_k$ is the top $15$ representative words for the $k$th topic, and the function ``freq'' counts the number of documents containing the words specified in the input argument. The third is the \emph{similarity}, computing the number of overlapping representative words across different topics:
$\sum_{1 \le k_1 < k_2 \le K^{(1)}} \sum_{v_1 \in V_{k_1}, v_2 \in V_{k_2}} I(v_1 = v_2).$ For all three measures, smaller values are better.

\cref{tab:perplexity} summarizes the results and shows the promising fit of DDE. Compared to other models with the same latent dimensions, DDEs have better test perplexity and similarity. The similarity measure shows that while the other model fits exhibit common representative words among different topics, the representative words learned from DDE are entirely disjoint and effectively represent different topics. In terms of coherence, the DDE fit is better than LDA but worse than DPFA. We have also fit models with larger dimensions for comparison, by considering LDA with $256$ latent variables, and DPFA with $K^{(1)} = 128, K^{(2)} = 64$ latent variables. The dimension for LDA is motivated by that the DDE has $2^8 = 256$ mixture components; while for DPFA, this is the same latent dimension specified in the original paper \citep{gan2015scalable}. We can see that while considering a larger latent dimension may help in terms of perplexity, this leads to a loss of the within-topic coherence as well as dilutes the boundary of each topic, and hence gives less interpretable results.

\begin{table}[h!]
\centering
\resizebox{0.8\textwidth}{!}{
\begin{tabular}{cccccc}
\toprule
Model & Dimension $\mck$ & Train perplexity & Test perplexity & Neg. coherence & Similarity \\ 
\midrule
LDA & $8$       & 499 & 512 & 276 & 43\\
LDA & $256$     & 269 & 515 & 321 & (41450) \\
\midrule
DPFA & $8-2$ & 289 & 499 & 211 & 5 \\
DPFA & $128-64$ & 175 & 232 & 280 & (3378) \\
\midrule
DDE & $8-2$ & 322 & 398 & 270 & 0 \\
DDE & $8-3$ & 322 & 399 & 275 & 0 \\
\bottomrule
\end{tabular}
}
\caption{Train and test perplexity scores of different models on the 20 Newsgroups dataset. For all measures, smaller values are better. We parenthesize  similarity scores for the models with different dimensions, as the measure is not normalized.}
\label{tab:perplexity}
\vspace{-5mm}
\end{table}

\vspace{-3mm}
\subsection{\darkblue{Multimodal Educational Data: Bernoulli-Lognormal-DDE}}
\vspace{-3mm}
We apply the DDE to an educational assessment dataset from the Trends in International Mathematics and Science Study (TIMSS)  \citep{fishbein2021timss}. 
We analyze the eighth-grade students' responses for an internet-based mathematics assessment. 
\darkblue{As the assessment is electronically conducted, multiple modalities of information are recorded. Here, we focus on two important modalities: \emph{binary} response accuracy and \emph{continuous} response time. For each individual student, our data consists of response accuracy (whether the student gave a correct answer) and response time (how long the student took) for each of the $J/2 = 29$ items. We use the same latent variables to model both data modalities and model the binary response accuracy via Bernoulli distributions, and model the continuous positive response times via Lognormal distributions: $Y_j \mid (\ma^{(1)} = \aaa^{(1)}) \sim \text{lognormal}\big(\beta_{j,0}^{(1)} + \sum_{k \in [K^{(1)}]} \beta_{j,k}^{(1)} \alpha_k^{(1)}, \gamma_j \big)$.}

    \begin{table}[h!]
    \centering
    \resizebox{\textwidth}{!}{
    \begin{tabular}{cccccc}
    \toprule
    Response \textbackslash{} Latent skill & $A_1^{(2)}$   & $A_1^{(1)}$: Number & $A_2^{(1)}$: Algebra & $A_3^{(1)}$: Geometry & $A_4^{(1)}$: Data and Prob \\ 
    \midrule
    Agree a lot                                & 0.62 & 0.59 & 0.62 & 0.42 & 0.57  \\
    Agree a little                             & 0.50 & 0.47  & 0.52 & 0.34 & 0.43 \\
    Disagree a little                          & 0.37 & 0.33  & 0.38 & 0.31 & 0.37 \\
    Disagree a lot                             & 0.29 & 0.29 & 0.32 & 0.22 & 0.29 \\
    \bottomrule
    \end{tabular}
    }
    \caption{\darkblue{Average latent variable estimate for each response category for the question ``Mathematics is one of my favorite subjects''.}}\label{tab:like math}
    \vspace{-4mm}
    \end{table}

We fit the two-latent-layer DDE with $K^{(1)}=7,~K^{(2)}=1$, and estimate the latent skills based on the posterior probability using \eqref{eq:latent variable estimation} (see Supplement S.5.2 for details). We compare the estimated latent variables with the held-out information of each student's categorical response to a survey question: ``Mathematics is one of my favorite subjects'', and display the results in \cref{tab:like math}. The first column shows that the higher-order latent variable, $A_1^{(2)}$, is highly correlated with the extent that students like math. This suggests that $A_1^{(2)}$ can be interpreted as a general indicator of the students' interest in math, while the fine-grained latent variables $\mathbf A^{(1)}$ represent students' specific skill mastery profiles.
In addition, we observe that the students who enjoy math tend to have a higher probability of mastering specific skills as well. This is coherent with the fact that the estimated $\BB^{(1)}$-coefficients are nonnegative for both modes. In Supplement S.5.5, we also illustrate a strong correlation of the estimated intercepts for both data modalities, compared to another held-out information (whether or not the items are multiple choice questions).

\section{Discussion}\label{sec:discussion}
This paper makes significant contributions to core AI problems from statisticians' perspective by proposing a broad family of interpretable DGMs with solid identifiability guarantees, scalable computational pipelines, and promising application potential.
It also opens up interesting directions for future research.
First, the current formulation of DDEs mainly focuses on binary latent variables and a multi-layer graphical structure, but it provides foundations for understanding more complex discrete latent variable models. 
We believe both the identifiability theory and estimation methods are extendable to general categorical/polytomous latent variables. Additionally, our theoretical identifiability guarantees extend to ``generalized DDEs'' with cross-level edges, which we illustrate in Supplement S.1.5. It would be interesting to propose suitable estimation methods for these more flexible model settings.

Another interesting problem pertains to high-dimensional settings, where the number of observed responses, $J$, may grow with the sample size $N$. Our current notion of identifiability focuses on identifying the population model parameters under the traditional asymptotics with a fixed $J$, where the latent variables are marginalized out in the likelihood. 
As modern datasets often comes with a large number of observed features, it would be interesting to explore whether our identifiability and estimability results can be generalized to such settings.

In terms of the methodology and applications, 
it would be interesting to extend DDEs to datasets with additional covariates. For example, the MNIST dataset comes with the actual digit labels as well as the spatial structure of the pixels in the image. 
Finally, it would be interesting to extend DDEs for identifiable causal representation learning \citep{scholkopf2021toward} to uncover causal structures among the higher-order latent variables.

\spacingset{1}

\paragraph{Supplementary Material.}
The Supplement contains all technical proofs, additional theoretical results, and additional details about algorithms, simulations, and real data analyses.

\paragraph{Acknowldgement.}
The authors are partially supported by NSF Grant DMS-2210796.

\vspace{2mm}
\spacingset{1}
\bibliographystyle{apalike}
\bibliography{ref_rev}


\clearpage
\renewcommand{\thesection}{S.\arabic{section}}  
\renewcommand{\thetable}{S.\arabic{table}}  
\renewcommand{\thefigure}{S.\arabic{figure}}
\renewcommand{\theequation}{S.\arabic{equation}}
\renewcommand{\thetheorem}{S.\arabic{theorem}}
\renewcommand{\thedefinition}{S.\arabic{definition}}
\renewcommand{\theremark}{S.\arabic{remark}}
\renewcommand{\theexample}{S.\arabic{example}}
\renewcommand{\thelemma}{S.\arabic{lemma}}

\setcounter{table}{0}
\setcounter{section}{0}
\setcounter{equation}{0}
\setcounter{figure}{0}
\setcounter{theorem}{0}
\setcounter{definition}{0}
\setcounter{example}{0}
\setcounter{lemma}{0}
\setcounter{remark}{0}

\begin{center}
    \LARGE Supplement to ``Deep Discrete Encoders: Identifiable Deep Generative Models for Rich Data with Discrete Latent Layers''
\end{center}
\vspace{5mm}

\spacingset{1.7}

This Supplementary Material is organized as follows. \cref{sec:supp proof} proves all main theorems and provides additional identifiability results under one-layer saturated models and generalized DDEs that goes beyond multi-layer structures. \cref{sec:supp initialization} provides details regarding the spectral initialization algorithm. \cref{sec:supp algo details} gives details regarding the EM algorithm) such as the penalized EM algorithm, M-step update formulas, implementation details, and selection of the number of latent variables. \cref{sec:supp simulation} provides various additional simulation results under (a) generically identifiable true parameters, (b) varying numbers of Monte Carlo samples in the SAEM algorithm, (c) unknown latent dimensions.
\cref{sec:supp data analysis} gives additional data analysis details such as preprocessing, latent dimension selection, and additional visualizations. Finally, \cref{sec:supp literature} discusses additional related works.

\section{Proof of Theorems}\label{sec:supp proof}
Recall from \cref{sec:identifiability} that our identifiability results for general DDEs with multiple latent layers build upon the identifiability of a model with only one latent layer (the shallowest latent layer), where the deeper latent layers have been marginalized out. Here, we formally state identifiability conditions for such one-latent-layer saturated models in \cref{subsec:one-layer}, before proving the identifiability results for general DDEs in \cref{subsec:main proofs}. We prove the claims stated for the one-latent-layer models in \cref{subsec:supp proof of one layer}, and \cref{thm:estimation consistency} in \cref{subsec:proof of consistency}. 
\cref{sec:local dependent} establishes identifiability for generalized DDEs that allow additional edges compared.
Additional identifiability results for selecting the latent dimension and results for related models with interaction effects (instead of the main-effect DDEs introduced in the main paper) are presented in \cref{subsec:additional identifiability}.

\subsection{Identifiability Under One-latent-layer Saturated Models}\label{subsec:one-layer}
The \emph{one-latent-layer saturated model} is defined as follows. 
\begin{definition}[One-latent-layer saturated model]\label{def:saturated model}
    The one-latent-layer saturated model with $K^{(1)}$ latent variables, responses $\YY \in \prod_{j \in [J]} \mcy_j$, and parameters $(\TT^{(1)}, \GG^{(1)})$ is defined by the distribution of $\YY \mid \ma^{(1)}$ in \eqref{eq:observed exp fam}, and the saturated latent distribution
    \begin{align}\label{eq:saturated pmf}
        \PP(\ma^{(1)} = \aaa) = \pi_{\aaa}, \quad \text{for all} \quad \aaa \in \{0,1\}^{K^{(1)}}.
    \end{align}
    Here, the parameter $\boldsymbol{\pi} := (\pi_{\aaa})_{\aaa}$ satisfies $\pi_{\aaa} \in (0,1)$ and $\sum_{\aaa} \pi_{\aaa} = 1$. The notation $\TT^{(1)} := (\bo\pi, \BB^{(1)}, \ggamma)$ collects all continuous parameters. We define the parameter spaces $\Omega_{K^{(1)}}(\TT^{(1)}; \GG^{(1)})$ and $\Omega_{K^{(1)}}(\TT^{(1)}, \GG^{(1)})$ similar to \cref{def:parameter space} in the main paper.
\end{definition}
Here, the term ``saturated'' indicates that no additional distributional assumptions are imposed on the latent variables, except that they are discrete. 
Similar to \cref{def:identifiability up to permutation}, we define an equivalence relationship $\sim_{K^{(1)}}$ when the parameters are identical up to label switching, and use it to define identifiability. 

\begin{definition}[Equivalence relation]\label{def:identifiability up to permutation saturated}
    For the one-latent-layer saturated model, define an equivalence relationship ``$\sim_{K^{(1)}}$'' by setting $(\TT^{(1)}, \GG^{(1)}) \sim_{K^{(1)}} (\tilde{\TT}^{(1)}, \tilde{\GG}^{(1)})$ if and only if $ \bo\gamma = \tilde{ \bo\gamma}$ and there exist a permutation $\sigma^{(1)} \in S_{[K^{(1)}]}$ such that the following conditions hold:
    \begin{itemize}
        \item $\bo\pi_{(\alpha_{\sigma(1)}, \ldots, \alpha_{\sigma(K^{(1)})})} = \tilde{\bo\pi}_{\aaa}$ for all $\aaa \in \{0, 1\}^{K^{(1)}}$
        \item ${g}_{j,k}^{(1)} = \tilde{{g}}_{j ,\sigma^{(1)}(k)}^{(1)}$ and $\beta_{j,k}^{(1)} = \tilde{\beta}_{j,\sigma^{(1)}(k)}^{(1)}$ for all $k \in [K^{(1)}], j \in [J]$.
    \end{itemize}
    We say that the one-latent-layer saturated model with true parameters $(\TT^{(1)\star}, \GG^{(1)\star})$ is identifiable up to $\sim_{K^{(1)}}$, if for any alternate parameter value $(\TT^{(1)}, \GG^{(1)}) \in \Omega_{K^{(1)}}(\TT^{(1)}, \GG^{(1)})$ with $\PP_{\TT^{(1)}, \GG^{(1)}} = \PP_{\TT^{(1)\star}, \GG^{(1)\star}}$, it holds that $(\TT^{(1)},\GG^{(1)}) \sim_{\mck} (\TT^{(1)\star}, \GG^{(1)\star})$. Here, $\PP_{\TT^{(1)},\GG^{(1)}}$ is the marginal distribution of $\YY$, which follows from \eqref{eq:observed exp fam} and \eqref{eq:saturated pmf}.
\end{definition}

Under these definitions, we state identifiability results for the one-latent-layer saturated model. \cref{prop:sid} and \cref{prop:gid one-layer} are one-layer analogues of \cref{thm:deep id} and \cref{thm:deep gid}, respectively. We postpone the proofs of these results to \cref{subsec:supp proof of one layer}.

\begin{proposition}\label{prop:sid}
Given the knowledge of $K^{(1)}$, the one-latent-layer saturated model with parameters $(\TT^\star, \GG^{(1)\star}) \in \Omega_{K^{(1)}}(\TT^{(1)}, \GG^{(1)})$ is identifiable up to $\sim_{K^{(1)}}$ when the true parameters $\B^{(1)\star}, \GG^{(1)\star}$ satisfy conditions A, B from \cref{thm:deep id}.
In particular, condition B holds when $\GG^{(1)\star}$ contains another identity matrix.
\end{proposition}

\begin{proposition}\label{prop:gid one-layer}
Consider the one-latent-layer saturated model where all parametric families and link functions $g_j$s in \eqref{eq:observed exp fam} are analytic, and the true parameter lives in $\Omega_{K^{(1)}} (\TT^{(1)}; \GG^{(1) \star})$. Then, the model is generically identifiable when $\GG^{(1) \star}$ satisfies condition C from \cref{thm:deep gid}.
\end{proposition}

\subsection{Proof of Theorems \ref{thm:deep id} and \ref{thm:deep gid}}\label{subsec:main proofs}
We prove the identifiability results for DDEs using Propositions \ref{prop:sid} and \ref{prop:gid one-layer}.

\begin{proof}[Proof of \cref{thm:deep id}]
Our argument is based on applying Proposition \ref{prop:sid} in a layer-wise manner. First, consider the bottom two layers with 
\begin{align}\label{eq:prop d layer}
    \pi^{(1)}_{ \aaa^{(1)}} = \PP(\ma^{(1)} = \aaa^{(1)}),\quad \forall\aaa^{(1)}\in\{0,1\}^{K_1}
\end{align}
defined by marginalizing out the deeper latent layers. Now, we can consider this as the one-layer model with proportion parameters $\bo{\pi}^{(1)} = (\pi^{(1)}_{\aaa^{(1)}}, \aaa^{(1)}\in\{0,1\}^{K_1})$. Then, Proposition \ref{prop:sid} gives the identifiability of $\B^{(1)}, \GG^{(1)}, \bo{\pi}^{(1)}$ up to $\sim_{K^{(1)}}$.

Having identified $\bo{\pi}^{(1)}$, the marginal distribution of the shallowest latent layer $\ma^{(1)}$ is uniquely identified. We generalize the notation in \eqref{eq:prop d layer} and define $\bo\pi^{(d)}$ similarly.
Inductively, for $1 \le d < D$, we apply Proposition \ref{prop:sid} by considering $\ma^{(d)}$ and $\bo{\pi}^{(d)}$ as the ``observed'' binary response vector and its proportion parameters that characterize its marginal probability mass function:
$$
\pi^{(d)}_{ \aaa^{(d)}} = \PP(\ma^{(d)} = \aaa^{(d)}),\quad 
\forall\aaa^{(d)}\in\{0,1\}^{K_d}.
$$
Consequently, the parameters $\B^{(d+1)}, \GG^{(d+1)}, \bo{\pi}^{(d+1)}$ between the $d$th and $(d+1)$th layers are identified up to $\sim_{K^{(d+1)}}$. In particular, when $d = D-1$, it remains to determine $\pp$ from the unstructured proportion parameter vector $\bo{\pi}^{(D)}$. Since we already have identified the $D$th layer labels up to $\sim_{K^{(D)}}$, this simply follows by marginalizing out the irrelevant coordinates:
$$p_k = \PP(A_k^{(D)} = 1) = \sum_{\aaa^{(D)}: \alpha^{(D)}_k = 1} \pi_{\aaa^{(D)}}^{(D)}.$$
The proof is complete.
\end{proof}

\begin{proof}[Proof of \cref{thm:deep gid}]
    \cref{prop:gid one-layer} shows that the non-identifiable measure-zero set of the one-latent-layer saturated model only depends on the coefficients $\BB^{(1)}$ and $\ggamma$. By marginalizing out all layers except the bottom two layers, the DDE becomes a one-latent-layer saturated model with parameters $\TT^{(1)}:= (\pi^{(1)}, \BB^{(1)}, \GG^{(1)}, \ggamma)$. Since we assume that $\GG^{(1)}$ satisfies condition C, the parameters $\TT^{(1)}$ are identifiable (up to a permutation $\sigma^{(1)} \in S_{[K^{(1)}]}$) as long as $(\BB^{(1)},\ggamma) \not \in N^{(1)}$. Here, $N^{(1)}$ is a measure-zero subset of the coefficient space $\Omega(\BB^{(1)},\ggamma ; \GG^{(1)})$.

    Now, assuming $\BB^{(1)} \not \in N^{(1)}$, we can use a similar argument for deeper layers inductively. For $2 \le d \le D-1$, let $N^{(d)}$ be the non-identifiable measure-zero subset of the $d$th layer coefficient space $\Omega(\BB^{(d)} ; \GG^{(d)})$.
    Note that we can still apply \cref{prop:gid one-layer} since the conditional distribution of $\ma^{(d)} \mid \ma^{(d+1)}$ is modeled as a Bernoulli distribution with a logistic link $g_{\text{logistic}}$, which is indeed analytic and satisfy Assumption \ref{assmp:monotone family}. Consequently, as long as $\BB^{(d)} \not \in N^{(d)}$ for all $1 \le d \le D-1$, the $D$-layer DDE is identifiable up to permutations of the latent variables within each layer. The proof is complete since $\cup_{d =1}^D (N^{(d)})^c$ is a union of a finite number of measure-zero sets, and hence again measure-zero.
\end{proof}

\subsection{Proof of Propositions \ref{prop:sid} and \ref{prop:gid one-layer}}\label{subsec:supp proof of one layer}
\subsubsection{Additional Notations}
We introduce additional notations that will be used to prove identifiability results for the one-latent-layer saturated model. Most of these notations are consistent with \cite{lee2024new}.
First, we omit the superscript ``${(1)}$'' that indicates the first latent layer when dealing with the one-latent-layer saturated model. Let $j, k, \aaa$ denote typical indices for $j \in [J], k \in [K], \aaa \in \{0,1\}^K$.
Write $\GG = \GG^{(1)}= \{\mathbf{g}_1^\top, \ldots, \mathbf{g}_J^\top \}^\top$ and let $H_j := \{k \in [K]: G_{j,k} = 1\}$ be the index of the parent latent variables for $Y_j$. Recall that for each $j$, the sample space $\mathcal{Y}_j$ is a separable metric space. Let $m_j$ be a base measure on $\mathcal{Y}_j$, this will be the counting measure for discrete sample spaces and the Lebesgue measure for continuous cases.
Given $j$ and $\aaa$, define the measure $\PP_{j,\aaa}$ on $\mathcal{Y}_j$ by setting
\begin{align}\label{eq:p_j,aaa def}
    \PP_{j,\aaa}(S) := \PP (Y_j \in S \mid \ma = \aaa) = \int_{S} p_j(y; \beta_{j,0}+\sum_k \beta_{j,k} \alpha_k, \gamma_j) dm_j(y).
\end{align}
In other words, $\PP_{j,\aaa}$ denotes the conditional distribution of $Y_j \mid \ma = \aaa$ in \eqref{eq:observed exp fam}.

For each $j \in [J]$, construct measurable subsets $S_{1,j}, \ldots, S_{\kappa_j, j} \subseteq \mathcal{Y}_j$ with $\kappa_j \ge 2$, where the collection of vectors $\mathbf{s}_j(\aaa):= \big(\PP_{j,\aaa} (S_{1, j}), \ldots, \PP_{j,\aaa}(S_{\kappa_j, j}) \big)_{\aaa \in \{0,1\}^K}$ is ``faithful'' in the following sense: 
\begin{enumerate}[(a)]
    \item for $\aaa, \aaa'$ with $\aaa_{H_j} = \aaa'_{H_j}$, it holds that $\mathbf{s}_j(\aaa) = \mathbf{s}_j(\aaa')$,
    \item there exists $\aaa, \aaa'$ with  $\aaa_{H_j} \neq \aaa'_{H_j}$ such that $\mathbf{s}_j(\aaa) \neq \mathbf{s}_j(\aaa')$.
\end{enumerate}
This construction is possible since Assumption \ref{assmp:proportion}(a) on the parameter space lead to a faithful graphical model. Without loss of generality, suppose that $S_{\kappa_j,j} = \mcy_j$ for all $j$. Also, define the following (unordered) set
\begin{align}\label{eq:S_j definition}
    \mathcal{S}_j := \Big\{ \mathbf{s}_j(\aaa): \aaa \in \{0,1\}^K \Big\}.
\end{align}

Define $\NN_1$ to be a $\kappa_1 ... \kappa_K \times 2^K$ matrix by setting
$$\NN_1((l_1, ..., l_K), \aaa) := \PP(Y_1 \in S_{l_1, 1}, ..., Y_K \in S_{l_K, K} \mid \aaa).$$ 
Here, we index the $2^K$ columns of $\NN_1$ using the binary vector $\aaa \in \{0, 1\}^K$, and the rows by $\xi_1 = (l_1, \ldots, l_K)$, where $l_j \in [\kappa_j]$. 
Similarly, let $\NN_2$ be a $\kappa_{K+1} \ldots \kappa_{2K} \times 2^K$ matrix whose $((l_{K+1}, \ldots, l_{2K}), \aaa)$-th entry is $\PP(Y_{K+1} \in S_{l_{K+1}, K+1}, ..., Y_{2K} \in S_{l_{2K}, 2K} \mid \aaa)$, and $\NN_3$ be a $\kappa_{2k+1} \ldots \kappa_{J} \times 2^K$ matrix whose $((l_{2K+1}, \ldots, l_{J}), \aaa)$-th entry is $\PP(Y_{2K+1} \in S_{l_{2K+1}, 1}, ..., Y_J \in S_{l_{J}, J} \mid \aaa)$. Similar to $\NN_1$, we index the rows of $\NN_2$ and $\NN_3$ by $\xi_2 = (l_{K+1}, \ldots, l_{2K})$ and $\xi_3 = (l_{2K+1}, \ldots, l_J)$, respectively. For notational simplicity, let $\upsilon_1 = \prod_{k = 1}^{K} \kappa_{k}, ~\upsilon_2 = \prod_{k = K+1}^{2K} \kappa_{k}, ~\upsilon_3 = \prod_{k = 2K+1}^{J} \kappa_{k}$. Note that the assumption $S_{\kappa_j, j} = \mcy_j$ implies forces the last row in all $\NN_a$s to be $\mathbf{1}_{2^K}^\top$.

Next, let $\mathbf{P}_0$ be a 3-way marginal probability tensor with size $\upsilon_1 \times \upsilon_2 \times \upsilon_3$, defined as
\begin{align*}
    \mathbf{P}_0(\xi_1, \xi_2, \xi_3) &= \PP(Y_1 \in S_{l_1}, \ldots, Y_J \in S_{l_J}) \\ &= \sum_{\aaa} \pi_{\aaa} \NN_1((l_{1}, \ldots, l_{K}), \aaa) \NN_2((l_{K+1}, \ldots, l_{2K}), \aaa) \NN_3((l_{2K+1}, \ldots, l_{J}), \aaa).
\end{align*}
We introduce an additional notation for tensor products as follows. For $a = 1, 2, 3$, consider $\nu_a \times r$ matrices $\M_a$ whose $l$th column is indexed as $\mathbf{m}_{a,l}$. Also, let $\circ$ denote the outer product between vectors. Then, we define the tensor product of $\M_1, \M_2, \M_3$ as
$$[\M_1, \M_2, \M_3] := \sum_{l=1}^r \mathbf{m}_{1,l} \circ \mathbf{m}_{2,l} \circ \mathbf{m}_{3,l}.$$
Using this notation, we can write $\mathbf{P}_0$ as  follows:
\begin{align}\label{eq:tensor decomposition}
    \mathbf{P}_0 = [\NN_1\text{Diag}(\bo\pi), \NN_2, \NN_3].
\end{align}

Now, our notation is almost identical to that in the proof of Theorem 1 in \cite{lee2024new}. Under conditions A and B from \cref{thm:deep id}, we can apply \emph{step 3} and the first two paragraphs of \emph{step 4} there to argue that the decomposition \eqref{eq:tensor decomposition} is unique up to a column permutation. We summarize this in the below \cref{thm:kruskal}.

\begin{lemma}[Theorem 1 in \cite{lee2024new}]\label{thm:kruskal}
    Consider the one-latent-layer saturated model, where the true parameters satisfy conditions A and B. Let $\mathbf{P}_0$ be the 3-way marginal probability tensor under these parameters. Then, the tensor decomposition $\mathbf{P}_0 = [\M_1, \M_2, \M_3]$ is unique up to a common column permutation. Here, $\M_a$s are $\upsilon_a \times 2^K$ matrices, whose last rows are the all-one vector $\mathbf{1}_{2^K}^\top$ for $a = 2,3$.
    Additionally, assuming that the graphical matrix $\GG$ is known, the model parameters $\TT$ are identifiable up to sign-flipping.
\end{lemma}
\noindent While \cite{lee2024new} establishes identifiability of the conditional distributions $\{\PP_{j,\aaa}\}$ (see \eqref{eq:p_j,aaa def}), recovering $\BB, \ggamma$ from $\{\PP_{j,\aaa}\}$ is straightforward since each $\PP_{j,\aaa}$ is the distribution of an identifiable parametric family.

\subsubsection{Proof of Proposition \ref{prop:sid}}

The following Lemma will be crucially utilized to identify and recover $\tilde{\GG}$. \darkblue{\cref{lem:H_j cardinality} generalizes the fact that when $Y_j$ is a pure child of $A_k^{(1)}$ (in other words $|H_j| =1$), we have $$|\mathcal{S}_j| = |\{\PP(Y_j \mid \ma^{(1)} = \aaa^{(1)}): \aaa^{(1)} \in \{0,1\}^{K^{(1)}} \}| = |\{\PP(Y_j \mid A_k^{(1)} = \alpha_k^{(1)}): \alpha_k^{(1)} = 0,1 \}| = 2.$$
Thus, by partitioning the set of conditional distributions corresponding to $Y_j$ which are pure children, the binary indices $\{\aaa^{(1)}: \alpha_k^{(1)} = 1\}$ and $\tilde{\GG}$ can be recovered (up to the permutation $\sigma \in S_{[K^{(1)}]}$).}
We present its proof after proving the main Proposition.
\begin{lemma}\label{lem:H_j cardinality}
Suppose that the parameters $(\B, \GG)$ satisfy part (a) in Assumption \ref{assmp:proportion}. For any $j$, define $H_j, \mathcal{S}_j, \bs_j(\aaa)$ as above.
\begin{enumerate}[(a)]
    \item $|H_j|$ and $|\mathcal{S}_j|$ satisfies:
    \begin{itemize}
    \item $|H_j| = 0$ if and only if $|\mathcal{S}_j| = 1$
    \item $|H_j| = 1$ if and only if $|\mathcal{S}_j| = 2$
    \item $|H_j| \ge 2$ if and only if $|\mathcal{S}_j| \ge 3$.
\end{itemize}
    \item $\Big\{\bs_j(\aaa): \alpha_k = 1 \Big\} = \Big\{\bs_j(\aaa): \alpha_k = 0 \Big\}$ if and only if $g_{j,k}$ = 0.
\end{enumerate}
\end{lemma}

\begin{proof}[Proof of \cref{prop:sid}]
Our proof builds upon \cref{thm:kruskal}, which proved a more general notion of nonparametric identifiability of the one-layer model, but under a given graphical matrix $\GG$. \cref{thm:kruskal} proves that the continuous parameters are identifiable up to sign flipping, given $\GG$. In our setting, Assumption \ref{assmp:proportion}(c) resolves the sign flipping issue and the continuous parameters can be uniquely determined. Consequently, it suffices to show that $\GG$ is identifiable up to the equivalence relation $\underset{K^{(1)}}{\sim}$. We separate this proof into two steps. 

\emph{Step 1: Tensor decomposition and setup.}
Consider a one-latent-layer saturated model with true parameters $(\TT, \GG)$ that satisfies conditions A, B. Suppose there exists an alternative set of parameters $(\tilde{\TT}, \tilde{\GG})$ that define a same marginal distribution of $\mathbf Y$. Define the notation $\tilde{\PP}_{j,\tilde{\aaa}}$ similar as ${\PP}_{j,{\aaa}}$ in \eqref{eq:p_j,aaa def}, and also define
$\tilde{\NN}_1, \tilde{\NN}_2, \tilde{\NN}_3$ as the conditional probability matrices that specify the value of $$\PP\big((Y_1, \ldots, Y_K) \mid \ma=\aaat \big), \quad \PP\big((Y_{K+1} \ldots, Y_{2K}) \mid \ma=\aaat \big), \quad \PP\big((Y_{2K+1}, \ldots, Y_J) \mid \ma=\aaat \big)$$ under the alternative parameters. Here, each column of $\tilde{\NN}_a$ is denoted by $\aaat \in \{0,1\}^K$, and the last row of each $\tilde{\NN}_a$ is the all-one vector $\mathbf{1}_{2^K}^\top$. Then, the marginal probability tensor $\mathbf{P}_0$ defined in \eqref{eq:tensor decomposition}
can be written as
\begin{align}\label{eq:tensor def}
    \mathbf{P}_0 = [{\NN}_1 \diag({\bo\pi}), {\NN}_2, {\NN}_3] = [\tilde{\NN}_1 \diag(\tilde{\bo\pi}), \tilde{\NN}_2, \tilde{\NN}_3].
\end{align}

By applying \cref{thm:kruskal}, the tensor decomposition in \eqref{eq:tensor def} is unique, so $\tilde{\NN}_1 \diag({\bo\pi}), \tilde{\NN}_2, \tilde{\NN}_3$ and $\NN_1\diag({\bo\pi}), \NN_2, \NN_3$ are identical up to a common column permutation, say $\mathfrak{S} \in S_{\{0,1\}^K}$. 
In particular, as the last row of $\tilde{\NN}_1 \diag(\tilde{\bo\pi})$ and $\NN_1\diag({\bo\pi})$ is exactly $\tilde{\bo\pi}^\top$ and $\bo\pi^\top$, $\big(\tilde{\bo\pi}, \tilde{\NN}_1\big)$ and $\big(\bo\pi, \NN_1 \big)$ are also identical up to the same permutation $\mathfrak{S}$.

We use this observation to prove that 
$\tilde{\GG}$ is equivalent to $\GG$ under $\underset{K^{(1)}}{\sim}$.
One subtlety for identifying $\GG$ is that there is no information about the $2^K$ column indices of $\tilde{\NN}_1, \tilde{\NN}_2, \tilde{\NN}_3$, so we cannot read off the conditional dependence structure directly. We tackle this problem based on the key observation that the unordered set $\mathcal{S}_j$ (defined in \eqref{eq:S_j definition}) can be written as
\begin{align}
    \mathcal{S}_j =& \Big\{\big(\PP_{j,\aaa} (S_{1, j}), \ldots, \PP_{j,\aaa} (S_{\kappa_j, j}) \big): {\aaa} \in \{0,1\}^K\Big\} \label{eq:S_j expression} \\
    = &\Big\{\big( \NN_{a_j}((\kappa_1, \ldots, \kappa_{j-1}, {1}, \kappa_{j+1}, \ldots, \kappa_J), \aaa ), \cdots, \NN_{a_j} ((\kappa_1, \ldots, \kappa_{j-1}, {\kappa_j}, \kappa_{j+1}, \ldots, \kappa_J),\aaa) \big): \aaa\Big\} \notag \\
    = &\Big\{\big( \tilde{\NN}_{a_j}((\kappa_1, \ldots, \kappa_{j-1}, {1}, \kappa_{j+1}, \ldots, \kappa_J), \aaat ), \cdots, \tilde{\NN}_{a_j} ((\kappa_1, \ldots, \kappa_{j-1}, {\kappa_j}, \kappa_{j+1}, \ldots, \kappa_J),\aaat) \big): \aaat \Big\} \notag \\
    =& \Big\{\big(\tilde{\PP}_{j,\aaat} (S_{1, j}), \ldots, \tilde{\PP}_{j,\aaat} (S_{\kappa_j, j}) \big): {\aaat}\Big\}, \label{eq:S_j expression 2}
\end{align}
and provides enough information about $\tilde{\mathbf{g}}_j$. Here, the index $a_j = 1,2,3$ can be understood from the context, for example $a_j = 1$ when $j \le K$, $a_j = 2$ when $K<j\le 2K$.

\emph{Step 2: Proving the equivalence by constructing a column permutation.}
Now, we show that there exists a permutation $\sigma \in S_{[K]}$ such that $\tilde{g}_{j,k} = g_{j,\sigma(k)}$ for all $j \in [J], k \in [K]$. We first construct such a permutation $\sigma$ by observing $\mathcal{S}_1, \ldots, \mathcal{S}_K$. For $k \in [K]$, \eqref{eq:S_j expression} implies $|\mathcal{S}_k|=2$, and part (a) of \cref{lem:H_j cardinality} constraints $\tilde{\mathbf{g}}_k$ to be a standard basis vector. Hence, we can define $\sigma(k)$ such that $\tilde{\mathbf{g}}_k = \mathbf{e}_{\sigma(k)}$.
To see that $\sigma$ is indeed a permutation, we have to show that $\sigma(k) \neq \sigma(l)$ for $k \neq l$. But this is immediate by noting that the cardinality of the pmf vector of $\{\PP(Y_k \in S_{1, k}, Y_l \in S_{1, l} \mid \ma = \aaa): \aaa \}$ is 2 if and only if $k = l$. For future purposes, let us partition the $2^K$ row indices $\aaat$ in \eqref{eq:S_j expression 2} into two groups $T_k$ and $T_k^c$ based on the value of $\big(\tilde{\PP}_{j,\aaat} (S_{1, j}), \ldots, \tilde{\PP}_{j,\aaat} (S_{\kappa_j, j}) \big)$. Then, the columns of $\tilde{\NN}$ that correspond to the $\aaat \in T_k$ must be identical to the columns of ${\NN}$ indexed by $\aaa$s such that $\alpha_{\sigma(k)} = 0$ (or 1).

Now, for each $j > K$ and $k \in [K]$, we show that $\tilde{g}_{j,k} = g_{j,\sigma(k)}$. By part (b) of \cref{lem:H_j cardinality}, we have 
$\tilde{g}_{j,k}$ = 0 if and only if \begin{align}\label{eq:probability under alpha tilde}
    \big\{\big(\tilde{\PP}_{j,\aaat} (S_{1, j}), \ldots, \tilde{\PP}_{j,\aaat} (S_{\kappa_j, j}) \big): \aaat \in T_k\big\} = \big\{\big(\tilde{\PP}_{j,\aaat} (S_{1, j}), \ldots, \tilde{\PP}_{j,\aaat} (S_{\kappa_j, j}) \big): \aaat \not\in T_k\big\}.
\end{align}
For notational simplicity, let $\aaa_\sigma := (\alpha_{\sigma(1)}, \ldots, \alpha_{\sigma(k)})$. Then, by the construction of $T_k$, \eqref{eq:probability under alpha tilde} simplifies to 
$$\Big\{\big(\PP_{j,\aaa_\sigma} (S_{1, j}), \ldots, \PP_{j,\aaa_\sigma} (S_{\kappa_j, j}) \big): \alpha_{\sigma(k)} = 1 \Big\} = \Big\{\big(\PP_{j,\aaa_\sigma} (S_{1, j}), \ldots, \PP_{j,\aaa_\sigma} (S_{\kappa_j, j}) \big): \alpha_{\sigma(k)} = 0 \Big\}.$$
Another application of part (b) of \cref{lem:H_j cardinality} shows that this is equivalent to $g_{j, \sigma(k)} = 0$. Hence, $\tilde{g}_{j,k} = g_{j,\sigma(k)}$ and we have shown $\GG \underset{K^{(1)}}{\sim} \GG$. This completes the proof.
\end{proof}

Below, we provide a short proof of \cref{lem:H_j cardinality}.
\begin{proof}[Proof of \cref{lem:H_j cardinality}]
\begin{enumerate}[(a)]
    \item The proof is immediate from the faithfulness of the graphical matrix $\GG$. For example, the third bullet point follows from noting that $|H_j| \ge 2$ implies $|\{\sum_{k \in H_j} \beta_{j,k}\alpha_k: \aaa \}| \ge 3$, and that $|\{\sum_{k \in H_j} \beta_{j,k}\alpha_k: \aaa \}| \le 2$ implies $|H_j| \le 1$.
    \item The ``if'' part immediately follows by conditional independence. For the ``only if'' part, by considering the contrapositive statement, it suffices to show that $g_{j,k} = 1$ implies $$\{\bs_j(\aaa): {\alpha_k = 1}\} \neq \{\bs_j(\aaa): {\alpha_k = 0}\}.$$
    By writing out the parametrization and using the identifiability of the parametric family in \eqref{eq:observed exp fam}, it suffices to show that 
    $\{\sum_{l \neq k} \beta_{j,l} \alpha_l + \beta_{j,k}\} \neq \{\sum_{l \neq k} \beta_{j,l} \alpha_l\}$.
    But this is immediate since $\beta_{j,k} \neq 0$.
\end{enumerate}
\end{proof}

\subsubsection{Proof of \cref{prop:gid one-layer}}
For the sake of notational simplicity, for a given coefficient matrix $\BB$, define
\begin{align}
    \eta_{j,\aaa} := \beta_{j,0} + \sum_{k=1}^K \beta_{j,k} \alpha_k
\end{align}
for each $j$ and $\aaa \in \{0,1\}^K$. For an alternative coefficient matrix $\tilde{\BB}$, we similarly define
$$\tilde{\eta}_{j,\tilde{\aaa}} := \tilde{\beta}_{j,0} + \sum_{k=1}^K \tilde{\beta}_{j,k} \tilde{\alpha}_k$$
for $\tilde{\aaa} \in \{0,1\}^K$.

Before presenting the proof of \cref{prop:gid one-layer}, we state two lemmas whose proof is postponed to the end of the subsection. Our first lemma is a relaxation of \cref{thm:kruskal}, and guarantees an \emph{almost sure} unique tensor decomposition of $\mathbf{P}_0$. Recall the definition of $\mathbf{P}_0$ from \eqref{eq:tensor decomposition}. 

\begin{lemma}[Modification of Theorem 2 in \cite{lee2024new}]\label{lem:kruskal gid}
    Consider the one-latent-layer saturated model with a true graphical matrix $\GG^{(1) \star}$ that satisfies condition C, and analytic parametric families $p(\cdot; \eta, \gamma)$ and link functions $g_j$.
    Then, the rank $2^K$ tensor decomposition $\mathbf{P}_0 = [\M_1, \M_2, \M_3]$ is unique up to column permutations for $\Omega_K(\TT^{(1)} ; \GG^{(1) \star}) \setminus \mathcal{N}_1$. Here, $\M_a$s are $\upsilon_a \times 2^K$ matrices, whose last rows are the all-one vector $\mathbf{1}_{2^K}^\top$ for $a = 2,3$. The set $\mathcal{N}_1$ is a measure-zero subset of $\Omega_K(\TT^{(1)} ; \GG^{(1) \star})$ that only imposes restrictions on $\BB, \ggamma$ and not on the mixture proportion parameters $\bo\pi$.
\end{lemma}

Our next lemma provides additional structure regarding the $2^K$ column indices in the tensor decomposition in \cref{lem:kruskal gid}. This is a nontrivial problem as we wish to identify the parameters up to a label switching of $K$ binary latent variables. Note that we can no longer utilize the pure children to address this problem, which was the strategy under the setting of strict identifiability.
Here, to formally state the label switching, recall the notation of $\sim_K$ \cref{def:identifiability up to permutation saturated} and write $\BB \sim_K \tilde{\BB}, \GG \sim_K \tilde{\GG}$, $\aaa \sim_K \tilde{\aaa}$ when there exists a permutation $\sigma \in S_{[K]}$ such that $\beta_{j,l} = \tilde{\beta}_{j,\sigma(l)}$, $g_{j,l} = \tilde{g}_{j,\sigma(l)}$, and 
$\aaa = (\tilde{\alpha}_{\sigma(1)}, \ldots, \tilde{\alpha}_{\sigma(K)}).$

\begin{lemma}\label{lem:sign flip}
    Suppose that there exist two sets of parameters $(\BB, \GG)$, $(\tilde{\BB}, \tilde{\GG})$ that satisfies Assumption \ref{assmp:proportion} and defines an identical $\eta_{j,\aaa}$ in the following sense:
    there exists a permutation $\mathfrak{S} \in S_{\{0,1\}^K}$ such that 
    \begin{align}
        \eta_{j,\aaa} = \tilde{\eta}_{j,\mathfrak{S}(\aaa)},
    \end{align}
    for all $j$, $\aaa$. 
    Then, we have $(\BB,\GG,\aaa) \sim_K (\tilde{\BB},\tilde{\GG}, \mathfrak{S}(\aaa))$ for ``generic'' parameters $\BB \not\in \Omega(\BB;\GG) \setminus \mathcal{N}_2$, where $\mathcal{N}_2$ is a measure-zero subset of $\Omega(\BB;\GG)$.
\end{lemma}

\begin{proof}[Proof of \cref{prop:gid one-layer}]
We work under the same notations introduced at the beginning of the section. We make one additional assumption regarding the finite subsets $\mathcal{D}_j := (S_{1,j}, \ldots, S_{\kappa_j,j})$ of the sample space $\mathcal{Y}_j$ as follows. We assume that $\mathcal{D}_j \subset \mathcal{C}_j$, where $\mathcal{C}_j$ is a countable separating class whose values determine probability measure on $\mathcal{Y}_j$. The existence of such a separating class is a consequence of $\mathcal{Y}_j$ being a separable metric space; see Step 1 in the proof of Theorem 1 in \cite{lee2024new} for a proof.

Suppose that $\GG$ satisfy condition C, $\TT \in \Omega_{\mck}(\TT;\GG)$, and that there exists an alternate parameter $\tilde{\TT}, \tilde{\GG}$ that defines the same marginal likelihood. Here, we are dropping all superscripts in $\GG = \GG^{(1)\star}$ for simplicity.
We show that $(\TT, \GG) \sim_{\mck} (\tilde{\TT}, \tilde{\GG})$ for $\TT \in \Omega_{\mck}(\TT;\GG) \setminus (\mathcal{N}_1 \cup \mathcal{N}_2)$, where $\mathcal{N}_1, \mathcal{N}_2$ are measure-zero sets that will be defined later.
Recall that for $a = 1,2,3$, $\NN_a$ are the $\upsilon_a \times 2^K$ conditional probability matrices that describe $\YY \mid (\ma = \aaa)$ under the true parameters $\TT, \GG$. Similarly, define $\tilde{\NN}_a$ as the corresponding matrices under the alternative parameters $\tilde{\TT}, \tilde{\GG}$. 
Similar to the proof of \cref{prop:sid}, we start from the rank $2^K$ decomposition of the marginal probability tensor in \eqref{eq:tensor decomposition}. By applying \cref{lem:kruskal gid}, the decomposition
\begin{equation}\label{eq:tensor decomposition alternative}
    \mathbf{P}_0 = [\NN_1\diag(\bo\pi), \NN_2, \NN_3] = [\tilde{\NN}_1 \diag(\tilde{\bo\pi}), \tilde{\NN}_2, \tilde{\NN}_3]
\end{equation}
is unique up to (the $2^K$) column permutations, for true parameters $\TT \in \Omega_K(\TT; \GG) \setminus \mathcal{N}_1$. Here, $\mathcal{N}_1$ is a measure-zero subset of $\Omega_K(\TT; \GG)$ that is defined in \cref{lem:kruskal gid}.
This implies that $(\bo\pi, \NN_1, \NN_2, \NN_3)$ and $(\tilde{\bo\pi}, \tilde{\NN}_1, \tilde{\NN}_2, \tilde{\NN}_3)$ are identical up to a common column permutation $\mathfrak{S} \in S_{\{0,1\}^K}$. In other words, for any $\aaa \in \{0,1\}^K$, the $\aaa$th column in $\NN_a$ is identical to the $\mathfrak{S}(\aaa)$th column of $\tilde{\NN}_a$ and $\pi_{\aaa} = \tilde{\pi}_{\mathfrak{S}(\aaa)}$.

By considering the row of $\NN_a$ that corresponds to the set $S_{l_j,j}$, we have
\begin{align}
    \PP_{j,\aaa}(S_{l_j, j}) &=  \NN_a(\kappa_{\cdot}, \ldots, \kappa_{j-1}, l_{j}, \kappa_{j+1}, \ldots, \kappa_{\cdot}, \aaa) \label{eq:probability comparing} \\
    &= \tilde{\NN}_a(\kappa_{\cdot}, \ldots, \kappa_{j-1}, l_{j}, \kappa_{j+1}, \ldots, \kappa_{\cdot}, \mathfrak{S}(\aaa)) = \tilde{\PP}_{j,\mathfrak{S}(\aaa)}(S_{l_j,j}) \notag
\end{align}
for all $j, l_j \in [\kappa_j]$, and $\aaa$. Furthermore, note that this identity holds for any finite subset $\mathcal{D}_j$ of $\mathcal{C}_j$, as the column indices of $\tilde{\NN}_a$ are determined by the alternative model and do not depend on the discretization $\mathcal{D}_j$. Hence, \eqref{eq:probability comparing} holds for all $S_{l_j,j} \in \mathcal{C}_j$. Since $\mathcal{C}_j$ is a separating class, we have $\PP_{j,\aaa} = \tilde{\PP}_{j,\mathfrak{S}(\aaa)}$. Recalling that $\PP_{j,\aaa} = \text{ParFam}_j(\eta_{j,\aaa}, \gamma_j)$ for some identifiable parametric family (see \eqref{eq:p_j,aaa def}), we must have $\eta_{j,\aaa} = \tilde{\eta}_{j,\mathfrak{S}(\aaa)}$ and $\gamma_j = \tilde{\gamma}_j$ for all $j, \aaa$. 

Now, additionally assuming that $\TT \not \in \mathcal{N}_2$, we can apply \cref{lem:sign flip}. Here, $\mathcal{N}_2$ is the null set defined in \cref{lem:sign flip}. Then, we have $(\BB, \GG) \sim_K (\tilde{\BB}, \tilde{\GG})$. Also, because $\aaa \sim_K \mathfrak{S}(\aaa)$, $\tilde{\pi}_{\mathfrak{S}(\aaa)} = \pi_{\aaa}$ implies $\tilde{\bo\pi} \sim_K \bo\pi$ and the proof is complete.
\end{proof}

We finally present the postponed proof of Lemmas \ref{lem:kruskal gid} and \ref{lem:sign flip}. While the main proof idea of \cref{lem:kruskal gid} is similar to that of Theorem 2 in \cite{lee2024new}, we provide a detailed proof for the sake of completeness.

\begin{proof}[Proof of \cref{lem:kruskal gid}]
For a matrix $\NN$, let $rk_k(\NN)$ be the Kruskal column-rank of $\NN$, that is, the largest integer $r$ such that any $r$ columns of $\NN$ are linearly independent.
We claim that it suffices to show that 
\begin{align}\label{eq:Kruskal rank checking}
    rk_k(\NN_1) = 2^K, ~rk_k(\NN_2) = 2^K, ~rk_k(\NN_3) \ge 2
\end{align}
for generic parameters in $\Omega(\TT; \GG) \setminus \mathcal{N}$. Assuming this, one can apply Kruskal's Theorem \citep{kruskal1977three}, which guarantees the uniqueness of the three-way tensor decomposition of $\mathbf{P}_0$ up to a universal column permutation and gives the desired result. Along the way, we show that the candidate set of non-identifiable parameters, $\mathcal{N}$, only imposes restrictions on $\BB$ and $\ggamma$ but not on $\bo\pi$. 

For the remainder of the proof, let $\bbeta_{I_1, I_2}$ denote the sub-matrix of the $J \times (K+1)$ coefficient matrix $\bbeta$ whose rows and columns are indexed by $I_1$ and $I_2$, respectively. Similarly, $\ggamma_{I_1}$ denotes the sub-vector of $\ggamma$ by collecting the entries indexed by $I_1$.

\paragraph{Proof of $rk_k(\NN_1) = 2^K$.} 
First, write the parameter space $$\Omega(\bbeta_{1:K, 0:K}, \ggamma_{1:K}; \GG_1) = \{\bbeta_{1:K, 0:K}, \ggamma_{1:K}: \beta_{j,k} \neq 0 \text{ for } g_{j,k} = 1\}$$ as a finite union of open, connected subsets of Euclidean space $\mathbb{R}^{\sum_{j,k \le K} g_{j,k}} \times {\mathbb{R}}^K$. Without loss of generality, let $$\Omega_{\text{positive}}(\bbeta_{1:K, 0:K}, \ggamma_{1:K}; \GG_1) := \{\bbeta_{1:K, 0:K}, \ggamma_{1:K}: \beta_{j,k} > 0 \text{ for } g_{j,k} = 1 \}$$ be our domain.
Here, we consider $\NN_1 = \NN_1(\bbb_{1:K,0:K}, \ggamma_{1:K})$ to be a matrix-valued function of $(\bbb_{1:K,0:K}, \ggamma_{1:K})$. 
Because $\NN_1$ has full column rank if and only if $\det (\NN_1^{\top} \NN_1) \neq 0$, it suffices to show that 
$$\{ \bbeta_{1:K, 0:K} \in \Omega_{\text{positive}}(\bbeta_{1:K, 0:K}, \ggamma_{1:K}; \GG_1) : \det (\NN_1^{\top} \NN_1) = 0 \}$$ is a measure-zero set in $\Omega(\bbeta_{1:K, 0:K}, \ggamma_{1:K}; \GG_1 )$. Note that the following argument also holds for other connected sub-domains of $\Omega(\bbeta_{1:K, 0:K}, \ggamma_{1:K}; \GG_1)$, and the union of a finite number of measure-zero sets are still measure-zero.

In particular, when $\GG_1 = \I_K$, the proof of Theorem 1 in \cite{lee2024new} showed that $\NN_1(\bbb_{1:K,0:K}, \ggamma_{1:K})$ always have full column rank. We use technical real analysis arguments to claim that this statement can be generalized to an arbitrary $\GG_1$ that satisfies $\diag(\GG_1) = \I_K$.
Consider the mapping
\begin{equation}\label{eq:determinent}
(\det (\NN_1^{\top} \NN_1)) (\bbb_{1:K,0:K}, \ggamma_{1:K}) : \Omega(\bbb_{1:K,0:K}, \ggamma_{1:K}; \GG_1) \rightarrow \mathbb{R}.
\end{equation}
Observe that $\det (\NN_1^{\top} \NN_1)$ defined in \eqref{eq:determinent} is a polynomial of entries of $\NN_1$. Each entry in $\NN_1$ can be written in the form of 
$$\NN_1((l_1, \ldots, l_K), \aaa) = \prod_{j=1}^K \int_{S_{l_j, j}} p_j(y_j ; g_j(\beta_{j,0}+\sum_k \beta_{j,k}\alpha_k), \gamma_j) m(dy_j).$$
Since we assume that the density $p_j$ and link function $g_j$ are analytic, each entry of $\NN_1$ is analytic. Consequently, $\det (\NN_1^{\top} \NN_1)$ is also analytic.

Next, note that $\det (\NN_1^{\top} \NN_1)$ cannot be identically zero, in other words, there exists $(\bbb^{\star}_{1:K,0:K}, \ggamma^{\star}_{1:K})$ such that $\det (\NN_1^{\top} \NN_1)(\bbb^{\star}_{1:K,0:K}, \ggamma^{\star}_{1:K}) \neq 0$. 
This is because one can make small perturbations from a parameter value in $\Omega(\bbb_{1:K,0:K}, \ggamma_{1:K}; \I_K)$ by setting $\beta_{j,k} = \epsilon$ for indices with $g_{j,k} = 1$ without making the determinant become zero.
Hence, by the following technical Lemma, we conclude that $$\{\bbb_{1:K,0:K}, \ggamma_{1:K} \in \Omega(\bbb_{1:K,0:K}, \ggamma_{1:K}; \GG_1): (\det (\NN_1^{\top} \NN_1))(\bbb_{1:K,0:K}, \ggamma_{1:K}) = 0 \}$$ is a null set in $\Omega(\bbb_{1:K,0:K}, \ggamma_{1:K}, \GG_1)$. The conclusion $rk_k(\NN_2) = 2^K$ automatically follows.

\begin{lemma}[\citet{mityagin2020zero}]
    Let $f: \Omega \rightarrow \mathbb R$ be a real analytic function defined on an open, connected domain $\Omega \in \mathbb R^d$ that is not identically zero. Then, $\{\omega \in \Omega: f(\omega) = 0\}$ is a measure-zero set in $\Omega$.
\end{lemma}

\paragraph{Proof of $rk_k(\NN_3) \ge 2$.}
Theorem 1 in \cite{lee2024new} showed that $rk_k(\NN_3) \ge 2$ under condition B. Hence, it suffices to show that condition B holds for generic parameters in $\Omega(\bbeta_{2K+1:J, 1:K}; \GG_3)$.
Fix any $\aaa \neq \aaa'$, and let $l = l(\aaa, \aaa')$ be an index among $[K]$ such that $\alpha_l \neq \alpha_l'$. Since we assume that all columns of $\GG_3$ are nonzero, there exists $j = j(\aaa, \aaa') > 2K$ such that $g_{j,l} = 1$. As $\beta_{j,l}(\alpha_l - \alpha_l') \neq 0$, 
$$\{\bbeta_{j,1:K} \in \Omega(\bbeta_{j,1:K}; \mathbf{g}_j): \sum_{k=1}^K \beta_{j,k} g_{j,k} (\alpha_k - \alpha_k') = 0 \}$$
is a measure-zero set in $\Omega(\bbeta_{j,1:K}; \mathbf{g}_j)$. Consequently, 
$$\Phi_{\aaa, \aaa'} := \{\bbeta_{2K+1:J, 1:K} \in \Omega(\bbeta_{2K+1:J, 1:K}; \GG_3) : \sum_{k=1}^K \beta_{j,k} g_{j(\aaa, \aaa'), k} (\alpha_k - \alpha_k') = 0 \}$$
is a measure-zero set in $\Omega(\bbeta_{2K+1:J, 1:K}; \GG_3)$. The proof is complete by taking a union over all $\aaa \neq \aaa'$.
\end{proof}

\begin{proof}[Proof of \cref{lem:sign flip}]
    We first assume that the true parameter $\BB$ belongs in the coefficient space 
    $$\Omega^g(\BB ; \GG) := \{\BB: \text{for each $j$, all elements of the form } \sum_{k \in H_j} \beta_{j,k} a_k, a_k = 0, 1, \text{ are distinct}\}.$$
    It is clear that the complement $\Omega(\BB ; \GG) \setminus \Omega^g(\BB ; \GG)$ can be spelled out as a finite union of lower dimensional subspaces, and has measure-zero with respect to $\Omega(\BB ; \GG)$. The following proof is somewhat technical, but the main idea is to show that one can entirely recover $\BB, \GG$, $\mathfrak{S}$ up to label switching, based on the \emph{values} of $\eta_{j,\aaa}$ while \emph{not using the indices} $\aaa$. 
    Steps 1 and 2 characterizes the coefficients $\BB$ based on the values $\{\eta_{j,\aaa}\}_{\aaa}$, which allows us to prove equivalence of $\BB, \GG$ and $\tilde{\BB}, \tilde{\GG}$ in step 3.

    \paragraph{Step 1: recovering $|H_j|$ and $|\BB|$.}
    We first claim that for each $j$, there uniquely exists an integer $I_j$, coefficients $c_{j,0}$, $c_{j,1} > \ldots > c_{j,I_j} > 0$ that satisfy
    \begin{align}\label{eq:eta_j decomposition}
        \eta_{j,\aaa} = \beta_{j,0} + \sum_{l \in H_j} \beta_{j,l} \alpha_l = c_{j,0} + \sum_{k \le I_j} c_{j,k} (\mathfrak{T}_j(\aaa))_k
    \end{align}
    for all $\aaa$ and some permutation $\mathfrak{T}_j \in S_{\{0,1\}^{K}}$.
    Here, we also prove that that $I_j = |H_j|$.
    
    The \emph{existence} directly follows as we can take $I_j = |H_j|$ and define $\{c_{j,k}: k \le I_j\}$ to be the coefficients $\{|\beta_{j,l}|: l \in H_j\}$ in decreasing order. Then, for each $k \le I_j$, there must be an $l_k \in H_j$ such that $c_{j,k} = |\beta_{j,l_k}| > 0$.
    Consequently, we can construct the permutation $\mathfrak{T}_j$ by setting $$(\mathfrak{T}_j(\aaa))_k := \begin{cases}
        \alpha_{l_k} & \text{if  } \beta_{j,l_k} > 0\\
        1 - \alpha_{l_k} & \text{if  } \beta_{j,l_k} < 0,
    \end{cases}$$
    for $k \le I_j$ and arbitrarily defining the remaining coordinates. By plugging in these definitions, we get
    \begin{align*}
        \sum_{k \le I_j} c_{j,k} (\mathfrak{T}_j(\aaa))_k &= \sum_{l \in H_j} \text{sgn}(\beta_{j,l}) |\beta_{j,l}| \alpha_{l} + \sum_{l \in H_j, \beta_{j,l}<0} |\beta_{j,l}| \\
        &= \sum_{l \in H_j} \beta_{j,l} \alpha_l - \sum_{l \in H_j, \beta_{j,l}<0} \beta_{j,l}.
    \end{align*}
    Hence, \eqref{eq:eta_j decomposition} holds for    $c_{j,0} := \beta_{j,0} + \sum_{l \in H_j, \beta_{j,l}<0} \beta_{j,l}$.

    For \emph{uniqueness}, we provide an inductive characterization of the $c_{j,k}$s based on the values $U_j := \{\eta_{j,\aaa}: \aaa \in \{0,1\}^K\}$. First, $c_{j,0}$ must be the minimum of the set $U_j$. Next, we define $c_{j,I_j}$ by noting that $c_{j,0} + c_{j,I_j}$ should be the smallest element in $\{\eta_{j,\aaa}: \eta_{j,\aaa} > c_{j,0}\}$. Inductively for each $k < I_j$, given $c_{j,0}$ and $c_{j,k+1}, \ldots, c_{j,I_j}$, $c_{j,0} + c_{j,k}$ must be the smallest element in $U_j \setminus \{c_{j,0}+\sum_{l > k} c_{j,l} a_l : a_l = 0, 1 \}$. We continue the induction until the set $U_j \setminus \{c_{j,0}+\sum_{l > k} c_{j,l} a_l : a_l = 0, 1 \}$ is empty. 
    By considering the true parameterization and the fact that $\BB \in \Omega^g(\BB;\GG)$, we have $|U_j| = 2^{|H_j|}$, and hence $|H_j| = I_j$.
    Additionally, this observation gives $c_{j,1} > \ldots > c_{j, |H_j|} > 0$.

    \paragraph{Step 2: Recovering the indexing $\aaa$.}
    While we have recovered the absolute values of the coefficients $|\beta_{j,k}|$ in Step 1, they are ordered based on their absolute values. Thus, the ordering is different across different $j$. In this step, we fully recover $\beta_{j,k}$ up to a column permutation $\tau \in S_{[K]}$, by representing each $\beta_{j,k}$ in terms of $\eta$ and $c$-values.
    To this end, we inductively construct sets $T_{j,k}, V_{j,k} \subseteq [2^K]$ for all $j \in [J]$ and $k \le |H_j|$. We claim that they correspond to ${\aaa}$s with identical ${\alpha}_{\sigma_j(k)}$ values, for some injective mapping $\sigma_j:[|H_j|] \to [K]$ (see \eqref{eq:t_jk v_jk} for a precise statement). This claim will be crucially utilized to recover the coefficients $\beta_{j,k}$.
    
    First, fix $j \in [J]$, $k \le |H_j|$, and define $T_{j,k}^{(1)} = \{\aaa: \eta_{j, \aaa} = c_{j,0}\}$ and $V_{j,k}^{(1)} = \{{\aaa}: \eta_{j,{\aaa}} = c_{j,0} + c_{j,k}\}$. Since there exists a unique index $l_k$ such that $|\beta_{j, l_k}| = c_{j,k} > 0$, these two sets have distinct $\alpha_{l}$ values if and only if $l = l_k$. Define $\sigma_j(k)$ as this $l_k$. Since $l_k \neq l_{k'}$ for $k \neq k'$, $\sigma_j$ is injective. Without the loss of generality, suppose $\beta_{j,l_k} > 0$. This is only for the sake of explicitly characterizing the sets $T_{j,k}$ and $V_{j,k}$, and does not affect their construction. Under this assumption, $\aaa \in T_{j,k}^{(1)}$ must satisfy $\alpha_{\sigma_j(k)} = 0$, $\aaa \in V_{j,k}^{(1)}$ must satisfy $\alpha_{\sigma_j(k)} = 1$. Also, as the values of $\alpha_l$ for $l \not \in H_j$ does not change the value of $\eta_{j,\aaa}$, we have $|T_{j,k}^{(1)}| = |V_{j,k}^{(1)}| = 2^{K-|H_j|}$.
    
    Inductively for $t > 1$, let
    \begin{align}
        \mathcal{T}_{j,k}^{[t]} &:= \{\aaa \not \in T_{j,k}^{(t-1)}, V_{j,k}^{(t-1)}: \eta_{j,\aaa} = \min_{\aaa' \not \in T_{j,k}^{(t-1)}, V_{j,k}^{(t-1)}} \eta_{j,\aaa'} \}, \label{eq:T_j,k} \\
        \mathcal{V}_{j,k}^{[t]} &:= \{\aaa \not \in T_{j,k}^{(t-1)}, V_{j,k}^{(t-1)}: \eta_{j,\aaa} = \min_{\aaa' \not \in T_{j,k}^{(t-1)}, V_{j,k}^{(t-1)}} \eta_{j,\aaa'} + c_{j,k} \}, \label{eq:V_j,k}
    \end{align}
    and define 
    $T_{j,k}^{[t]} = T_{j,k}^{(t-1)} \cup \mathcal{T}_{j,k}^{[t]}$ and $V_{j,k}^{[t]} = V_{j,k}^{(t-1)} \cup \mathcal{V}_{j,k}^{[t]}$. By the construction in \eqref{eq:T_j,k} and \eqref{eq:V_j,k}, again because $\beta_{j, \sigma_j(k)} > 0$,
    we must have $\alpha_{\sigma_j(k)} = 0$ for $\aaa \in \mathcal{T}_{j,k}^{[t]} $, and $\alpha_{\sigma_j(k)} = 1$ for $\aaa \in \mathcal{V}_{j,k}^{[t]} $. We also have $|\mathcal{T}_{j,k}^{[t]}| = |\mathcal{V}_{j,k}^{[t]}| = 2^{K-|H_j|}$, which inductively gives $|T_{j,k}^{[t]}| = |V_{j,k}^{[t]}| =  2^{K-|H_j|} t$.
    We continue the construction until $T_{j,k}^{[t]} \cup V_{j,k}^{[t]} = \{0,1\}^K$, that is when $t = 2^{|H_j|-1}$. 
    
    Finally, we define $T_{j,k} := T_{j,k}^{\left(2^{|H_j|-1}\right)}$ and $V_{j,k} := V_{j,k}^{\left(2^{|H_j|-1}\right)}$ as the final induction outputs. 
    Then, $\{T_{j,k}, V_{j,k}\}$ is a partition of $\{0,1\}^K$ with equal cardinality, where $T_{j,k} = \{\aaa: \alpha_{\sigma_j(k)} = 0\}$ and $V_{j,k} = \{\aaa: \alpha_{\sigma_j(k)} = 1\}$. In general, without the positivity assumption $\beta_{j,\sigma_j(k)} > 0$, we can conclude that 
    \begin{align}\label{eq:t_jk v_jk}
        \{T_{j,k}, V_{j,k}\} = \big\{\{\aaa: \alpha_{\sigma_j(k)} = 0\}, \{\aaa: \alpha_{\sigma_j(k)} = 1\} \big\},
    \end{align}
    as an \emph{unordered} set.
    As all columns of $\GG$ are not empty (see Assumption \ref{assmp:proportion} (b)), for each $l$, we must have at least one $j$ such that $g_{j,l} = 1$.
    Hence, the set of all possible partitions $\{\{T_{j,k}, V_{j,k}\}:j \in [J], k \in [|H_j|]\}$ must take exactly $K$ distinct values. Let us index each element by $\{T_l, V_l\}$ for ${l \in [K]}$, and let $\tau \in S_{[K]}$ be a permutation such that 
    \begin{align}\label{eq:t_l,v_l}
        \{T_l, V_l\} = \big\{\{\aaa: \alpha_{\tau(l)} = 0\}, \{\aaa: \alpha_{\tau(l)} = 1\} \big\}.
    \end{align}
    Without loss of generality, suppose that $T_l$ and $V_l$ are defined so that the mean of the vector $(\eta_{j,\aaa} : \aaa \in T_l)$ is strictly smaller than the mean of $(\eta_{j,\aaa} : \aaa \in V_l)$.
    Then, the monotonicity assumption Assumption \ref{assmp:proportion}(c) implies that
    \begin{align}\label{eq:t_l,v_l characterization}
        T_l = \{\aaa: \alpha_{\tau(l)} = 0\}, \quad V_l = \{\aaa: \alpha_{\tau(l)} = 1\},
    \end{align}
    and resolves the sign-flipping for the latent variable $A_{\tau(l)}$.  
    Next, to recover the parameter $\beta_{j,\tau(l)}$ for each $j$ and $l$, we claim that
    \begin{enumerate}[(a)]
        \item if $\{T_l, V_l\} \not\in \{\{T_{j,k}, V_{j,k}\}: k \le |H_j| \}$, $\beta_{j,\tau(l)} = 0$. 
        \item if $\{T_l, V_l\} = \{T_{j,k}, V_{j,k}\}$ for some $ k \le |H_j|$, $|\beta_{j,\tau(l)}| = c_{j,k}$ and $$\text{sgn}(\beta_{j,\tau(l)}) = \begin{cases}
            1 & \text{ if  } T_{l} = T_{j,k}, \\
            -1 & \text{ if  } T_{l} = V_{j,k}.
        \end{cases}$$
    \end{enumerate}
    Here, to show (a), suppose $\{T_l, V_l\} \not\in \{\{T_{j,k}, V_{j,k}\}: k \le |H_j| \}$. Then, by the characterization in \eqref{eq:t_jk v_jk} and \eqref{eq:t_l,v_l}, we must have $\sigma_j(k) \neq \tau(l)$ for all $k \le |H_j|$. Recalling that $\sigma_j$ collects all indices such that $|\beta_{j, \sigma_j(k)}| > 0$, we must have $\beta_{j,\tau(l)} = 0$. 
    To prove part (b), note that the assumption implies $\sigma_j(k) = \tau(l)$, so $|\beta_{j,\tau(l)}| = |\beta_{j,\sigma_j(k)}| = c_{j,k}$. The claim regarding the sign holds because $T_{j,k}$ has a smaller average of $\eta_{j, \cdot}$ values than $V_{j,k}$ by construction. Thus, $T_l = T_{j,k}$ if and only if $\beta_{j,\tau(l)} > 0$.
    
    \paragraph{Step 3: explicit representation of all parameters.}
    Finally, we prove the desired result by applying the same characterization of the parameters under the alternative values $\tilde{\eta}_{j,\tilde{\aaa}}$. Here, note that $\BB \in \Omega^g(\BB; \GG)$ implies that $\tilde{\BB} \in \Omega^g(\BB; \GG)$.
    We apply steps 1-2 using $\tilde{\eta}_{j,\tilde{\aaa}}$ instead of $\eta_{j,\aaa}$.
    Step 1 holds without any change, and we obtain the same coefficients $c_{j,k}$s as well as $|H_j| = |\tilde{H}_j|$.
    Next, in step 2, the definition of $T_{j,k}, V_{j,k}$ in \eqref{eq:T_j,k}, \eqref{eq:V_j,k} needs to be modified to be the collection of $\tilde{\aaa}$s instead of $\aaa$s. Let $\tilde{T}_{j,k}$ and $\tilde{V}_{j,k}$ be the corresponding sets under the alternative parametrization. As the coefficients $c_{j,k}$ are identical, $({T}_{j,k}, V_{j,k})$ and $(\tilde{T}_{j,k}, \tilde{V}_{j,k})$ are consistent in the sense that 
    \begin{align}\label{eq:t_jk tilde characterization}
        \tilde{T}_{j,k} = \mathfrak{S}({T}_{j,k}), \quad \tilde{V}_{j,k} = \mathfrak{S}({V}_{j,k}).
    \end{align}
    In other words, by viewing $T_{j,k}$ as the collection of \emph{indices} $\aaa$ in $\eta_{j,\aaa}$, $\tilde{T}_{j,k}$ and $T_{j,k}$ are the same collection of indices. 
    Thus, in terms of defining $\tilde{T}_l$ and $ \tilde{V}_l$, we can take the same indexing for $l \in [K]$ as in \eqref{eq:t_l,v_l} by letting $\tilde{T}_l := \mathfrak{S}(T_l)$ and $\tilde{V}_l := \mathfrak{S}(V_l)$. Hence, for some permutation $\tilde{\tau} \in S_{[K]}$, we must have 
    \begin{align}\label{eq:tilde t_l characterization}
        \tilde{T}_l &= \{\tilde{\aaa}: \tilde{\alpha}_{\tilde{\tau}(l)} = 0\} = \mathfrak{S}(\{\aaa: \alpha_{\tau(l)} = 0\}) = \mathfrak{S}(T_l), \\ 
        \quad \tilde{V}_l &= \{\tilde{\aaa}: \tilde{\alpha}_{\tilde{\tau}(l)} = 1\} = \mathfrak{S}(\{\aaa: \alpha_{{\tau}(l)} = 1\}) = \mathfrak{S}(V_l).
    \end{align}

    Now, we finish the proof by proving $\BB \sim \tilde{\BB}$ and $\aaa \sim \mathfrak{S}(\aaa)$. By taking $\sigma := \tilde{\tau} \cdot \tau^{-1}$ in the definition of the equivalence relations, it suffices to show $\beta_{j,\tau(l)} = \tilde{\beta}_{j, \tilde{\tau}(l)}$ for all $j,l$, and $\alpha_{\tau(l)} = (\mathfrak{S}(\aaa))_{\tilde{\tau}(l)}$ for all $\aaa, l$.
    The conclusion for $\beta_{j,\tau(l)}$ follows from its characterization in Step 2. For any $j, l$, the relationship \eqref{eq:t_jk tilde characterization} and \eqref{eq:tilde t_l characterization} shows that $\{T_l, V_l\} \in \{\{T_{j,k}, V_{j,k}\}: k \le |H_j| \}$ if and only if $\{\tilde{T}_l, \tilde{V}_l\} \in \{\{\tilde{T}_{j,k}, \tilde{V}_{j,k}\}: k \le |H_j| \}$. For case (a), we have $\beta_{j,\tau(l)} = \tilde{\beta}_{j,\tilde{\tau}(l)}  = 0$. For case (b), we have $|\beta_{j,\tau(l)}| = c_{j,k} = |\tilde{\beta}_{j,\tilde{\tau}(l)}|$ as well as $\text{sgn}(\beta_{j,\tau(l)}) = \text{sgn}(\tilde{\beta}_{j,\tilde{\tau}(l)})$, thus $\beta_{j,\tau(l)} = \tilde{\beta}_{j,\tilde{\tau}(l)}$.
    We next show the claim for $\alpha$. Without loss of generality, suppose $\alpha_{\tau(l)} = 0$. Then, by applying \eqref{eq:t_l,v_l characterization} and \eqref{eq:tilde t_l characterization}, we have $\mathfrak{S}(\aaa) \in \tilde{T}_l$ and $(\mathfrak{S}(\aaa))_{\tilde{\tau}(l)} = 0$. Similarly, for $\alpha_{\tau(l)} = 1$, we get $(\mathfrak{S}(\aaa))_{\tilde{\tau}(l)} = 1$ and the proof is complete.
\end{proof}

\subsection{Technical Conditions and Proof of Theorem \ref{thm:estimation consistency}}\label{subsec:proof of consistency}
We first state the technical assumptions regrading the penalty function $p_{\lambda_N,\tau_N}$ and tuning parameters $(\lambda_N,\tau_N)$ that are imposed for \cref{thm:estimation consistency}.
For some tuning parameters $\lambda_N, \tau_N > 0$, $p_{\lambda_N,\tau_N}:\mathbb{R}\to [0,\infty)$ is a sparsity-inducing symmetric penalty that is nondecreasing on $[0, \infty)$, nondifferentiable at $0$, differentiable at $(0, \tau_N)$, $p_{\lambda_N,\tau_N} \propto {\lambda_N}/{\tau_N}$ around $0$ and satisfy
\begin{align*}
    p_{\lambda_N,\tau_N} (b) = 0, ~\text{if} ~ b = 0; \quad
    p_{\lambda_N,\tau_N}' (b) \le \frac{C\lambda_N}{\tau_N},~ \text{if} ~ |b| \le \tau_N; \quad
    p_{\lambda_N,\tau_N} (b) = \lambda_N, \quad \text{if} ~ |b| \ge \tau_N.
\end{align*}
Note that $\lambda_N$ is the magnitude of the penalty, and $\tau_N$ is the point of truncation. We also assume that $\lambda_N$ and $\tau_N$ depend on $N$ such that ${1}/{\sqrt{N}} \ll \tau_N \ll {\lambda_N}/{\sqrt{N}} \ll 1.$ Here, for nonnegative sequences $\{(a_N, b_N)\}_{N\ge 1}$, we write $a_N \ll b_N$ when ${a_N}/{b_N} \to 0$ as $N \to \infty$.

\begin{proof}
For simplicity, we omit displaying the dependence of data in the likelihood $\ell(\TT) = \ell(\TT \mid \YY)$ and omit the subscript in the equivalence relation $\sim_{\mck}$.
We also define the objective function in \eqref{eq:penalized optimization} as 
$$Q_N(\TT) := \frac{1}{N} \Big[- \ell (\TT) + \sum_{d=1}^D p_{\lambda_N,\tau_N} (\B^{(d)}) \Big].$$
Below, we separately prove the consistency of the continuous parameters $\TT$ and discrete parameters $\mcg$. {In the proof, we use the usual Bachmann-Landau asymptotic notations for deterministic sequences as well as its analog for random sequences.}

We first show that $\hat{\TT}$ is consistent up to label permutations. This follows by modifying the usual consistency proof of M-estimators \citep{van2000asymptotic} under our identifiability notion. Fix $\epsilon > 0$, and define the pointwise limit of $Q_N(\TT)$ as 
$$Q_\infty (\TT) := - \E_{\TT^\star} \log \PP (\YY \mid \TT).$$
Define the $\epsilon$-ball around $\TT^\star$ under the equivalence relation $\sim$ as 
$$B_\sim (\epsilon, \TT^\star) := \{\TT: \lVert \TT' - \TT^\star \rVert \le \epsilon \text{ for some } \TT' \sim \TT\}.$$
Also, define the constant $\eta:= \inf_{\TT \not \in B_\sim (\epsilon, \TT^\star)} Q_\infty (\TT) - Q_\infty(\TT^\star)$. Since we assume the model identifiability up to $\sim$, we must have $Q_\infty (\TT) - Q_\infty(\TT^\star) = \text{KL}\big(\PP(\cdot \mid \TT) \| \PP(\cdot \mid \TT^\star) \big) > 0$ for all $\TT \not\sim \TT^\star$. Here, KL denotes the Kullback–Leibler divergence.
Then, we get $\eta > 0$, as we are considering a compact parameter space. 
By a standard argument, we have
\begin{align}
    \PP( \min_{\TT' \sim \hat{\TT}} &\lVert \TT'-\TT^\star \rVert > \epsilon) \le \PP\big(\min_{\TT \not \in B_\sim (\epsilon, \TT^\star)} Q_N(\TT) \le Q_N(\TT^\star) \big) \label{eq:1st line lemma} \\
    &\le \PP\big(\min_{\TT \not \in B_\sim (\epsilon, \TT^\star)} Q_N(\TT) \le Q_N(\TT^\star) , ~ \sup_{\TT} \big|Q_N(\TT) - Q_\infty(\TT) \big| < \frac{\eta}{2} \big) + o(1). \label{eq:2nd line lemma}
\end{align}
Here, \eqref{eq:1st line lemma} uses the definition that $\hat{\TT}$ is a minimizer of $Q_N(\TT)$.
The inequality in \eqref{eq:2nd line lemma} follows from the uniform law of large numbers, which holds under a compact parameter space and a vanishing penalty with $\lambda_N = o(N)$. Now, the proof is complete by noting that the first term in \eqref{eq:2nd line lemma} is zero, since $\TT \not \in B_\sim (\epsilon, \TT^\star)$ and $\sup_{\TT} |Q_N(\TT) - Q_\infty(\TT)| < \frac{\eta}{2}$ implies
$$ Q_N(\TT) >  Q_\infty(\TT) - \frac{\eta}{2} \ge Q_\infty(\TT^\star) + \frac{\eta}{2} > Q_N(\TT^\star).$$
Hence, there exists some $\tilde{\TT} \sim \hat{\TT}$ that is consistent for $\TT^\star$.

Now, we prove that $\tilde{\TT}$ is $\sqrt{N}$-consistent, additionally using the assumption on the Fisher information. We first re-write the inequality $Q_N (\tilde{\TT}) \le Q_N (\TT^\star)$ as
\begin{equation}\label{eq:Q inequality}
    -\ell (\tilde{\TT}) + \ell(\TT^\star) \le p_{\lambda_N,\tau_N}(\B_0) -p_{\lambda_N,\tau_N}(\tilde{\B}).
\end{equation}
By a Taylor expansion, we can bound the LHS of \eqref{eq:Q inequality} as
\begin{align*}
    -\ell (\tilde{\TT} ) + \ell(\TT^\star) &= - (\tilde{\TT} - \TT^\star)^{\top} \ell'(\TT^\star) + \frac{1}{2} (\tilde{\TT} - \TT^\star)^{\top} (NI(\TT^\star)+o_p(N)) (\tilde{\TT} - \TT^\star) \\
    &\ge \|\tilde{\TT} - \TT^\star\|_2 O_p(\sqrt{N}) + N \|\tilde{\TT} - \TT^\star\|_2^2 \left( \frac{\lambda_{\min} (I(\TT^\star))}{2} + o_p(1) \right).
\end{align*}
Here, $\lambda_{\min} (I(\TT^\star)) > 0$ denotes the smallest eigenvalue of the  positive definite Fisher information $I(\TT^\star)$. 
On the other hand, the RHS of \eqref{eq:Q inequality} must be negative since $\tau_N \rightarrow 0$. Indeed, for $\beta_{l,k}^{(d)\star} \neq 0$ and a large enough $N$, consistency gives $|\tilde{\beta}_{l,k}^{(d)}| = |\beta_{l,k}^{(d)\star}| + o_p(1) > \tau_N$, and the bound
\begin{align*}
    p_{\lambda_N,\tau_N}(\B_0^{(d)}) -p_{\lambda_N,\tau_N}(\tilde{\B}^{(d)}) &\le \sum_{l,k:\beta_{l,k}^{(d)\star} \neq 0} \left[ p_{\lambda_N,\tau_N} (\beta_{l,k}^{(d)\star}) - p_{\lambda_N,\tau_N} (\tilde{\beta}_{l,k}^{(d)}) \right] = 0
\end{align*}
holds with high probability.
Hence, we have
$$\|\tilde{\TT} - \TT^\star\|_2 O_p(\sqrt{N}) + \frac{\lambda_{\min} (I(\TT^\star))}{2} N \|\tilde{\TT} - \TT^\star\|_2^2 \le 0$$
with high probability,
which is impossible when $\|\tilde{\TT} - \TT^\star\|_2 \gg \frac{1}{\sqrt{N}}$. Thus, $\|\tilde{\TT} - \TT^\star\|_2 = O_p \left( \frac{1}{\sqrt{N}} \right)$.

Finally, we prove the estimation consistency for the discrete graph structures $\GG^{(d)}$s.
It suffices to show that for any fixed $d, l, k$, $\tilde{g}_{l,k}^{(d)} = g_{0; l,k}^{(d)}$ with high probability.
As a first case, suppose that $g_{0; l,k}^{(d)} = 1$. Then, the consistency result implies $\tilde{\beta}_{l,k}^{(d)} \xrightarrow{p} \beta_{0; l,k}^{(d)} \neq 0$. Hence, $\tilde{\beta}_{l,k}^{(d)} \neq 0$ with high probability, so $\tilde{g}_{l,k}^{(d)} = 1$.
Next, consider the case when $g_{0; l,k}^{(d)} = 1$. Assume the converse, and suppose $\tilde{\beta}_{l,k}^{(d)} \neq 0$. By the $\sqrt{N}$-consistency and the assumption that $\tau_N \gg \frac{1}{\sqrt{N}}$, we must have $|\tilde{\beta}_{l,k}^{(d)}| \ll \tau_N$ with high probability. Then, the first-order conditions (KKT conditions) give that
$$\partial_{\beta_{l,k}^{(d)}} \ell(\tilde{\TT}) := \frac{\partial \ell(\TT)}{\partial \beta_{l,k}^{(d)}} \mid_{\TT = \tilde{\TT}} = p'_{\lambda_N, \tau_N}(\tilde{\beta}_{l,k}^{(d)}) = \Theta_p \left(\frac{\lambda_N}{\tau_N} \right).$$
But we have a contradiction because a Taylor expansion of the partial derivative gives
$$\partial_{\beta_{l,k}^{(d)}} \ell (\tilde{\TT}) = \partial_{\beta_{l,k}^{(d)}} \ell (\TT^\star) + N O_p(\tilde{\TT} - \TT^\star) = O_p \left( \sqrt{N}\right),$$
and $\sqrt{N} \ll {\lambda_N}/{\tau_N}.$ Hence, we must have $\tilde{g}_{l,k}^{(d)} = 0$, and the proof is complete.
\end{proof}

\begin{remark}\label{rem:mixture}
    One natural question is to whether our estimator would still be consistent when the number of latent variables are unknown. This extension is not straightforward since the number of the top-layer latent variable, $K^{(D)}$, determines the number of deepest mixture components of DDEs. Estimating the number of mixture components is a challenging problem even in simple parametric models, and often leads to a slower (than ${1}/{\sqrt{N}}$) rate of convergence in parameter estimation \citep{goffinet1992testing, ho2016convergence}. Note that for such cases, the Fisher information becomes singular, and \cref{thm:estimation consistency} cannot be applied. 
\end{remark}

\subsection{\darkblue{Identifiability of Generalized DDEs with More Complex Latent Structures}}\label{sec:local dependent}

As pointed out by a reviewer, the current architecture of DDEs do not allow cross-level edges, which may limit the models' representational power. Here, we illustrate that DDEs can still be identified, even under the presence of cross-level dependencies. 

Our main idea is to re-formulate multi-layer graphical structures with \emph{cross-level edges} into a structure with \emph{within-layer edges}. See \cref{fig:local dependence} for a visual illustration; where the latent variable $A_1$ with cross-level edges (in the left panel) is moved to the lower latent layer (in the right panel). Based on this re-formulation, we instead establish identifiability of multi-layer latent structures where within-layer arrows are permitted. In other words, layer-wise local dependence is allowed. 

\begin{figure}[h!]
\centering
\begin{minipage}{0.45\textwidth}
\centering
\resizebox{\textwidth}{!}{
    \begin{tikzpicture}[scale=1.5]
\tikzset{
    node/.style={circle, draw, minimum size=1.2cm, inner sep = 0pt, align=center,}
    shaded/.style={circle, draw, fill=gray!30, minimum size=1.2cm, inner sep = 0pt, align=center},
}

    \node (hb1)[draw, line width=2pt] at (0,-1) {$A_{1}$};
    \node (hb2)[hidden] at (2,-1) {$A_{2}$};
    
       
    \node (v1)[hidden] at (-1,-2.2) {$A_3$};
    \node (v2)[hidden] at (1,-2.2) {$A_4$};
    \node (v3)[hidden] at (3,-2.2) {$A_5$};

    \node (vv1)[neuron] at (-1.8,-3.4) {$Y_1$};
    \node (vv2)[neuron] at (-1,-3.4) {$Y_2$};
    \node (vv3)[neuron] at (-0.2,-3.4) {$Y_3$};
    \node (vv4)[neuron] at (0.6,-3.4) {$Y_4$};
    \node (vv5)[neuron] at (1.4,-3.4) {$Y_5$};
    \node (vv6)[neuron] at (2.2,-3.4) {$Y_6$};
    \node (vv7)[neuron] at (3.0,-3.4) {$Y_7$};
    \node (vv8)[neuron] at (3.8,-3.4) {$Y_8$};

    \draw[arr, red] (hb1) -- (v1);
    \draw[arr] (hb2) -- (v1);
    \draw[arr] (hb2) -- (v3);
    \draw[arr] (hb2) -- (v2);
    \draw[arr] (hb2) -- (v3);
    \draw[arr] (v1) -- (vv3);
    \draw[arr] (v1) -- (vv4);
    \draw[arr] (v2) -- (vv4);
    \draw[arr] (v2) -- (vv5);
    \draw[arr] (v2) -- (vv6);
    \draw[arr] (v3) -- (vv7);
    \draw[arr] (v3) -- (vv8);
    
    \draw[->, blue,dashed, bend right=70, thick] (hb1) to (vv1);
    \draw[->, blue,dashed, thick] (hb1) to (vv2);
    \draw[->, blue,dashed, thick] (hb1) to (vv3);

\end{tikzpicture}
}
\end{minipage}
\quad
\begin{minipage}{0.45\textwidth}
\centering
\resizebox{\textwidth}{!}{
    \begin{tikzpicture}[scale=1.5]
\tikzset{
    node/.style={circle, draw, minimum size=1.2cm, inner sep = 0pt, align=center,}
    shaded/.style={circle, draw, fill=gray!30, minimum size=1.2cm, inner sep = 0pt, align=center},
}

    \node (hb1)[draw, line width=2pt] at (-1,-2.2) {$A_{1}$};
    \node (hb2)[hidden] at (1,-1) {$A_{2}$};
    
    \node at (4.6,-1) {$d=2$};
    \node at (4.6,-2.2) {$d=1$};
    \node at (4.6,-3.4) {$d=0$};
       
    \node (v1)[hidden] at (0,-2.2) {$A_3$};
    \node (v2)[hidden] at (1,-2.2) {$A_4$};
    \node (v3)[hidden] at (2,-2.2) {$A_5$};

    \node (vv1)[neuron] at (-1.8,-3.4) {$Y_1$};
    \node (vv2)[neuron] at (-1,-3.4) {$Y_2$};
    \node (vv3)[neuron] at (-0.2,-3.4) {$Y_3$};
    \node (vv4)[neuron] at (0.6,-3.4) {$Y_4$};
    \node (vv5)[neuron] at (1.4,-3.4) {$Y_5$};
    \node (vv6)[neuron] at (2.2,-3.4) {$Y_6$};
    \node (vv7)[neuron] at (3.0,-3.4) {$Y_7$};
    \node (vv8)[neuron] at (3.8,-3.4) {$Y_8$};

    \draw[arr,red] (hb1) -- (v1);
    \draw[arr] (hb2) -- (v1);
    \draw[arr] (hb2) -- (v3);
    \draw[arr] (hb2) -- (v2);
    \draw[arr] (hb2) -- (v3);
    \draw[arr] (v1) -- (vv3);
    \draw[arr] (v1) -- (vv4);
    \draw[arr] (v2) -- (vv4);
    \draw[arr] (v2) -- (vv5);
    \draw[arr] (v2) -- (vv6);
    \draw[arr] (v3) -- (vv7);
    \draw[arr] (v3) -- (vv8);
    
    \draw[->, blue,dashed, thick] (hb1) to (vv1);
    \draw[->, blue,dashed, thick] (hb1) to (vv2);
    \draw[->, blue,dashed, thick] (hb1) to (vv3);

\end{tikzpicture}
}
\end{minipage}
    \caption{\textbf{Left}: Example graphical structure without within-level edges, but allowing cross-level edges. \textbf{Right}: Equivalent formulation by moving $A_1$ to the middle layer; now we have within-level edges $\Lambda^{(1)} = \{A_1 \to A_3\}$ instead of cross-level edges.}
    \label{fig:local dependence}
\end{figure}

\paragraph{Generalized DDEs.}
To formalize each layer in such \emph{generalized DDEs}, we assume that this multi-layer latent structure is ordered in a manner such that $$\pa(\ma_{1:K^{(d)}}^{(d)}) \setminus \ma_{1:K^{(d)}}^{(d)} = \ma_{1:K^{(d+1)}}^{(d+1)}, \quad \forall 0 \le d \le D-1,$$
where we write $\ma^{(0)} = \YY$ for notational simplicity.
That is, the parents of variables in the $d$th layer must belong either in the current ($d$th) or directly upper ($d+1$th) layer. 

Assume that there are no edges $Y_j \to Y_{j'}$ in the observed (bottom) layer, in other words, assume that local dependence occurs only among the latent variables. For each latent layer (indexed by $d$), assume without loss of generality that $\ma^{(d)}$ are indexed in a topological manner, so that there exist no arrows of the form $A_k^{(d)} \to A_{k'}^{(d)}$ for $k \ge k'$. Let $\Lambda^{(d)}$ denote the graph on the vertex set $\ma^{(d)}$; for example, $\Lambda^{(1)} = \{A_1 \to A_3\}$ in \cref{fig:local dependence}.
Also, for each $A_k^{(d)}$, extend the conditional probabilities in \eqref{eq:latent conditional distribution} by including $\delta$-parameters to model the dependence of the same ($d$th) layer variables:
\begin{align}\label{eq:conditional probability local dependence}
    A_k^{(d)} \mid \ma_{1:(k-1)}^{(d)}, \ma^{(d+1)} \sim \text{Ber}\Big(\gl(\beta_{k,0}^{(d+1)} + \sum_{l=1}^{K^{(d+1)}} \beta_{k,l}^{(d+1)} A_l^{(d+1)} + \sum_{k'=1}^{k-1} \delta_{k',k}^{(d)} A_{k'}^{(d)}) \Big).
\end{align}
For simplicity, assume monotonicity of the conditional distributions in the sense that all coefficients $\beta_{k,l}^{(d+1)}, \delta_{k',k}^{(d)}$ are nonnegative.

In the following, we establish identifiability of generalized DDEs under the class of ``topologically double triangular'' models. This notion follows from a recent work \cite{lee2025identifiability}, where a single latent layer was considered and no arrows between the observed variables were allowed. Here, we extend the ``double triangular'' graphical matrices by additionally incorporating the edges $\Lambda^{(d)}$ in each layer as follows.

\begin{definition}[Triangular graphical matrix]\label{def:double triangular}
    A $L \times K^{(d)}$ matrix $\GG_1^{(d)}$ with binary entries is ``triangular'' when it takes the following form:
    $$\GG_1^{(d)} = \begin{pmatrix}
        \GG_{1,1}^{(d)} \\ \GG_{1,2}^{(d)} \\ \vdots \\ \GG_{1,L}^{(d)}
    \end{pmatrix}, \quad \GG_{1,k}^{(d)} = \begin{pmatrix}
    0      & \cdots & 0      & * & \cdots & * \\
    0      & \cdots & 0      & \vdots & \ddots & \vdots \\
    0      & \cdots & 0      & * & \cdots & * \\
    \multicolumn{3}{c}{\underbrace{0 \quad \cdots \quad 0}_{k-1}} & 1 & \cdots & * \\
    \end{pmatrix}.$$
    Here, the starred entries are allowed to take any values 0 or 1, and each $\GG_{1,k}^{(d)}$ can be a row-vector.
    
    For each $d \ge 2$, we say that the $K^{(d-1)} \times K^{(d)}$ binary matrix $\GG^{(d)}$ is ``topologically double-triangular'' when there exists a topological ordering of $[K^{(d-1)}]$ (with respect to the graph $\Lambda^{(d-1)}$) such that:
    $$\GG^{(d)} = \begin{pmatrix}
        \GG_1^{(d)} \\ \GG_2^{(d)} \\ \GG_3^{(d)}
    \end{pmatrix},$$
    where 
    \begin{enumerate}[(i)]
        \item the matrices $\GG_1^{(d)}, \GG_2^{(d)}$ are triangular after individual column permutations,
        \item $\GG_3^{(d)}$ does not have empty columns,
        \item there exists no arrows between $\mathcal{I}_1, \mathcal{I}_2, \mathcal{I}_3$ in $\Lambda^{(d-1)}$ (i.e. $\mathcal{I}_1, \mathcal{I}_2, \mathcal{I}_3$ are $d$-separated given $\ma$). Here, each $\mathcal{I}_a$ denotes the set of row indices corresponding to $\GG_a^{(d)}$.
    \end{enumerate}
\end{definition}

We elaborate on the details of \cref{def:double triangular}. For simplicity, assume that each $\GG_{1,k}^{(d)}$ is a row vector $(\underbrace{0, \ldots, 0}_{k-1}, 1, \ldots, *)$. Then, $\GG_1^{(d)}$ simplifies to the upper triangular matrix $\begin{pmatrix}
    1 & * & \cdots & * \\
    0 & 1 & \cdots & * \\
    \vdots & \vdots & \ddots & \vdots \\
    0 & 0 & \cdots & 1 \\
\end{pmatrix}.$ Thus, the topologically double-triangular condition is requiring two such triangular matrices. Here, note that condition (i) requires $\GG_{1}^{(d)}, \GG_2^{(d)}$ to be a triangular matrix after arbitrary column permutations. For example, in \cref{fig:double triangular}, by letting 
$$\GG_{1}^{(d)} = \begin{pmatrix}
    1 & 1 \\ 0 & 1
\end{pmatrix}, \quad \GG_{2}^{(d)} = \begin{pmatrix}
    1 & 1 \\ 1 & 0
\end{pmatrix}$$
respectively denote the $(1,2)$nd and $(3,4)$th rows of $\GG^{(d)}$, both matrices are triangular. We also have $\GG_3^{(d)} = (1, 1)$, so condition (ii) as well as (iii) are also satisfied, and \cref{fig:double triangular} defines an topologically double-triangular graphical model.

\begin{figure}[h!]
\centering
\resizebox{0.35\textwidth}{!}{
    \begin{tikzpicture}[scale=2]

    \node (v1)[hidden] at (0, 0) {$A_1^{(d-1)}$};
    \node (v2)[hidden] at (0.8, 0) {$A_2^{(d-1)}$};
    \node (v3)[hidden] at (1.6, 0) {$A_3^{(d-1)}$};
    \node (v4)[hidden] at (2.4, 0) {$A_4^{(d-1)}$};
    \node (v5)[hidden] at (3.2, 0) {$A_5^{(d-1)}$};
       
    \node (h1)[hidden] at (1.07, 1) {$A_1^{(d)}$};
    \node (h2)[hidden] at (2.14, 1) {$A_2^{(d)}$};

    \draw[ultra thick, qedge, blue] (h1) -- (v1) node [midway,above=-0.12cm,sloped] {}; 
    \draw[ultra thick, qedge, blue] (h2) -- (v2) node [midway,above=-0.12cm,sloped] {};  
    \draw[ultra thick, qedge, red] (h1) -- (v3) node [midway,above=-0.12cm,sloped] {}; 
    \draw[ultra thick, qedge, red] (h2) -- (v3) node [midway,above=-0.12cm,sloped] {}; 
    \draw[thick, qedge, blue] (h2) -- (v1) node [midway,above=-0.12cm,sloped] {}; 
    \draw[thick, qedge, red] (h1) -- (v4) node [midway,above=-0.12cm,sloped] {};
    \draw[qedge] (h1) -- (v5) node [midway,above=-0.12cm,sloped] {}; 
    \draw[qedge] (h2) -- (v5) node [midway,above=-0.12cm,sloped] {}; [midway,above=-0.12cm,sloped] {};
    \draw[thick, qedge, green!70!black] (v1) -- (v2) node [midway,above=-0.12cm,sloped] {};
    \draw[thick, qedge, green!70!black] (v3) -- (v4) node [midway,above=-0.12cm,sloped] {};
\end{tikzpicture}
}
\caption{Graphical illustration of a topologically double-triangular structure, with $\mathcal{I}_1 = \{1,2\}, \mathcal{I}_2 = \{3,4\}, \mathcal{I}_3 = \{5\}$.}
\label{fig:double triangular}
\end{figure}

\begin{definition}
    A $D$-latent-layer generalized DDE is a statistical model defined by the conditional distributions in \eqref{eq:top layer}, \eqref{eq:conditional probability local dependence}, and \eqref{eq:observed exp fam}. 
\end{definition}

Now, we establish identifiability of generalized DDEs, under the class of topologically double triangular models.

\begin{theorem}\label{thm:local dependence}
    Consider a $D$-latent-layer generalized DDE where the number of latent variables per layer, $\mck$, is known. Among the class of topologically double-triangular generalized DDEs, the model is identifiable up to (i) Markov equivalence for tree-structures within each layer and (ii) label-switching, when the observed (bottom) layer satisfy conditions A and B (see \cref{thm:deep id}).
\end{theorem}

Here, the restriction to Markov equivalence is a fundamental ambiguity for identifying DAG models. For example, suppose that the true graph is $A_1^{(1)} \to A_2^{(1)} \to A_3^{(1)} \leftarrow A_1^{(2)}$, where the graph in the first latent layer is a tree. Then, this structure is indistinguishable from $A_1^{(1)} \leftarrow A_2^{(1)} \to A_3^{(1)} \leftarrow A_1^{(2)}$. Note that colliders with multiple parents (e.g. $A_3^{(1)}$) do not suffer from this ambiguity, thanks to the parametric assumption in \eqref{eq:conditional probability local dependence}.
Compared to identifiability results in the main text, we additionally impose structural restrictions on the class of alternative graphical models. In other words, both the true and alternative models are required to be topologically double-triangular.

\paragraph{Proof.}
To prove \cref{thm:local dependence}, it suffices to show the following proposition for the generalized DDE with one-latent layer. Then, \cref{thm:local dependence} follows by the exact same layer-wise identifiability argument as that in the proof of \cref{thm:deep id}.

\begin{proposition}\label{prop:local dependence}
    Consider a saturated generalized DDE with only one latent layer, and known $K$. Among the class of topologically double-triangular saturated generalized DDEs, the model is identifiable up to Markov equivalence.
\end{proposition}

We prove \cref{prop:local dependence} using the following key property for matrix ranks, whose proof is deferred to the end of the section.
\begin{lemma}\label{lem:local dependence}
    Suppose that $\GG_{\mathcal{I}}$ is triangular. Then, $\PP(Y_{\mathcal{I}} \mid \ma)$ has full column-rank.
\end{lemma}

\begin{proof}[Proof of \cref{prop:local dependence}]
    We use the notations from \cref{subsec:supp proof of one layer}, and denote the latent variables as $\ma$ and observed variables as $\YY$.
    Let $\cup_{a=1}^3 I_a = [J]$ denote the partition that corresponds to the double triangular structure in \cref{def:double triangular}. As we assume that each $\ma_{\mathcal{I}_a}^{(d)}$ are disconnected (or d-separated) by $\Lambda^{(d)}$, we have the conditional independence
    $$\YY_{I_1} \perp \YY_{I_2} \perp \YY_{I_3} \mid \ma.$$
    Thus, the tensor decomposition in \eqref{eq:tensor decomposition} still holds. We separate the proof into two steps below.
    
    \emph{Step 1: Kruskal's theorem.}
    \cref{lem:local dependence} ensures that $\NN_1, \NN_2$ in the tensor decomposition \eqref{eq:tensor decomposition} are full rank. Additionally, all $2^K$ columns in $\NN_3$ are distinct. To see this, note that for any $\aaa \neq \aaa' \in \{0,1\}^K$, there exists some $k$ such that $\alpha_k \neq \alpha_k'$. By the assumption that $\GG_3$ (corresponding to the rows indexed by $\mathcal{I}_3$) does not have empty columns, there exists some $j \in \mathcal{I}_3$ such that $A_k \to Y_j$. Then, as $\beta_{j,k} \neq 0$, $\PP(Y_j \mid \ma = \aaa) \neq \PP(Y_j \mid \ma = \aaa')$. Hence, the corresponding columns in $\NN_3$ are distinct.
    
    Hence, \eqref{eq:Kruskal rank checking} holds. Applying Kruskal's theorem shows that the decomposition \eqref{eq:tensor decomposition} is unique up to column permutations. In other words, $(\NN_a, \boldsymbol{\pi})$ can be recovered up to a common column permutation $\mathfrak{S} \in S_{\{0,1\}^K}$. For more details, see Step 1 in the proof of \cref{prop:sid}.

    \emph{Step 2: structuring the permutation $\mathfrak{S}$.}
    The above argument recovers $\PP(\YY \mid \ma)$ and $\PP(\ma)$ up to a common column permutation $\mathfrak{S}$ for the latent configurations $\ma$. Now, it suffices to structure $\mathfrak{S}$ up to the trivial label permutation ambiguities. Then, as the probability distribution $\PP(\YY, \ma)$ is faithful to the underlying DAG $\GG \cup \Lambda$, we can identify all components up to Markov equivalence. This allows us to construct the graphical structure up to the completed partially DAG (CPDAG). Directions can be assigned to the undirected edges using (i) the assumption that there are no directed edges of the type $A_k \to Y_j$, (ii) the linearity of the conditional distributions in \eqref{eq:conditional probability local dependence}, excluding the fundamental ambiguities regarding tree structures among $\YY$ in the observed graph $\Lambda$.
    
    Without loss of generality, let $Y_1, \ldots, Y_{J_1}$ be the observations corresponding to $\mathcal{I}_1$, that are indexed via topological ordering under true graph $\Lambda$. Index the latent variables $\ma$ according to the triangular graphical structure in \cref{def:double triangular}, and let ${j_k} := \min \{j \in \mathcal{I}_1: A_k \to Y_j\}$ denote the child of $A_k$ with the largest index. Note that we assume $\pa(\YY_{\mathcal{I}_1}) \supseteq \ma$, so $j_k$ is well-defined.

    We prove by backwards induction and show that we can recover the set $T_k = \{\aaa: \alpha_k = 1\}$ only from the conditional probability matrix $\NN_1$ and $\boldsymbol{\pi}$ (without the correct column labels). 
    For the base case with $k = K$, $Y_{j_K}$ has an unique latent parent $A_K$ (excluding the observed parents in $\YY_{1:j_K-1}$). Hence, for any $\mathbf{y}_{1:j_K} \in \{0,1\}^{j_K}$,
    $$\PP(Y_{j_K} = y_{j_K} \mid \ma = \aaa, \YY_{1:j_K-1} = \yy_{1:j_K-1}) = \PP(Y_{j_K} = y_{j_K} \mid A_K = \alpha_K, \YY_{1:j_K-1} = \yy_{1:j_K-1}).$$
    By fixing $\mathbf{y} = \mathbf{0}$ and considering $\aaa \in \{0,1\}^K$, the above takes exactly two values, depending on $\alpha_K = 0/1$. As the LHS of the above display can be computed from $\NN_1$ and $\boldsymbol{\pi}$, $T_K$ must correspond to the column indices of $\NN_1$ with the larger value. Here, we have used the monotonicity assumption that $\beta_{j_K, K} > 0$.

    Next, we recover $T_k$, assuming that we are given $T_l$ for all $k < l \le K$. We similarly proceed by noting that $Y_{j_k}$ must have $A_{k}$ as a latent parent, and other latent parents are a subset of $\ma_{k+1:K}.$ This implies that for any $\mathbf{y}_{1:j_k} \in \{0,1\}^{j_k}$,
    $$\PP(Y_{j_k} = y_{j_k} \mid \ma = \aaa, \YY_{1:j_k-1} = \yy_{1:j_k-1}) = \PP(Y_{j_k} = y_{j_k} \mid \ma_{k:K} = \aaa_{k:K}, \YY_{1:j_k-1} = \yy_{1:j_k-1}).$$ 
    By fixing $\mathbf{y}_{1:j_k} = \mathbf{0}$ and considering any $\aaa \in \{0,1\}^K$ with $\aaa_{k+1:K} = \mathbf{0}$, the above takes exactly two values, depending on $\alpha_k = 0/1$. Similar as above, the LHS can be computed from $\NN_1, \boldsymbol{\pi}$, and hence $T_k$ is recovered by selecting the column indices of $\NN_1$ corresponding to a larger value. This completes the induction, and we have recovered all $T_k$s. This completes the proof.
\end{proof}

We finally prove \cref{lem:local dependence} using the following key property, which allows us to ignore redundant observables within $\YY_\mathcal{I}$.
\begin{lemma}\label{lem:rank}
    For any $\mathcal{I} \subset \mathcal{I}' \subseteq [J]$
    we have 
    $\rank(\PP(\YY_{\mathcal{I}} \mid \ma)) \le \rank(\PP(\YY_{\mathcal{I}'} \mid \ma))$.
\end{lemma}
\noindent The proof directly follows by the fact that $\rank(\mathbf{M} \mathbf{N}) \le \rank(\mathbf{N})$ for compatible matrices, as one can recover $\PP(\YY_{\mathcal{I}} \mid \ma)$ by appropriately marginalizing out rows in $\PP(\YY_{\mathcal{I}'} \mid \ma)$.

\begin{proof}[Proof of \cref{lem:local dependence}]
    For notational convenience, let $\mathcal{I} = \{1, \ldots, J\}$ and suppose that $\YY_\mci$ is indexed via topological ordering so that $Y_j \to Y_{j'}$ if and only if $j \le j'$. We proceed in an induction on $K$. First, the case when $K = 1$ is straightforward.

    Assuming that the statement holds when there are $K-1$ latent variables, we show the claim for models with $K$ latent variables. Suppose that we have observed variables $Y_1, \ldots, Y_J$ and latent variables $A_1, \ldots, A_{K}$. For each $k$, let $Y_{j_k}$ denote the minimum element in $\ch(A_k)$. 
    Index the latent variables $\ma$ so that $j_k < j_l$ for $k < l$. By \cref{lem:rank}, we can assume without loss of generality that $j_1 = 1$, since ignoring the observations $Y_j$s for $j < j_1$ do not increase the rank of $\PP(\YY \mid \ma)$.
    
    We show that the $2^{J} \times 2^{K}$ matrix $\mathbf{P} := \PP(Y_1, \ldots, Y_{J} \mid A_1, \ldots, A_{K})$ has full column rank (of $2^{K}$). Fix the row/column ordering of $\mathbf{P}$ by mapping each binary configuration $(a_1, \ldots, a_K)$ to an integer $\sum_{k=1}^K a_k 2^{k-1}$ and assume that this integer is increasing (e.g. $(0,0) < (1,0) < (0,1) < (1,1)$). Let $\theta_{\aaa} := \PP(Y_{1} = 1 \mid \ma = \aaa)$ denote the conditional probability of $Y_{1}=1$ given its parents. Also, define conditional probability matrices 
    $$\mathbf{M}_0 := \PP(\YY_{2:J} \mid \ma_{2:K} = \aaa_{2:K}, Y_1 = 0), \quad \mathbf{M}_1 := \PP(\YY_{2:J} \mid \ma_{2:K} = \aaa_{2:K}, Y_1 = 1).$$
    
    Our main observation is that, under the topologically triangular assumption, we have the conditional independence $A_1 \perp \YY_{2:J} \mid Y_1, \ma_{2:K}$. Using this, we have 
    $$\PP(\YY \mid \ma) = \PP(Y_1 \mid \ma) \PP(\YY_{2:J} \mid Y_1, \ma_{2:K}),$$
    which gives the following decomposition of $\mathbf{P}$:
    \begin{eqnarray*}
        &\mathbf{P}\Big((0, \yy), (0, \aaa)\Big) = \mathbf{M}_0(\yy, \aaa) (1-\theta_{0,\aaa}), \quad &\mathbf{P}\Big((0, \yy), (1, \aaa)\Big) = \mathbf{M}_0(\yy, \aaa) (1-\theta_{1, \aaa}), \\
        &\mathbf{P}\Big((1, \yy), (0, \aaa)\Big) = \mathbf{M}_1(\yy, \aaa) \theta_{0,\aaa}, \quad &\mathbf{P}\Big((1, \yy), (1, \aaa)\Big) = \mathbf{M}_1(\yy, \aaa) \theta_{1, \aaa}.
    \end{eqnarray*}

    Let $\mathbf{p}_{(1,\aaa)}$ and $\mathbf{p}_{(0,\aaa)}$ denote the corresponding columns in $\mathbf{P}$. Let $\bar{\mathbf{P}}$ denote the matrix by performing the following column operation on $\mathbf{P}$: for each $\aaa \in \{0,1\}^{K-1}$, subtract $\frac{1-\theta_{1,\aaa}}{1-\theta_{0,\aaa}}\mathbf{p}_{(0,\aaa)}$ from $\mathbf{p}_{(1,\aaa)}$.
    Then, we have
    \begin{eqnarray*}
        &\bar{\mathbf{P}}\Big((0, \yy), (0, \aaa)\Big) = \mathbf{M}_0(\yy, \aaa) (1-\theta_{0,\aaa}), \quad &\bar{\mathbf{P}}\Big((0, \yy), (1, \aaa)\Big) = 0, \\
        &\bar{\mathbf{P}}\Big((1, \yy), (0, \aaa)\Big) = \mathbf{M}_1(\yy, \aaa) \theta_{0,\aaa}, \quad &\bar{\mathbf{P}}\Big((1, \yy), (1, \aaa)\Big) = \mathbf{M}_1(\yy, \aaa) \frac{\theta_{1,\aaa} - \theta_{0,\aaa}}{1-\theta_{0,\aaa}}.
    \end{eqnarray*}
    Since $\rank(\mathbf{P}) = \rank(\bar{\mathbf{P}})$ and viewing $\bar{\mathbf{P}}$ as a $2 \times 2$ block matrix, it suffices to show that the conditional probability matrices $\mathbf{M}_0, \mathbf{M}_1$ have full column rank. 
    Here, we use the fact that $\theta_{1,\aaa} - \theta_{0,\aaa} \neq 0$ for all $\aaa \in \{0,1\}^{K-1}$ under our parametrization for $Y_1 \mid \ma$ (see \eqref{eq:conditional probability local dependence}), since the triangular assumption gives $g_{1, 1} = 1$ (in other words, $A_1 \to Y_1$). 
    
    Finally, the full rankness of $\mathbf{M}_0, \mathbf{M}_1$ follows from the induction hypothesis. To see this, use \eqref{eq:conditional probability local dependence} to spell out the conditional distribution of $Y_2 \mid \ma_{\ma_{2:K}}, Y_1$:
    \begin{align*}
        \PP(Y_2 = 1 \mid \ma_{{2:K}}, Y_1 = 0) &= \gl(\beta_{2,0} + \sum_{k=2}^K \beta_{2,k} A_k), \\
        \PP(Y_2 = 1 \mid \ma_{{2:K}}, Y_1 = 1) &= \gl(\beta_{2,0} + \delta_{1,2} + \sum_{k=2}^K \beta_{2,k} A_k).
    \end{align*}
    Each conditional distributions above can be viewed as that arising from separate reduced models with $K-1$ latent variables $\ma_{2:K}$, where the second equation considers a combined intercept parameter $\beta_{2,0} + \delta_{1,2}$. The same logic applies to all conditional distributions $Y_k \mid \ma_{k:K}, Y_{2:k-1}, Y_1$, allowing us to use the induction hypothesis.
\end{proof}

\subsection{Additional Identifiability Results}
\label{subsec:additional identifiability}
\paragraph{\darkblue{Identifying the latent dimension.}}
Our main identifiability results in the main theorems have assumed that the latent dimension $\mck$ is known. Here, we illustrate that one can additionally establish the identifiability of $\mck$ under a weaker notion of identifiability. To elaborate, we identify $\mck$ under the class of DDEs that satisfy the two-pure-children condition A.
\begin{theorem}[Modification of Theorem 1 in \cite{lee2025identifiability}]\label{thm:select K}
    Assume a $D$-latent layer DDEs satisfying the two-pure-children condition A (see Theorem 3.1), where $D$ is given but $\mck$ is unknown. Then, the number of latent variables $\mck$ is identifiable.
\end{theorem}

The above result follows directly from a more general claim from \cite{lee2025identifiability} alongside the layerwise identifiability in the proof of \cref{thm:deep id}.

\paragraph{Detour: Identifiability of one-latent-layer saturated models with interaction effects of binary latent variables}
    One may ask whether the linear/additive parametrization $\beta_{j,0} + \sum_{k \in [K]} \beta_{j,k} A_k$ in the one-latent-layer saturated model is necessary for its identifiability. We next show that this is not the case, and prove that identifiability of $\GG$ can be established under two more flexible nonlinear (in terms of the dependence on $\ma$) parametric models commonly used in psychometrics, using the exact same conditions A and B.
    Here, to handle different parametrizations, define $\eta_{j,\aaa}$ to be the nonlinear parameter for $Y_j \mid \ma$: 
    $$Y_j \mid (\ma = \aaa) \sim \text{ParFam}_j (\eta_{j,\aaa}),$$
    and rewrite condition B as follows:
    \begin{enumerate}
        \item[B'.] For any $\aaa \neq \aaa'$, there exists $j > 2K$ such that $\eta_{j,\aaa} \neq \eta_{j,\aaa'}$.
    \end{enumerate}

    We first consider the \emph{ExpDINA model} \citep[see Definition 4 in][]{lee2024new}\footnote{While this model is originally named ``Exponential-family based'', this is not required for our identifiability conclusion. In other words, the family Parfam$_j$ in \eqref{eq:expdina} can be any paramteric family.}, which considers the following conjunctive form of $\eta_{j,\aaa}$:
    \begin{equation}\label{eq:expdina}
        Y_j \mid (\ma = \aaa) \sim \textnormal{ParFam}_j \Big(g_j \big(\beta_{j,0} + \beta_{j,1} \prod_{k=1}^K \alpha_k^{q_{j,k}}, \gamma_j \big) \Big).
    \end{equation}
    In other words, the conditional distribution has two possible parameter values based on whether $\prod_{k=1}^K \alpha_k^{q_{j,k}}=1$ (or equivalently, whether $\ma \succeq \mathbf{q}_j$).
    To resolve the sign flipping ambiguity, let as assume $\beta_{j,1} > 0$ for all $j$ instead of Assumption \ref{assmp:proportion}(c).
    Under this model, both parts of \cref{lem:H_j cardinality} does not hold since $|\mathcal{S}_j| = 2$ whenever $|H_j| \ge 1$. Hence, we focus on partitioning $\{0,1\}^K$ by grouping the binary patterns $\aaa$ that takes the same values of \eqref{eq:S_j expression}, instead of just looking at the cardinality of $\mathcal{S}_j$. We formally state this claim in the following Lemma. Consequently, this Lemma can be used in place of \cref{lem:H_j cardinality} to prove identifiability of the exploratory ExpDINA model. The proof is a direct modification of Step 2 above (first, use part (a) of \cref{lem:mono dina} to construct a permutation $\sigma \in S_{[K]}$ based on the first $K$ rows, and use part (b) to prove $\tilde{g}_{j,k} = g_{j,\sigma(k)}$ for the other rows), and we omit the details.

    \begin{lemma}\label{lem:mono dina}
    Consider an ExpDINA model. Fix any $j$ such that $|H_j| \ge 1$, and partition $\{0,1\}^K$ into two sets $T_j$ and $T_j^c$ based on the value of $\mathcal{S}_j = \{\mathbf{s}_j(\aaa)\}_{\aaa}$, so that $|T_j| \le |T_j^c|$. If $|T_j| = |T_j^c|$, we break the symmetry by additionally assuming that $\mathbf{1}_K \in T_j$. Then, the following holds.
    \begin{enumerate}[(a)]
        \item $T_j = \{\aaa: \aaa \succeq \mathbf{g}_j \}$ and  $|T_j| = 2^{K - |H_j|}.$
        \item Suppose that the first $K$ rows of $\GG$ is a row-permutation of $\I_K$, in other words, there exists $\sigma \in S_{[K]}$ such that $\mathbf{g}_{k} = \mathbf{e}_{\sigma(k)}$ for all $k$. Then, for any $j$, $T_j \subseteq T_k$ if and only if $g_{j,\sigma(k)} = 1$.
    \end{enumerate}
    \end{lemma}

    \begin{proof}
    \begin{enumerate}[(a)]
        \item Note that the parameter of the ExpDINA model takes two values, depending on the value of $\prod_{k} \alpha_k^{g_{j,k}} = \mathbf{1}(\aaa \succeq \mathbf{g}_j)$. This value is equal to one for the $2^{K-|H_j|}$ configurations of $\aaa$ such that $\aaa \succeq \mathbf{g}_j$, and these are exactly the elements of $T_j$.
        \item The ``if'' part is immediate since $g_{j,\sigma(k)} = 1$ implies $\mathbf{g}_j \succeq \mathbf{g}_k$. \\ For the ``only if'' part, we prove the contrapositive. Suppose $g_{j,\sigma(k)} = 0$. Then, $\aaa = \mathbf{1}_K$ and $\aaa' := (\alpha_1, \ldots, \alpha_{\sigma(k)-1}, 0, \alpha_{\sigma(k)+1}, \ldots, \alpha_K)$ have the same conditional distribution as in \eqref{eq:expdina}. Hence, $\aaa$ and $\aaa'$ must both belong in $T_j$ but $\aaa' \not \in T_k$, so $T_j \not \subseteq T_k$.
        \end{enumerate}
    \end{proof}

    Next, as a second example, we consider the flexible \emph{ExpGDM} \citep[see Definition 1 in the Supplementary Material of][]{lee2024new} and prove its identifiability under the same conditions A and B'. This model considers all possible linear and interaction effects between the latent variables:
    \begin{align*}
        Y_j \mid \ma \sim &\textnormal{ParFam} \Big( \eta_{j,\aaa}, ~\gamma_j
        \Big), \\
        \text{where} \quad \eta_{j, \aaa}= & \beta_{j,\varnothing} + {\sum}_{k=1}^{K} \beta_{j,k} \left\{q_{j,k} \alpha_{k}\right\}
        + {\sum}_{1\leq k_1 < k_2\leq K} \beta_{j, k_1 k_2} \left\{q_{j,k_1}\alpha_{k_1}\right\} \left\{q_{j,k_2}  \alpha_{k_2}\right\} 
        \\ \notag
        &\qquad
        + \cdots + \beta_{j,H_j} {\prod}_{k\in H_j} \left\{q_{j,k} \alpha_{k}\right\}.
    \end{align*}
    The above model incorporates all possible main and interaction effects of the parent latent variables in the conditional distribution of $Y_j \mid \ma$.
    It is clear that this model generalizes both the one-latent-layer saturated model in \cref{def:saturated model} and the ExpDINA model.
    To address the sign-flipping issue, we assume that 
    \begin{align}\label{eq:mono gdina}
        \eta_{j, \aaa} > \eta_{j, \aaa'} \quad \text{for} \quad \aaa \succeq \mathbf{g}_j, ~ \aaa' \not\succeq \mathbf{g}_j.
    \end{align}
    This is a stronger assumption compared to Assumption \ref{assmp:proportion}(c). We impose this modified monotonicity condition as Assumption \ref{assmp:proportion}(c) cannot resolve the sign-flipping ambiguity, as we illustrate this in the following example.
    \begin{example}
         Consider a toy setting of $J=K=2$ and $\GG = \I_2$ with $\beta_{1,1} = \beta_{2,2} = 1$. Suppose that there all intercepts and interaction effects are zero for each $j = 1,2$: $\beta_{j,0} = \beta_{j,12} = 0$. Consider a stronger monotonicity assumption that all main-effects are nonnegative, that is $\beta_{j,k} \ge 0$. Define an alternative model with $\tilde{\GG} = \begin{pmatrix}
        1 & 0 \\
        1 & 1
    \end{pmatrix}$ and positive main-effects: $\tilde{\beta}_{1,1} = 1, \tilde{\beta}_{2,1} = \tilde{\beta}_{2,2} = 1$ and no intercepts, but with a negative interaction effect $\tilde{\beta}_{2,12} = -2$. Then, the matrix $\{\eta_{j,\aaa}\}_{j,\aaa}$ and $\{\tilde{\eta}_{j,\aaa}\}_{j,\aaa}$ are identical up to column permutation, so these two distinct parameters are non-distinguishable.
    \end{example}
    Also assuming a monotone condition on the parametric family (see below), we can modify \cref{lem:mono dina} as follows. Consequently, under all these assumptions, the ExpGDM is also identifiable under conditions A and B'.
    \begin{definition}[Monotone family]
        We say that a parametric family $p(\cdot \mid \eta)$ is a \emph{monotone family} with a \emph{monotone set} $U$ when there exists a measurable set $U \subset \mathcal{Y}$ not depending on $\eta$ such that $\PP(Y \in U \mid \eta)$ (or $\PP(Y \in U \mid \eta, \gamma)$) is a strictly increasing function in $\eta$. 
    \end{definition}
    Note that all parametric families considered in this paper are monotone families. For example, we can take $U = \{1\}, \{1, 2, \ldots\}, (0, \infty)$ for the Bernoulli/Poisson/Normal distribution with mean $\eta$, respectively.
    
    \begin{lemma}
        Consider an ExpGDM. Suppose all $p_j$s are monotone families with monotone sets $U_j$, and the monotonicity condition \eqref{eq:mono gdina} holds. Let $T_j:= \{\aaa: \PP_{j, \aaa}(U_j) = \max_{\aaa'} \PP_{j,\aaa'}(U_j) \}$. Then, the same conclusions as in \cref{lem:mono dina} hold.
    \end{lemma}

    \begin{proof}
    \begin{enumerate}[(a)]
        \item Since we consider a monotone family, we can write $T_j = \{\aaa: \eta_{j,\aaa} = \max_{\aaa'} \eta_{j,\aaa'} \}$. Then, by assumption \eqref{eq:mono gdina}, $\aaa \not \in T_j$ for $\aaa \not\succeq \mathbf{g}_j$. Also, by the faithfulness assumption, we have $\aaa \in T_j$ for all $\aaa \succeq \mathbf{g}_j$.
        \item Assuming part (a), the argument in \cref{lem:mono dina} can be applied.
    \end{enumerate}
    \end{proof}

\section{Details of the Layerwise Double-SVD Initialization}\label{sec:supp initialization}

\subsection{Details of Algorithm \ref{algo-init}}
We describe full details of the layerwise double-SVD initialization in \cref{algo-init}.
We first consider the noiseless scenario to motivate the general procedure. The setting is that we are given a $N \times J$ matrix $\E [\YY]$, and wish to recover $\ma, \BB$. 
For simplicity, we consider the one-latent-layer saturated model, and omit the layer-wise superscript to simplify the notation. We also assume identical parametric families in the observed layer and a nonlinear function $\mu \circ g$.

The first step involves denoising the nonlinearity and rewriting the data matrix as a low-rank approximation. Recall that $\mu:H \to \mathbb{R}$ computes the mean of the observed-layer parametric family, $g:\mathbb{R} \to H$ is the link function, and define a function $\gt := \mu \circ g$.
Since $$\E (Y_{i,j} \mid \ma_i) = \mu(g(\beta_{j,0} + \sum_k \beta_{j,k} A_{i,k}, \gamma_j)) = \gt(\beta_{j,0} + \sum_k \beta_{j,k} A_{i,k}),$$
we have $$\ZZ:= \gt^{-1} (\E (\YY \mid \ma) ) = [\mathbf{1}_N, ~ \ma] \BB_1^{\top}.$$ 
Let $\ZZ_0$ be the centered version of $\ZZ$ so that the column sums are zero, in other words, ${z}_0(i,j) := {z}(i,j) - \frac{1}{N} \sum_{i'=1}^N {z}(i',j)$. Similarly, let $\ma_0$ be the column-centered version of $\ma$ with $A_0(i,k) := A_(i,k) - \frac{1}{N} \sum_{i'=1}^N A(i',l)$. Then, we have 
\begin{equation}\label{eq:centered spectral decomposition}
    \ZZ_0 = \ma_0 \BB^{\top}.
\end{equation}
This is a rank $K$ decomposition, and we can write the SVD of $\ZZ_0$ as $\ZZ_0 = \UU \bo\Sigma \VV^{\top}$. Here, $\UU, \bo\Sigma, \VV$ are $N \times K$, $K \times K$, $J \times K$ matrices, respectively.

The second step addresses the rotation invariance of the SVD by finding the sparse representation in \eqref{eq:centered spectral decomposition}.
Motivated by the sparsity of $\BB$, we perform the Varimax rotation on $\VV$. As Varimax finds a sparse rotation \citep[see][for a theoretical justification]{rohe2023vintage}, we expect it to find a rotation matrix $\mathbf{R}$ such that $\hat{\BB} := \VV \mathbf{R}$ has the same sparsity pattern as $\BB$ (and the graphical matrix $\GG$). Consequently, we can use $\hat{\BB}$ as a rough estimate for $\BB$, and the sparsity pattern of $\hat{\BB}$ to estimate $\GG$. The estimate for $\ma$ follows from solving \eqref{eq:centered spectral decomposition} for $\ma_0$, using the estimated $\hat{\BB}$. 

There are several subtleties to address when extending this procedure to the sample-based setting. One immediate issue is applying the inverse-link function $\gt^{-1}$. For discrete samples, observed data may take values in the boundary of the sample space, such as $Y_{j} = 0$ when $\mcy = \mathbb{N} \cup \{0\}$. This makes the inverse function not well-defined. To resolve this, we apply the double SVD-based procedure in \cite{zhang2020note} as mentioned in the main text, which denoises and truncates the data into a subset of the sample space. The final output of this procedure is the sample-version SVD for the matrix $\mathbf{Z}_0$ (see step 6 in \cref{algo-spectral-initialization}).

Another subtlety arises while estimating $\mathbf{A}$ and $\BB$ after rotating $\mathbf{V}$, since Varimax does not account for the scaling of the row/columns. Here, we exploit the discreteness of the latent variables in $\ma$ to rescale $\BB$. Using the decomposition \eqref{eq:centered spectral decomposition}, we can crudely estimate $\ma_0$ using the sparse Varimax output $\hat{\BB}$. While the estimates $\hat{\ma}_0$ also suffer from the same scaling issue, we can still estimate the \textit{binary} $\ma$ by noting that $A_0(i,k) < 0$ if and only if $A(i,k) = 0$. Finally, using the estimated $\hat{\ma}$, we re-estimate $\BB$ via \eqref{eq:centered spectral decomposition}, now with correct scaling; see steps 9-10 in \cref{algo-spectral-initialization}. 

The entire procedure is summarized in Algorithm \ref{algo-spectral-initialization}, where we also specify the choice of tuning parameters.

\begin{algorithm}[h!]
\caption{Spectral initialization for Algorithm \ref{algo-penalized}}
\label{algo-spectral-initialization}
\SetKwInOut{Input}{Input}
\SetKwInOut{Output}{Output}

\KwData{$\YY, \{K^{(d)}\}_{d}, D$, function $\tilde{g}$, truncation parameters $\epsilon = 10^{-4}, \delta = \frac{1}{2.5 \sqrt{J}}$}

\begin{enumerate}[1.]
    \item Apply SVD to $\YY$ and write $\YY = \UU \Sigma \VV^{\top}$, where $\Sigma = \diag(\sigma_i)$ and $\sigma_1 \ge \ldots \ge \sigma_J$.
    \item Let ${\YY}_{\tilde{K}^{(1)}} = \sum_{k=1}^{\tilde{K}^{(1)}} \sigma_k \bu_k \bv_k^{\top}$, where $\tilde{K}^{(1)} := \max\{K^{(1)}+1, \arg\max_k \{ \sigma_k \ge 1.01 \sqrt{N} \} \}$.
    \item Define $\hat{\YY}_{\tilde{K}^{(1)}}$ by truncating $\YY_{\tilde{K}^{(1)}}$ to the range of responses, at level $\epsilon$. See \cref{rmk:truncation} for details.
    
    
    \item Define $\hat{\ZZ}$ by letting $ \hat{z}(i,j) =  \tilde{g}^{-1}(\hat{y}_{\tilde{K}^{(1)}}(i,j))$.

    \item Let $\hat{\ZZ}_0$ be the centered version of $\hat{\ZZ}$, that is, $\hat{z}_0(i,j)= \hat{z}(i,j) - \frac{1}{N} \sum_{k=1}^N \hat{z}(k,j)$.

    \item Apply SVD to $\hat{\ZZ}_0$ and write its rank-$K^{(1)}$ approximation as $\hat{\ZZ}_0 \approx \hat{\UU} \hat{\Sigma} \hat{\VV}$.

    \item Let $\tilde{\VV}$ be the rotated version of $\hat{\VV}$ according to the Varimax criteria.

    \item Entrywise threshold $\tilde{\VV}$ at $\delta$ to induce sparsity, and flip the sign of each column so that all columns have positive mean.
    Let $\hat{\GG^1}$ be the estimated sparsity pattern. 

    \item Estimate the centered $\ma_0$ by $\hat{\ma}_0 := \hat{\ZZ}_0 \tilde{\VV} (\tilde{\VV}^{\top} \tilde{\VV})^{-1}$, and estimate $\ma$ by reading off the signs: $\hat{A}(i,k) = \mathbbm{1}(A_0(i,k) > 0).$ 

    \item Let $\hat{\ma}_{\text{long}} := [\mathbf{1}, \hat{\ma}]$.
    Estimate $\BB_1$ by $\hat{\BB}_1 := C_g ( ( \hat{\ma}^{\top}_{\text{long}} \hat{\ma}_{\text{long}})^{-1} \hat{\ma}_{\text{long}}^{\top} \hat{\ZZ}_0) \cdot \hat{\GG}^1$, where $\cdot$ is the element-wise product and $C_g$ is a positive constant that is defined in \cref{rmk:scaling constant}

    \item Let $\YY = \hat{\ma}$ and $g$ be the logistic function. Go back to Step 1 to estimate the next layer coefficient matrix. Continue until reaching the deepest layer.

    \item For the deepest layer, estimate $\pp$ by setting $\hat{\pp}_k := \frac{1}{N} \sum_{i=1}^N \hat{A}(i,k)$. 
\end{enumerate}
  
 \Output{$\hat{\pp}, \{\hat{\BB}^{(d)}\}_d$.}
\end{algorithm}

\begin{remark}[Truncating $\hat{\YY}$]\label{rmk:truncation}
    We explain more on the truncation details in Step 3 by considering specific response types. For Normal responses, the original sample space is $\mathbb{R}$ and the truncation (Steps 1-4 in \cref{algo-spectral-initialization}) may be omitted.
    For Binary responses, we set 
    $$\hat{y}_{K^{(1)}}(i,j) = \begin{cases}
        \epsilon, &\text{ if } y_{K^{(1)}}(i,j) = 0, \\
        1-\epsilon, &\text{ if } y_{K^{(1)}}(i,j) = 1.
    \end{cases}$$
    For Poisson responses, we set $$\hat{y}_{K^{(1)}}(i,j) = \begin{cases}
        \epsilon, &\text{ if } y_{K^{(1)}}(i,j) < \epsilon, \\
        y_{K^{(1)}}(i,j), &\text{ otherwise.}
    \end{cases}$$
    In terms of implementing the method, we follow the suggestions of \cite{zhang2020note} with $\epsilon = 10^{-4}$.
\end{remark}

\begin{remark}[Constant $C_g$]\label{rmk:scaling constant}
    In Step 10 of \cref{algo-spectral-initialization}, $C_g > 0$ is an artificial scaling constant that depends on the link function $\tilde{g}$. This is introduced to better adjust the scaling of $\BB$, as the nonlinear transform $\tilde{g}^{-1}$ leads to a potentially biased estimate for $\hat{\mathbf{Z}}_0$. This adjustment is unnecessary for Normal-based DDEs, where $g$ is the identity link, and in such cases, one can simply set $C_g = 1$. For the Bernoulli or Poisson-based DDE where $\tilde{g}$ is the logistic or exponential function, we choose $C_g = \frac{1}{2}$ as the scaling factor based on simulation results.
\end{remark}

\subsection{Estimation Accuracy Without the Spectral Initialization}\label{subsec:random initialization}
We illustrate the effectiveness of our spectral initialization for by comparing it with EM parameter estimates obtained under random initialization. We show that even for the low dimensional setting $(J, K^{(1)}, K^{(2)}) = (18,6,2)$, the overall model complexity may be too large for an EM algorithm with random initialization to converge to the global optimum. In contrast, the spectral initialization provides a reliable starting point. In \cref{tab:sim-random init}, we compare the accuracy of the PEM estimates under (a) random initialization and (b) spectral initialization. Here, we consider the same three parametric families as the main paper, the identifiable true parameter values $\mcb_s$ (see eq. \eqref{eq:true parameter sid}), and two sample sizes $N = 1000, 4000$. We set the random initialization as follows:
\begin{align*}
    p_k, B^{(1)}_{j,k}, B^{(2)}_{k,l} &\sim \text{Unif}(0, 1), \quad \mbox{for all } k \in [K^{(1)}], ~ j \in [J], ~ l \in [K^{(2)}], \\
    B^{(1)}_{j,0}, B^{(2)}_{k,0} &\sim \text{Unif}(-1, 0), \quad \gamma_j \sim \text{Unif}(0.5, 1.5),
\end{align*}
where Unif$(a,b)$ denotes the uniform distribution on the interval $(a,b)$.

\begin{table}[h!]
\begin{tabular}{cccccccccc}
\toprule
\multirow{2}{*}{ParFam} & & \multicolumn{2}{c}{Accuracy($\mcg$)} & \multicolumn{2}{c}{RMSE$(\TT)$} & \multicolumn{2}{c}{Time (s)} & \multicolumn{2}{c}{\# iterations} \\
\cmidrule{2-10}
 & Initialization \textbackslash{}$N$ & 1000 & 4000 & 1000 & 4000 & 1000 & 4000 & 1000 & 4000 \\
 \midrule
\multirow{2}{*}{Bernoulli} & Random &  0.617 & 0.547 & 1.37 & 1.32 & 23.6 & 48.4 & 20.1 & 27.11 \\
 & \textbf{Spectral} & 0.966 & 0.992 & 0.30 & 0.20 & 6.7 & 37.5 & 4.1 & 4.2 \\
\midrule
\multirow{2}{*}{Poisson} & Random & 0.743 & 0.725 & 1.47 & 1.49 & 20.4 &26.8 & 21.9 & 23.2 \\
 & \textbf{Spectral} & 0.999 & 1 & 0.16 & 0.08 & 3.6 & 6.4 & 4.4 & 4.0 \\
 \midrule
\multirow{2}{*}{Normal} & Random &  0.595 & 0.581 & 1.71 & 1.83 & 36.6 & 347.7 & 14.7 & 16.9  \\
 & \textbf{Spectral} & 0.996 & 1 & 0.13 & 0.06 & 1.2 & 3.0 & 4.1 & 4.4 \\
 \bottomrule
\end{tabular}
\caption{Accuracy measures for $\mcg$ and $\TT$, computation time and iterations for 2-layer DDE estimates under different initializations. For the Accuracy$(\mcg)$ column, larger is better. For the other columns, smaller is better.}
\label{tab:sim-random init}
\end{table}

The results in \cref{tab:sim-random init} clearly illustrates that random initialization does not effectively converge to the true parameter values. In contrast, the spectral initialization results in a significantly smaller estimator error as well as shorter time and smaller numbers of iterations, demonstrating its superior performance.

It may be worth mentioning that a significant proportion of local optimizers arise from boundary cases of the identifiability condition. For example, in the Normal case, the local optimizers emptied out one or more columns in the $\BB^{(1)}$ and $\BB^{(2)}$ matrices, or exhibit two linearly dependent columns. This corresponds to the set of parameters excluded by Assumption \ref{assmp:proportion}.

\section{Additional Algorithm Details}\label{sec:supp algo details}

\subsection{Penalized EM Algorithm}
This supplement presents the details of the standard penalized EM algorithm (PEM) that was mentioned in the main text, which does not use stochastic approximation. The PEM computes the penalized maximum likelihood estimator in \eqref{eq:penalized optimization} to estimate the potentially sparse coefficients $\BB^{(1)}$ and $\BB^{(2)}$.
The PEM is an iterative procedure that consists of an expectation step followed by a penalized maximization step \citep{green1990use, chen2015q}.
In the $(t+1)$th iteration, the E-step computes the expectation of the complete data penalized-log-likelihood
\begin{align}\notag
    \ell_c(\YY, \ma^{(1)}, \ma^{(2)}; \TT) - &\sum_{d=1}^2 p_{\lambda_N, \tau_N}(\BB^{(d)}) \notag
    =~ \sum_{i=1}^N \Bigg[\log \PP(\ma_i^{(2)}; \pp) + \log \PP(\ma_i^{(1)} \mid \ma_i^{(2)}; \BB^{(2)}) \\
    &+ \log \PP(\YY_i \mid \ma_i^{(1)}, \ma_i^{(2)}; \BB^{(1)}, \ggamma)\Bigg]  - \sum_{d=1}^2 p_{\lambda_N, \tau_N}(\BB^{(d)}). \notag
\end{align}
This requires calculating the conditional probability for each latent configuration using the previous parameter estimates; that is, calculating $\PP (\ma_i^{(1)} = \aaa^{(1)}, \ma_i^{(2)} = \aaa^{(2)} \mid \YY; \TT^{[t]})$ for all $i\in[N]$, $\aaa^{(1)}\in\{0,1\}^{K_1}$, and $\aaa^{(2)}\in\{0,1\}^{K_2}$. Here, the superscript $[t]$ denotes the $t$th iteration estimates.
In the M-step, we update the parameters by maximizing the expectation computed in the E step, which boils down to solving the following three maximizations:
\begin{align}
    \pp^{[t+1]}:= &\argmax_{\pp}\sum_{i=1}^N \E\left[\log \PP(\ma_i^{(2)}; \pp) ; \pp^{[t]}\right], \label{eq:EM maximization p}\\
    \BB^{(2),[t+1]}:= &\argmax_{\BB^{(2)}} \sum_{i=1}^N \E\left[\log \PP(\ma_i^{(1)} \mid \ma_i^{(2)}; \BB^{(2)}) ;  \BB^{(2),[t]} \right] - p_{\lambda_N,\tau_N}(\BB^{(2)}) , \label{eq:EM maximization B2}\\
    (\BB^{(1),[t+1]}, \ggamma^{[t+1]}):= &\argmax_{\BB^{(1)}, \ggamma} \sum_{i=1}^N \E\left[ \log \PP(\YY_i \mid \ma_i^{(1)}, \ma_i^{(2)} ); \BB^{(1),[t]}, \ggamma^{[t]})\right]- p_{\lambda_N,\tau_N}(\BB^{(1)}).\label{eq:EM maximization B1}
\end{align}
Here, the optimizations for each layer are separated; we update the top-layer proportion parameters $\pp$ in \eqref{eq:EM maximization p}, the middle latent layer coefficients $\BB^{(2)}$ in \eqref{eq:EM maximization B2}, and the bottom layer coefficients $(\BB^{(1)}, \ggamma)$ in \eqref{eq:EM maximization B1}.
Additionally, due to the conditional independence assumption in each layer, the maximizations can be further simplified into low-dimensional optimizations over each row of $\BB$. \cref{algo-penalized} summarizes the PEM algorithm. 

\begin{algorithm}[h!]
\caption{Penalized EM (PEM) algorithm for the two-latent-layer DDE}
\label{algo-penalized}
\setstretch{0.7}
\SetAlgoLined
\SetKwInOut{Input}{Input}
\SetKwInOut{Output}{Output}

\KwData{$\YY, \mck$, tuning parameters $\lambda_{N}, \tau_N$.}
Initialize $\TT^{[0]}$ using the layerwise initialization in \cref{algo-init}. 

 \While{log-likelihood has not converged}{
 In the $[t+1]$th iteration,
 
 \texttt{// E-step}
 
    \For{$(i, \aaa^{(1)}, \aa2) \in [N]\times \{0, 1 \}^{K^{(1)}}\times \{0, 1 \}^{K^{(2)}}$}{
\begin{align*}
&\varphi^{[t+1]}_{i, \aaone, \aa2} = \PP (\ma_i^{(1)} = \aaa^{(1)}, \ma_i^{(2)} = \aaa^{(2)} \mid \YY; \TT^{[t]})
\end{align*}
}

    \texttt{// M-step} \\
    update $\TT^{[t+1]} = (\pp^{[t+1]}, \BB^{(1), [t+1]}, \BB^{(2), [t+1]}, \ggamma^{[t+1]})$ by solving \eqref{eq:EM maximization p}-\eqref{eq:EM maximization B1}. 
  }
  Estimate $\GG$ based on the sparsity structure of $\hat{\BB}$ according to \eqref{eq:graphical matrix estimator}.\\
 \Output{Estimated continuous parameters $\hat{\TT}$ and graphical matrices $\hat{\GG}^{(1)}$, $\hat{\GG}^{(2)}$.}

\end{algorithm}

\subsection{Detailed Update Formulas for \cref{algo-saem}}

Continuing from Section 4, we describe \cref{algo-saem} under $D=2$ latent layers. The extension to more latent layers is straightforward and we omit the detailed updates for simplicity.

\paragraph{Simplified simulation step for the SAEM Algorithm.}
We display the complete conditionals for the simulation step in the $(t+1)$th iteration of the SAEM (see \cref{algo-saem}). First, we sample each $A_{i,l}^{(2),[t+1]}$ from the following distribution:
\begin{align*}
    \PP(A_{i,l}^{(2)}=\alpha_l^{(2)} \mid (-)) &\propto p_l^{[t] \alpha_l^{(2)}} (1-p_l^{[t]})^{1 - \alpha_l^{(2)}} \prod_{k=1}^{K^{(1)}} g_{\text{logistic}}(A_{i,k}^{(1), [t]}; e^{\eta_{k,\ma_i^{(2),[t]}}}) \\
    &\propto p_l^{[t] \alpha_l^{(2)}} (1-p_l^{[t]})^{1 - \alpha_l^{(2)}} \prod_{k=1}^{K^{(1)}} \frac{e^{A_{i,k}^{(1), [t]} \beta^{(2),[t]}_{k,l} \alpha_l^{(2)}}}{1 + e^{\eta_{k,\ma_i^{(2),[t]}}}}.
\end{align*}
Here, (-) denotes the samples/parameter values computed in the previous ($t$th) iteration, excluding the random variable of interest, $A_{i,l}^{(2), [t]}$. The notation 
$$\eta_{k,\ma_i^{(2),[t]}} := \beta_{k,0}^{(2),[t]} + \sum_{l' \neq l} \beta_{k,l'}^{(2),[t]} A_{i,l'}^{(2), [t]} + \beta_{k,l}^{(2),[t]} \alpha_l^{(2)}$$
denotes the linear combinations computed under $\TT^{[t]}, \ma^{(2),[t]}$. As $\alpha_l^{(2)} = 0/1$, sampling from this distribution is straightforward by computing the above expression.

Next, we sample each $A_{i,k}^{(1),[t+1]}$ similarly from the complete conditionals:
\begin{align*}
    \PP(A_{i,k}^{(1)} = \alpha_{k}^{(1)} \mid (-)) &\propto  e^{\alpha_{k}^{(1)} \eta_{k,\ma_i^{(2),[t]}}} \prod_{j=1}^J \PP \left(Y_{i,j} ; g(\eta_{j,\ma_i^{(1),[t]}}, \gamma_j^{[t]}) \right).
\end{align*}
Here, $\eta_{j,\ma_i^{(1),[t]}}$ is similarly defined as $ \beta_{j,0}^{(1),[t]} + \sum_{k' \neq k} \beta_{j,k'}^{(1),[t]} A_{i,k'}^{(1), [t]} + \beta_{j,k}^{(1),[t]}\alpha_{k}^{(1)}$.

\paragraph{Simplified M-step for the SAEM Algorithm.}
In the main paper, we have motivated the SAEM M-step in terms of the parameter $\BB^{(2)}$ (see \eqref{eq:SAEM maximization B2}). Here, for completeness, we present the fully expanded formulas for updating each parameter.

First, for each $l \in [K^{(2)}]$, we update $p_l$ in closed form as follows: 
     $$p_l^{[t+1]} := \frac{\sum_{i} A_{i,l}^{(2),[t+1]}}{N}.$$
Next, for $k \in [K^{(1)}]$, we update the $k$th row of $\BB^{(2)}$ (denoted as $\bbeta_{k}^{(2)}$) as follows:
\begin{align*}
    Q_k^{(2),[t+1]}(\bbeta_k^{(2)}) &:= 
    (1 - \theta_{t+1}) Q_k^{(2),[t]}(\bbeta_k^{(2)}) + \theta_{t+1} \sum_{i=1}^N \log \PP(A_{i,k}^{(1)} = A_{i,k}^{(1), [t+1]} \mid \ma_i^{(2)} = \ma_{i}^{(2), [t+1]}; \bbeta_k^{(2)}),\\
    \bbeta_k^{(2),[t+1]} &:= \argmax_{\bbeta_k^{(2)}} \left[Q_k^{(2),[t+1]}(\bbeta_k^{(2)}) - p_{\lambda_N,\tau_N}(\bbeta_k^{(2)}) \right],
\end{align*}
Finally, for each $j \in [J]$, we update the $j$th row of $\BB^{(1)}$ (denoted as $\bbeta_j^{(1)}$) and $\gamma_j$ (if it exists) as follows:
\begin{align}
    Q_j^{(1),[t+1]}(\bbeta_j^{(1)},\gamma_j) &:= 
    (1 - \theta_{t+1}) Q_j^{(1),[t]}(\bbeta_j^{(1)},\gamma_j) + \theta_{t+1} \sum_{i=1}^N \log \PP \left(Y_{i,j} \mid \ma_i^{(1)} = \ma_i^{(1),[t+1]}; \bbeta_j^{(1)},\ggamma_j \right), \notag \\
    (\bbeta_j^{(1),[t+1]},\gamma_j^{[t+1]}) &:= \argmax_{\bbeta_j^{(1)},\gamma_j} \left[Q_j^{(1),[t+1]}(\bbeta_j^{(1)},\gamma_j) - p_{\lambda_N,\tau_N}(\bbeta_j^{(1)}) \right]. \label{eq:m-step saem normal}
\end{align}

\subsection{Alternatives for Estimating the Number of Latent Variables}
\label{subsec:unknown K estimation}
Recall that we have proposed a spectral-ratio estimator to select the latent dimension $\mck$ in \cref{sec:select k}. In this supplement, we propose two alternative estimators motivated by popular methods for selecting the number of latent variables in other statistical problems \citep{shen1994convergence, melnykov2012initializing, chen2022determining}. The performance of these will be later assessed in a simulation study in \cref{subsec:estimate K}.

Continuing from the setup in \cref{sec:select k}, we address the estimation of $\mck$ under the two-latent-layer DDE. We first focus on the scenario where the number of deepest latent variables, $K^{(2)}$, is known, and our goal is to select $K^{(1)}$ from a candidate grid $\mathfrak{K}$. This assumption is often justified in real-world applications where prior knowledge of the number of latent labels or classes in the dataset is available. The objective is to uncover finer-grained latent structures that capture additional generative information beyond the known labels.

\begin{itemize}
    \item EBIC: We compute the extended BIC \citep[EBIC;][]{EBIC} for each $k \in \mathfrak{K}$, and select the model with the smallest EBIC: $$K^{(1)}:= \argmin_{k \in \mathfrak{K}} \text{EBIC}(k).$$ 
    We postpone the detailed formula of EBIC$(k)$ to the end of this subsection.
    Here, we consider EBIC instead of more traditional information criteria such as AIC and BIC since it is designed to handle large parameter spaces. 
    \item LRT: The class of models with $k \in \mathfrak{K}$ can be viewed as a nested set of regular parametric models, and one can apply the method of sieves to select $K^{(1)}$ \citep{shen1994convergence}. In other words, we start from the smallest $k$ and sequentially conduct level-$\alpha$ likelihood ratio tests for $H_0: K^{(1)} = k$ v.s. $H_1: K^{(1)} = k+1$ using the $\chi^2$ limiting distributions, until when we cannot reject the alternative; define $\hat{K}^{(1)}$ as this $k$.
    
\end{itemize}

Below is a discussion regarding the theoretical and computational properties of the three proposed estimators (spectral-ratio, EBIC, LRT). 
Theoretically, consistency of all three methods can be justified under varying assumptions. Additionally assuming that $K^{(1)}$ is identifiable,
the EBIC is known to be consistent even under a diverging $J$ \citep{EBIC}. The LRT estimator can be justified by standard arguments invoking Wilk's theorem \citep{wilks1938large}. Finally, under the asymptotic regime $N, J \to \infty$ and assuming that $\mu \circ g$ is linear, $\sigma_1, \ldots, \sigma_r$ is close to the top $r$ singular values of $\YY$ and standard eigenvalue perturbation arguments (Weyl's theorem) guarantee that the spectral ratio estimator for $K^{(1)}$ is consistent.

In terms of computation, the spectral-ratio estimator is the most desirable as it just requires computing the SVD of the denoised data matrix once. The other two methods (EBIC, LRT) require re-fitting the model (using \cref{algo-saem}) for each candidate value of $K^{(1)}$ as well as computing the likelihood. This limits their usage when the upper bound for $K^{(1)}$ is large. 
Furthermore, the model re-fitting issue for EBIC and LRT is amplified when the number of the top layer latent variables $K^{(2)}$ is also unknown, as the cardinality of the candidate set for $\mck = (K^{(1)}, K^{(2)})$ increases. On the other hand, the spectral-ratio estimates the number of latent variables in a layerwise manner, and is computationally efficient even under an unknown $K^{(2)}$. 

For analyzing real data, we recommend making the final selection of $\mck$ by also incorporating qualitative aspects of the data such as domain knowledge and interpretability.

\paragraph{Definition of EBIC.}
Consider a 2-latent layer DDE with $\mck = (K^{(1)}, K^{(2)})$ latent variables. This is a parametric model with $|\TT| = J(K^{(1)}+1) + K^{(1)}(K^{(2)}+1) + K^{(2)} + J$ (if there exists a dispersion parameter) or $J(K^{(1)}+1) + K^{(1)}(K^{(2)}+1) + K^{(2)}$ (otherwise) parameters. Here, we do not count the discrete parameters $\mcg$, as they are implied by the coefficients $\mcb$. Let $\hat{\TT}_{\mck}$ and $\ell_{\mck} (\hat{\TT}_{\mck}; \YY)$ respectively denote the MLE and the log-likelihood under the 2-layer DDE with $\mck$ parameters.
Also, let df$(\mck)$ be the number of non-zero parameters in $\hat{\TT}_{\mck}$. Then, following \cite{EBIC}, the EBIC objective in \cref{subsec:unknown K estimation} is defined as follows:
$$EBIC(\mck) := -2 \ell_{\mck} (\hat{\TT}_{\mck}; \YY) + \text{df}(\mck) \log N + 2 {|\TT|  \choose \text{df}(\mck)}.$$

\subsection{Implementation Details}
We elaborate on the choices made to implement the EM Algorithms \ref{algo-penalized} and \ref{algo-saem}.

\paragraph{Tuning parameter selection.} For practical implementation, the penalty function and the tuning parameters $\lambda_N, \tau_N$ must be specified. In this work, we consider the truncated Lasso penalty function (TLP) proposed by \cite{shen2012likelihood}, that is $p_{\lambda, \tau}(b) := \lambda \min(|b| , \tau)$. Preliminary simulations revealed that the results are very similar under other penalties such as SCAD.

Based on the consistency result in \cref{thm:estimation consistency}, we consider $\lambda_{N} = N^{1/4}, \tau_N = \max(3N^{-0.3}, 0.3)$ for the simulation studies. For real data analysis, tuning parameters were selected from the following grid:
$$\lambda_{1,N}, \lambda_{2,N} \in \{N^{1/8}, N^{2/8}, N^{3/8} \}, \quad \tau_N \in 2 \{N^{-1/8}, N^{-2/8}, N^{-3/8} \}.$$
Here, $\lambda_{a,N}$ ($a=1,2$) denotes the penalty size for the $a$th latent layer coefficient $\BB^{(a)}$. Note that distinct $\lambda_N$-values are used for each layer, which is to better accommodate the larger uncertainty in the deeper layer.

While there are many ways to select tuning parameters such as exact or approximate cross-validation and information criteria-based methods, we choose to use the latter in order to encourage sparsity. 
We select the model with the smallest extended Bayesian information criterion (EBIC). This approach is preferred over cross-validation in unsupervised settings, where likelihood-based loss functions often yield non-sparse, overfitted models \citep{chetverikov2021cross}.

Additionally, to implement the SAEM, the step size $\theta_t$ that determines the weight of the stochastic averaging needs to be specified. Here, following the standard Robbins-Monro condition \citep{robbins1951stochastic}, we set $\theta_t = 1/t$.

\paragraph{Convergence criteria.} For the penalized EM algorithm, we set the convergence criteria such that we terminate the algorithm when the difference between the log-likelihood function values of two consecutive iterations is less than $N/500$, as the log-likelihood is proportional to $N$. For the SAEM, since we do not directly compute the log-likelihood, convergence is declared when the difference of the vectorized $L^2$ norm of the continuous parameters is smaller than $K^{(2)}/2$. Under the spectral initialization, the PEM and SAEM generally took less than 10 and 5 iterations until convergence, respectively. 

In terms of implementing the PEM algorithm to select the number of latent variables, $K^{(1)}$ or $K^{(2)}$, we consider a more generous threshold of $3\times N/500$. This is for the sake of faster implementation, as we only use the resulting likelihood to implement the EBIC and LRT method.

\paragraph{M-step implementation.}
As our M-step in both algorithms (PEM and SAEM) boils down to solving multiple lower-dimensional maximization of dimensions less or equal to $K^{(1)}+1$, we choose to simply apply a built-in optimizing package that directly computes the global maximizers. This is because there are already other approximations being made in the SAEM, and we did not want to increase the randomness for our simulation studies. 

For the simulation studies under deeper models with $D \ge 3$ (see Section 5 in the main paper), we have slightly modified the M-step of the Normal observed-layer (see \eqref{eq:m-step saem normal}) by considering the hard-thresholding penalty $p_{\tau}(b) := \frac{\tau^2}{2} I(|b| \neq 0)$. This leads to a closed-form simplification of \eqref{eq:m-step saem normal}, as spelled-out below. Letting $\mathbf{y}_j$ denote the $j$th column of $\YY$, $\ma^{[t+1]}_{\text{long}} := (\mathbf{1}_N, \ma^{(1),[t+1]})$, and $\text{thres}_\tau(b) := b I(|b| > \tau)$, we can write the updates for each $\beta_j, \gamma_j$ as
$$\hat{\bbeta}_{j}^{(1),[t+1]} = \text{thres}_\tau\Big((\ma^{[t+1] \top}_{\text{long}} \ma^{[t+1]}_{\text{long}})^{-1} \ma^{[t+1] \top}_{\text{long}} \mathbf{y}_j \Big), \quad \hat{\gamma}_j^{[t+1]} = \sqrt{\frac{\sum_{i} (Y_{i,j} - \eta_{j,\ma_i^{(1),[t+1]}})^2}{N}},$$
which speeds up the SAEM algorithm.

For practical implementation for a larger dataset with general observed-layer distributions, we recommend the user to modify the M-step to a faster but approximate optimization procedure. For example, one may choose to do a one-step gradient ascent in each iteration of the M-step.

\paragraph{Addressing latent variable permutation for simulations.}
In our simulations, resolving the latent variable permutation is necessary to accurately compute the errors. A naive approach involves computing the error across all possible permutations, but this quickly becomes computationally infeasible, even for moderate values of $K^{(1)} = 18, K^{(2)} = 6$. To address this, we formulate the optimal latent variable permutation as an assignment problem and solve it efficiently using the following bottom-up approach as follows.

First, we construct a $K^{(1)} \times K^{(1)}$ cost matrix, where each entry is the squared $L^2$ norm between the corresponding column vectors of $\BB^{(1)}$ and $\hat{\BB}^{(1)}$. Next, use the Hungarian algorithm \citep{kuhn1955hungarian} to find a column permutation that minimizes the total assignment cost. As the column indices of $\BB^{(1)}$ correspond to the row indices of $\BB^{(2)}$, we permute the rows of $\BB^{(2)}$ accordingly. Finally, we apply the same procedure to find an optimal permutation of the $\BB^{(2)}$ columns. This method ensures computational efficiency while accurately resolving the latent variable permutation.

\section{Additional Simulation Results}\label{sec:supp simulation}
\subsection{True Parameter Values}
In terms of the true parameter values, we consider two sets of values based on the strict and generic identifiability conditions in Theorems \ref{thm:deep id} and \ref{thm:deep gid}, respectively. 
First, we define $\mcb_s = \{\BB_s^{(d)}\}_{d \in [D]}$ that satisfy \cref{thm:deep id} as follows:
\begin{equation}\label{eq:true parameter sid}
\BB_s^{(d)} = \begin{pmatrix}
-2 \mathbf{1}_{K^{(d)}} & 4 \mathbf I_{K^{(d)}} \\
-4 \mathbf{1}_{K^{(d)}} & 4 \mathbf I_{K^{(d)}} \\
-2 \mathbf{1}_{K^{(d)}} & \BB_{1}^{(d)} 
\end{pmatrix},\quad \text{where} \quad \beta_{1; j,k}^{(d)} := \begin{cases}
    4 & \text{if } k=j, \\
    4/3 & \text{ else if } |k-j| = K^{(d)}/2, \\
    0 & \text{otherwise}.
\end{cases} 
\end{equation}
We also consider $\mcb_g  = \{\BB_g^{(d)}\}_{d \in [D]}$ that satisfy \cref{thm:deep gid}, defined as:
\begin{equation}\label{eq:true parameter gid}
\begin{aligned}
    \BB_g^{(d)} = \begin{pmatrix}
    -2 \mathbf{1}_{K^{(d)}} & \BB_{2}^{(d)} \\
    -4 \mathbf{1}_{K^{(d)}} & \BB_{2}^{(d)\top} \\
    -2 \mathbf{1}_{K^{(d)}} & \BB_{1}^{(d)}
    \end{pmatrix}, \quad \text{where} \quad 
\beta_{2; j,k}^{(d)} &:= \begin{cases}
    4 & \text{if } k=j, \\
    4/3 & \text{ else if } 0 < k-j \le \lceil K^{(d)}/3 \rceil, \\
    0 & \text{otherwise},
\end{cases}
\end{aligned}
\end{equation}
and $\BB_1^{(d)}$ is the matrix defined in \cref{eq:true parameter sid}.
Note that $\BB_g^{(d)}$ is a less sparse version of $\BB_s^{(d)}$, where the identity matrices are modified to $\BB_2^{(d)}$.
While $\BB_2^{(d)}$ and $\mcb_g$ have the same value of 4 on the main diagonals, $\mcb_s$ is sparser. In particular, $\mcb_s$ has two pure children per latent variable whereas $\mcb_g$ has none. 

Regarding the remaining parameters, we set the top-layer proportion parameters as $p_{k} = 0.5$ for all $k \in [K^{(D)}]$, and set the dispersion parameters for the Normal distribution as $\gamma_j = \sigma_j^2 = 1$ for all $j \in [J]$.

\subsection{\darkblue{Additional Details on Ablation Studies}}
We provide additional details regarding ablation studies. We separately considered data generated from a DDE and a DBN with the same coefficients $\mathcal{B}_s$ in \eqref{eq:true parameter sid}. We implemented the DBN using the \texttt{Deep Neural Network} toolbox in \texttt{MATLAB} \citep{tanaka2014novel}.

In addition to the analysis in the main text, we compare the RMSE values for the continuous parameters under sparse, identifiable models (panel (d) in \cref{fig:ablation true DDE}) versus that under fully-connected, non-identifiable models (\cref{fig:nonidentifiable DBN} below). Interestingly, the RMSE values returned by the DBN algorithm are similar. This indicates the unreliability of parameter estimation via estimation algorithms for learning DBNs. In contrast, our proposed estimation method for DDEs encourages sparsity as well as adopts a wiser initialization strategy that is close to the true parameters, which leads to better parameter estimation and structure recovery.

\begin{figure}[h!]
    \centering
    \includegraphics[width=0.32\linewidth]{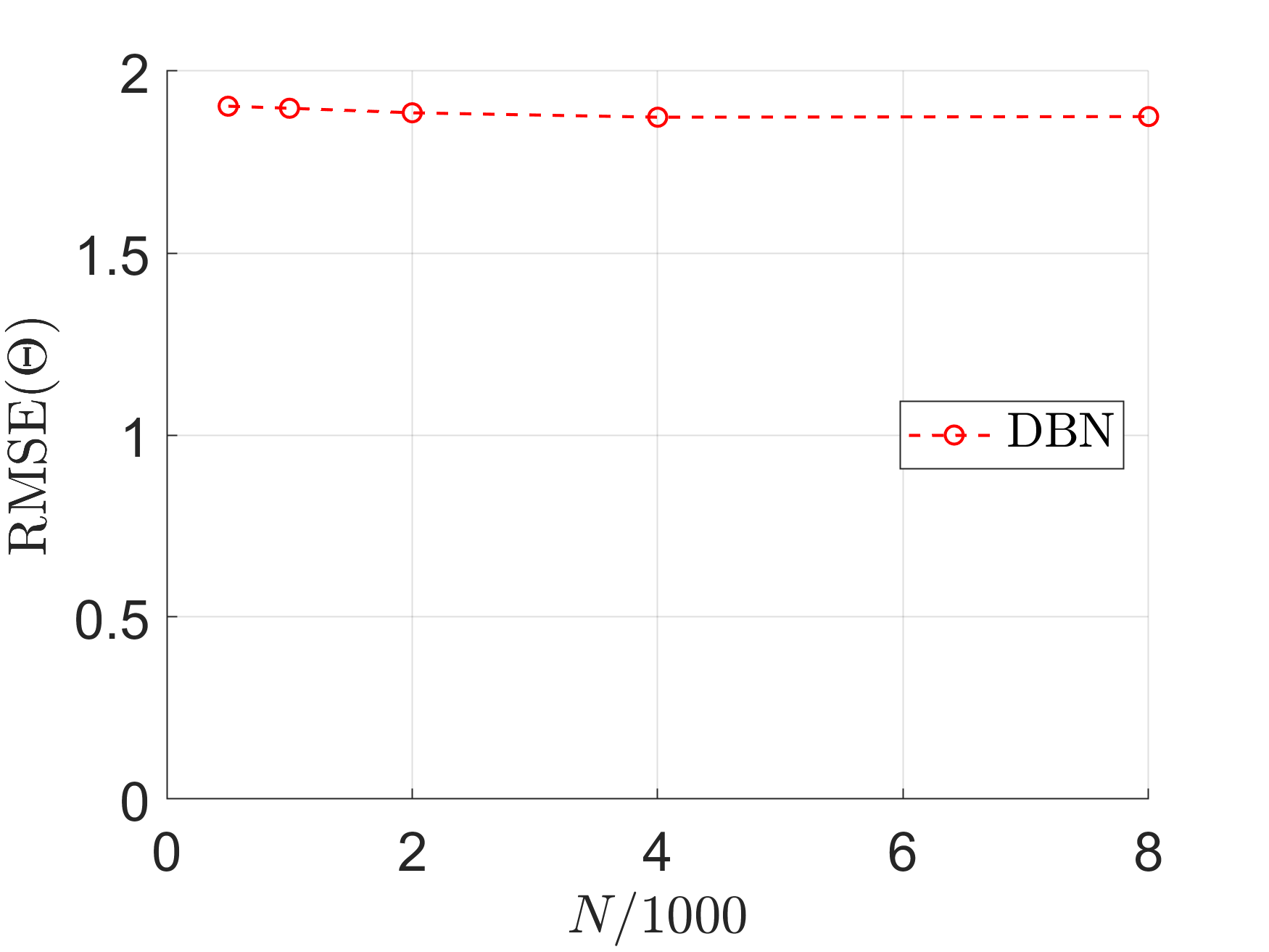}
    \caption{Parameter estimation accuracy for fully-connected DBNs with $J = K^{(1)} = K^{(2)} = 18$. The true coefficients are generated iid from $N(0, 2)$.}
    \label{fig:nonidentifiable DBN}
\end{figure}

\subsection{Simulation Results Under the Generic Identifiable Parameters $\mathcal{B}_g$ and Computation Time}\label{subsec:additional figures}

Here, we present the omitted details from \cref{sec:simulations} in the main paper, regarding (a) the estimation accuracy under generically identifiable parameters, (b) performance of the PEM algorithm, and (c) computation time.

\paragraph{Estimation accuracy under generically identifiable parameters.}
We present the analogs of Figure \ref{fig:acc g sid} under the generically identifiable true parameters $\mcb_g$. 
While all other simulation settings are identical to that described in the main text, we make the following two changes. First, we do not consider the largest parameter dimension of $(J, K^{(1)}, K^{(2)}) = (90, 30, 10)$ in this generically identifiable setting. This is because the layerwise initialization turned out to be less effective in this high-dimensional but less-sparse scenario.
Second, for Poisson-based-DDEs, we modify the intercept values to be smaller compared to the values in \eqref{eq:true parameter sid} as follows.
When $K^{(2)} = 2$, we consider a smaller intercept parameter for $\BB_g^{(1)}$ as follows:
$$\beta_{2; j,0}^{(1)} := \begin{cases}
    -3 & \text{ if } k \le K^{(1)}, \\
    -5 & \text{ else if } K^{(1)} < k \le 2K^{(1)}, \\
    -2 & \text{ otherwise}.\end{cases}$$
Similarly for $K^{(2)} =6$, we work under the smaller intercept values of
$$\beta_{2; j,0}^{(1)} := \begin{cases}
    -10 & \text{ if } \sum_{k=1}^{K^{(1)}} \beta_{j,k} \ge 8, \\
    -5 & \text{ otherwise}.\end{cases}$$
This adjustment is to make the Poisson parameters not explode, as using the original intercept values in \eqref{eq:true parameter gid} makes some Poisson parameters for $Y_j \mid \ma$ very large and generates unrealistic data. 

\begin{figure}[h!]
    \centering
    \includegraphics[width=0.9\linewidth]{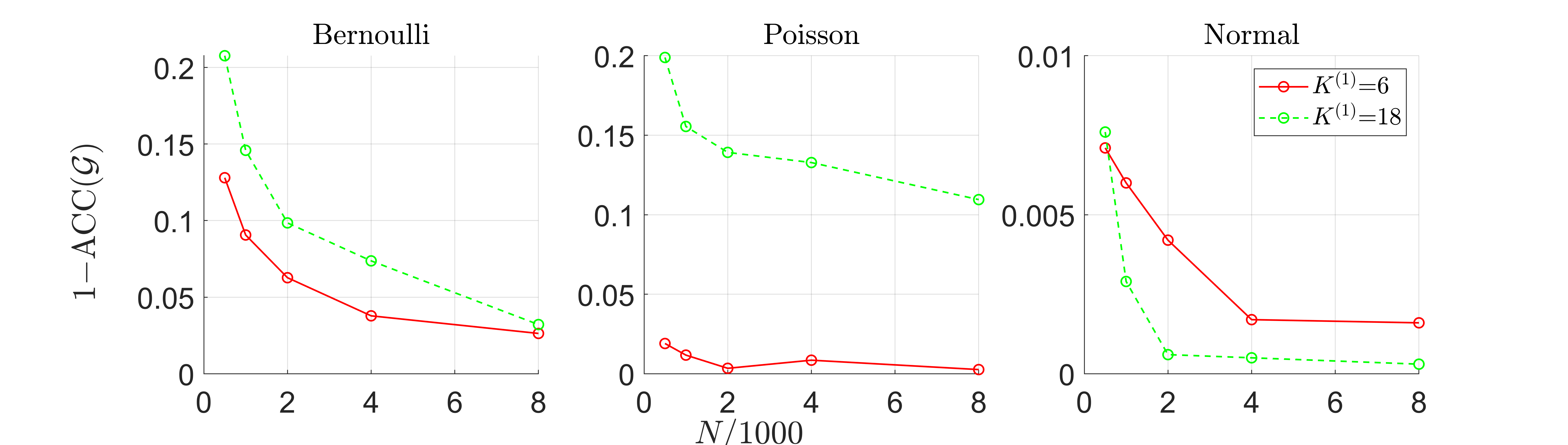}
    \includegraphics[width=0.9\linewidth]{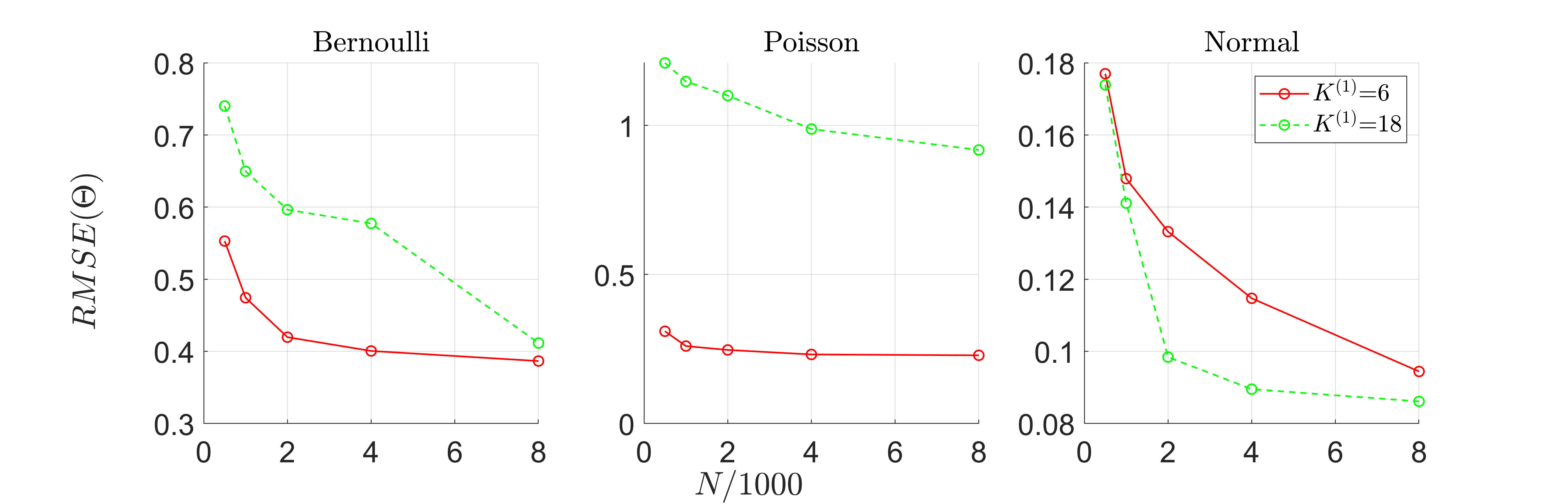}
    \caption{Estimation error for $\mcg$ and $\TT$ under the 2-layer DDE with parameters $\mathcal{B}_g$ and various observed-layer parametric families. Note that the y-axis varies across plots.}
    \label{fig:acc g gid}
\end{figure}

The estimation accuracy for $\mcg$ and $\TT$ is displayed in \cref{fig:acc g gid}. While we see a similar trend as in the results under the true parameters $\mcb_s$ (Figure \ref{fig:acc g sid}), the overall error values are larger and indicate that estimation under the less sparse model is more challenging. Recall that we are considering a different Poisson intercept compared to that under $\mcb_s$, so we cannot directly compare the Poisson results. 

\paragraph{Estimation accuracy of the PEM algorithm.}
We compare the performance of PEM (\cref{algo-penalized}) and SAEM (\cref{algo-saem}) for each parametric family. We implement the PEM under the same simulation setup as SAEM, described in \cref{sec:simulations}. However, the PEM implementation failed for higher latent dimensions, specifically under the latent dimension of $(J, K^{(1)}, K^{(2)}) = (54, 18, 6)$ or higher, due to memory allocation issues. As a result, PEM was only applied in the low-dimensional scenario with with $(J, K^{(1)}, K^{(2)}) = (18, 6, 2)$. 

The results, presented in Tables \ref{tab:ber-simulation}-\ref{tab:normal-simulation} (error of estimating $\mcg$) and \ref{tab:ber-simulation-Theta}-\ref{tab:normal-simulation-Theta} (error of estimating $\TT$) in \cref{subsec:exact values}, show that the PEM estimates exhibit a faster rate of convergence as $N$ increases. In particular, the RMSE values decrease at the parametric rate ${1}/{\sqrt{N}}$, as established in \cref{thm:estimation consistency}. We believe that the lower estimation accuracy under SAEM is due to multiple approximations within the SAEM algorithm such as approximate sampling and the objective function updates. We recommend using PEM instead of SAEM to estimate DDEs when the latent dimension is small.

\paragraph{Computation time.} 
\cref{fig:time} reports the computation times for all simulations in Section 5.1 and \cref{subsec:additional figures}. The results show that computation time varies across parametric families. A common trend is that both SAEM and PEM slow down as the sample size $N$ increases. Additionally, SAEM becomes slower as the parameter dimension increases.
Comparing SAEM and PEM is somewhat subtle, as their relative computation times depend on the response type and different convergence criteria. However, preliminary simulations under the dimension $(J, K^{(1)}, K^{(2)}) = (45, 15, 5)$ indicate that PEM is slower than SAEM across all three responses types. Furthermore, PEM fails to operate under larger dimensions due to high memory requirements. These observations support the conclusion in the main paper that SAEM is desirable for scenarios involving large latent dimensions.

\begin{figure}[h!]
    \centering
    \includegraphics[width=\linewidth]{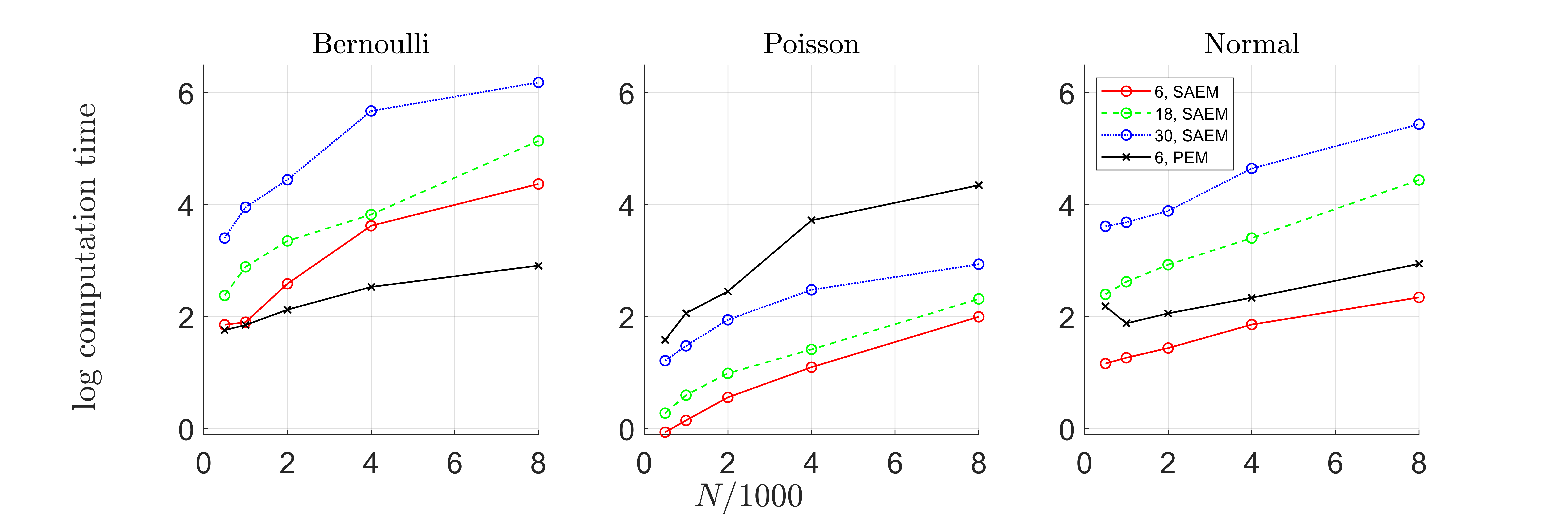} \\
    \includegraphics[width=\linewidth]{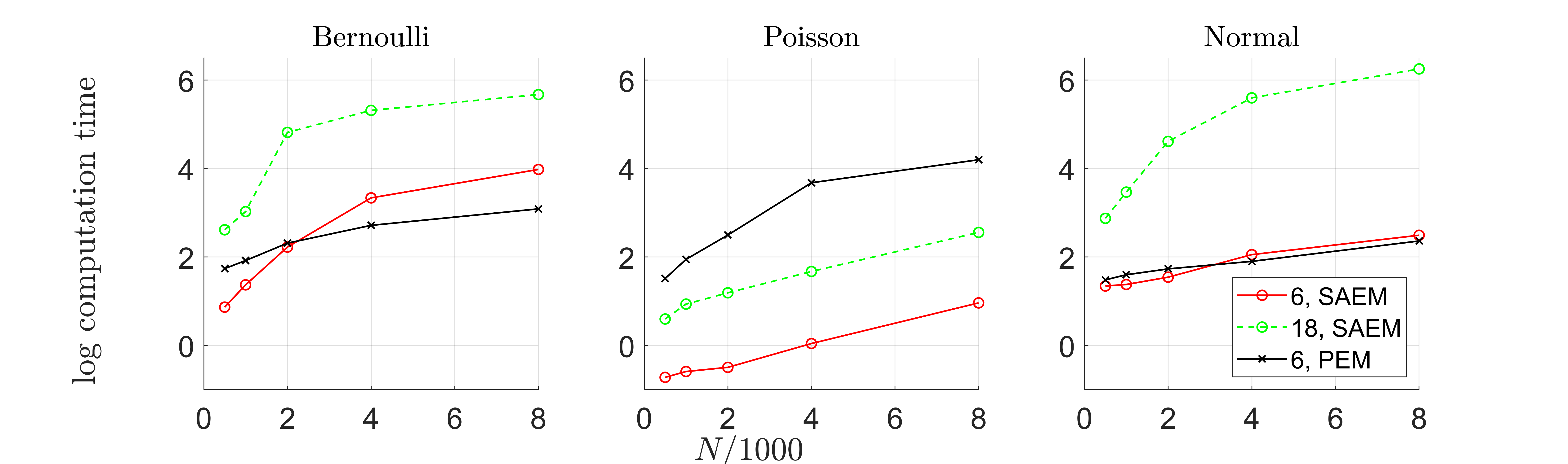}
    \caption{Log computation time (in seconds) for the simulation results in \cref{sec:simulations} and \cref{subsec:additional figures}, under true parameters $\mcb_s$ (top row) and $\mcb_g$ (bottom row). The legends indicate the value of $K^{(1)}$ and estimation algorithm.}
    \label{fig:time}
\end{figure}

\subsection{Experiments With Varying Numbers of Gibbs Samples in \cref{algo-saem}}
Here, we conduct additional experiments to assess the effect of the number of Gibbs samples, $C$, on both estimation accuracy and computation time. These experiments follow the setup in \cref{subsec:random initialization}, but we implement the SAEM algorithm (\cref{algo-saem}) varying $C = 1, 5, 25$. 

The simulation results in Table \ref{tab:Simulation different C} 
indicate that smaller values of $C$ result in faster computation across all parametric families, with minimal loss in estimation accuracy.
Here, the dependence of computation time on $C$ arises at the M-step, where the maximization objective is a sum of $O(C)$ functions. A larger $C$ complicates the objective function and makes the optimization slower. Based on these findings, we selected $C = 1$ as a baseline for the SAEM procedure.

\begin{table}[h!]
\centering
\begin{tabular}{cccccccccc}
\toprule
\multirow{2}{*}{ParFam} & & \multicolumn{2}{c}{Accuracy($\mcg$)} & \multicolumn{2}{c}{RMSE$(\TT)$} & \multicolumn{2}{c}{time (s)} & \multicolumn{2}{c}{\# iterations} \\
\cmidrule{2-10}
 & \# Gibbs $C$ \textbackslash{}$N$ & 1000 & 4000 & 1000 & 4000 & 1000 & 4000 & 1000 & 4000 \\
 \midrule
\multirow{3}{*}{Bernoulli} & 1 & 0.916 & 0.962 & 0.43 & 0.36 & 3.1 & 10.0 & 3.0 & 3.6 \\
 & 5 & 0.920 & 0.961 & 0.43 & 0.38 & 5.3 & 41.4 & 2.6 & 3.2 \\
 & 25 & 0.923 & 0.962 & 0.43 & 0.37 & 26.0 & 198.5 & 2.5 & 3.1 \\
\midrule
\multirow{3}{*}{Poisson} & 1 & 0.993 & 0.999 & 0.26 & 0.24 & 3.5 & 7.0 & 3.2 & 3.2 \\
 & 5 & 0.999 & 1 & 0.24 & 0.24 & 6.5 & 25.4 & 3.3 & 3.2 \\
 & 25 & 0.998 & 1 & 0.24 & 0.24 & 34.7 & 120.3 & 3.4 & 3.2 \\
\midrule
\multirow{3}{*}{Normal} & 1 & 0.994 & 1 & 0.16 & 0.14 & 0.5 & 1.1 & 2.0 & 2.0 \\
 & 5 & 0.995 & 1 & 0.15 & 0.13 & 1.0 & 2.2 & 2.0 & 2.0 \\
 & 25 & 0.994 & 1 & 0.16 & 0.13 & 2.2 & 8.8 & 2.0 & 2.0 \\
 \bottomrule
\end{tabular}
\caption{Accuracy measures for $\mcg$ and $\TT$, computation time and iterations for 2-layer DDE estimates under varying number of Gibbs samples $C$. For the first column, larger is better. For the other columns, smaller is better.}
\label{tab:Simulation different C}
\end{table}

\subsection{Simulations for Selecting the Number of Latent Variables}
\label{subsec:estimate K}
We first evaluate the performance of the three estimators (denoted as EBIC, LRT, Spectral) for selecting $K^{(1)}$, as introduced in \cref{subsec:unknown K estimation}, assuming the knowledge of $K^{(2)}$. 
Since the EBIC and LRT estimators require likelihood computation, we restrict the simulations to the configuration $(J, K^{(1)}, K^{(2)}) = (18, 6, 2)$. For each candidate model with $k \in \mathfrak{K}$ first-latent-layer variables, the likelihood is computed using parameter estimates from the PEM algorithm. For the LRT estimator, we have set the significance level to be $\alpha = 0.01$ and used the $\chi^2$ limiting distribution for each sequential test. 
Continuing from the simulation settings in \cref{sec:simulations}, we consider two sets of true parameters: $\mcb_s, \mcb_g$ and three sets of observed-layer parametric families: Bernoulli, Poisson, and Normal.
In addition to the sample sizes $N$ in the previous section, $N = 6000$ is also considered to better assess the large-sample accuracy.
Here, we assume that the number of the top layer latent variables $K^{(2)} = 2$ is known, and select $K^{(1)}$ from the candidate set $\mathfrak{K} = [2K^{(2)}, J/2) \cap \mathbb{N} = \{4, 5, 6, 7, 8\}$. The equality in $2K^{(2)} \le K^{(1)}$ is allowed to consider an equal number of underfitted/overfitted models.

\begin{figure}[h!]
    \centering
    \includegraphics[width=\linewidth]{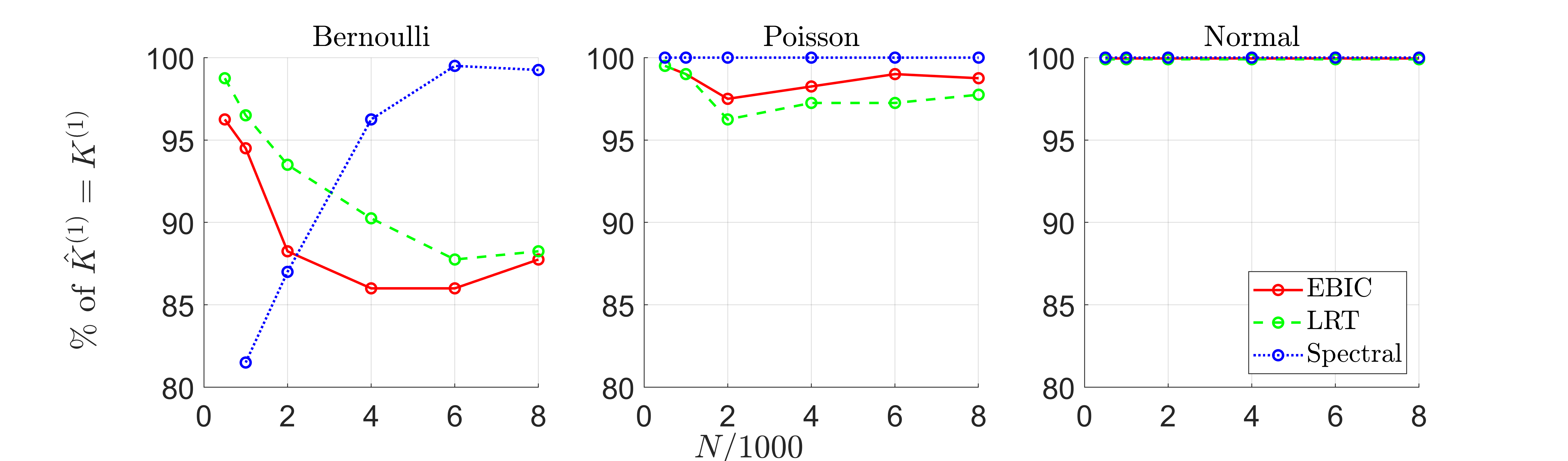}
    \caption{Selection accuracy of $K^{(1)}$ under the two-latent-layer DDE with true parameters $\mcb_s$. The spectral ratio estimator shows near-perfect accuracy for large $N$.}
    \label{fig:Select_K1}
\end{figure}

\begin{figure}[h!]
    \centering
    \includegraphics[width=\linewidth]{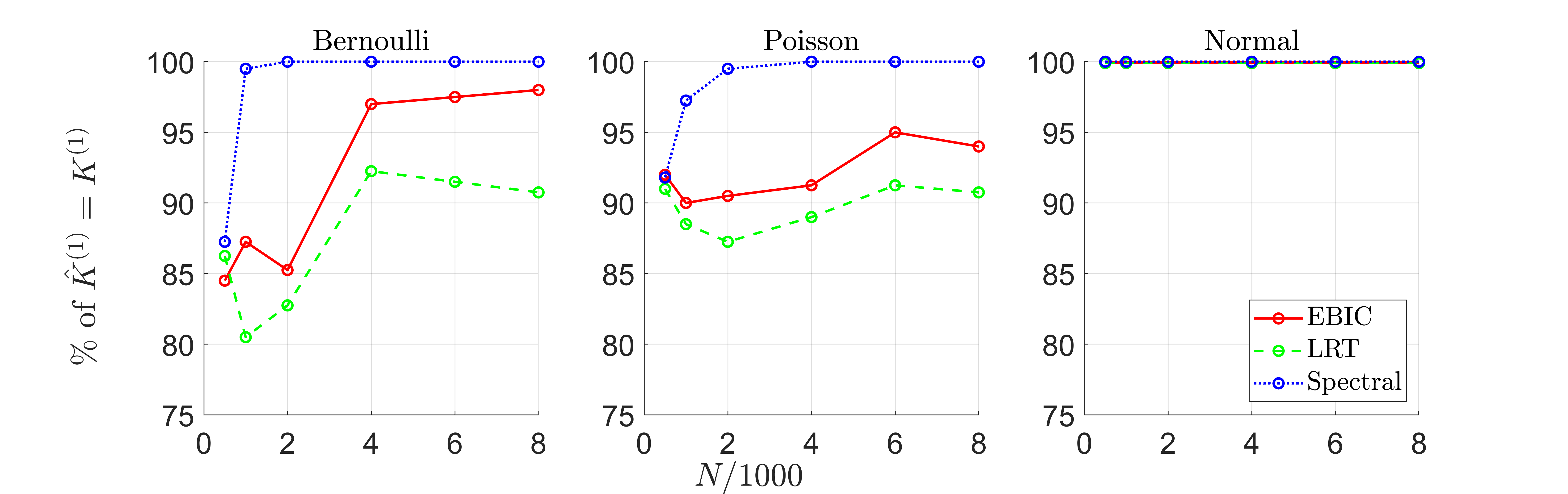}
    \caption{Simulation results for selecting $K^{(1)}$ under the 2-layer DDE with true parameters $\mcb_g$.}
    \label{fig:Select_K1 gid}
\end{figure}

We fit the three estimators $400$ times for each scenario, and display the correct selection percentage for the Bernoulli, Poisson, and Normal-based DDEs with true parameters $\mcb_s$/$\mcb_g$ in \cref{fig:Select_K1}/\cref{fig:Select_K1 gid}.
Comparing the three estimators, the spectral ratio-based estimator has near-perfect accuracy when $N$ is large enough, say $4000$. This demonstrates the empirical consistency of the spectral ratio estimator.
In contrast, the consistency of the EBIC and LRT estimators is not clear for Bernoulli and Poisson responses\footnote{To be more precise, as we are using approximate level $\alpha = 0.01$ tests, the correct selection percentage of the LRT estimator should converge to 0.99.}. Thus, we conclude that for a large enough $N$, it is desirable to select $K^{(1)}$ using the spectral ratio estimator. 
Among the three response types, the Bernoulli case with its nonlinear link function and limited response values is the most challenging. In contrast, the Normal case with a linear observed layer and continuous responses achieves near-perfect selection accuracy, which is consistent with earlier observations regarding parameter estimation accuracy.

\begin{figure}[h!]
    \centering
    \includegraphics[width=\linewidth]{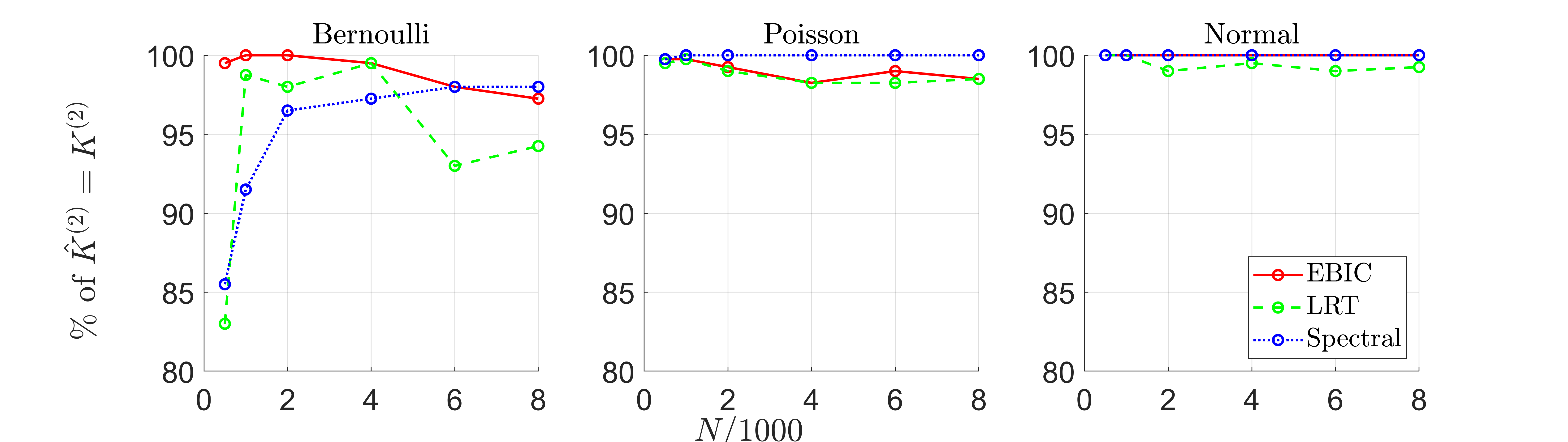}
    \caption{Selection accuracy of $K^{(2)}$ under two-latent-layer DDEs with true parameters $\mcb_s$.}
    \label{fig:select K2}
\end{figure}

\begin{figure}[h!]
    \centering
    \includegraphics[width=\linewidth]{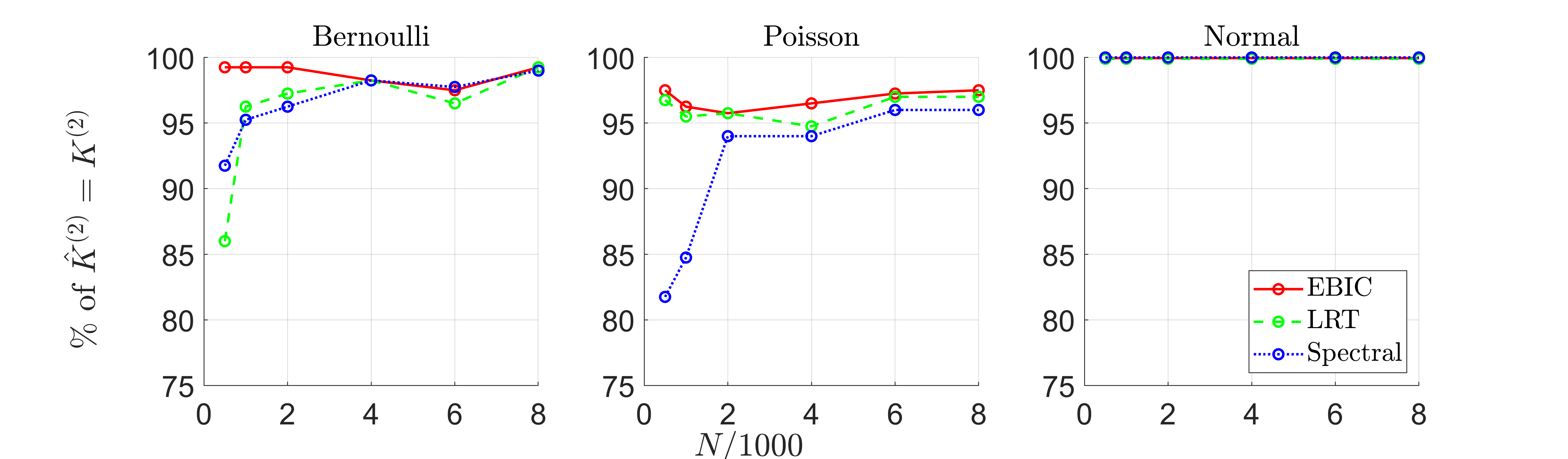}
    \caption{Simulation results for selecting $K^{(2)}$ under the 2-layer DDE with true parameters $\mcb_g$.}
    \label{fig:select K2 gid}
\end{figure}

We also apply these three estimators to select $K^{(2)}$, assuming that $K^{(1)}=6$ is correctly estimated or known. The candidate set for $K^{(2)}$ is $\mathfrak{K} := \{1,2,3\}$. \cref{fig:select K2}/\cref{fig:select K2 gid} displays the correct selection percentage under the true parameters $\mcb_s$/$\mcb_g$. Here, we implement the LRT estimator by naively assuming that Wilk's theorem holds. The overall trends are similar to those observed in the previous figures for selecting $K^{(1)}$. Under the sparse true parameter $\mcb_s$, the spectral ratio estimator performs well across all response types and the LRT estimator has the lowest, but still decent accuracy. Interestingly, the EBIC estimator outperforms the spectral ratio estimator in more challenging scenarios, such as Bernoulli responses with small $N$ or under the Poisson responses with less sparse true parameters $\mcb_g$. Unlike the case for selecting $K^{(1)}$, it is not clear that the spectral ratio estimator outperforms the EBIC estimator for selecting $K^{(2)}$.

Based on these experiments, we conclude that in the most general setting of a two-latent-layer DDE where both $K^{(1)}$ and $K^{(2)}$ are unknown, a two-step approach is effective. First, the spectral ratio estimator can be used to select $K^{(1)}$. Then, $K^{(2)}$ can be determined using either the EBIC or the spectral ratio estimator, incorporating domain knowledge if needed.

\subsection{\darkblue{Simulations for Deeper Models with $D\ge 3$ Latent Layers}}
\paragraph{Estimation accuracy of graphical structures}
We display the postponed tables from the the main text (Section 5) in Tables \ref{tab:D=3} and \ref{tab:D=4}, which correspond to experiments for deeper models with Normal responses and $D=3,4$ latent layers.
\begin{table}[h!]
\centering
\begin{tabular}{cccccccccc}
\toprule
Layer \textbackslash $N$ & 500 & 1000 & 2000 & 4000 & 8000& 16000 \\
\midrule
$\GG^{(1)}$ & 1.00 & 1.00 & 1.00 & 1.00 & 1.00 & 1.00 \\
$\GG^{(2)}$ & 0.89 & 0.89 & 0.90 & 0.93 & 0.94 & 0.98 \\
$\GG^{(3)}$ & 0.87 & 0.89 & 0.88 & 0.89 & 0.90 & 0.91 \\
\midrule
runtime (s) & 3 & 3 & 3 & 5 & 8 & 15 \\
 \bottomrule
\end{tabular}
\caption{\darkblue{Average entrywise-accuracy of estimating graphical matrices and runtime when $D=3$.}}
\label{tab:D=3}
\end{table}

\begin{table}[h!]
\centering
\begin{tabular}{ccccccccccc}
\toprule
Layer \textbackslash $N$ & 500 & 1000 & 2000 & 4000 & 8000 & 16000 \\
\midrule
$\GG^{(1)}$ & 1.00 & 1.00 & 1.00 & 1.00 & 1.00 & 1.00 \\
$\GG^{(2)}$ & 0.83 & 0.89 & 0.93 & 0.95 & 0.96 & 0.97 \\
$\GG^{(3)}$ & 0.72 & 0.74 & 0.75 & 0.77 & 0.78 & 0.80 \\
$\GG^{(4)}$ & 0.64 & 0.65 & 0.70 & 0.70 & 0.71 & 0.73 \\
\midrule
runtime (s) & 75 & 96 & 74 & 89 & 186 & 312 \\
 \bottomrule
\end{tabular}
\caption{\darkblue{Average entrywise-accuracy of estimating graphical matrices and runtime when $D=4$.}}
\label{tab:D=4}
\end{table}

\paragraph{Selecting the number of latent variables}
To assess our spectral-ratio estimator's performance, we have conducted additional experiments for deeper models with $D=3$. Table \ref{tab:D=3 select K} illustrates the satisfactory performance of Algorithm \ref{algo:select_K} for all layers, where each entry reports the accuracy across 400 replications.We have also evaluated the spectral-ratio estimator for more challenging cases with $D \ge 4$ latent layers, where the accuracy degrades due to accumulation of uncertainty and larger grid size.

\begin{table}[h!]
\centering
\begin{tabular}{cccccccccc}
\toprule
Layer \textbackslash $N$ & 500 & 1000 & 2000 & 4000 & 8000& 16000 & $|\mathfrak{K}^{(d)}|$\\
\midrule
$K^{(1)}$ & 100 & 100 & 100 & 100 & 100 & 100 & 14 \\
$K^{(2)}$ & 28 & 44 & 68 & 79 & 87 & 92 & 5 \\
$K^{(3)}$ & 65 & 71 & 82 & 89 & 92 & 95 & 2 \\
 \bottomrule
\end{tabular}
\caption{\darkblue{Correct selection percentage for $\mathcal{K}$, under a true model with Normal responses and $D=3, K^{(D)}=2$. The last column reports the typical grid size for each $K^{(d)}$.}}
\label{tab:D=3 select K}
\end{table}

\subsection{Exact Values of Estimation Errors}\label{subsec:exact values}
We display the complete results of the simulation outputs corresponding to plots in \cref{sec:simulations}, \cref{subsec:additional figures}, and \cref{subsec:estimate K}. 

\paragraph{Accuracy of Estimating $\mcg$.} Tables \ref{tab:ber-simulation}-\ref{tab:normal-simulation} display the average entrywise estimation accuracy for $\mcg$ under each parametric family. These values correspond to the upper rows of Figures \ref{fig:acc g sid} and \ref{fig:acc g gid}.

\begin{table}[h!]
\centering
\resizebox{\textwidth}{!}{
\begin{tabular}{ccccccccc}
\toprule
Parfam                     & True parameter      & $(J, K^{(1)}, K^{(2)})$               & Algorithm$\setminus N$ & 500   & 1000  & 2000  & 4000  & 8000  \\
\midrule
\multirow{7}{*}{Bernoulli} & \multirow{4}{*}{$\mcb_s$} & \multirow{2}{*}{(18,6,2)} & SAEM      & 0.899 & 0.931 & 0.956 & 0.974 & 0.985 \\
                           &                     &                           & PEM       & 0.926 & 0.966 & 0.986 & 0.992 & 0.992 \\ \cmidrule{3-9}
                           &                     & (54,18,6)                 & SAEM      & 0.910 & 0.967 & 0.990 & 0.996 & 0.998 \\ \cmidrule{3-9}
                           &                     & (90,30,10)                & SAEM      & 0.914 & 0.971 & 0.991 & 0.996 & 0.998 \\
                           \cmidrule{2-9}
                           & \multirow{3}{*}{$\mcb_g$} & \multirow{2}{*}{(18,6,2)} & SAEM      & 0.872 & 0.909 & 0.937 & 0.962 & 0.974 \\
                           &                     &                           & PEM       & 0.891 & 0.937 & 0.970 & 0.986 & 0.992 \\
                           \cmidrule{3-9}
                           &                     & (54,18,6)                 & SAEM      & 0.792 & 0.854 & 0.902 & 0.926 & 0.968 \\
                           \bottomrule
\end{tabular}
}
\caption{$\text{Acc}(\mcg)$ under the Bernoulli 2-layer DDE. } 
\label{tab:ber-simulation}
\end{table}

\begin{table}[h!]
\centering
\resizebox{\textwidth}{!}{
\begin{tabular}{ccccccccc}
\toprule
Parfam                     & True parameter      & $(J, K^{(1)}, K^{(2)})$               & Algorithm$\setminus N$ & 500   & 1000  & 2000  & 4000  & 8000  \\
\midrule
\multirow{7}{*}{Poisson} & \multirow{4}{*}{$\mcb_s$} & \multirow{2}{*}{(18,6,2)} & SAEM      & 0.985 & 0.993 & 0.999 & 0.999 & 0.999 \\
                           &                     &                           & PEM       & 0.994 & 0.999 & 1 & 1 & 1 \\ \cmidrule{3-9}
                           &                     & (54,18,6)                 & SAEM      & 0.987 & 0.996 & 0.999 & 1 & 1 \\ \cmidrule{3-9}
                           &                     & (90,30,10)                & SAEM      & 0.986 & 0.996 & 1 & 1 & 1 \\
                           \cmidrule{2-9}
                           & \multirow{3}{*}{$\mcb_g$} & \multirow{2}{*}{(18,6,2)} & SAEM      & 0.981 & 0.988 & 0.997 & 0.992 & 0.997 \\
                           &                     &                           & PEM       & 0.994 & 0.995 & 0.998 & 0.998 & 1 \\
                           \cmidrule{3-9}
                           &                     & (54,18,6)                 & SAEM      & 0.801 & 0.845 & 0.861 & 0.867 & 0.891 \\
                           \bottomrule
\end{tabular}
}
\caption{$\text{Acc}(\mcg)$ under the Poisson 2-layer DDE. } 
\label{tab:poi-simulation}
\end{table}

\begin{table}[h!]
\begin{tabular}{ccccccccc}
\toprule
Parfam                     & True parameter      & $(J, K^{(1)}, K^{(2)})$               & Algorithm & 500   & 1000  & 2000  & 4000  & 8000  \\
\midrule
\multirow{7}{*}{Normal} & \multirow{4}{*}{$\mcb_s$} & \multirow{2}{*}{(18,6,2)} & SAEM      & 0.985 & 0.993 & 0.999 & 0.999 & 0.999 \\
                           &                     &                           & PEM       & 0.992 & 0.996 & 1 & 1 & 1 \\ \cmidrule{3-9}
                           &                     & (54,18,6)                 & SAEM      & 0.996 & 0.998& 0.999& 1&1 \\ \cmidrule{3-9}
                           &                     & (90,30,10)                & SAEM      & 0.995 & 0.998 & 1 & 1 & 1 \\
                           \cmidrule{2-9}
                           & \multirow{3}{*}{$\mcb_g$} & \multirow{2}{*}{(18,6,2)} & SAEM      & 0.993 & 0.994 & 0.996 & 0.998 & 0.998 \\
                           &                     &                           & PEM       & 0.992 & 0.997 & 0.999 & 1 & 1 \\
                           \cmidrule{3-9}
                           &                     & (54,18,6)                 & SAEM      & 0.993 & 0.995 & 0.998 & 0.999 & 1 \\
                           \bottomrule
\end{tabular}
\caption{$\text{Acc}(\mcg)$ under the Normal 2-layer DDE.} 
\label{tab:normal-simulation}
\end{table}

\paragraph{Accuracy of Estimating $\TT$.}
Tables \ref{tab:ber-simulation-Theta}--\ref{tab:normal-simulation-Theta} reports the RMSE values for estimating continuous parameters $\TT$. These values correspond to the bottom rows of Figures \ref{fig:acc g sid} and \ref{fig:acc g gid}.

\begin{table}[h!]
\centering
\resizebox{\textwidth}{!}{
\begin{tabular}{ccccccccc}
\toprule
Parfam                     & True parameter      & $(J, K^{(1)}, K^{(2)})$               & Algorithm$\setminus N$ & 500   & 1000  & 2000  & 4000  & 8000  \\
\midrule
\multirow{7}{*}{Bernoulli} & \multirow{4}{*}{$\mcb_s$} & \multirow{2}{*}{(18,6,2)} & SAEM & 0.498 & 0.431 & 0.386 & 0.359 & 0.341 \\
& & & PEM & 0.404 & 0.304 & 0.231 & 0.203 & 0.184 \\ \cmidrule{3-9}
                           &                     & (54,18,6)                 & SAEM      & 0.392 & 0.289 &0.231 &0.213 &0.208 \\ \cmidrule{3-9}
                           &                     & (90,30,10)                & SAEM      & 0.356 & 0.240 & 0.204 & 0.185 & 0.177 \\
                           \cmidrule{2-9}
                           & \multirow{3}{*}{$\mcb_g$} & \multirow{2}{*}{(18,6,2)} & SAEM      & 0.553 & 0.474& 0.420& 0.400& 0.387 \\
                           &                     &                           & PEM       & 0.520 & 0.398& 0.288& 0.206& 0.165 \\
                           \cmidrule{3-9}
                           &                     & (54,18,6)                 & SAEM      & 0.740 & 0.650 & 0.596 & 0.577 & 0.412 \\
                           \bottomrule
\end{tabular}
}
\caption{$\text{RMSE}(\TT)$ under the Bernoulli 2-layer DDE. } 
\label{tab:ber-simulation-Theta}
\end{table}

\begin{table}[h!]
\centering
\resizebox{\textwidth}{!}{
\begin{tabular}{ccccccccc}
\toprule
Parfam                     & True parameter      & $(J, K^{(1)}, K^{(2)})$               & Algorithm$\setminus N$ & 500 & 1000  & 2000  & 4000  & 8000  \\
\midrule
\multirow{7}{*}{Poisson} & \multirow{4}{*}{$\mcb_s$} & \multirow{2}{*}{(18,6,2)} & SAEM & 0.281 & 0.258 & 0.245 & 0.243 & 0.242 \\
& & & PEM & 0.219 & 0.158 & 0.108 & 0.080 & 0.061 \\ \cmidrule{3-9}
                           &                     & (54,18,6)                 & SAEM      & 0.289& 0.168& 0.153& 0.147& 0.145 \\ \cmidrule{3-9}
                           &                     & (90,30,10)                & SAEM      & 0.281 & 0.195& 0.125& 0.117& 0.116 \\
                           \cmidrule{2-9}
                           & \multirow{3}{*}{$\mcb_g$} & \multirow{2}{*}{(18,6,2)} & SAEM      & 0.309 & 0.260& 0.247 & 0.231 & 0.229 \\
                           &                     &                           & PEM       & 0.206& 0.145 & 0.101& 0.072 & 0.053 \\
                           \cmidrule{3-9}
                           &                     & (54,18,6)                 & SAEM      & 1.210& 1.147& 1.100& 0.988& 0.918 \\
                           \bottomrule
\end{tabular}
}
\caption{$\text{RMSE}(\TT)$ under the Poisson 2-layer DDE. } 
\label{tab:poi-simulation-Theta}
\end{table}

\begin{table}[h!]
\centering
\resizebox{\textwidth}{!}{
\begin{tabular}{ccccccccc}
\toprule
Parfam                     & True parameter      & $(J, K^{(1)}, K^{(2)})$               & Algorithm$\setminus N$ & 500   & 1000  & 2000  & 4000  & 8000  \\
\midrule
\multirow{7}{*}{Normal} & \multirow{4}{*}{$\mcb_s$} & \multirow{2}{*}{(18,6,2)} & SAEM &0.189 &0.156& 0.147& 0.137& 0.133 \\
& & & PEM & 0.170& 0.128& 0.080& 0.063& 0.054 \\ \cmidrule{3-9}
                           &                     & (54,18,6)                 & SAEM      & 0.117& 0.090& 0.072& 0.063& 0.057 \\ \cmidrule{3-9}
                           &                     & (90,30,10)                & SAEM      & 0.108& 0.074& 0.054& 0.045& 0.041 \\
                           \cmidrule{2-9}
                           & \multirow{3}{*}{$\mcb_g$} & \multirow{2}{*}{(18,6,2)} & SAEM      & 0.177& 0.148& 0.133& 0.115& 0.094\\
                           &                     &                           & PEM       & 0.179& 0.137& 0.094& 0.069& 0.055\\
                           \cmidrule{3-9}
                           &                     & (54,18,6)                 & SAEM      & 0.174 & 0.141& 0.098& 0.090& 0.086 \\
                           \bottomrule
\end{tabular}
}
\caption{$\text{RMSE}(\TT)$ under the Normal 2-layer DDE. }
\vspace{-5mm}
\label{tab:normal-simulation-Theta}
\end{table}

\paragraph{Accuracy of selecting $\mck$.} Tables \ref{tab:Select_K1 Ber}-\ref{tab:Select_K1 Poi} reports the accuracy of selecting $\mck$ under Bernoulli and Poisson-based DDEs, which correspond to \cref{fig:Select_K1} and \cref{fig:Select_K1 gid}. Accuracy values for Normal-based DDEs are omitted, as all methods demonstrated near-perfect accuracy in this setting. Similarly, Tables \ref{tab:select K2 Ber}-\ref{tab:select K2 Poi} display the results for selecting $K^{(2)}$, which corresponds to Figures \ref{fig:select K2} and \ref{fig:select K2 gid}.

These tables provide additional insights beyond those presented in figures, which only display the correct selection probability for the events $\hat{K}^{(1)} = 6$ and $\hat{K}^{(2)} = 2$. 
The percentage of incorrect estimates in the tables reveal systematic tendencies of the estimators: the EBIC and LRT estimators frequently \textit{overselect} ($\hat{K}^{(1)} \ge K^{(1)}$) whereas the spectral estimator sometimes \textit{underselects} ($\hat{K}^{(1)} \le K^{(1)}$).

\begin{table}[h!]
\centering
\resizebox{0.75\textwidth}{!}{
\begin{tabular}{cccccccc}
\toprule
ParFam, value                     & $N$                   & Method \textbackslash{} $\hat{K}^{(1)}$    & 4                    & 5                    & 6                    & 7                    & 8                    \\
\midrule
\multirow{18}{*}{Bernoulli, $\mcb_s$} & \multirow{3}{*}{500}  & EBIC     &           0 & 0.25                 & 96.25                & 3.5                  & 0 \\
                            &                       & LRT      & 0.5                  & 0 & \textbf{98.75} & 0.75 & 0 \\
                            &                       & Spectral & 1.75 & 30.0 & 67.5                 & 0.75                 &  0                    \\
                            \cmidrule{2-8}
                            & \multirow{3}{*}{1000} & EBIC     &  0                    &       0               & 94.5                 & 5.5                  &   0                   \\
                            &                       & LRT      &  0                    &           0           & \textbf{96.5}                 & 3.5                  &    0                  \\
                            &                       & Spectral & 0.25                 & 18.25                & 81.5                 &         0             &        0              \\ \cmidrule{2-8}
                            & \multirow{3}{*}{2000} & EBIC     &   0                   &          0            & 88.25                & 11.75                &       0               \\
                            &                       & LRT      &  0                    &       0               & \textbf{93.5}                & 6.25                 & 0.25                 \\
                            &                       & Spectral &     0                 & 13.0                 & 87.0                 &            0          &    0                  \\ \cmidrule{2-8}
                            & \multirow{3}{*}{4000} & EBIC     &    0                  &         0             & 86.0                 & 12.25                & 1.75                 \\
                            &                       & LRT      &     0                 &             0         & 90.25                & 9.25                 & 0.5                  \\
                            &                       & Spectral &  0                    & 3.75                 & \textbf{96.25}                &     0                 &             0         \\ \cmidrule{2-8}
                            & \multirow{3}{*}{6000} & EBIC     &    0                  &              0        & 86.0                 & 13.5                & 1.5                 \\
                            &                       & LRT      &     0                 &             0         & 87.75                & 12.25                 &   0                \\
                            &                       & Spectral &   0                   & 0.5                 & \textbf{99.5}                &          0            &     0                 \\ \cmidrule{2-8}
                            & \multirow{3}{*}{8000} & EBIC &0 &0 & 87.75 & 11.0 & 1.25 \\
                            &                       & LRT  &0 &0 & 88.25 &0 10.75 & 1.00 \\
                            &                       & Spectral  &0 & 0.75 & \textbf{99.25} & 0& 0 \\
    \midrule
\multirow{18}{*}{Bernoulli, $\mcb_g$} 
& \multirow{3}{*}{500}  & EBIC & 0& 0& 84.5 & 13.75 & 1.75 \\
                    & & LRT & 0.75 & 1.0 & 86.25 & 12.0 & 0\\
                    & & Spectral & 8.5 & 4.25 & \textbf{87.25} & 0& 0\\
\cmidrule{2-8}
& \multirow{3}{*}{1000}  & EBIC & 0& 0& 87.25 & 10.5 & 2.5 \\
                    & & LRT & 0& 0& 80.5 & 18.5 & 1.0 \\
                    & & Spectral & 0.5 & 0& \textbf{99.5} & 0& 0\\ 
\cmidrule{2-8}
& \multirow{3}{*}{2000}  & EBIC & 0& 0& 85.25 & 11.5 & 3.25 \\
                    & & LRT & 0& 0& 82.75 & 16.25 & 1.0 \\
                    & & Spectral & 0& 0& \textbf{100.0} & 0& 0\\ 
\cmidrule{2-8}
& \multirow{3}{*}{4000} & EBIC & 0& 0& 97.0 & 3.0 & 0\\
                    & & LRT &0 &0 & 92.25 & 7.75 & 0\\
                    & & Spectral &0 &0 &\textbf{100.0} &0 &0 \\ 
\cmidrule{2-8}
& \multirow{3}{*}{6000}  & EBIC &0 &0 & 97.5 & 2.5 &0 \\
                    & & LRT &0 &0 & 91.5 & 8.5 &0 \\
                    & & Spectral &0 &0 & \textbf{100.0} &0 &0 \\ 
\cmidrule{2-8}
& \multirow{3}{*}{8000}  & EBIC &0 &0 & 98.0 & 2.0 &0 \\
                    & & LRT &0 &0 & 90.75 & 9.25 &0 \\
                    & & Spectral &0 &0 & \textbf{100.0} &0 &0 \\ 
                            \bottomrule
\end{tabular}}
\caption{Empirical distribution of the estimated $\hat{K}^{(1)}$ values under the Bernoulli DDE with true $K^{(1)} = 6$ and parameters $\mcb_s, \mcb_g$. For each sample size, we present the method with the highest accuracy in bold. ``ParFam'' is short for ``parametric family''.}
\label{tab:Select_K1 Ber}
\end{table}

\begin{table}[h!]
\centering
\resizebox{0.75\textwidth}{!}{
\begin{tabular}{cccccccc}
\toprule
ParFam, value                      & $N$                   & Method \textbackslash{} $\hat{K}^{(1)}$    & 4                    & 5                    & 6                    & 7                    & 8                    \\
\midrule
\multirow{18}{*}{Poisson, $\mcb_s$} & \multirow{3}{*}{500}  & EBIC     &  0 & 0  & 99.5  & 0.5  & 0 \\
                            &                       & LRT      &0 &0 &  99.5  & 0.5  &0  \\
                            &                       & Spectral &0 &0 &  \textbf{100.0} &0  &0  \\
                            \cmidrule{2-8}
                            & \multirow{3}{*}{1000} & EBIC       &0 &0 &   99.0  & 1.0  &0  \\
                            &                       & LRT      &0 &0 &   99.0  & 1.0  &0  \\
                            &                       & Spectral     &0 &0 &   \textbf{100.0}  &0  & 0 \\ \cmidrule{2-8}
                            & \multirow{3}{*}{2000} & EBIC       & 0& 0&   97.5  & 2.5  &0  \\
                            &                       & LRT      &0 &0 &   96.25  & 3.75  &0  \\
                            &                       & Spectral     &0 &0 &   \textbf{100.0}  &0  &0  \\ 
                            \cmidrule{2-8}
                            & \multirow{3}{*}{4000} & EBIC       &0 &0 &   98.25  & 1.5  & 0.25 \\
                            &                       & LRT      &0 &0 &   97.25  & 2.75  & 0 \\
                            &                       & Spectral     & 0&0 &   \textbf{100.0}  & 0 & 0\\ 
                            \cmidrule{2-8}
                            & \multirow{3}{*}{6000} & EBIC       &0 &0 &   99.0  & 1.0  &0 \\
                            &                       & LRT      &0 &0 &   97.25  & 2.75  &0  \\
                            &                       & Spectral     &0 &0 &   \textbf{100.0}  &0  &0 \\ \cmidrule{2-8}
                            & \multirow{3}{*}{8000} & EBIC       &0 &0 &   98.75  & 0.75  & 0.5 \\
                            &                       & LRT      &0 &0 &   97.75  & 1.75  & 0.5 \\
                            &                       & Spectral     &0 &0 &   \textbf{100.0}  &0  &0 \\
\midrule
\multirow{18}{*}{Poisson, $\mcb_g$}
& \multirow{3}{*}{500}  & EBIC &0 &0 &\textbf{92.0} & 7.75 & 0.25 \\
                    & & LRT &0 & 1.0 & 91.0 & 8.0 &0 \\
                    & & Spectral & 8.5 &0 & 91.5 &0 &0 \\
\cmidrule{2-8}
& \multirow{3}{*}{1000}  & EBIC &0 &0 & 90.0 & 9.5 & 0.5 \\
                    & & LRT &0 &0 & 88.5 & 10.75 & 0.75 \\
                    & & Spectral & 2.75 &0 & \textbf{97.25} &0 &0 \\ 
\cmidrule{2-8}
& \multirow{3}{*}{2000}  & EBIC &0 &0 & 90.5 & 8.5 & 1.0 \\
                    & & LRT &0 & 0.25 & 87.25 & 11.5 & 1.0 \\
                    & & Spectral & 0.5 &0 & \textbf{99.5} &0 &0 \\ 
\cmidrule{2-8}
& \multirow{3}{*}{4000} & EBIC &0 &0 & 91.25 & 7.25 & 1.5 \\
                    & & LRT &0 &0 & 89.0 & 9.25 & 1.75 \\
                    & & Spectral &0 &0 & \textbf{100.0} &0 &0 \\ 
\cmidrule{2-8}
& \multirow{3}{*}{6000}  & EBIC &0 &0 & 95.0 & 4.5 & 0.5 \\
                    & & LRT &0 &0 & 91.25 & 7.75 & 1.0 \\
                    & & Spectral &0 &0 & \textbf{100.0} &0 &0 \\ 
\cmidrule{2-8}
& \multirow{3}{*}{8000}  & EBIC &0 &0 & 94.0 & 5.0 & 1.0 \\
                    & & LRT &0 &0 & 90.75 & 8.25 & 1.0 \\
                    & & Spectral &0 &0 & \textbf{100.0} &0 &0 \\ 
                            \bottomrule
\end{tabular}
}
\caption{Empirical distribution of the estimated $\hat{K}^{(1)}$ values under the Poisson DDE with true $K^{(1)} = 6$ and parameters $\mcb_s,\mcb_g$. For each sample size, we present the method with the highest accuracy in bold. ``ParFam'' is short for ``parametric family''.}
\label{tab:Select_K1 Poi}
\end{table}

\begin{table}[h!]
\centering
\resizebox{0.65\textwidth}{!}{
\begin{tabular}{cccccc}
\toprule
ParFam, value                      & $N$                   & Method \textbackslash{} $\hat{K}^{(2)}$    & 1 & 2 & 3 \\
\midrule
\multirow{18}{*}{Bernoulli, $\mcb_s$} & \multirow{3}{*}{500}  & EBIC & 0.5 & \textbf{99.5}  & 0 \\
                            &                       & LRT      & 16.75 &  83.0  & 0.25  \\
                            &                       & Spectral & 12.0 & 85.5 & 2.5 \\
                            \cmidrule{2-6}
                            & \multirow{3}{*}{1000} & EBIC &0 & \textbf{100.0}  &0 \\
                            &                       & LRT  & 0.5 & 98.75  & 0.75 \\
                            &                       & Spectral & 6.5 & 91.5 & 2.0 \\ \cmidrule{2-6}
                            & \multirow{3}{*}{2000} & EBIC &0 & \textbf{100.0}  &0 \\
                            &                       & LRT  &0 & 98.0  & 2.0 \\
                            &                       & Spectral & 3.25& 96.5 & 0.25 \\
                            \cmidrule{2-6}
                            & \multirow{3}{*}{4000} & EBIC &0 & \textbf{99.5}  & 0.5 \\
                            &                       & LRT  &0 & 96.0  & 4.0 \\
                            &                       & Spectral & 2.75 & 97.25 & 0 \\
                            \cmidrule{2-6}
                            & \multirow{3}{*}{6000} & EBIC &0 & \textbf{98.0}  & 2.0  \\
                            &                       & LRT  &0 & 93.0 & 7.0 \\
                            &                       & Spectral & 2.0 & \textbf{98.0} & 0 \\ \cmidrule{2-6}
                            & \multirow{3}{*}{8000} & EBIC &0 & 97.25  & 2.75 \\
                            &                       & LRT  &0 & 94.25  & 5.75 \\
                            &                       & Spectral & 2.0 & \textbf{98.0} & 0 \\
\midrule
\multirow{18}{*}{Bernoulli, $\mcb_g$}
& \multirow{3}{*}{500} & EBIC & 0.75 & \textbf{99.25} &0 \\
                       & & LRT & 13.5 & 86.0 & 0.5  \\
                       & & Spectral & 7.0 & 91.75 & 1.25\\
                            \cmidrule{2-6}
& \multirow{3}{*}{1000} & EBIC &0 & \textbf{99.25}  & 0.75 \\
                       & & LRT & 1.5 &  96.25  & 2.25  \\
                       & & Spectral & 3.5 & 95.25 & 1.25 \\ 
                            \cmidrule{2-6}
& \multirow{3}{*}{2000} & EBIC &0 & \textbf{99.25}  & 0.75 \\
                       & & LRT & 0.5 &  94.0  & 5.5  \\
                       & & Spectral & 8.75 & 90.75 & 0.5 \\
                            \cmidrule{2-6}
& \multirow{3}{*}{4000} & EBIC &0 & \textbf{98.25}  & 1.75 \\
                       & & LRT &0 & \textbf{98.25}  & 1.75 \\
                       & & Spectral & 2.75 & 96.25 & 1.0 \\ 
                       \cmidrule{2-6}
& \multirow{3}{*}{6000} & EBIC &0 & \textbf{97.5}  & 2.5 \\
                       & & LRT &0 &  96.5  & 3.5  \\
                       & & Spectral & 1.75 & 97.75 & 0.5 \\
                       \cmidrule{2-6}
& \multirow{3}{*}{8000}  & EBIC &0 &  \textbf{99.25}  & 0.75  \\
                       & & LRT &0 &  \textbf{99.25}  & 0.75  \\
                       & & Spectral & 0.5 & 99.0 & 0.5 \\
                       \bottomrule
\end{tabular}
}
\caption{Empirical distribution of the estimated $\hat{K}^{(2)}$ values under the Bernoulli DDE with true $K^{(2)} = 2$ and parameters $\mcb_s, \mcb_g$. For each sample size, we present the method with the highest accuracy in bold. ``ParFam'' is short for ``parametric family''.}
\label{tab:select K2 Ber}
\end{table}

\begin{table}[h!]
\centering
\resizebox{0.65\textwidth}{!}{
\begin{tabular}{cccccc}
\toprule
Parfam, value                      & $N$                   & Method \textbackslash{} $\hat{K}^{(2)}$    & 1 & 2 & 3 \\
\midrule
\multirow{18}{*}{Poisson, $\mcb_s$} & \multirow{3}{*}{500}  & EBIC & 0& \textbf{99.75}  & 0.25 \\
                            &                       & LRT      & 0&  99.5  & 0.5  \\
                            &                       & Spectral & 0.25 &\textbf{99.75} & 0\\
                            \cmidrule{2-6}
                            & \multirow{3}{*}{1000} & EBIC &0 & 99.75  & 0.25 \\
                            &                       & LRT  &0 & 99.75  & 0.25 \\
                            &                       & Spectral &0 & \textbf{100.0} & 0 \\ \cmidrule{2-6}
                            & \multirow{3}{*}{2000} & EBIC &0 & 99.25  & 0.75 \\
                            &                       & LRT  &0 & 99.0  & 1.0 \\
                            &                       & Spectral &0 & \textbf{100.0} & 0 \\
                            \cmidrule{2-6}
                            & \multirow{3}{*}{4000} & EBIC &0 & 98.25  & 1.75 \\
                            &                       & LRT  &0 & 98.25  & 1.75 \\
                            &                       & Spectral &0 & \textbf{100.0} & 0 \\
                            \cmidrule{2-6}
                            & \multirow{3}{*}{6000} & EBIC &0 & 99.0  & 1.0  \\
                            &                       & LRT  &0 & 98.25 & 1.75 \\
                            &                       & Spectral &0 & \textbf{100.0} & 0 \\ \cmidrule{2-6}
                            & \multirow{3}{*}{8000} & EBIC &0 & 98.5  & 1.5  \\
                            &                       & LRT  &0 & 98.5  & 1.5 \\
                            &                       & Spectral & 0& \textbf{100.0} &  0\\
\midrule
\multirow{18}{*}{Poisson, $\mcb_g$}
& \multirow{3}{*}{500} & EBIC & 1.0 & \textbf{97.5}  & 1.5 \\
                       & & LRT & 1.75 &  96.75  & 1.5  \\
                       & & Spectral & 16.5 &81.75 & 1.75\\
                            \cmidrule{2-6}
& \multirow{3}{*}{1000} & EBIC & 0.25 & \textbf{96.25}  & 3.5 \\
                       & & LRT & 0.75 &  95.5  & 3.75  \\
                       & & Spectral & 13.75 &84.75 & 1.5 \\ 
                            \cmidrule{2-6}
& \multirow{3}{*}{2000} & EBIC &0 & \textbf{95.75}  & 4.25 \\
                       & & LRT & 0.5 &  94.0  & 5.5  \\
                       & & Spectral & 8.75 & 90.75 & 0.5 \\
                            \cmidrule{2-6}
& \multirow{3}{*}{4000} & EBIC & 0.25 & \textbf{96.5}  & 3.25 \\
                       & & LRT & 0.75 &  94.75  & 4.5  \\
                       & & Spectral & 5.5 & 94.0 & 0.5 \\
                            \cmidrule{2-6}
& \multirow{3}{*}{6000} & EBIC &0 & \textbf{97.25}  & 2.75 \\
                       & & LRT &0 &  97.0  & 3.0  \\
                       & & Spectral & 3.5 & 96.0 & 0.5 \\ 
                       \cmidrule{2-6}
& \multirow{3}{*}{8000}  & EBIC & 0.25 & \textbf{97.5}  & 2.25 \\
                       & & LRT & 0.25 &  97.0  & 2.75  \\
                       & & Spectral & 3.5 & 96.0 & 0.5 \\
                            \bottomrule
\end{tabular}
}
\caption{Empirical distribution of the estimated $\hat{K}^{(2)}$ values under the Poisson DDE with true $K^{(2)} = 2$ and parameters $\mcb_s,\mcb_g$. For each sample size, we present the method with the highest accuracy in bold. ``ParFam'' is short for ``parametric family''.}
\label{tab:select K2 Poi}
\end{table}

\clearpage

\section{Data Analysis Details and Additional Results}\label{sec:supp data analysis}
\subsection{Preprocessing}

\paragraph{MNIST Data.} 
We work with the preprocessed version of the default training set, which consists of $60,000$ images, each containing information on the $28^2$ pixel values. Initially, each pixel takes integer values between $0$ and $255$. 
As the data values are highly concentrated around values near 0 or larger than 200, we transform the data into binary responses by thresholding at a value of 128. In other words, pixels with values exceeding 128 are assigned a binary value of 1, while the rest are set to 0.
For the sake of easier presentation and computational efficiency, we work on the subset of the dataset whose true digit labels are equal to $0, 1, 2, 3$, and also discard some pixels with uniformly small values by selecting $J = 264$ pixels whose average pixel values are larger than 40. After preprocessing, the training set consists of $N = 20,679$ unlabeled images.
We identically preprocess the test set, which leads to a total of $N=4,157$ images.

\paragraph{20 Newsgroups Data.}
The dataset provides a default partition of the train and test sets based on chronological order, and we use this partition.
The preprocessing of the training data was carried out in three main steps. First, to reduce the signal-to-noise ratio and enhance interpretability, we focus on a subset of labels. To elaborate, we consider the newsgroup articles that belong to the large class of computer, recreation, and science (see the top-latent layer labels in \cref{fig:topic model true label hierarchy}). A manual inspection revealed that the newsgroup with label \texttt{sci.crypt} has a wide range of topics such as government and politics and was often cross-referenced to those newsgroups, so we did not include these documents in our dataset. 
Second, documents of extreme lengths were filtered out by removing the shortest $5\%$ and longest $1\%$ of all documents. This procedure is standard in the literature, and has been shown to increase the signal-to-noise ratio \citep{ke2024using}. 
Third, we construct our dictionary by excluding infrequent words (with less than 100 occurrences) and screening out stop-words (such as the, he, in) and topic-irrelevant words (such as like, must, since) using the \texttt{R} package \texttt{tm}.
As the original dataset consists of email texts, we performed a manual screening to remove uninformative email-related vocabulary, such as edu, com, net, and cmu.
Finally, we conducted a secondary filtering of the short documents, as the removal of stop-words led to some documents being uninformative. The resulting processed train dataset is a sparse $5883 \times 653$ count matrix. 

\begin{figure}[h!]
\resizebox{\textwidth}{!}{
\begin{tikzpicture}[
  auto,
  node distance=0.7cm,
  every node/.style={
    draw,
    rectangle,
    minimum size=1cm,
    font=\sffamily\bfseries
  },
  arrow/.style={
    ->,
    line width=#1
  },
  every text node part/.style={align=center}
  ]

  \node (T1) {recreation};
  \node (T2) [right=10.4cm of T1] {computer};
  \node (T3) [right=10.3cm of T2] {science};

  \node (M3) [below=of T1] {sports-baseball};
  \node (M2) [left=of M3] {cars};
  \node (M1) [left=of M2] {motorcycle};
  \node (M4) [right=of M3] {sports-hockey};
  \node (M5) [right=of M4] {os-windows};
  \node (M6) [right=of M5] {windows};
  \node (M7) [right=of M6] {graphics};
  \node (M8) [right=of M7] {hardware-ibm};
  \node (M9) [right=of M8] {hardware-mac};
  \node (M10) [right=of M9] {medicine};
  \node (M11) [right=of M10] {space};
  \node (M12) [right=of M11] {electronics};

  \draw[arrow=0.5mm] (T1) -- (M1);
  \draw[arrow=0.5mm] (T1) -- (M2);
  \draw[arrow=0.5mm] (T1) -- (M3);
  \draw[arrow=0.5mm] (T1) -- (M4);
  \draw[arrow=0.5mm] (T2) -- (M5);
  \draw[arrow=0.4mm] (T2) -- (M6);
  \draw[arrow=0.5mm] (T2) -- (M7);
  \draw[arrow=0.3mm] (T2) -- (M8);
  \draw[arrow=0.3mm] (T2) -- (M9);
  \draw[arrow=0.3mm] (T3) -- (M10);
  \draw[arrow=0.3mm] (T3) -- (M11);
  \draw[arrow=0.3mm] (T3) -- (M12);
\end{tikzpicture}
}
\caption{Nested structure of the true held-out labels for the 20 newsgroups dataset.}
\label{fig:topic model true label hierarchy}
\end{figure}

To process the test dataset, we go through the same first and second steps described above. We continue using the dictionary constructed from the train set, with $J = 653$ words. After filtering out the short/long documents, the resulting data matrix consists of $N = 3320$ documents.

\paragraph{TIMSS Assessment Data.} 
We use the preprocessed dataset provided in \cite{lee2024new}. Additionally, we remove all rows of the data matrix with missing entries, which resulted in a total of $N = 435$ observations. While our estimation procedure can be modified to handle missing data (missing at random) by modifying the likelihood function, we take this additional preprocessing step for the sake of consistency with other parts of the paper.

This dataset includes additional information regarding the latent structure $\GG^{(1)}$, which is the so-called $Q$-matrix in the cognitive diagnostic modeling literature \citep{von2008general, lee2024new}. The provisional $\GG^{(1)}$ matrix specifies $K^{(1)} = 7$ latent cognitive skills (Number, Algebra, Geometry, Data and Probability, Knowing, Applying, and Reasoning) as well as the skills that are required to solve each item. That is, $g^{(1)}_{j,k}=1$ if item $j$ requires latent skill $k$ to solve it. Hence, following the confirmatory latent variable modeling convention in psychometrics, we estimate the DDE parameters by fixing $\GG^{(1)}$ to this {given} structure.

\subsection{Selecting the Number of Latent Variables}
\begin{figure}
    \centering
    \includegraphics[width = 0.45\textwidth]{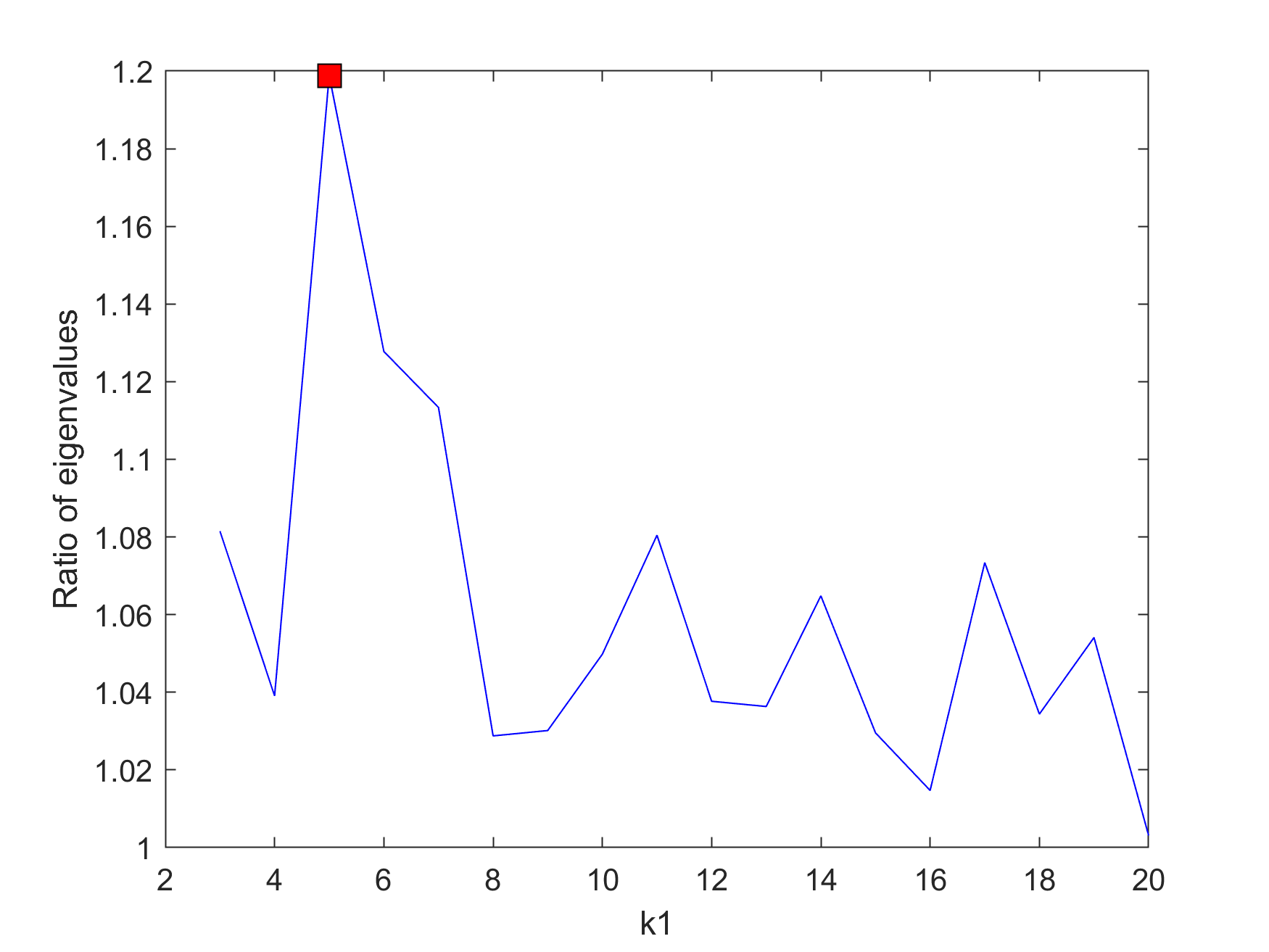}
    \quad
    \includegraphics[width = 0.45\textwidth]{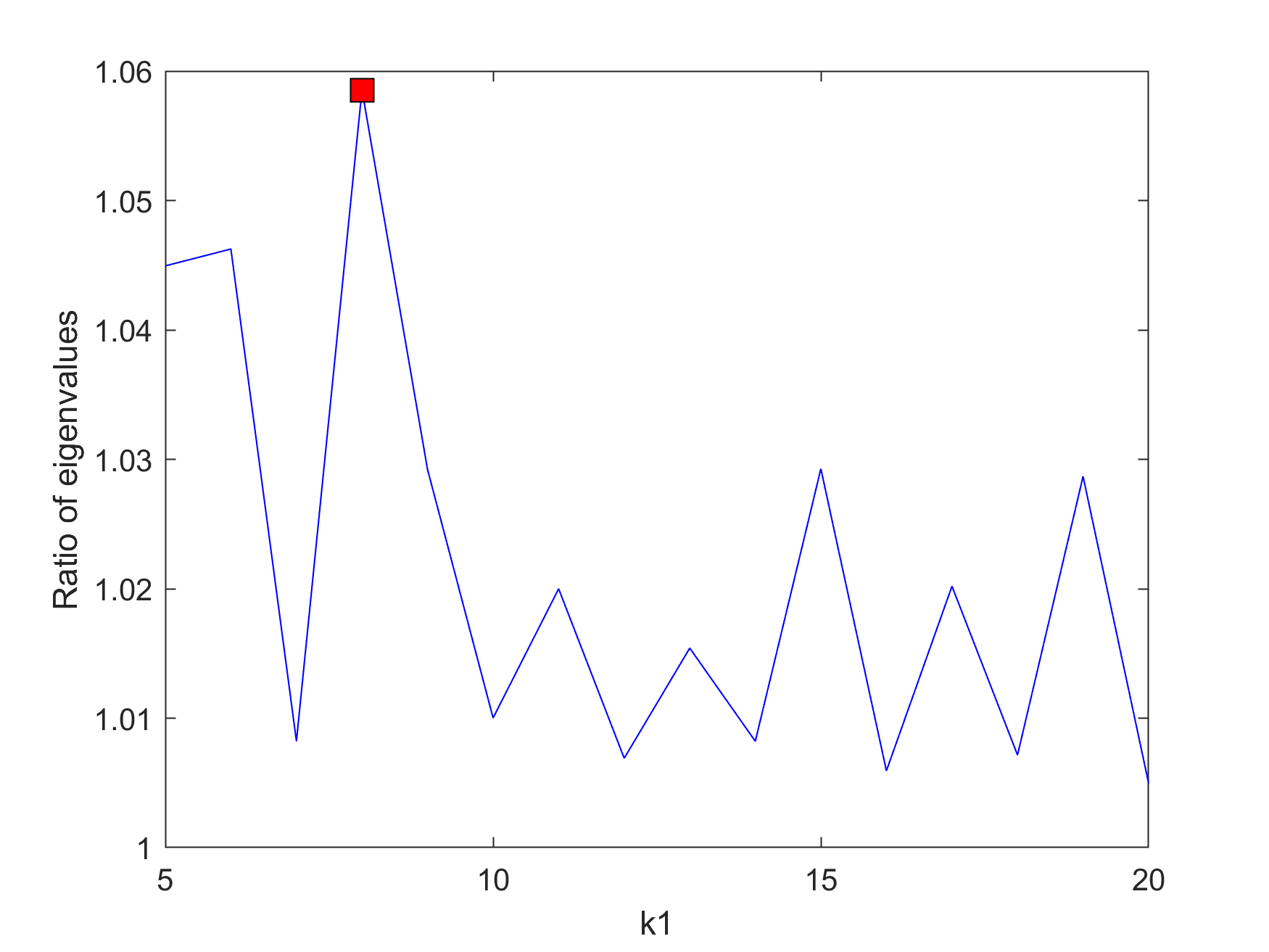}
    \caption{Based on the spectral estimator for $K^{(1)}$, (left) we select $K^{(1)}=5$ for MNIST, and (right) $K^{(1)}=8$ for 20 newsgroups. The peak, highlighted in red, correspond to the selected values. We omit displaying the first few eigenvalue ratios for better illustration.}
    \label{fig:eigenvalue ratio mnist}
    \vspace{5mm}
\end{figure}

\paragraph{MNIST Data.}
We select $K^{(1)} = 5$ based on the spectral-gap estimator, as illustrated in the left panel of \cref{fig:eigenvalue ratio mnist}. Also, we set $K^{(2)} = 2$ motivated by our identifiability requirement $2K^{(2)} < K^{(1)}$. This choice is also supported by the fact that there are four true labels. Each latent configuration $\ma^{(2)} = (A_1^{(2)}, A_2^{(2)}) \in \{0,1\}^2$ uniquely corresponds to each digit, as illustrated in the center panel of \cref{fig:mnist label estimation}.

\paragraph{20 Newsgroups Data.}
We select $K^{(1)} = 8$ based on the spectral-gap estimator, as illustrated in the right panel of \cref{fig:eigenvalue ratio mnist}.
For $K^{(2)}$, the EBIC values are quite similar for $K^{(2)} = 2, 3$ with values of $1.3273 \times 10^6$ and $1.3277 \times 10^6$, respectively. Also, the spectral-gap is similar for both values $K^{(2)} = 2, 3$. Hence, we select $K^{(2)}$ based on the interpretability of the inferred latent structure.
The estimated latent structure with $K^{(2)} = 3$ is displayed in \cref{fig:overfitted topic model}. Compared to \cref{fig:topic model} in the main text, the `technology' topic has split into two groups whose interpretation is somewhat blurry. Notably, these two `technology' latent variables do not match the labels of `computer' and `science' provided by the dataset (see \cref{fig:topic model true label hierarchy}). For instance, the `technology 2' variable has arrows to the finer topics of motorcycle, software, and space. Hence, we select $K^{(2)} = 2$ based on this lack of semantic distinction.

\begin{figure}
\resizebox{\textwidth}{!}{
\begin{tikzpicture}[
  auto,
  node distance=1cm,
  every node/.style={
    draw,
    rectangle,
    minimum size=1cm,
    font=\tiny\sffamily\bfseries
  },
  arrow/.style={
    ->,
    line width=#1
  },
  every text node part/.style={align=center}
  ]

  \node (T1) {recreation};
  \node (T2) [right=2cm of T1] {technology 1};
  \node (T3) [right=2cm of T2] {technology 2};

  \node (M3) [below=of T1] {cars};
  \node (M2) [left=of M3] {sports};
  \node (M1) [left=of M2] {motorcycle};
  \node (M4) [right=of M3] {software};
  \node (M5) [right=of M4] {location/names};
  \node (M6) [right=of M5] {hardware};
  \node (M7) [right=of M6] {graphics};
  \node (M8) [right=of M7] {space};

  \draw[arrow=1.5mm] (T3) -- (M1);
  \draw[arrow=1.65mm] (T1) -- (M2);
  \draw[red, arrow=1.08mm] (T2) -- (M2);
  \draw[arrow=0.67mm] (T1) -- (M3);
  \draw[arrow=0.86mm] (T2) -- (M3);
  \draw[arrow=1.15mm] (T1) -- (M4);
  \draw[arrow=0.73mm] (T2) -- (M4);
  \draw[arrow=0.4mm] (T3) -- (M4);
  \draw[arrow=0.51mm] (T1) -- (M5);
  \draw[red, arrow=-0.36mm] (T1) -- (M6);
  \draw[arrow=0.36mm] (T2) -- (M6);
  \draw[red, arrow=0.46mm] (T1) -- (M7);
  \draw[arrow=1.05mm] (T2) -- (M7);
  \draw[arrow=0.5mm] (T3) -- (M8);
\end{tikzpicture}
}
\caption{Graphical structure of the latent topics, where we fit the 2-layer DDE with $K^{(2)} = 3, K^{(1)} = 8$. The width of the upper layer arrows is proportional to the corresponding coefficients and the red arrow indicates negative values.}
    \label{fig:overfitted topic model}
\end{figure}

\paragraph{TIMSS Assessment Data.}
This dataset includes additional information regarding the latent structure $\GG^{(1)}$, which is the so-called $Q$-matrix in the cognitive diagnostic modeling literature \citep{von2008general, lee2024new}. The provisional $\GG^{(1)}$ matrix specifies $K^{(1)} = 7$ latent cognitive skills (Number, Algebra, Geometry, Data and Probability, Knowing, Applying, and Reasoning) as well as the skills that are required to solve each item. That is, $g^{(1)}_{j,k}=1$ if item $j$ requires latent skill $k$ to solve it.
Hence, following the confirmatory latent variable modeling convention in psychometrics, we estimate the DDE parameters by fixing $\GG^{(1)}$ to this {given} structure.

Additionally, the dataset comes with a rough classification of the seven latent skills into $K^{(2)}=2$ categories: content skills and cognitive skills. However, we observed that fitting $K^{(2)} = 2$ results in one of the proportion parameters exhibiting an unusually large value of $p_2 = 0.96$. This indicates that $A_2^{(2)} = 1$ for $96\%$ of the students, and this skill is redundant. Hence, we have adjusted the number of higher-order latent variables to $K^{(2)} = 1$, which partitions the students into two groups based on the value of $A^{(2)}_1$. The EBIC for the two cases were also comparable and justified this choice.

\subsection{Additional Visualization and Performance Evaluation for the MNIST Dataset}

\subsubsection{Additional visualization}
\begin{figure}[h!]
    \centering
\resizebox{0.5\textwidth}{!}{
\centering
    \begin{tikzpicture}[scale=2,
  auto,
  node distance=1cm,
  arrow/.style={
    ->,
    line width=#1}
  ] 

    \node (o1)[neuron] at (-0.8, -0.8) {$Y_1$};
    \node (o2)[neuron] at (0, -0.8) {$Y_2$};
    \node (o3)[neuron] at (0.8, -0.8) {$Y_3$};
    \node (o4)[neuron] at (1.6, -0.8) {$Y_4$};
    \node (o5)[neuron] at (2.4, -0.8) {$\cdots$};
    \node (o6)[neuron] at (3.2, -0.8) {$\cdots$};
    \node (o7)[neuron] at (4.0, -0.8) {$Y_J$};

    \node (v1)[hidden] at (0, 0) {$A_1^{(1)}$};
    \node (v2)[hidden] at (0.8, 0) {$A_2^{(1)}$};
    \node (v3)[hidden] at (1.6, 0) {$A_3^{(1)}$};
    \node (v4)[hidden] at (2.4, 0) {$A_4^{(1)}$};
    \node (v5)[hidden] at (3.2, 0) {$A_5^{(1)}$};
       
    \node (h1)[hidden] at (0.8, 0.8) {$A_1^{(2)}$};
    \node (h2)[hidden] at (2.4, 0.8) {$A_2^{(2)}$};
    
    \draw[arrow=2mm] (h1) -- (v1);
    \draw[arrow=2mm] (h1) -- (v2);
    \draw[arrow=0.67mm] (h1) -- (v3);
    \draw[arrow=0.22mm] (h1) -- (v4);
    \draw[arrow=1.7mm] (h1) -- (v5);
    \draw[arrow=0.26mm,red] (h2) -- (v1);
    \draw[arrow=0.63mm] (h2) -- (v2);

    \draw[arrow=0.5mm] (v5)--(o1);
    \draw[arrow=0.35mm] (v5)--(o2);
    \draw[arrow=0.46mm] (v5)--(o3);
    \draw[arrow=0.36mm,red] (v1)--(o4);
    \draw[arrow=0.35mm,red] (v2)--(o4);
    \draw[arrow=1.6mm] (v5)--(o4);
    \draw[arrow=1.9mm,red] (v2)--(o7);
    \draw[arrow=0.5mm,red] (v3)--(o7);
    \draw[arrow=0.71mm] (v5)--(o7);
\end{tikzpicture}
}
    \caption{Graphical model representation of the learned latent structure. The edge widths are proportional to coefficients' absolute values. The edge colors (black/red) imply positive/negative coefficients, respectively.}
    \label{fig:mnist graphical model}
\end{figure}

\noindent 
In \cref{fig:mnist graphical model}, we visualize the learned latent structure of the MNIST dataset. The corresponding graphical matrix $\GG^{(2)}$ was used to interpret the top layer latent variables in the main text. For example, $A_1^{(2)}$ is connected to all lower-layer variables, and we interpret this as an indicator for the \emph{pixel density} of the image. 

\begin{table}
\centering
\resizebox{\textwidth}{!}{
\begin{tabular}{c|*5{C}@{}}
\toprule
Label & True & LCM & 1-layer DDE & Spectral init &  2-layer DDE \\ 
\midrule
0 & \includegraphics[width=\linewidth]{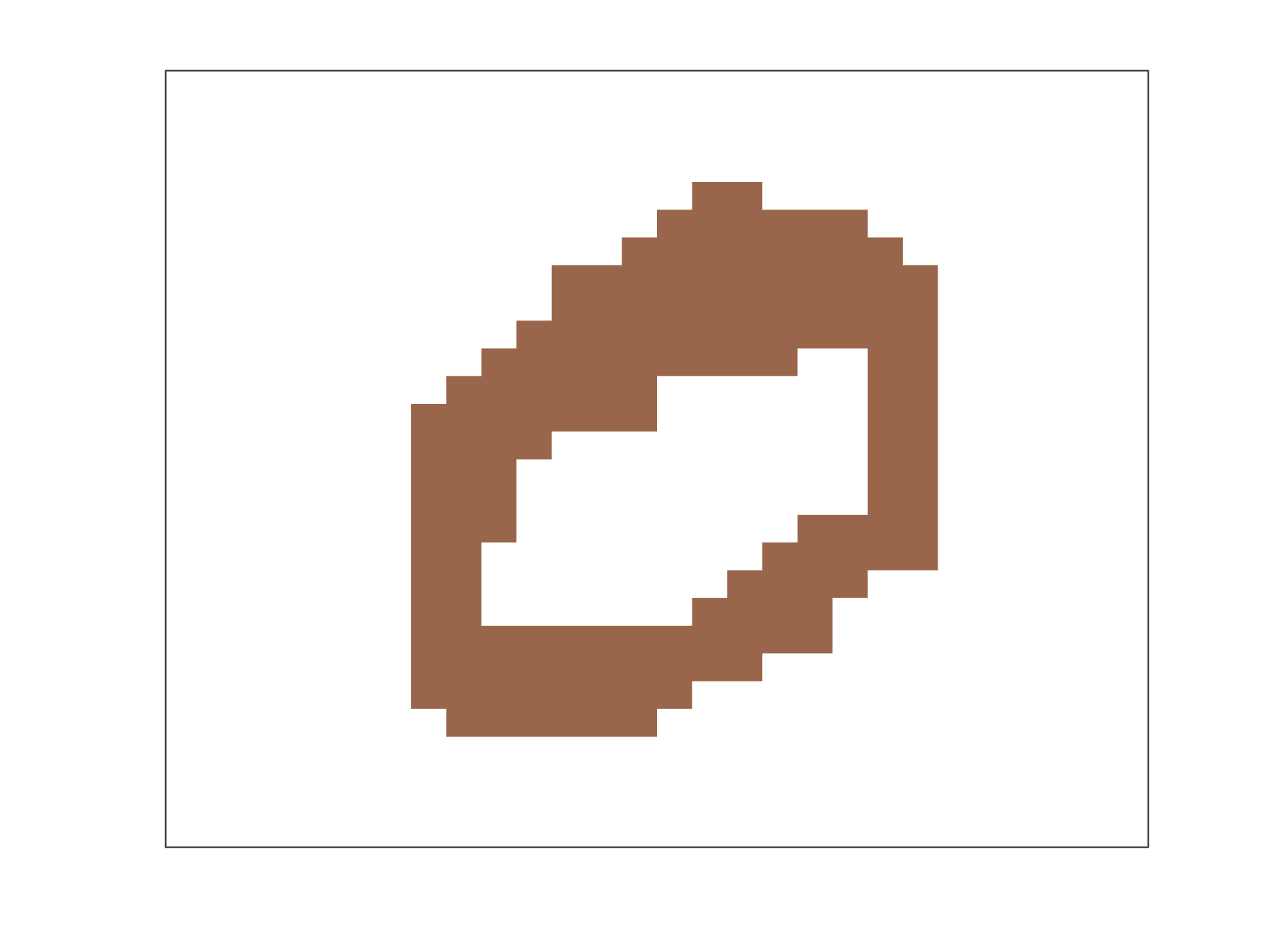} & \includegraphics[width=\linewidth]{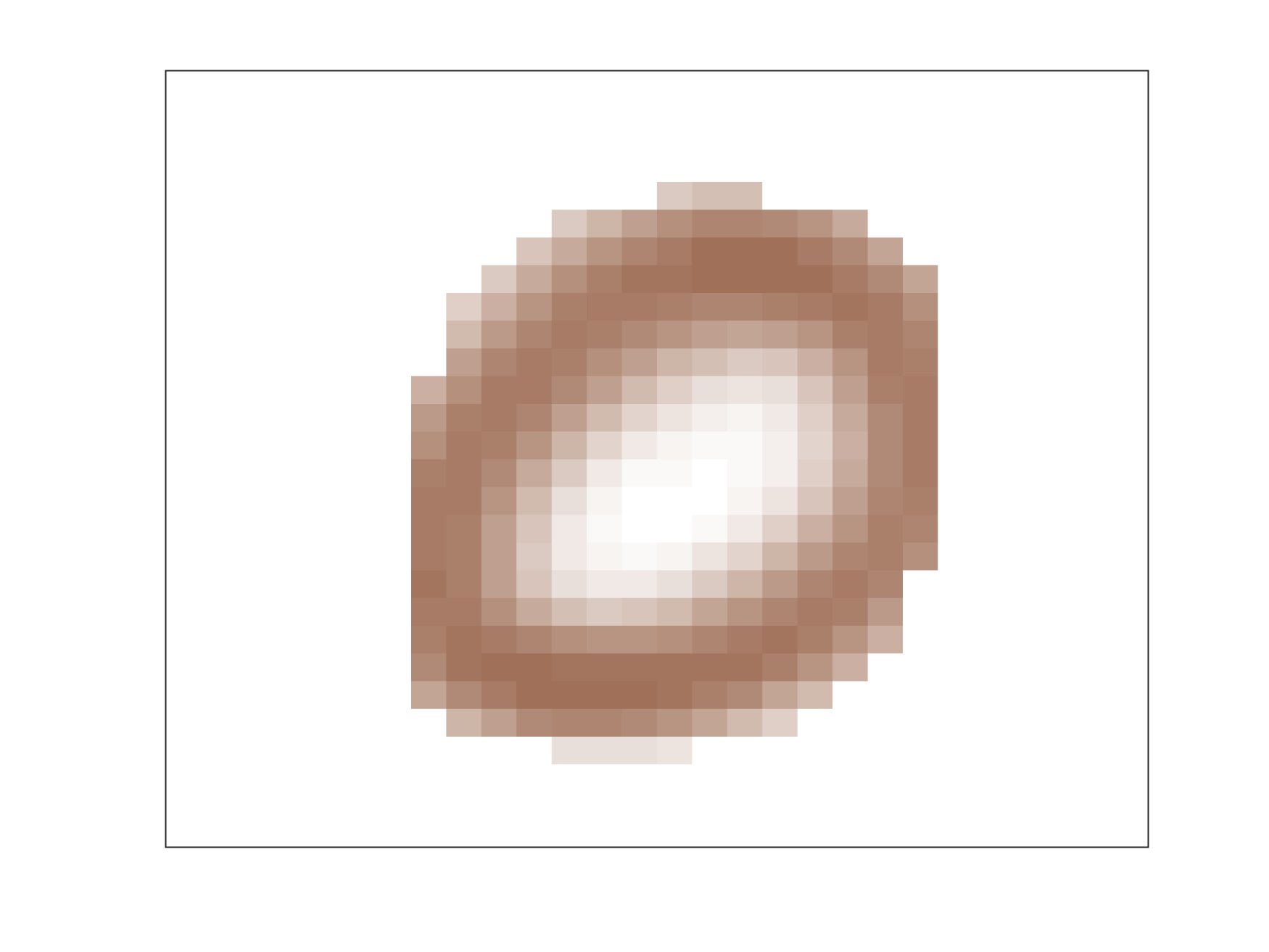}  & \includegraphics[width=\linewidth]{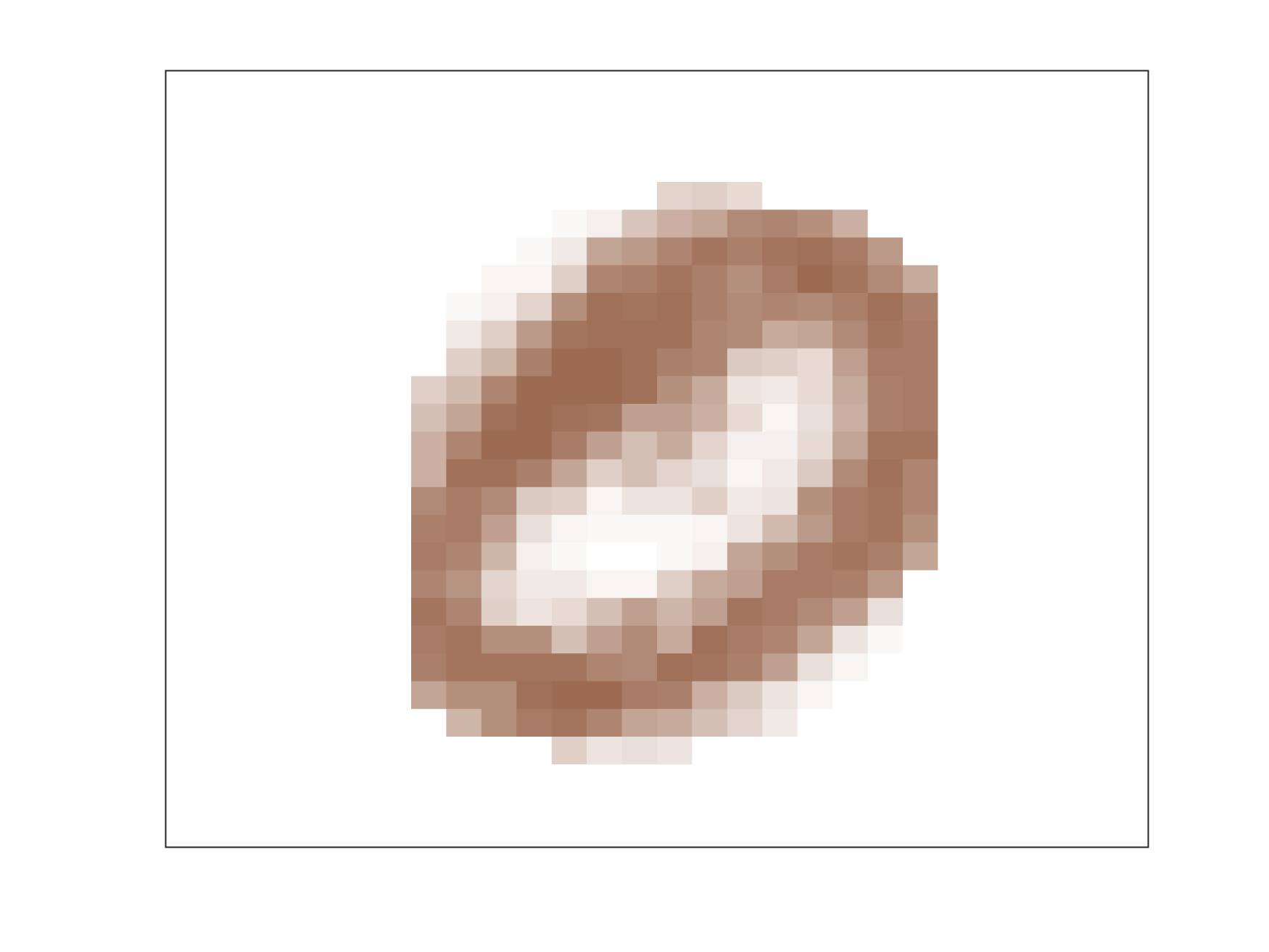} 
& \includegraphics[width=\linewidth]{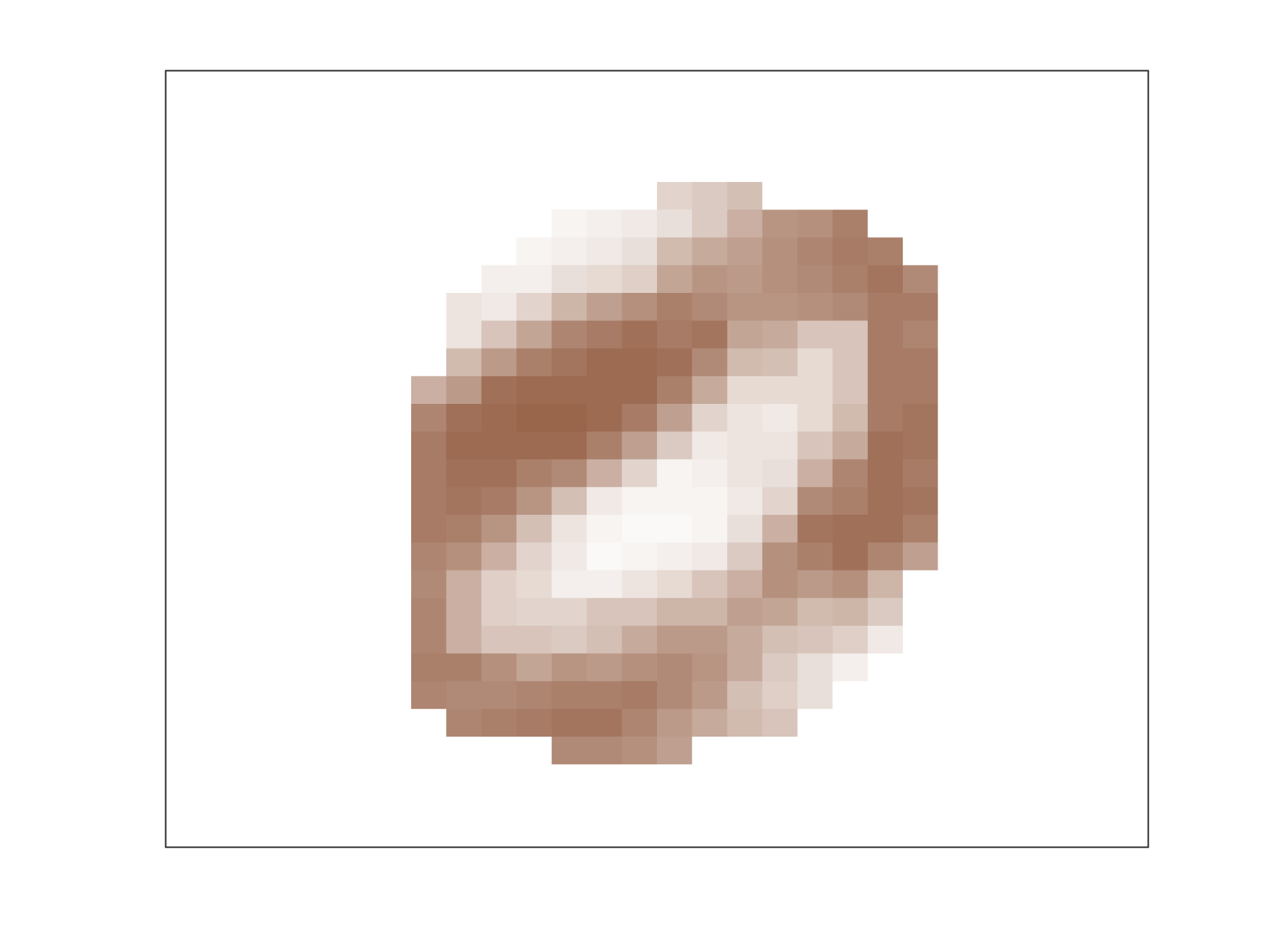} & \includegraphics[width=\linewidth]{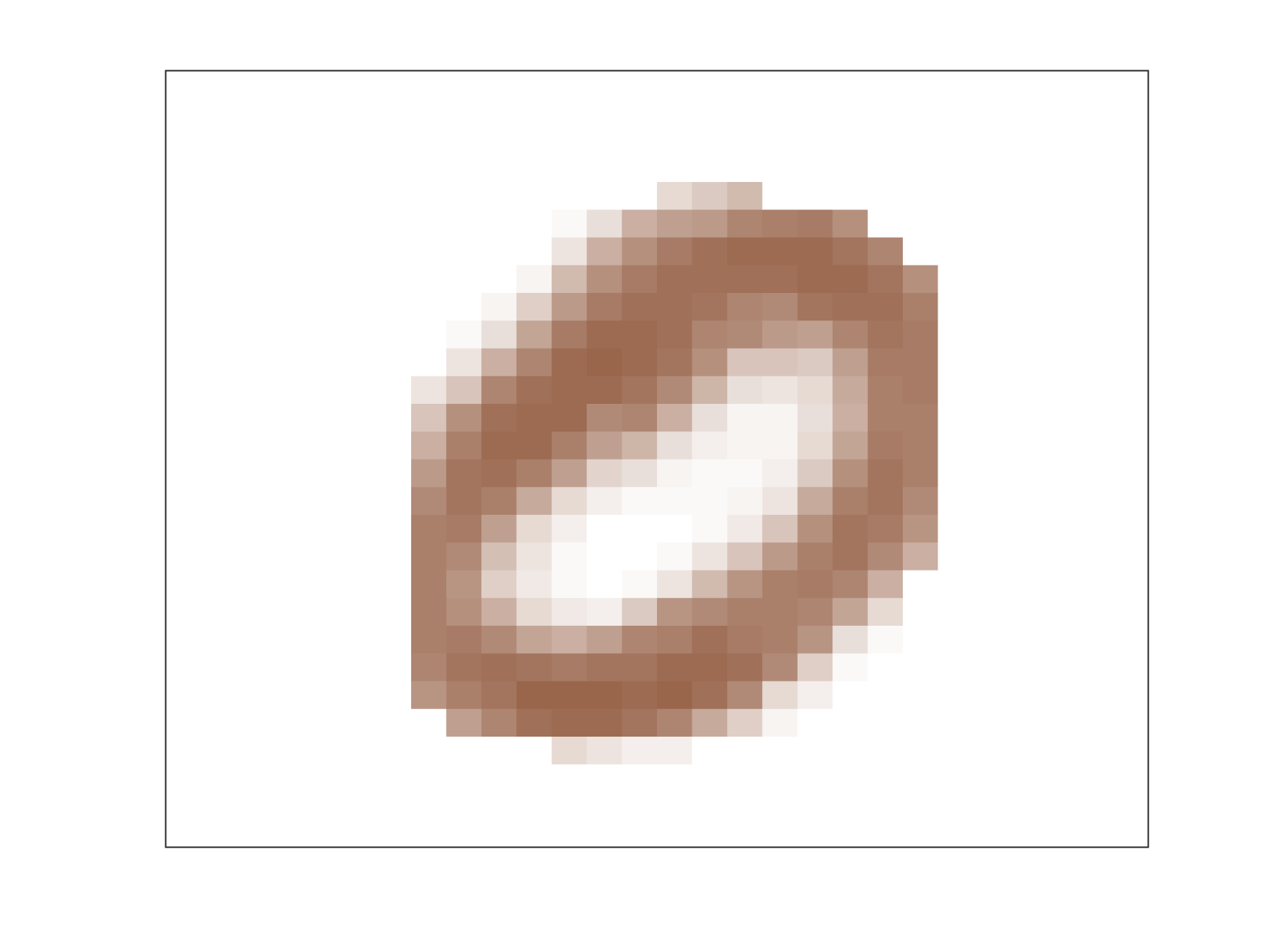} \\
1 & \includegraphics[width=\linewidth]{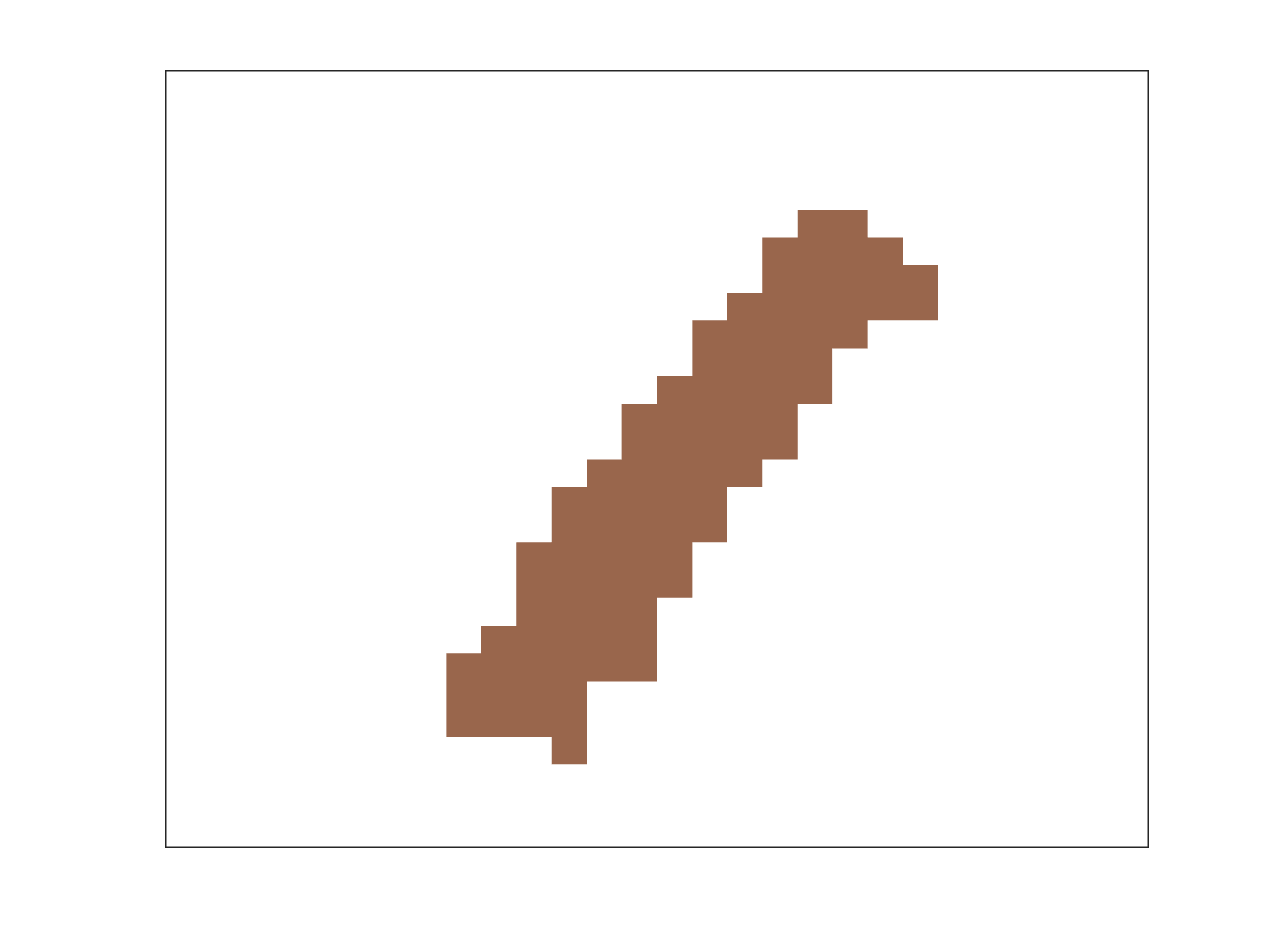}  & \includegraphics[width=\linewidth]{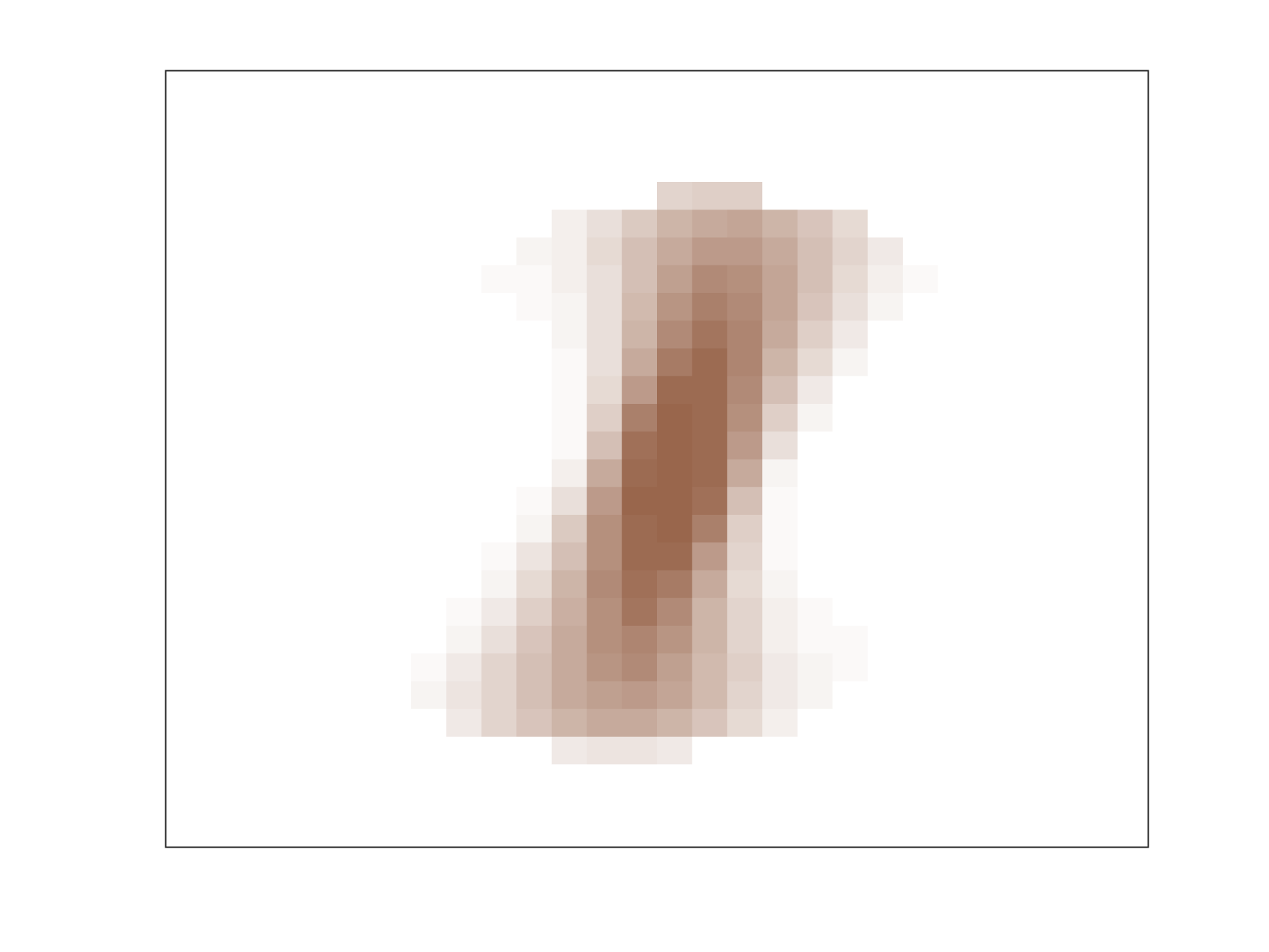} & \includegraphics[width=\linewidth]{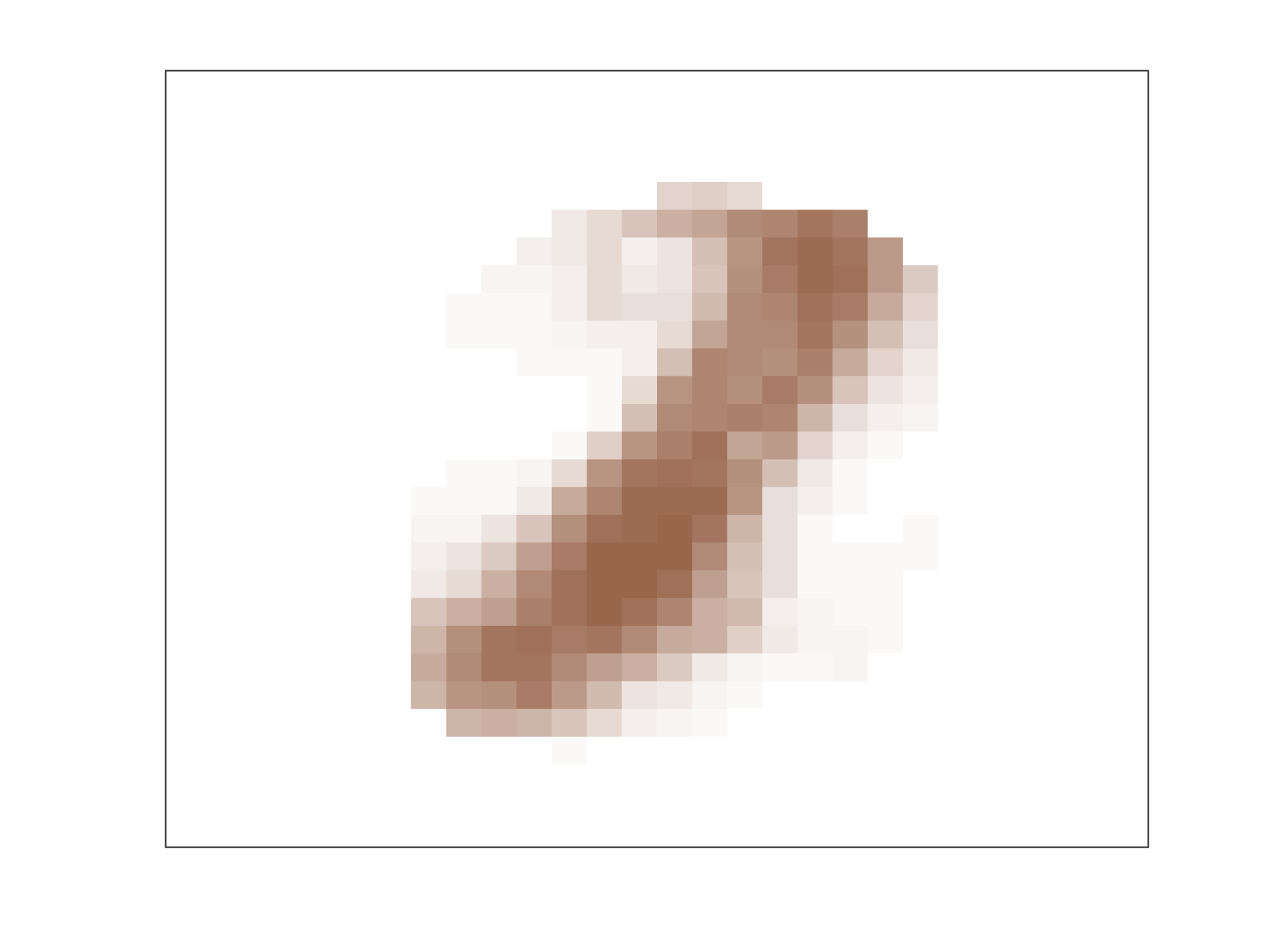} & \includegraphics[width=\linewidth]{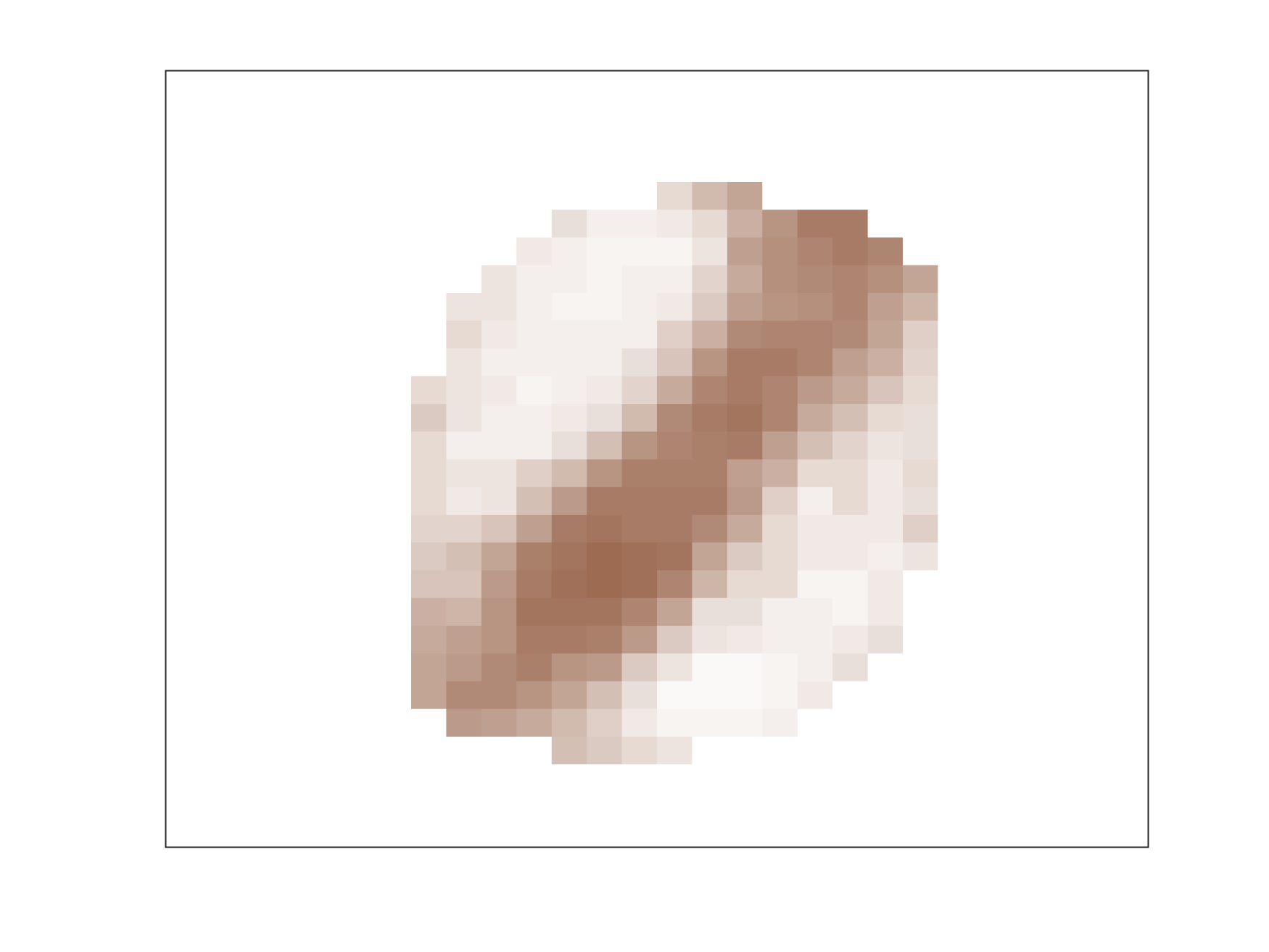} & \includegraphics[width=\linewidth]{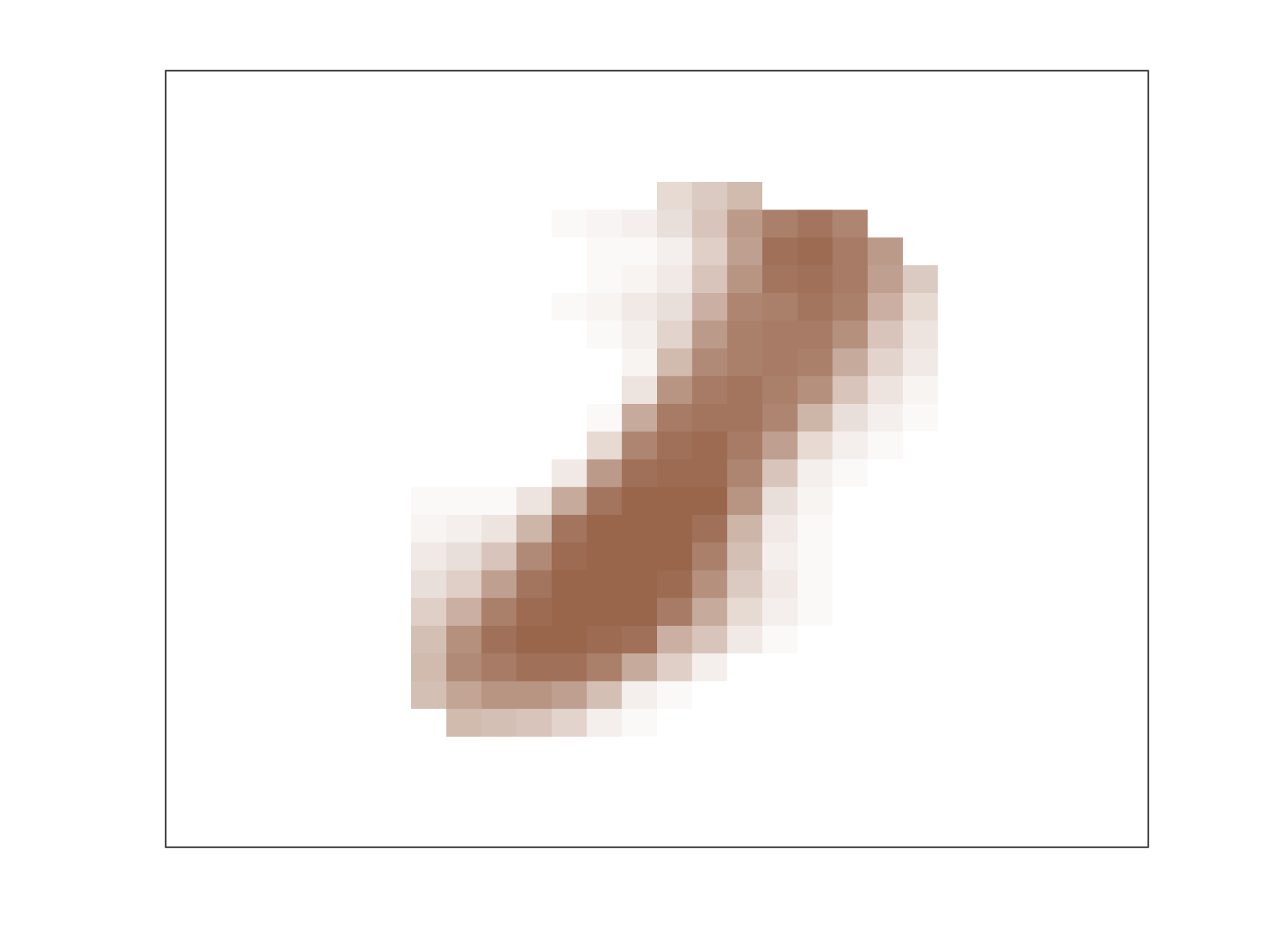} \\
2 & \includegraphics[width=\linewidth]{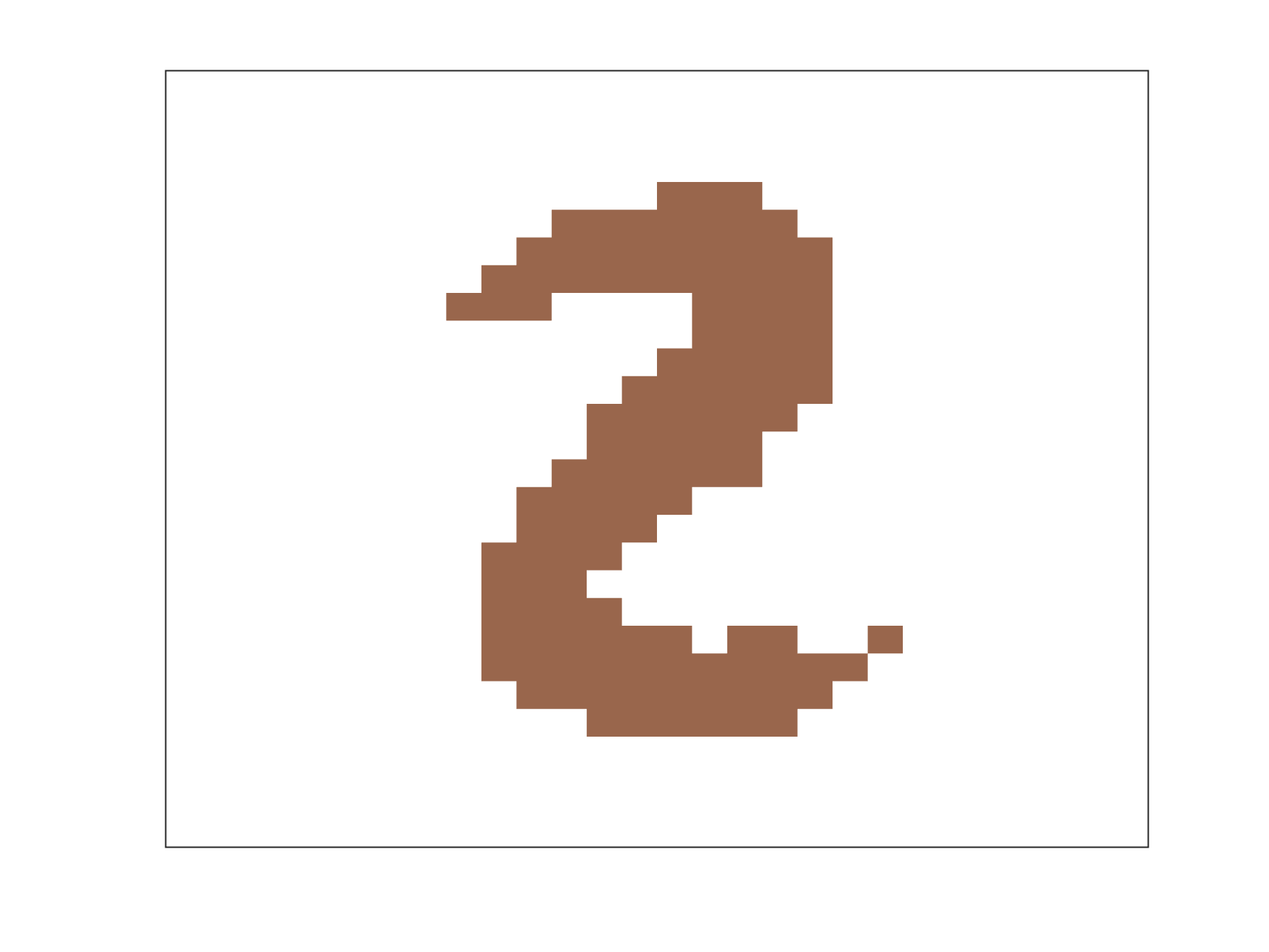} & \includegraphics[width=\linewidth]{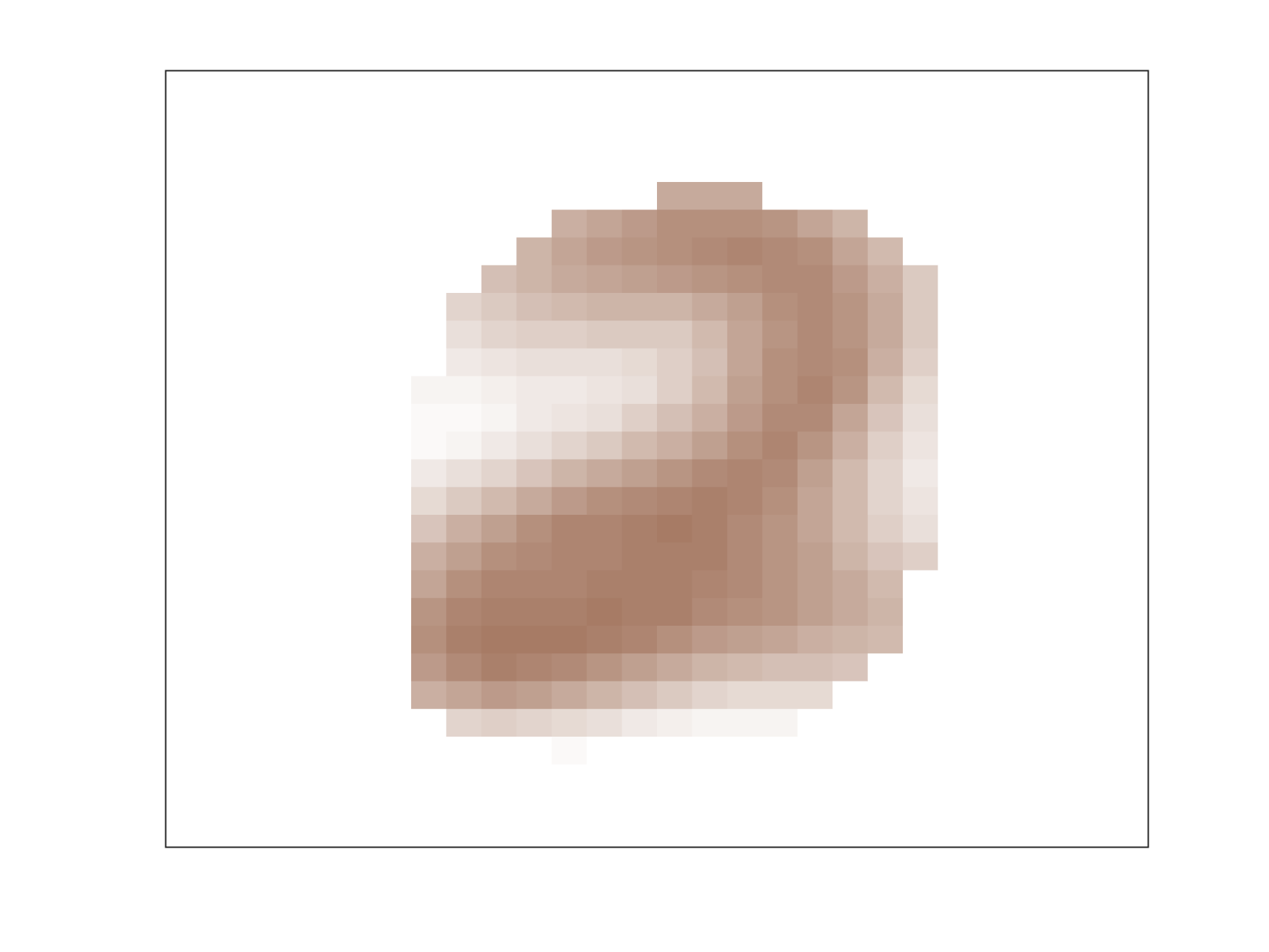} & \includegraphics[width=\linewidth]{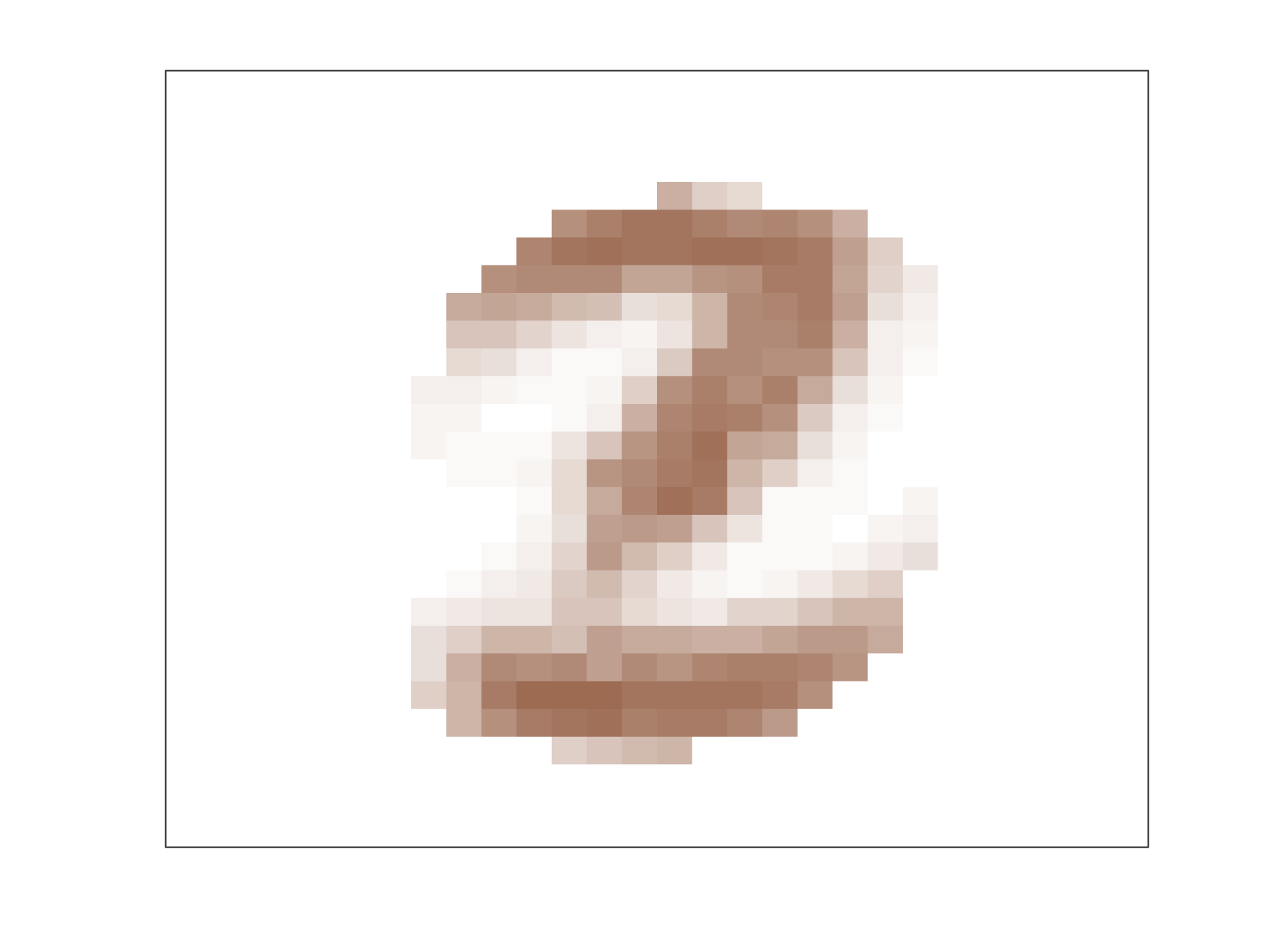} & \includegraphics[width=\linewidth]{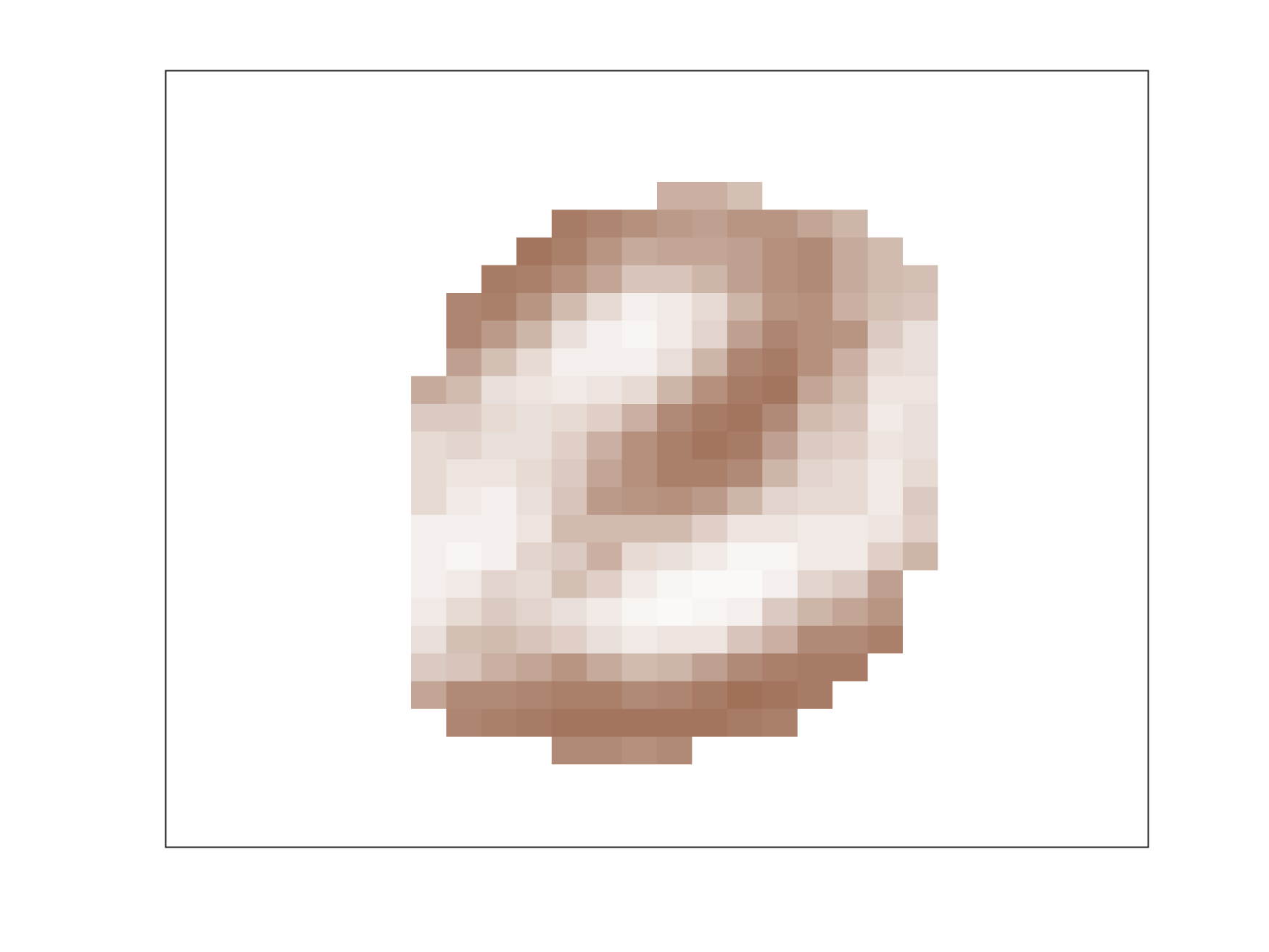} & \includegraphics[width=\linewidth]{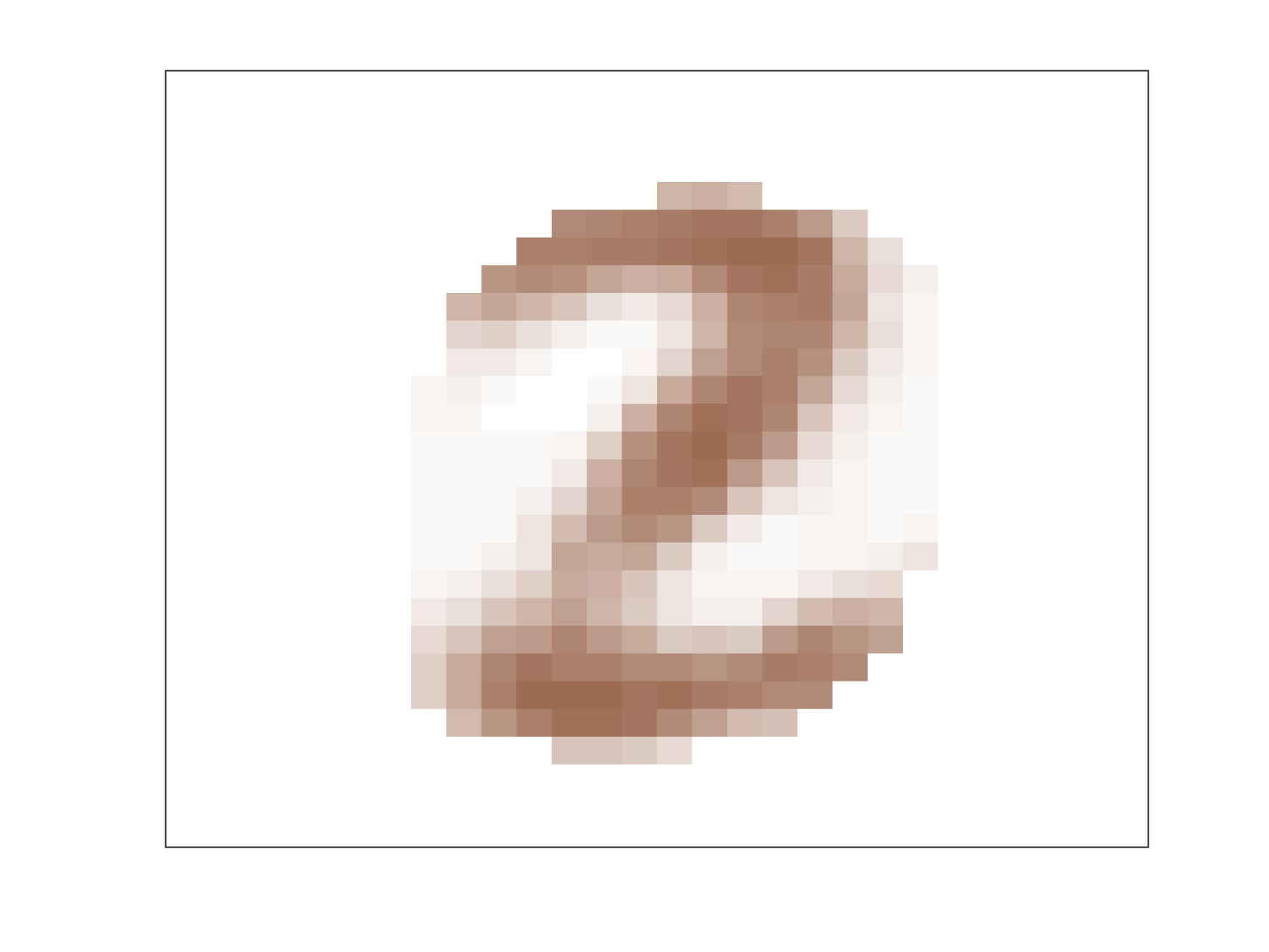}  \\
3 & \includegraphics[width=\linewidth]{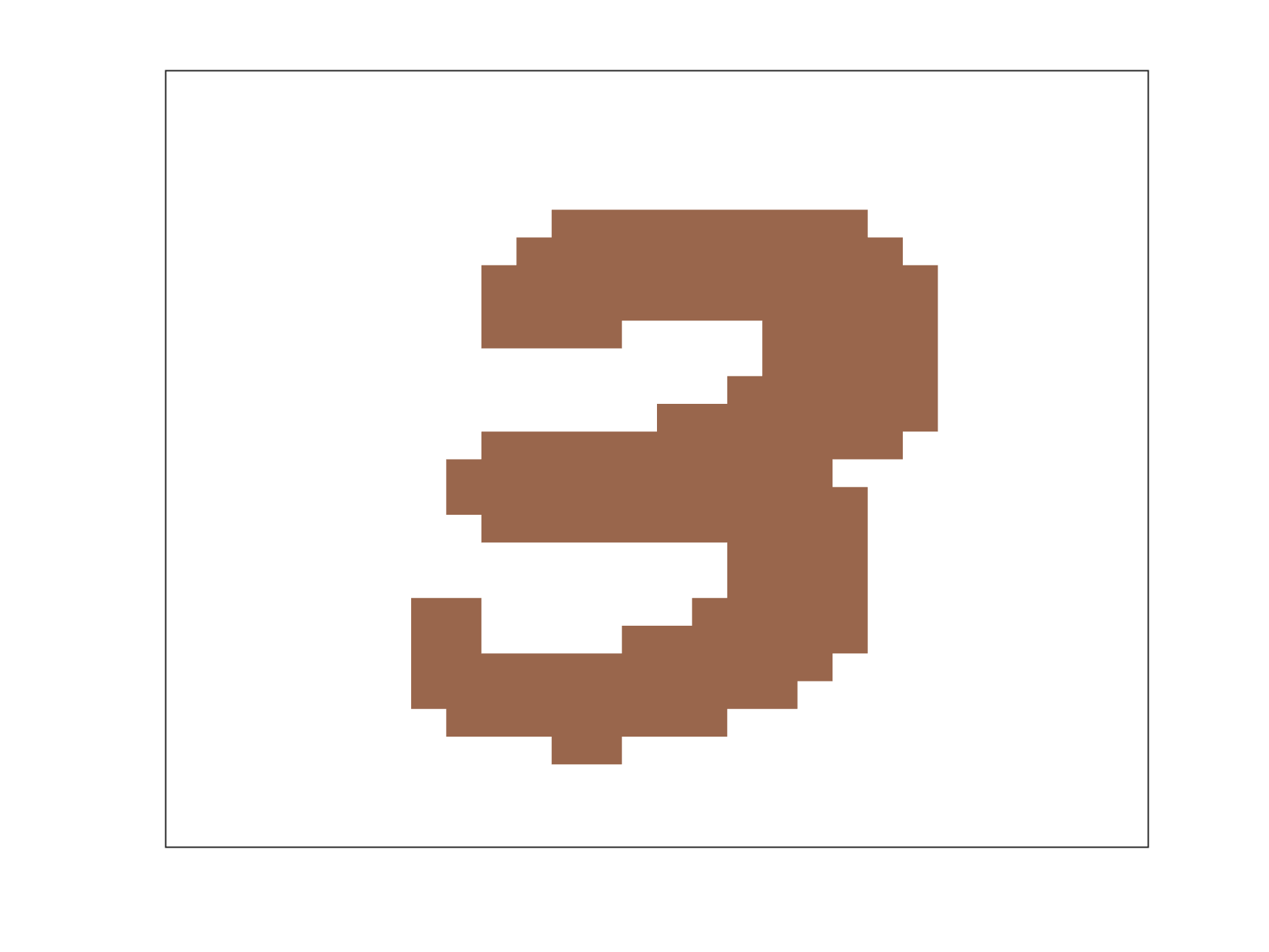} & \includegraphics[width=\linewidth]{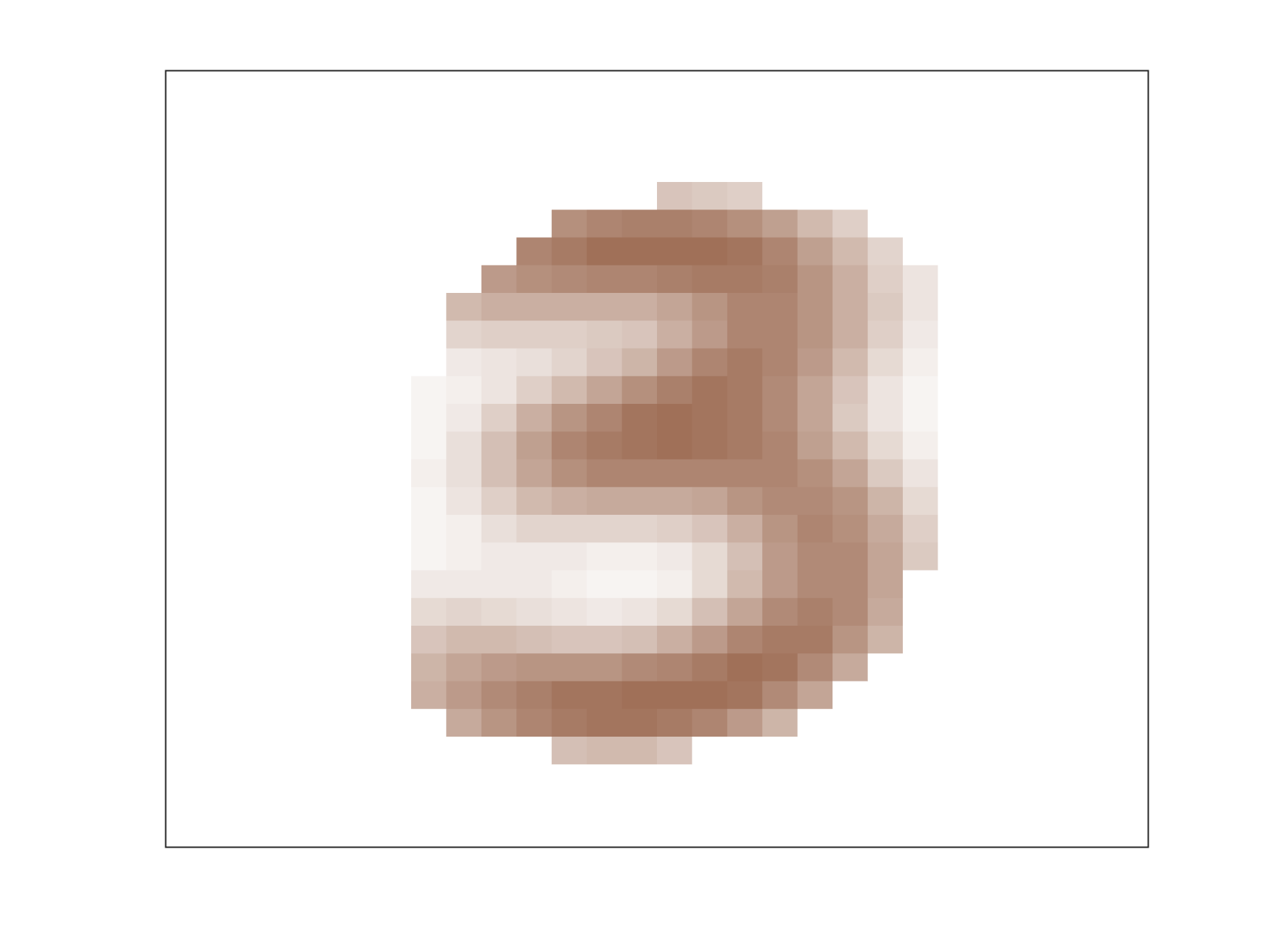} &\includegraphics[width=\linewidth]{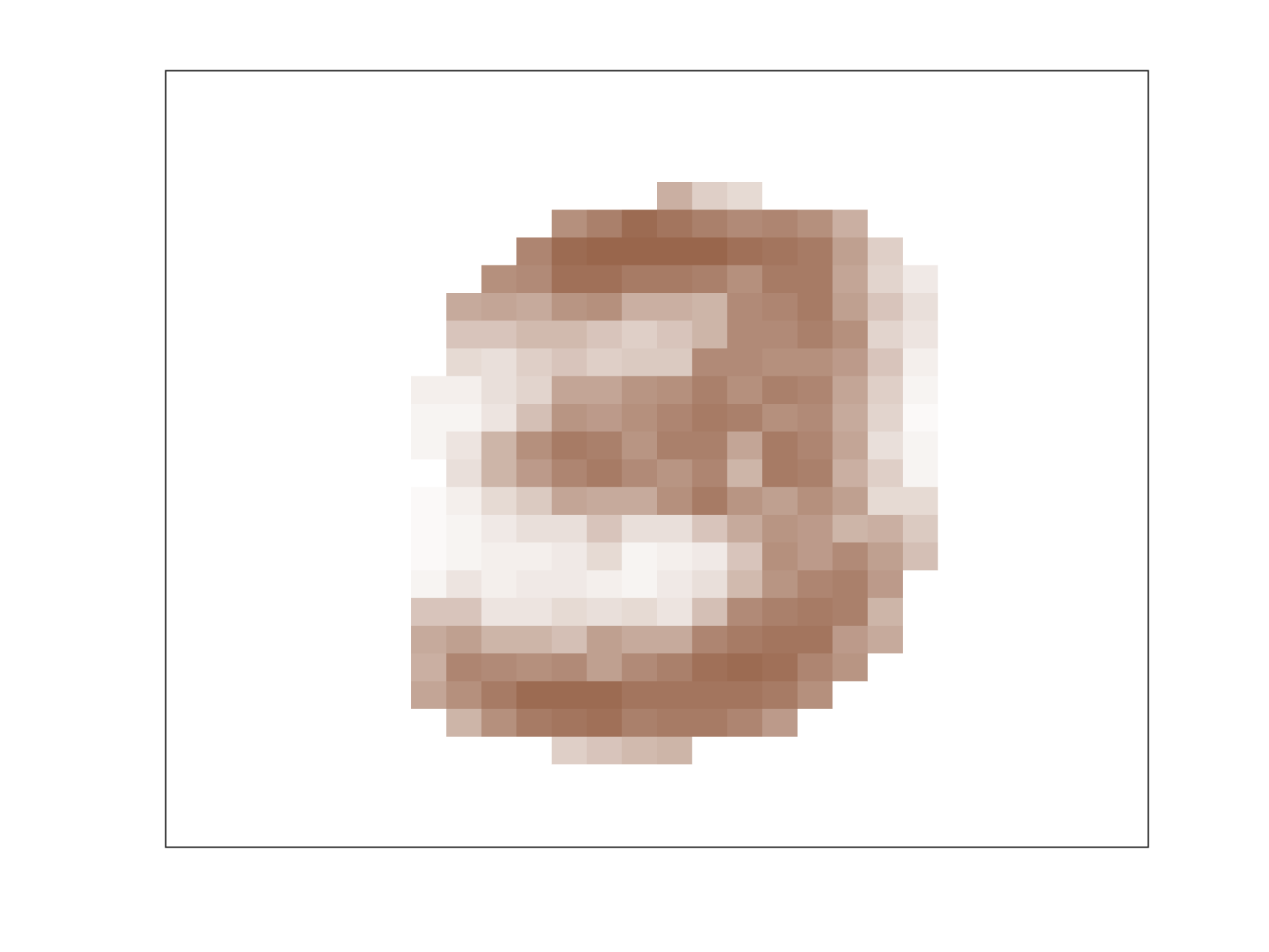} 
& \includegraphics[width=\linewidth]{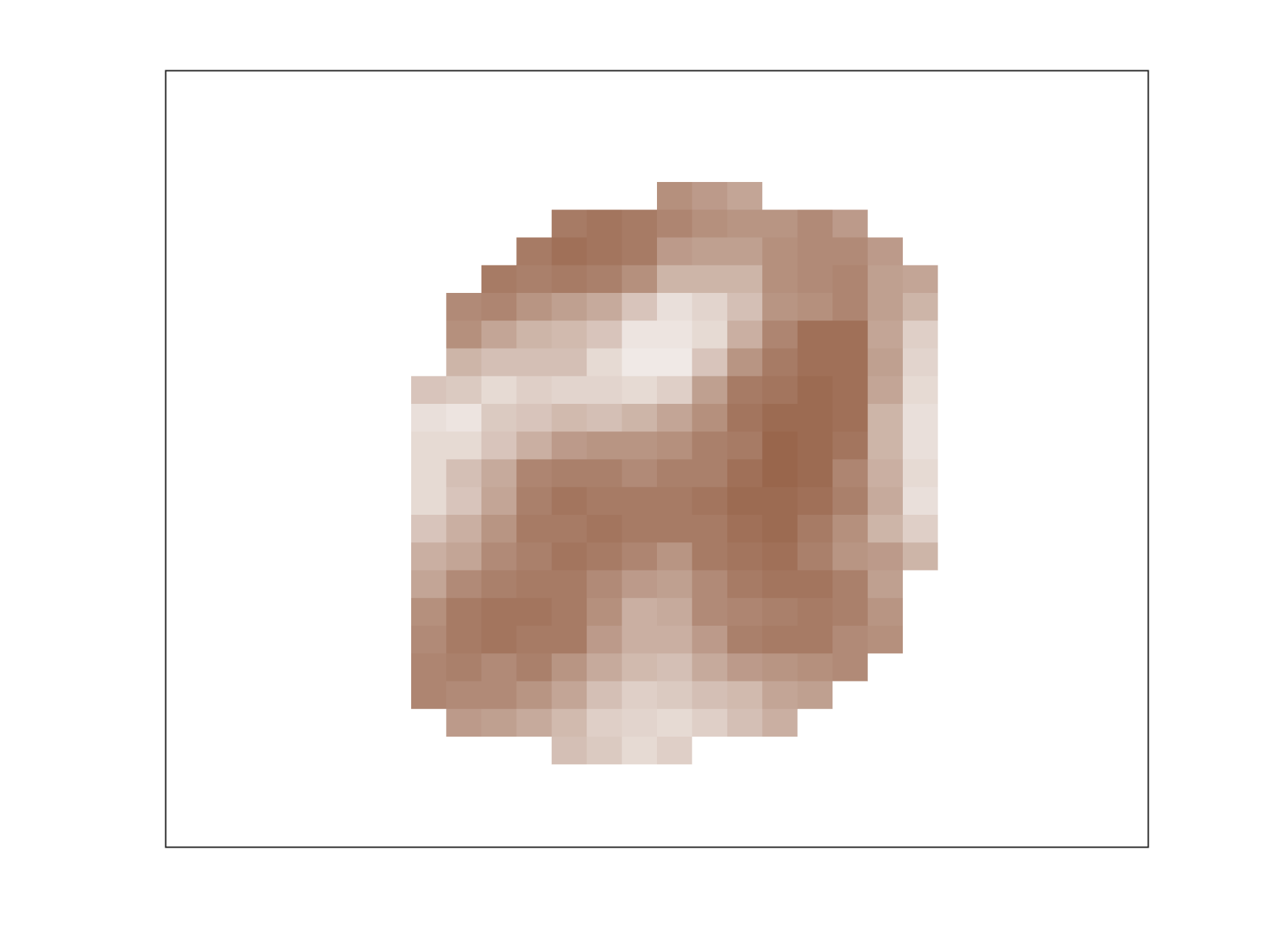} & \includegraphics[width=\linewidth]{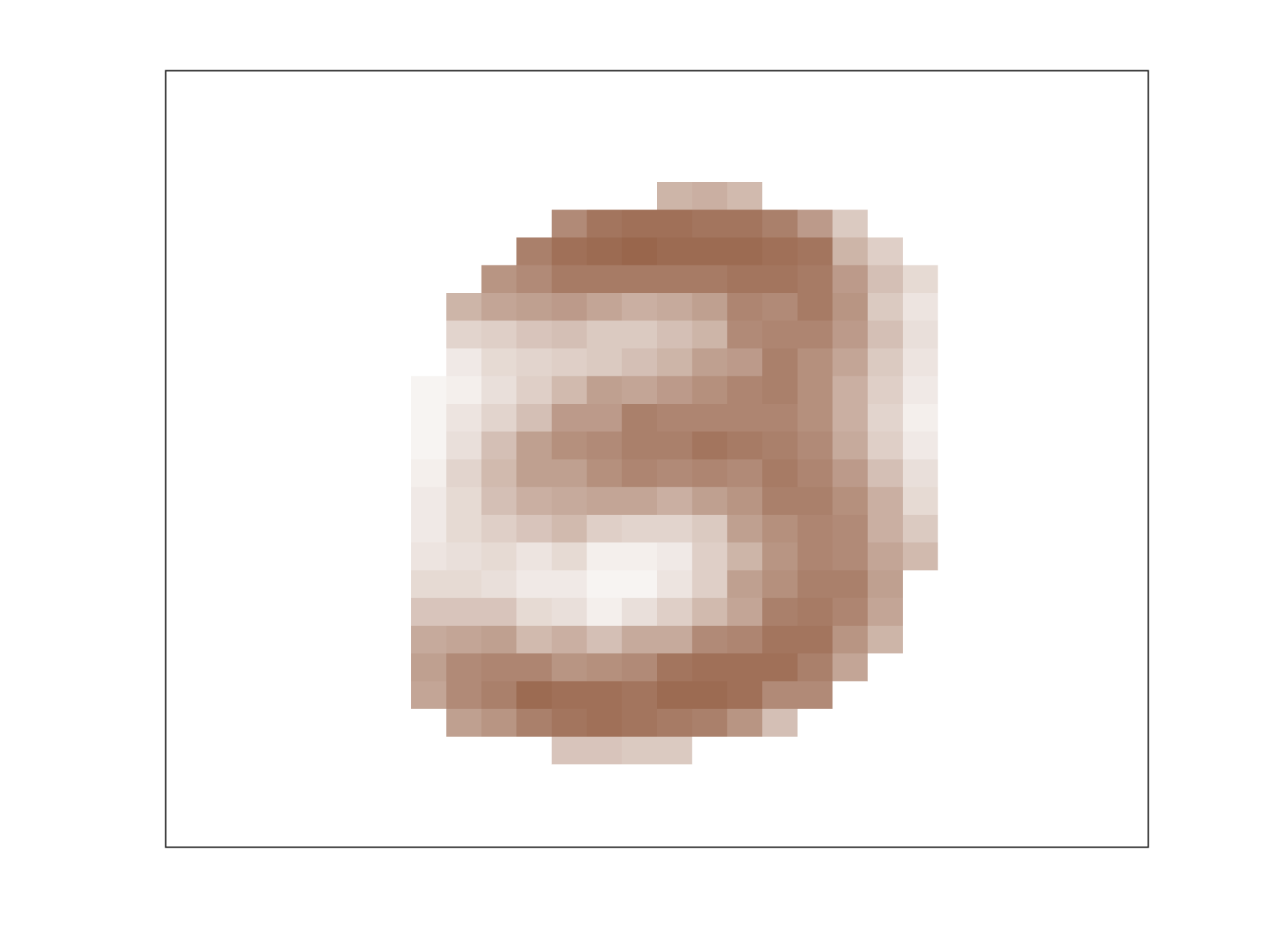} \\
\bottomrule 
\end{tabular}
}
\caption{True (first column) and reconstructed (other columns) images under various estimators. Here, each pixel takes values in $[0,1]$, where a darker shade indicates a larger value.}
\label{tab:reconstruction MNIST appendix}
\end{table}

We additionally visualize examples of reconstructed images under various models in \cref{tab:reconstruction MNIST appendix}, which have been used to compute the pixel-wise reconstruction accuracy in \cref{tab:reconstruction MNIST} in the main paper. The reconstructions are based on the corresponding conditional mean $\mathbb{E} \left(\YY \mid \ma^{(1)}\right)$ for each latent representation $\ma^{(1)}$, which is rearranged in the original $28 \times 28$ grid. The visualization illustrates that the reconstructions from the 2-layer DDE exhibit greater clarity compared to those from alternative models.

\subsubsection{\darkblue{Comparison with VAEs and iVAEs}}
We provide additional numerical comparison against VAEs and identifiable VAEs (iVAE), alongside implementation details. 
We use the standard VAE and iVAE in \cite{khemakhem2020variational} with varying latent dimensions $K = 2, 5$ (these numbers are motivated by the latent dimensions $K^{(1)}=5, K^{(2)}=2$ used for DDEs). The values presented in Table 4 in the main text correspond to the VAE with latent dimension $K=2$. While we use identical neural network architectures (two-layer perceptrons with 50 hidden variables in the middle layer) and factorized priors for both VAE and iVAE, the latter \textit{requires auxiliary information} to construct ``label priors'' while the former does not. Hence, we have provided the true digit labels as auxiliary data for iVAEs.

\begin{table}[h!]
\centering
\begin{tabular}[h!]{lccccc}
    \toprule
    Accuracy & 2-DDE & VAE ($K = 2$) & VAE ($K = 5$) & iVAE ($K = 2$) & iVAE ($K = 5$) \\ 
    \midrule
    Train classif. (\%) & 92.0 & 97.1 & 97.5 & 97.5 & 98.2 \\
    Test classif. (\%) & 92.6 & 92.0 & 92.1 & 92.5 & 93.2 \\
    Train recon.  (\%) & 79.6 & 82.5 & 83.4 & 82.5 & 86.6 \\
    Test recon. (\%)   & 79.9 & 82.7 & 83.6 & 82.8 & 86.6  \\
    \bottomrule 
\end{tabular}
\captionof{table}{Performance comparison of DDEs versus VAEs and iVAEs with varying latent dimensions $K$ on the MNIST dataset. Note that the iVAEs are trained in a supervised manner, where the true labels are incorporated as auxiliary information.}
\label{tab:reconstruction MNIST ivae}
\end{table}
We report the evaluation metrics in \cref{tab:reconstruction MNIST ivae}. \emph{Compared to VAEs}, DDEs have better test classification accuracy but lower reconstruction accuracy. We believe that the discrepancy in reconstruction accuracy results from the continuous latent variables in VAEs, which provide more detailed information. In contrast, for classification, the lack of a clear threshold in continuous latent variables seem to result in potential overfitting and lower test accuracy. 

\emph{Compared to iVAEs}, DDEs have lower performance in both metrics but this is not a fair comparison as iVAEs are weakly supervised (requiring the true label information), whereas DDEs are fully unsupervised. Focusing on the classification error, we note that iVAEs with $K = 2$ latent dimensions have similar performance compared to DDEs. This suggests that DDEs are able to compete with benchmark iVAEs as well.

\begin{table}[h!]
\centering
\resizebox{\textwidth}{!}{
\begin{tabular}{{l}*6{C}@{}}
\toprule
DDE & \includegraphics[width = \linewidth]{figs/MNIST/EM_basis_k=0_original.png} &
\includegraphics[width = \linewidth]{figs/MNIST/EM_basis_k=1_original.png} & \includegraphics[width = \linewidth]{figs/MNIST/EM_basis_k=2_original.png} & \includegraphics[width = \linewidth]{figs/MNIST/EM_basis_k=3_original.png} & \includegraphics[width = \linewidth]{figs/MNIST/EM_basis_k=4_original.png} & \includegraphics[width = \linewidth]{figs/MNIST/EM_basis_k=5_original.png} \\
\midrule
iVAE & \includegraphics[width = \linewidth]{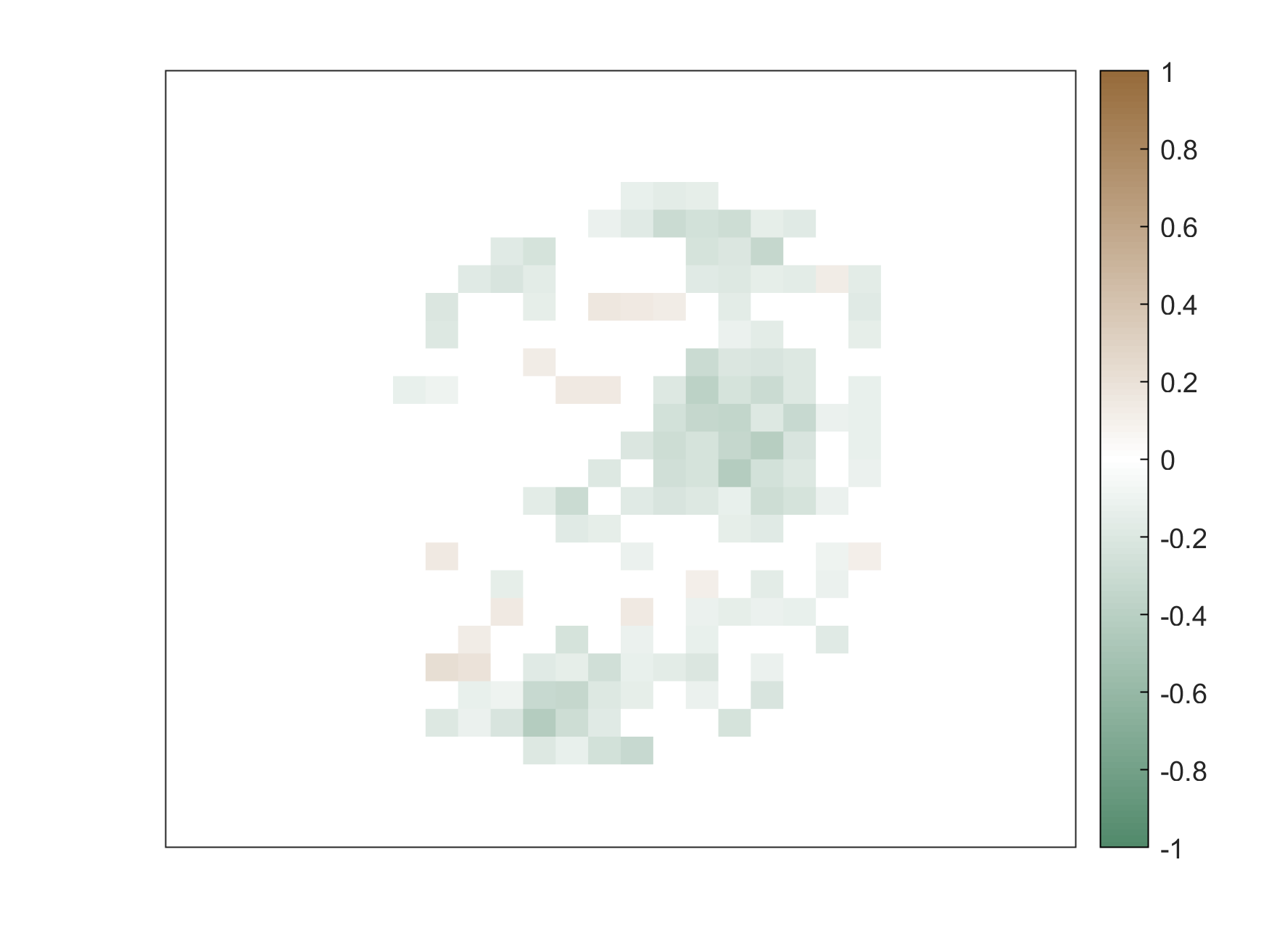} &
\includegraphics[width = \linewidth]{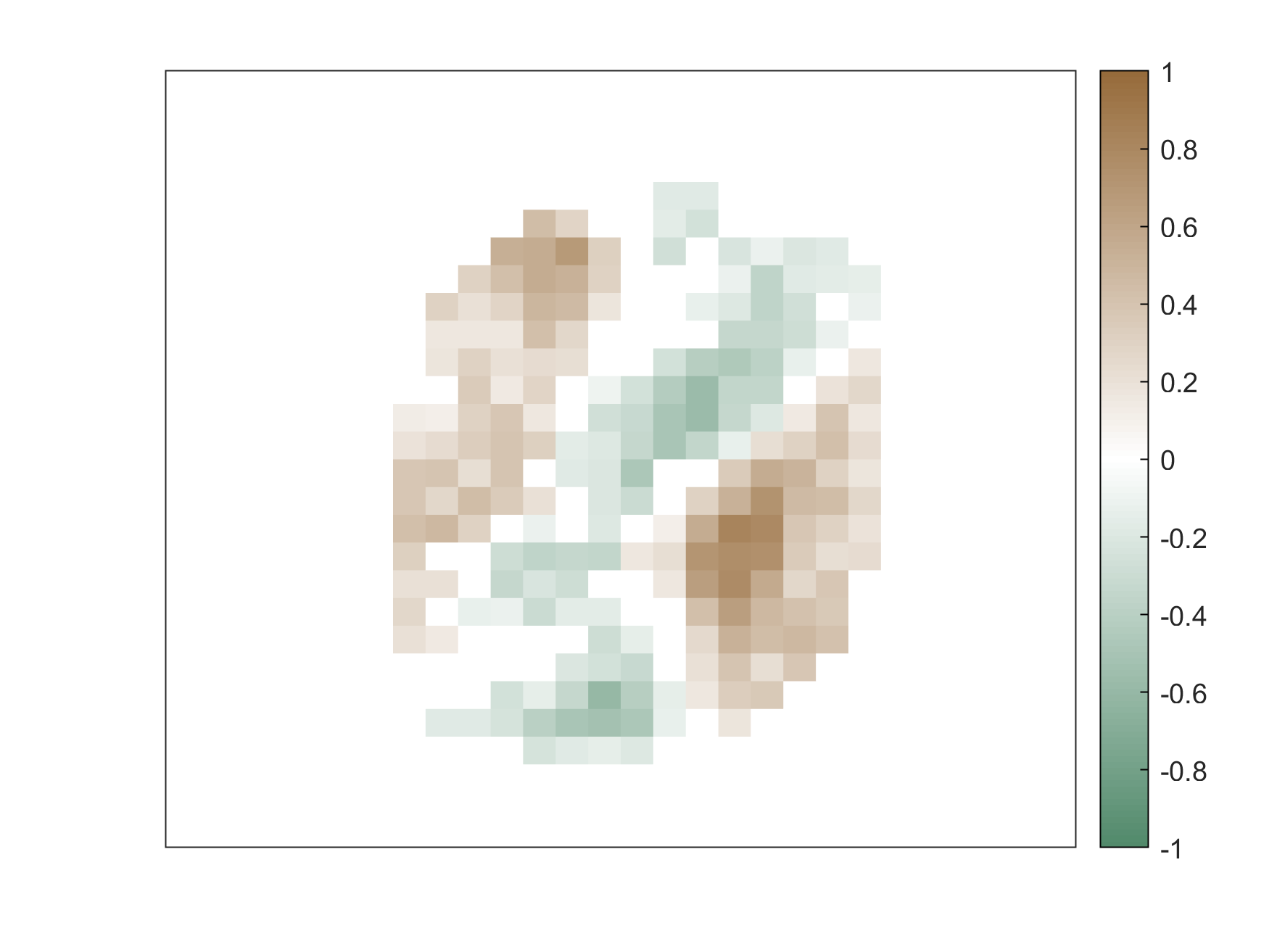} & 
\includegraphics[width = \linewidth]{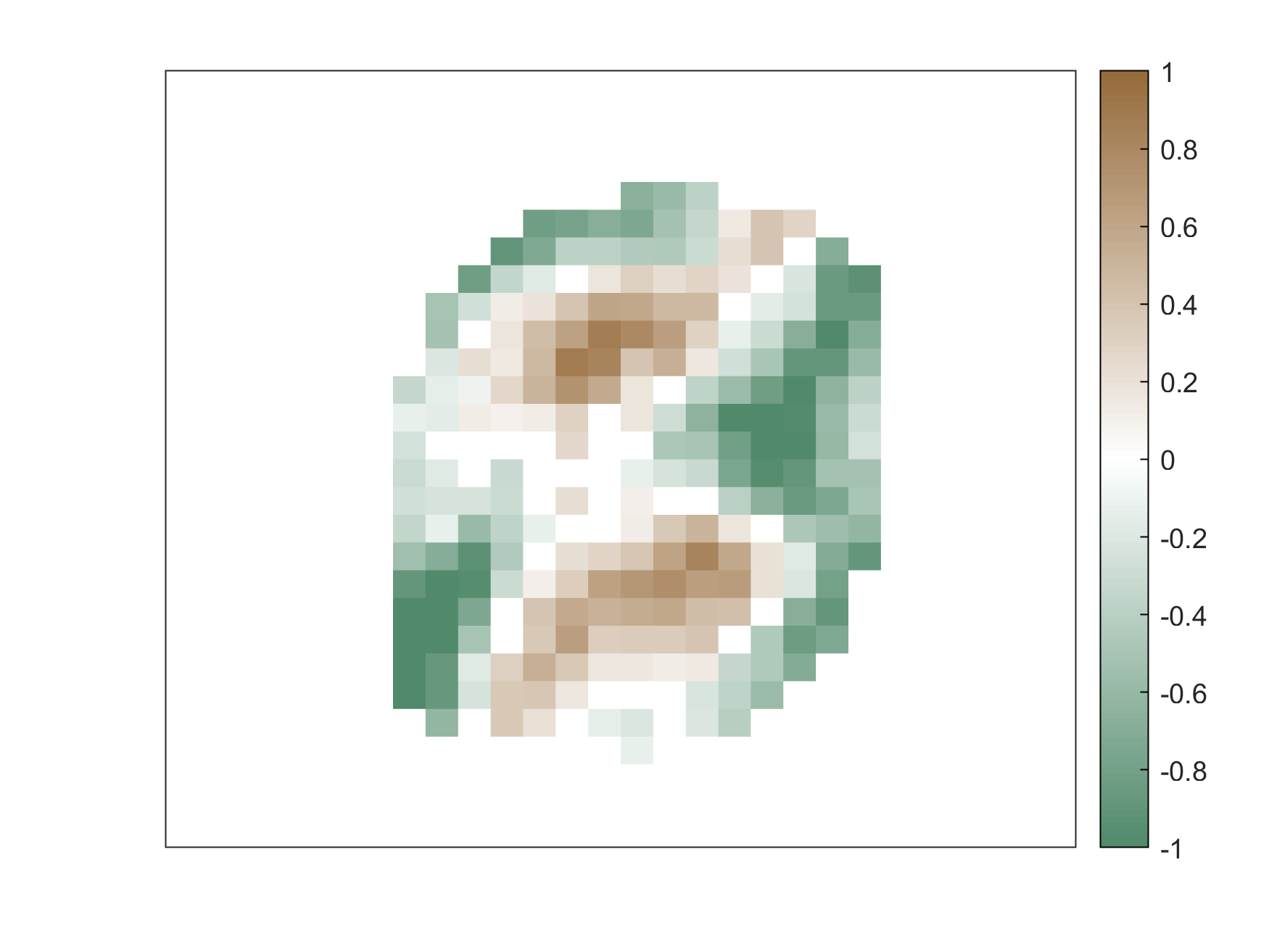} & 
\includegraphics[width = \linewidth]{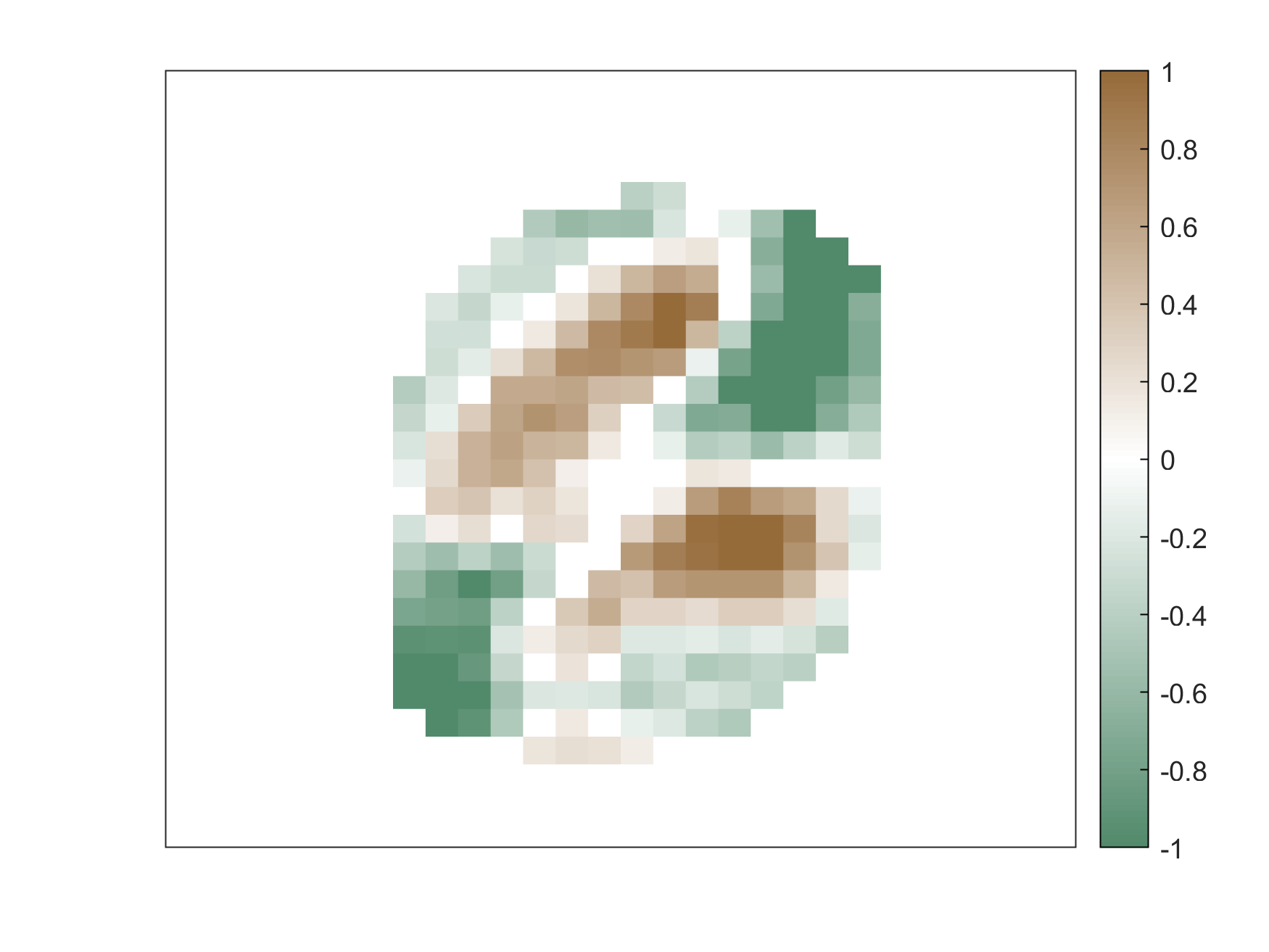} & 
\includegraphics[width = \linewidth]{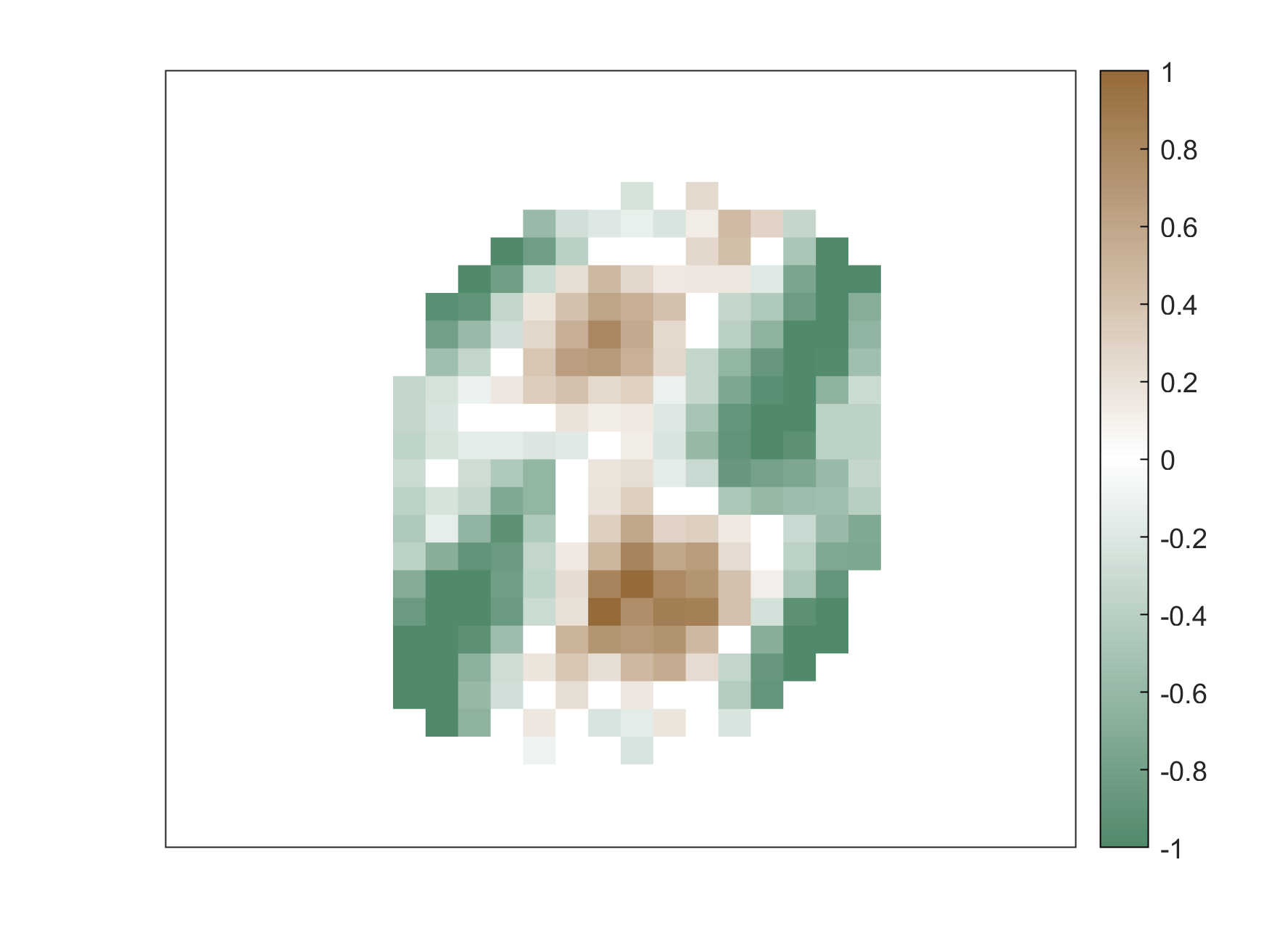} & 
\includegraphics[width = \linewidth]{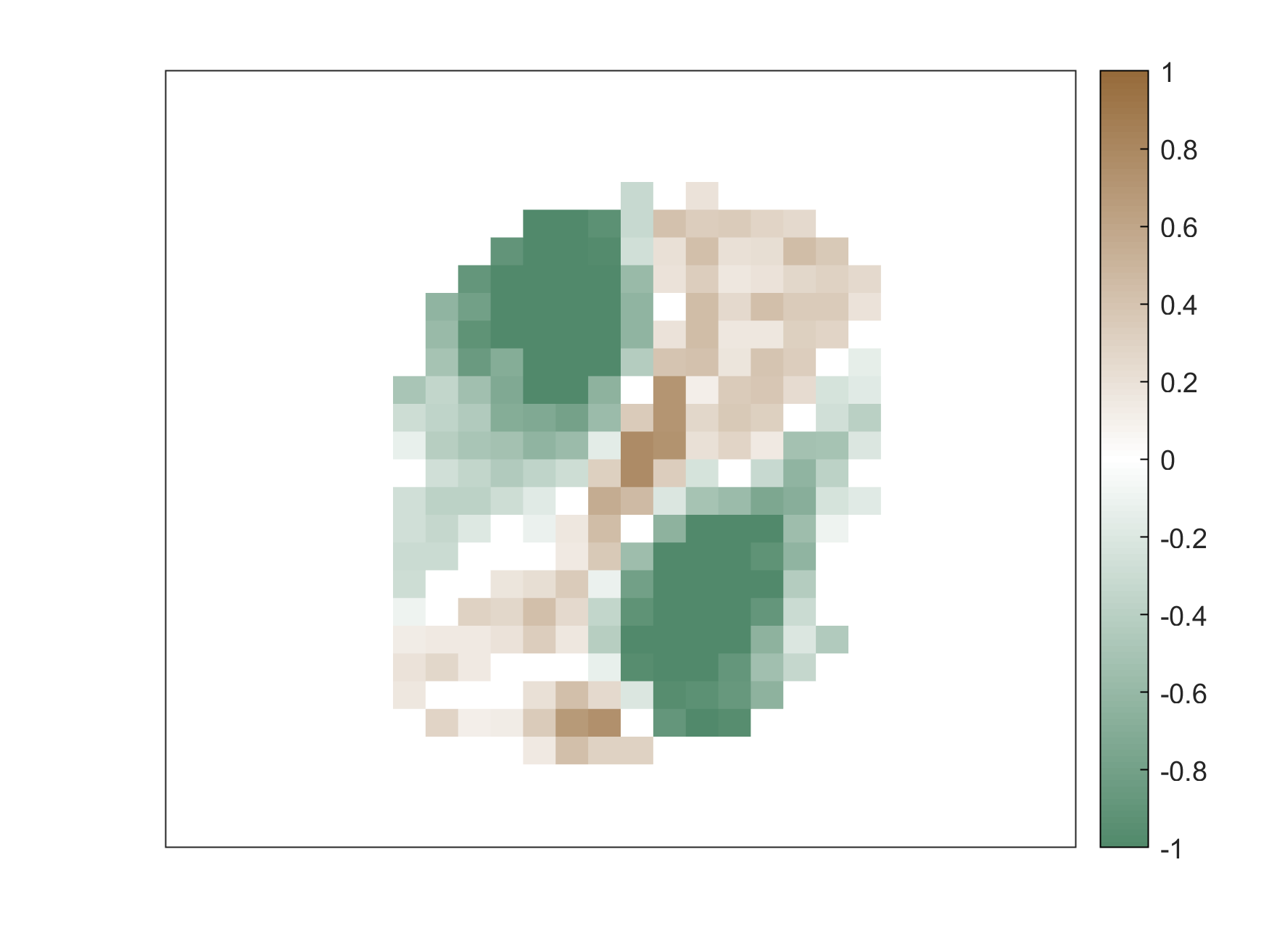} \\
\midrule
& $\bbeta_{0,0}^{(1)}$ & $\bbeta_{0,1}^{(1)}$ & $\bbeta_{0,2}^{(1)}$ & $\bbeta_{0,3}^{(1)}$ &
$\bbeta_{0,4}^{(1)}$ & $\bbeta_{0,5}^{(1)}$ \\
\end{tabular}
}
\caption{From left to right: Estimated basis images (reshaped last-layer coefficients) from the MNIST data. \textcolor{red!40!black}{Brown} indicates positive values and \textcolor{green!50!black}{green} indicates negative values. The basis images under DDEs are distinct, smoother, and more interpretable whereas that under the iVAE exhibit overlaps.}
\label{tab:basis images MNIST spp}
\end{table}

Our most interesting finding is the qualitative comparison of the basis images learned via DDEs versus iVAEs (with $K=2$) in \cref{tab:basis images MNIST}. Here, the iVAE basis images are the last-layer perceptron coefficients. We see that the third, fourth, fifth column of the iVAE panel are almost identical. This strongly indicates that the estimated perceptron (neural network) parameters suffer from overparametrization. Consequently, it is challenging to interpret the shallow layer in the perceptron. In contrast, the basis images under DDEs are distinct, smoother, as well as highly interpretable (e.g. the basis images in the first two columns directly resemble the digits 1 and 0).

\subsubsection{Implementation details}
The VAE and iVAE was implemented by the \texttt{Python} codes provided in \cite{khemakhem2020variational}, where no auxiliary information was used for VAEs and the true digit labels were used as auxiliary information for iVAEs. The DBN is implemented similar to the earlier ablation studies in Supplement 4.2. For both models, we have used default tuning parameters and architectures provided in the original codes. To compute the classification error from latent representations, we have fitted a decision tree analogous to that for DDEs as described in the main text.

\subsection{Performance Evaluation for the 20 Newsgroups Dataset}
\paragraph{Perplexity.}
Perplexity is a very popular measure to evaluate topic models and many other machine learning models. Following the original definition from \cite{blei2003latent} that considers each word as an individual sample, we define
\begin{align}\label{eq:perplexity}
    \text{perplexity}(\YY \mid \ma, \TT) := \exp \left[ - \frac{ \sum_{i,j} Y_{i,j} \log \left( \frac{\lambda_{i,j}}{\sum_{j'} \lambda{i,j'}}\right)}{\sum_{i,j} Y_{i,j}}\right].
\end{align}
Here, $\lambda_{i,j} = \exp\big(\beta_{j,0}^{(1)} + \sum_{k \in [K^{(1)}]} \beta_{j,k}^{(1)} A_{i,k}^{(1)}\big)$ is the Poisson parameter for the conditional distribution $Y_{i,j} \mid \ma_i^{(1)}$. 
The motivation for \eqref{eq:perplexity} is that under the Poisson likelihood for \eqref{eq:observed exp fam}, the joint distribution of $(Y_{i,1}, \ldots, Y_{i,J}) \mid \ma_i$ follows a Multinomial distribution $$\text{Multi}\left(\sum_{j} Y_{i,j}, \left(\frac{\lambda_{i,1}}{\sum_{j'} \lambda{i,j'}}, \ldots, \frac{\lambda_{i,J}}{\sum_{j'} \lambda{i,j'}} \right) \right).$$
This definition is consistent with that used for Poisson factor analysis \citep{zhou2012beta, gan2015scalable} as well as LDA \citep{blei2003latent}.

To evaluate \eqref{eq:perplexity}, we use the parameter estimates $\hat{\TT}$ from our proposed method. As \eqref{eq:perplexity} also requires the knowledge of each latent variable $\ma_i^{(1)}$, latent variable estimates needs to be computed. To compute train perplexity, we use the estimator \eqref{eq:latent variable estimation} from the main text. For test perplexity, we adopt the approach of \cite{gan2015scalable}, first computing the posterior distribution of $\ma_i^{(1)}$ by using $80\%$ of the words in each document, and evaluating perplexity based on the remaining $20\%$ words.

\paragraph{Implementations of existing topic modeling algorithms.}
We describe the implementation details of alternative topic modeling algorithms that were used for model comparison in \cref{tab:perplexity}. LDA is implemented by the function \texttt{LDA} in the \texttt{R} package \texttt{topicmodels} \citep{grun2011topicmodels}, using the variational EM algorithm. For DPFA, we have used the original \texttt{MATLAB} codes that were publicly available on the first author's GitHub page\footnote{\url{https://github.com/zhegan27/dpfa_icml2015/tree/master}} \citep{gan2015scalable}. 
Based on the empirical findings in \cite{gan2015scalable}, which indicate that the DPFA-SBN model optimized using the SGNHT method performs the best among their proposed methods, we chose this implementation. 
The method was run using the default tuning parameters.
Note while \cite{gan2015scalable} also analyzed the 20 Newsgroups dataset (see Table 1 in the cited paper), their reported perplexity values are not directly comparable to our results due to preprocessing. Notably, our preprocessing steps enhanced the signal-to-noise ratio and leads to a smaller perplexity.

\subsection{\darkblue{Additional Analysis of the TIMSS Assessment Dataset}}

We display the estimated coefficients for each modailty in \cref{fig:heatmap}, focusing on the first four skills (Number, Algebra, Geometry, Data\&Probability). The magnitude of the coefficients are larger for the response accuracy, indicating a larger effect of the latent variables on the responses. Additionally, we observe an interesting pattern learned from the intercept values, compared to the held-out information regarding the items (constructed-response versus multiple-choice).  To be specific, the constructed-response items (indexed by the rows 2, 4, 7, 15-17, 19-20, 24-26, 29) have smaller intercept values for the left panel (response accuracy) of \cref{fig:heatmap} and larger intercept values for the right panel (response time).  This is analogous to the previous analysis of the same dataset in \citep{lee2024new}, where each mode (response accuracy and time) were analyzed separately. 
\begin{figure}[h!]
    \centering
    \includegraphics[width=0.47\linewidth]{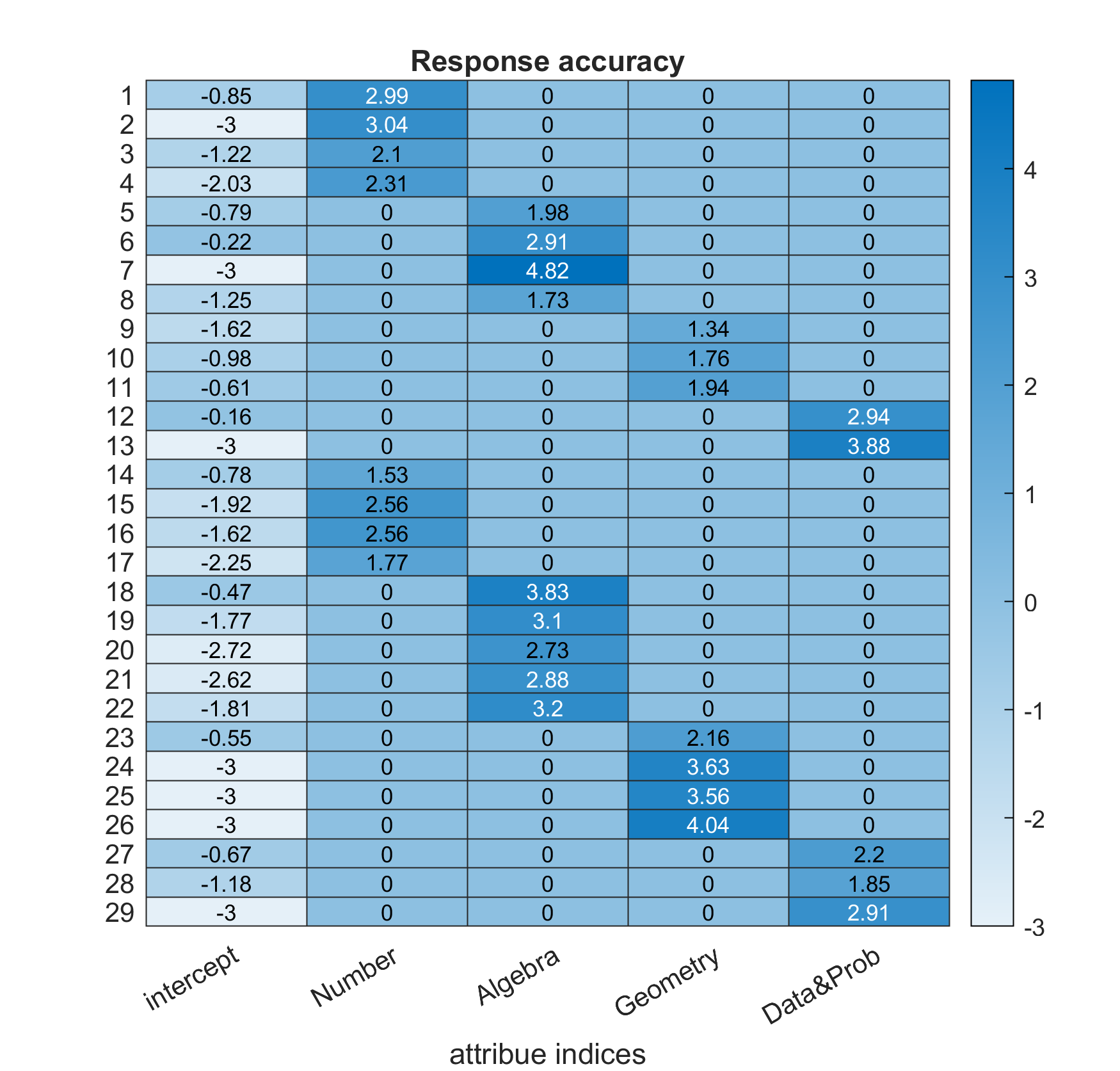}
    \includegraphics[width=0.47\linewidth]{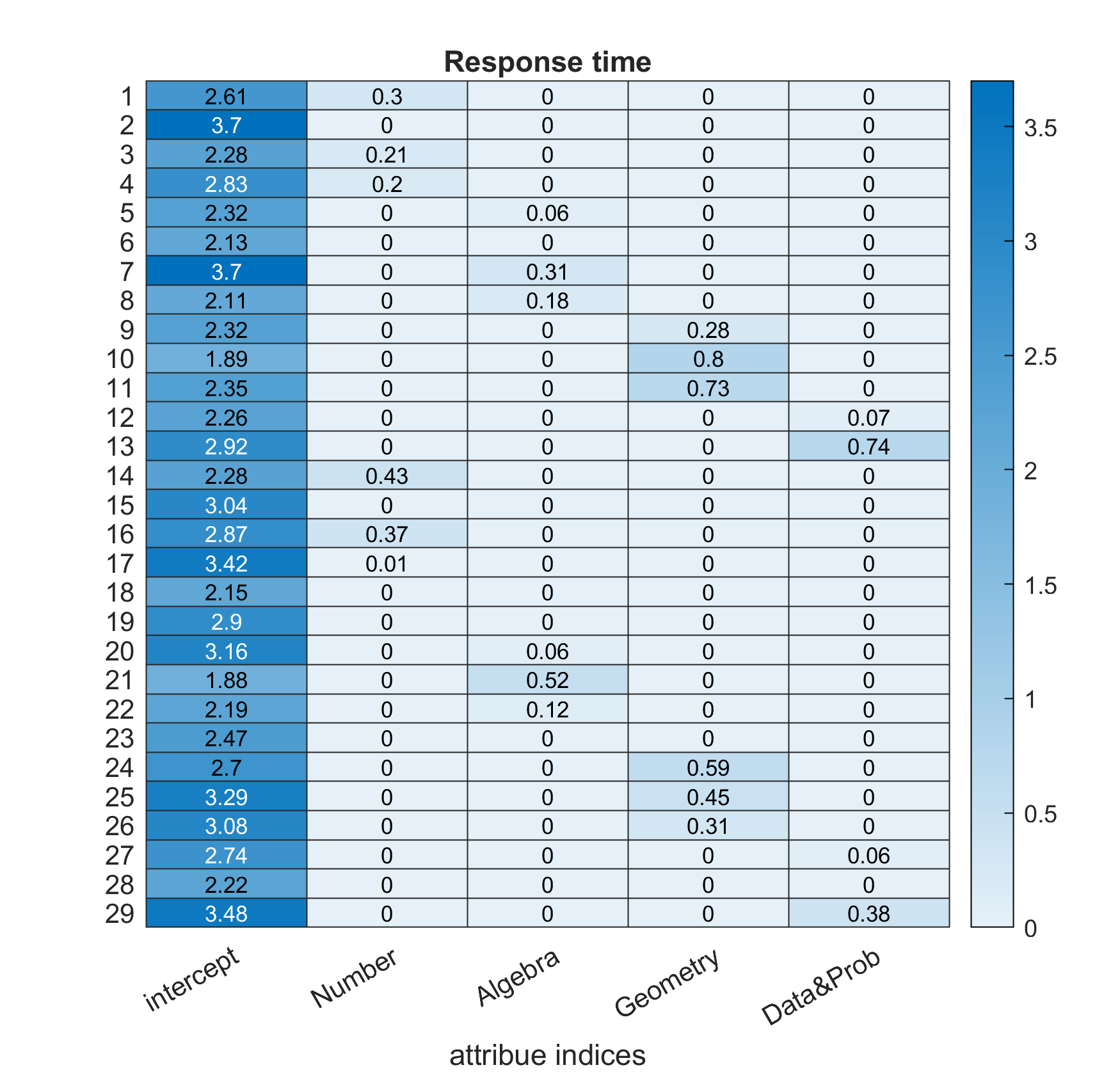}
    \caption{Estimated coefficients for (left) response accuracy, (right) response time.}
    \label{fig:heatmap}
\end{figure}

Additionally, we visualize the analog of \cref{tab:like math} solely-based on the response time in \cref{tab:like math original}, as opposed to the analysis using both modes. Here, the same lognormal-DDE was used with identical latent dimensions. While we observe a similar trend that students enjoying math are more likely to master each skill, the discrepancy between the students enjoying/not enjoying math is lower. For instance, the discrepancy between the category `agree a lot' versus `disagree a lot' in \cref{tab:like math} of the main paper is 0.35 versus the value of 0.14 in \cref{tab:like math original}. This illustrates considering both data modalities help to better learn the latent variables in a more interpretable manner.

\begin{table}[h!]
\centering
\resizebox{\textwidth}{!}{
\begin{tabular}{cccccc}
\toprule
Response \textbackslash{} Latent skill & $A_1^{(2)}$   & $A_1^{(1)}$: Number & $A_2^{(1)}$: Algebra & $A_3^{(1)}$: Geometry & $A_4^{(1)}$: Data and Prob \\ 
\midrule
Agree a lot                                & 0.83 & 0.72 & 0.88 & 0.86 & 0.87  \\
Agree a little                             & 0.80 & 0.72  & 0.85 & 0.84 & 0.85 \\
Disagree a little                          & 0.76 & 0.74  & 0.85 & 0.84 & 0.84 \\
Disagree a lot                             & 0.69 & 0.67 & 0.76 & 0.75 & 0.78 \\
\bottomrule
\end{tabular}
}
\caption{Average latent variable estimate for each response category for the question ``Mathematics is one of my favorite subjects'', solely based on response times.}\label{tab:like math original}
\vspace{-4mm}
\end{table}

\section{Additional Literature Review}\label{sec:supp literature}
\paragraph{iVAEs.}
We compare DDEs versus iVAEs and its extensions. One distinction is that iVAEs essentially have only \textit{one} latent layer of random variables transformed by deterministic deep neural networks, while DDEs allow \textit{multiple} latent layers. Thus, DDEs are able to model the hierarchical nature of uncertainty and are more suitable for tasks like hierarchical topic modeling \citep{ranganath2015deep}.

Also, iVAEs typically require \emph{continuous} latent variables, additional auxiliary variables, and its notion of identifiability still suffers from rotational ambiguity. Most iVAEs are used to model continuous observed responses \citep{khemakhem2020variational, moran2021identifiable, kivva2022identifiability}, although there exist a few variants exclusively proposed for discrete data (binary, count) as well \citep{hyttinen2022binary}. In particular, up to our knowledge, identifiability guarantees for \textit{discrete} data has been only established in \cite{hyttinen2022binary} under strong parametric restrictions on the encoder.\footnote{{The original paper \cite{khemakhem2020variational} only provided identifiability theory for continuous data, although their simulation studies suggest potential extensions to discrete data.}} In contrast, DDEs with \emph{discrete} latent variables do not require auxiliary variables, do not suffer from rotational ambiguity, and the identifiability of DDEs holds regardless of the observed data types. 

\paragraph{Psychometrics.}
Multidimensional binary latent variables are also popular for modeling students' item response data in educational measurement \citep{junker2001cognitive, von2008general, de2011generalized}. These psychometric models are known as cognitive diagnostic models (CDMs). A CDM uses a single latent layer of cognitive skills to model a student's binary responses to many questions in a test, with a sparse loading graph between the observed and latent layers. 
Here, each binary latent variable represents a student's mastery or deficiency of a cognitive skill, and the sparse graph between the observed item responses and latent skills encodes which skills each item is designed to measure in the test. 
The proposed DDEs substantially generalize this modeling idea by allowing (a) multilayer latent variables, which can model one's knowledge structure at multiple resolutions ranging from fine-grained details to general concepts \citep{gu2024deepCDM}, and (b) modeling a rich class of general responses, going beyond the typical binary correct/incorrect responses in educational tests.

\end{document}